\documentclass[a4paper,11pt]{article}

\usepackage[top=3.5cm, bottom=3.5cm, left=2.5cm, right=2.5cm]{geometry}

\usepackage[utf8]{inputenc}
\usepackage[T1]{fontenc}
\usepackage[french,english]{babel} 
\usepackage{enumerate}
\usepackage{enumitem}
\usepackage{subcaption}
\usepackage{xcolor}
\usepackage[]{pdfpages}

\usepackage[normalem]{ulem}

\usepackage{titling}
\usepackage[colorlinks=false,linkcolor=blue]{hyperref}
\usepackage{csquotes}
\usepackage[nottoc, notlof, notlot, numbib]{tocbibind}
\usepackage[style=alphabetic,giveninits,doi=false,isbn=false,url=false,eprint=false,clearlang=true,sorting=nyt,maxbibnames=99]{biblatex}
\DeclareSourcemap{
      \maps[datatype=bibtex]{
      \map[overwrite=false]{
       \step[fieldsource=note]
       \step[fieldset=addendum, origfieldval, final]
       \step[fieldset=note, null]
       }
   }
   }

\usepackage{amsthm,amsfonts,amsmath,amssymb}
\usepackage{mathtools}
\usepackage{thmtools}
\usepackage{bbm}
\usepackage{mathrsfs}
\usepackage{float}
\usepackage{stmaryrd}
\usepackage{setspace}
\usepackage{array}
\usepackage{tabularx}
\usepackage{ltablex} 
\usepackage{comment} 

\usepackage{todonotes}
\mathtoolsset{showonlyrefs}

\newcommand{\B}{\mathcal{B}}
\newcommand{\C}{\mathcal{C}}

\newcommand{\N}{\mathcal{N}}
\newcommand{\R}{\mathbb{R}}

\newcommand{\E}{\mathbb{E}}

\renewcommand{\d}{\mathrm{d}}
\newcommand{\F}{\mathcal{F}}

\renewcommand{\P}{\mathbb{P}} 
\renewcommand{\H}{\mathcal{H}}
\renewcommand{\L}{\mathcal{L}} 
\newcommand{\W}{\mathcal{W}}

\newcommand{\blue}[1]{\textcolor{blue}{#1}}

\newcommand{\purple}[1]{\textcolor{purple}{#1}}
\newcommand{\ndbar}{\overline{D}} 
\newcommand{\cstW}{C^{\overline{W}}} 
\newcommand{\cstWh}{C^{\overline{W}}} 
\newcommand{\llnW}{R^{\overline{W}}} 
\newcommand{\cstH}{C^{H,B}} 
\newcommand{\cstHhat}{\hat{C}^{H,B}} 
\newcommand{\Finit}{\F_W} 
\newcommand{\varu}{\sigma_u} 
\newcommand{\vpu}{\tilde{\sigma}_u} 
\newcommand{\varv}{\sigma_v} 
\newcommand{\vpv}{\tilde{\sigma}_v} 

\newcommand{\bfS}{\mathbf{S}}
\newcommand{\bfz}{\mathbf{z}}
\newcommand{\bfZ}{\mathbf{Z}}
\newcommand{\bfh}{\mathbf{h}}
\newcommand{\bfH}{\mathbf{H}}
\newcommand{\bfb}{\mathbf{b}}
\newcommand{\bfB}{\mathbf{B}}
\newcommand{\Win}{W_{\mathrm{in}}}
\newcommand{\Wout}{W_{\mathrm{out}}}
\newcommand{\dimin}{d_{\mathrm{in}}}
\newcommand{\dimout}{d_{\mathrm{out}}}
\newcommand{\varin}{\sigma_{\mathrm{in}}}
\newcommand{\varout}{\sigma_{\mathrm{out}}}
\newcommand{\vpin}{\tilde{\sigma}_{\mathrm{in}}}
\newcommand{\vpout}{\tilde{\sigma}_{\mathrm{out}}}
\newcommand{\vpk}{\tilde{\sigma}_{k}}
\newcommand{\bnabla}{\boldsymbol{\nabla}}

\newcommand{\Cov}{\operatorname{Cov}}

\newtheorem{theorem}{Theorem}[section]
\newtheorem{proposition}[theorem]{Proposition}
\newtheorem{lemma}[theorem]{Lemma}
\newtheorem{corollary}[theorem]{Corollary}
\newtheorem{assumption}{Assumption}

\theoremstyle{definition}
\newtheorem{definition}[theorem]{Definition}

\theoremstyle{remark}
\newtheorem{rem}{Remark}

\newcounter{paragraphe} 

\title{ ResNets of All Shapes and Sizes: \\ Convergence of Training Dynamics in the Large-scale Limit}

\author{
Louis-Pierre Chaintron,\quad
Lénaïc Chizat,\quad
Javier Maass\thanks{École Polytechnique Fédérale de Lausanne (EPFL),  Institute of Mathematics, 1015 Lausanne, Switzerland. Authors listed in alphabetical order.}
}

\date{\today}

\begin{document}

\maketitle

\begin{abstract}

We establish convergence of the training dynamics of residual neural networks (ResNets) to their joint infinite depth $L$, hidden width $M$, and embedding dimension $D$ limit. 
Specifically, we consider ResNets with two-layer perceptron blocks in the maximal local feature update (MLU) regime and prove that, after a bounded number of training steps, the error between the ResNet and its large-scale limit is $O\big(\frac{1}{L}+\frac{\sqrt{D}}{\sqrt{LM}}+\frac{1}{\sqrt{D}}\big)$. This error rate is empirically tight when measured in embedding space.
For a budget of $P=\Theta(LMD)$ parameters, this yields a convergence rate $O(P^{-1/6})$  for the scalings of $(L,M,D)$ that minimize the bound.
Our analysis exploits in an essential way the depth-two structure of residual blocks and applies formally to a broad class of state-of-the-art architectures, including Transformers with bounded key-query dimension.

From a technical viewpoint, this work completes the program initiated in the companion paper~\cite{chizat2025hidden} where it is proved that for a fixed embedding dimension $D$, the training dynamics converges to a \emph{Mean ODE} dynamics at rate $O\big(\frac{1}{L}+\frac{\sqrt{D}}{\sqrt{LM}}\big)$. Here, we study the large-$D$ limit of this Mean ODE model and establish convergence at rate $O\big(\frac{1}{\sqrt{D}}\big)$, yielding the above bound by a triangle inequality. 
To handle the rich probabilistic structure of the limit dynamics and obtain one of the first rigorous quantitative convergence for a DMFT-type limit, we combine the cavity method with propagation of chaos arguments at a functional level on so-called \emph{skeleton maps}, which express the weight updates as functions of CLT-type sums from the past.
\end{abstract}

\tableofcontents


\section{Introduction}

\newcommand{\tWin}{\mathtt{W_{in}}}
\newcommand{\tWout}{\mathtt{W_{out}}}
\renewcommand{\th}{\mathtt{h}}
\newcommand{\tU}{\mathtt{U}}
\newcommand{\tV}{\mathtt{V}}
\newcommand{\tf}{\mathtt{f}}
\newcommand{\ty}{\mathtt{y}}

Training artificial neural networks of ever increasing size has been a key driver of progress in artificial intelligence.
Yet, from a theoretical viewpoint, the structure and behavior of training dynamics of very large models remain poorly understood. In particular, beyond the  infinite-width case, it is still unclear whether and when other well-defined infinite-size limits can be shown to exist (especially in light of certain training instabilities arising in large scale settings~\cite{wortsman2023small}), under which conditions on the architecture's shape they arise, and how efficiently finite models approximate them.

 In this work, we focus on residual networks (ResNets) with depth-two blocks --- the backbone of many state-of-the-art architectures --- whose shape is determined by their depth $L$, hidden width $M$ and embedding dimension $D$. We address these questions by building a rigorous  theory of the large-scale limit of their training dynamics, for a fixed training time and number of training samples. In addition to providing the first proof of convergence for ResNets in feature-learning regimes as $L,M,D\to \infty$, our analysis treats the joint limit under general (not necessarily proportional) scalings and yields quantitative error rates that are empirically sharp.

\subsection{Gradient descent on ResNets with depth-two blocks}\label{sec:setting}
\paragraph{ResNets with 2LP blocks.} The model of interest is a residual neural network (ResNet)~\cite{he2015deepresiduallearningimage} with two-layer perceptron (2LP) blocks. For an input $x\in \R^{\dimin}$, the output $\mathtt{y}_\theta(x)\in \R^{\dimout}$ is given by 
\begin{align}\label{eq:fully_finite_resnet}
\left\{
\begin{aligned}
\th^{(0)}_\theta(x) &= \tWin x, \\
\th^{(\ell)}_\theta(x) &= \th^{(\ell-1)}_\theta(x)+ \frac{1}{ML}   \tV^{\ell} \rho\Big(\frac1D (\tU^{\ell})^\top \th_\theta^{(\ell-1)}(x)\Big),&\forall \ell\in [1:L]\\ 
\ty_\theta(x) &= \frac{1}{D}(\tWout)^\top \th_\theta^{(L)}(x) 
\end{aligned}
\right.
\end{align}
where $\rho:\R\to\R$ is a smooth nonlinearity, which applies component-wise to vectors, and the parameters/weights $\theta=((\tU^{\ell},\tV^{\ell})_{\ell \in [1:L]},\tWin,\tWout)\in \mathbb{R}^P$ include the embedding matrix \(\tWin := (\Win^d)_{d=1}^D \in \R^{D \times \dimin}\), the readout (or unembedding) matrix \(\tWout := (\Wout^d)_{d=1}^D \in \R^{D \times \dimout}\), and the input/output weights of each layer $(\tU^{\ell},\tV^{\ell})\in \R^{D\times M}$ for $\ell\in [1:L]$. The \emph{shape} hyper-parameters (HP) of the model are the embedding dimension $D$, the hidden width $M$  and the depth $L$. While prior works have often considered $D=M$, it is crucial in our analysis to distinguish these two shape HPs as they play different, in a sense competing, roles and need to be sent to $+\infty$ separately. The total number of parameters is
\[
P = D\cdot (\dimin+\dimout)+ 2LMD
\]
and scales as $\Theta(LMD)$ when ${\dimin}$ and ${\dimout}$ are constant. We note that our analysis could be extended to the case where the domain and codomain of $\rho$ are of dimension larger than $1$ but fixed, such as attention blocks with bounded key-query dimension $d_k$ in Transformers~\cite{vaswani2017attention}, see~\cite[Section~4.2]{chizat2025hidden} for details on how they fit in this framework with $\rho:\mathbb{R}^{3d_k}\to \mathbb{R}^{d_k}$.

\paragraph{Training dynamics.} We consider a training set of $N$ input samples $(x_i)_{i\in [1:N]}$ and their respective loss functions $\mathrm{loss}_i :\R^{d_{out}}\to \R$, assumed to have a Lipschitz gradient. We define the objective function $\mathrm{Loss}(\theta)=\frac1N \sum_{i=1}^N \mathrm{loss}_i (\mathtt{y}_\theta(x_i))$ and consider the following Gradient Descent (GD) training dynamics with scaled learning rates (LRs) $\eta_u,\eta_v,\eta_{\mathrm{in}}, \eta_{\mathrm{out}}\geq 0$ and scaled standard deviations (STDs) of the random initialization $\sigma_u,\sigma_v,\sigma_{\mathrm{in}}, \sigma_{\mathrm{out}}\geq 0$
\begin{align}\tag{GD}\label{eq:resnet-GD}
\left\{
\begin{aligned}
\tU^{\ell}_{0}[i,j] &\overset{\text{iid}}{\sim} \mathcal{SG}(0,D\sigma_u^2) \\
\tV^{\ell}_{0}[i,j] &\overset{\text{iid}}{\sim} \mathcal{SG}(0,D\sigma_v^2)\\
\tWin_{,0}[i,j] &\overset{\text{iid}}{\sim} \mathcal{SG}(0,\sigma_{\mathrm{in}}^2)\\
\tWout_{,0}[i,j] &\overset{\text{iid}}{\sim} \mathcal{SG}(0,\sigma_{\mathrm{out}}^2)
\end{aligned}
\right.,&&
\text{and}
& &
\left\{
\begin{aligned}
\tU^{\ell}_{k+1}&=\tU^{\ell}_{k} - \eta_u LMD \nabla_{\tU^{\ell}}\mathrm{Loss}(\theta_k)\\
\tV^{\ell}_{k+1}&=\tV^{\ell}_{k} - \eta_v LMD  \nabla_{\tV^{\ell}}\mathrm{Loss}(\theta_k)\\
\tWin_{,k+1}&=\tWin_{,k} - \eta_{\mathrm{in}} D \nabla_{\Win}\mathrm{Loss}(\theta_k)\\
\tWout_{,k+1}&=\tWout_{,k} - \eta_{\mathrm{out}} D \nabla_{\Wout}\mathrm{Loss}(\theta_k)
\end{aligned}
\right.
\end{align}
where $\mathcal{SG}(0,\sigma^2)$ represents a centered subgaussian distribution with variance $\sigma^2$ and variance proxy $\Theta(\sigma^2)$, such as $\mathcal{N}(0,\sigma^2)$ or $\mathrm{Unif}([-\sqrt{3}\sigma,\sqrt{3}\sigma])$. For simplicity, we assumed here that all scalar weights are initialized independently but our analysis allows for certain form of dependencies, for instance $\tWin_{,0}=\tWout_{,0}$ (weight-tying of embeddings) is allowed. Also, our analysis can easily be adapted to more general gradient-based training algorithms beyond GD.

Assuming that the LRs and initialization STDs are $\Theta(1)$, our choices of scalar multipliers and initialization scales are  characteristic of the Maximal Local Updates (MLU) regime~\cite[Section~4]{chizat2025hidden}. This regime is the unique regime where the one-step loss decay is $\Theta(1)$ at initialization, and where the weights $\tU^\ell$ at layer $\ell$ contribute to a non-vanishing fraction of the $\Theta(1)$ updates of the features\footnote{For a given $x$, the features at layer $\ell$ and step $k$ is the vector $\rho\Big(\frac1D (\tU^{\ell}_k)^\top \th_{\theta_k}^{(\ell-1)}(x)\Big) \in \mathbb{R}^M$. It is also known as the activation vector, and the argument of $\rho$ is the pre-activation vector.} at the same layer $\ell$. There exists other regimes with a well-behaved loss decay where the feature updates are non-local (the lazy ODE regime), or even vanishing (the lazy kernel regime), see~\cite[Section~4]{chizat2025hidden} for a discussion of the asymptotic phase diagram in the large $L,M,D$ limit, which builds on many prior works~\cite{dey2026dontlazycompletepenables, yang2023tensorprogramsvifeature, bordelon2024infinitelimitsmultiheadtransformer, bordelon2023depthwisehyperparametertransferresidual}.

\subsection{Main limit theorem}\label{sec:main-limit-theorem}

Our main limit theorem is concerned with a slightly modified version of~\eqref{eq:resnet-GD} which we call (ClippedGD) where certain quantities involved in the computation of the gradient are clipped to a bounded range. These clippings are necessary to stabilize the large-$L$ limit and are defined in~\cite[Rmk.~4.1]{chizat2025hidden}, but they are not needed for the large-$D$ limit developed in the present paper, so we ignore them in the following\footnote{More precisely, in the notation of~\cite[Section.~4.3]{chizat2025hidden}, Theorem~\ref{thm:main_theorem_intro} requires $\mathrm{cl}_u$ and $\mathrm{cl}_b$ to be bounded and to have bounded derivatives up to order $4$. It is immediate to check that the theorems in the present paper still hold under these assumptions, rather than with $\mathrm{cl}_u=\mathrm{cl}_b=\mathrm{id}$ as we implicitly assume throughout here.}.

\begin{theorem}[Quantitative large-scale limit of ResNets]\label{thm:main_theorem_intro}
Let $C_0>0$. Assume that $\rho^{(i)}$ (the $i$-th order derivative of $\rho$) is bounded for $i\in [1:5]$ and that $\nabla \mathrm{loss}_i$ is Lipschitz for $i\in [1:N]$. Consider the dynamics $\mathrm{(ClippedGD)}$ with $\eta_u,\eta_v,\sigma_u,\sigma_v,\sigma_{\mathrm{in}}, \sigma_{\mathrm{out}}\geq 0$ and $\eta_\mathrm{in},\eta_{\mathrm{out}}=0$ (for simplicity).

For all $k\in \mathbb{N}$, there exist deterministic vectors $y_k(x_i)\in \mathbb{R}^{\dimout}$ and stochastic processes $s \in [0,1] \mapsto H_k(s,x_i)\in \mathbb{R}$, $i\in [1:N]$ and there exist $c_1,c_2>0$ such that the following holds. 

Let $\Delta_k^y \coloneqq \max_{i\in [1:n]} \Vert \mathtt{y}_k(x_i)-y_k(x_i)\Vert_{2}$ and let $\Delta_k^h \coloneqq \max_{i\in [1:N]}\max_{\ell\in [1:L]} D^{-1/2} \Vert \th_k^{(\ell)}(x_i) - H_k^{(D)}(\ell/L,x_i) \Vert_2$ where $H_k^{(D)}(s,x_i)\in \mathbb{R}^D$ is a vector of $D$ independent copies of $H_k(s,x_i)$ suitably coupled to the randomness of the ResNet. For all $L,M,D\in \mathbb{N}^*$ such that  $\log(L) \vee D \leq C_0 M$ and for all $\delta \in (0,1) 
$, it holds with probability at least $1-\delta$
\begin{align}
 \Delta_{k}^y \vee \Delta_{k}^h \leq c_1\Big(\frac{1}{L}+\frac{\sqrt{D}}{\sqrt{ML}}+\frac{1+\log(1/\delta)}{\sqrt{D}}\Big)
\end{align}
provided that the right-hand side is smaller than $c_2$. 
\end{theorem}

Strictly speaking, the proof is written for $N=1$, but the arguments straightforwardly adapt to $N>1$, see Remark~\ref{rem:skeleton_maps_for_N_ge_1} below for further precisions. 
A key object to prove Theorem~\ref{thm:main_theorem_intro} is the \emph{Mean ODE} dynamics, defined in~\eqref{eq:original_finite_D_system} below, which is used as an intermediate term in a triangle inequality. As shown in~\cite{chizat2025hidden}, this dynamics corresponds to the infinite-depth limit $L\to \infty$ of (ClippedGD). More precisely, \cite[Theorem 3]{chizat2025hidden} gives a high-probability error bound of the form
\begin{equation}\label{eq:hidden-width-informal}
\max_{\ell\in [1:L]}\max_{i\in [1:N]} \frac{1}{\sqrt{D}}\Vert \mathtt{h}^{\ell}_{k}(x_i)-h_{k}^{(D)}(\ell/L,x_i)\Vert_{2} =O\Big(\frac{1}{L}+\frac{\sqrt{D}}{\sqrt{ML}}\Big),
\end{equation} 
where $(h_k^{(D)})_{k\geq 0}$ are the hidden representations in $\mathbb{R}^D$ of the Mean ODE model.
In the present paper, we prove a quantitative convergence result of the Mean ODE dynamics to its high-dimensional $D\to \infty$ limit, with an explicit high-probability error bound in $O(1/\sqrt{D})$. This is the content of Theorem~\ref{thm:quantitative_convergence_main}, the central result of this paper presented in detail in the next section. Put together with a triangle inequality, these two results give the bound on $\Delta^h_k$ in $O(\frac{1}{L}+\frac{\sqrt{D}}{\sqrt{ML}}+\frac{1}{\sqrt{D}})$. Then we get the same bound on $\Delta^y_k$ using that the linear map $\mathtt{h}^{(L)}_k(x_i)\to \mathtt{y}_k(x_i)$, i.e.~the multiplication by $\frac{1}{D}(\tWout)^\top$, and the asymptotic counterpart of this operation, have operator norms in $O(1)$ with high probability. The details of the proof of Theorem~\ref{thm:main_theorem_intro} can be found in Section~\ref{sec:proof_of_thms_quant_cvg}.

We can make the following first remarks:
\begin{itemize}
\item \emph{(Assumptions on $\rho$.)} We do not assume $\rho$ itself to be bounded. This allows to cover the case of linear activations, or smooth approximations of ReLU. It is necessary however to assume that $\rho$ has at most linear growth since otherwise the Mean ODE can explode in finite time. Such explosion issues can be mitigated by systematically projecting $\mathtt{h}^{(\ell)}_k(x_i)$ on the RMS unit ball (LayerNorm), but we do not consider such projections here. 
\item \emph{(Modes of convergence.)} The output of the ResNet $\mathtt{y}_k(x_i)$ converges to a deterministic limit $y_k(x_i)$ while the hidden state $\th_\theta^{(\ell)}(x_i) \in \mathbb{R}^D$ at layer $\ell$ converges towards an infinite-length vector of independent samples of a random variable $H_k(\ell/L,x_i)$. For the sake of readability, Theorem~\ref{thm:main_theorem_intro} bounds the convergence rate for these two objects only, but our analysis leads to convergence rates (in appropriate metrics) for essentially all objects appearing in the dynamics. 
\item \emph{(Shape constraints.)} We deduce from Theorem~\ref{thm:main_theorem_intro} the following sufficient condition to converge to the limit: it must hold $L,M,D\to \infty$ under the shape constraints (i) $D/(ML)\to 0$, (ii) $\log(L) =O(M)$ and (iii) $D=O(M)$. These conditions cover the scaling regimes of standard architectures, see for instance the case of Llama 3 models in~\cite{grattafiori2024llama} where for the MLP blocks $M/D\in [3,4]$ and $D/(ML)$ decreases with the model size. From a theoretical viewpoint, of all these constraints, we believe that only (i) is necessary and we have observed experimentally that the limit dynamics is indeed different when $D/(ML)\to \alpha>0$. We leave it as an open question to determine whether the constraints (ii) and (iii), which come from the analysis in~\cite{chizat2025hidden}, are necessary for convergence. We note that (iii) is the most stringent constraint and in particular it implies (i) when $L,M,D\to \infty$.
\item \emph{(Dependencies of the constants.)} Although we do not track it explicitly, it can be seen that the bound depends on $d_{\mathrm{in}}$ and $d_{\mathrm{out}}$ only via the RMS norm of the embedded vectors $\tWin x_i$ and of the (rescaled) embedded gradients $\tWout \nabla \mathrm{loss}_i(\ty_k(x_i))$. Therefore, under appropriate assumptions on the norm of the input data and of the gradients, our bound holds uniformly over $d_{\mathrm{in}}$  and $d_{\mathrm{out}}$. In contrast, our bounds depend super-exponentially on the training time $k$ and the number of training samples $N$ -- we use $N = 1$ in the proof so in many places dependencies on $k$ should be read as dependencies on $k\cdot N$, see Remark~\ref{rem:skeleton_maps_for_N_ge_1}. It would be an interesting direction to improve these dependencies in $k$ and $N$. In~\cite{chizat2025hidden}, the dependency in $n$ is only logarithmic and in particular we believe that there is ample room for improvement for the $N$ dependency with a slightly more careful treatment.
\item \emph{(Trainable embeddings.)} We assumed $\tWin$ and $\tWout$ to be fixed to their initial random values -- through $\eta_\mathrm{in},\eta_{\mathrm{out}}=0$ -- for simplicity, but it is straightforward to extend our analysis to the case where these matrices are trained as well, and the same result would hold with $\eta_{\mathrm{in}},\eta_{\mathrm{out}}=\Theta(1)$. See details in Appendix~\ref{sec:extensions_of_results}.
\end{itemize}

\begin{figure}
\centering
\begin{subfigure}{0.32\linewidth}
\centering
\includegraphics[scale=0.4]{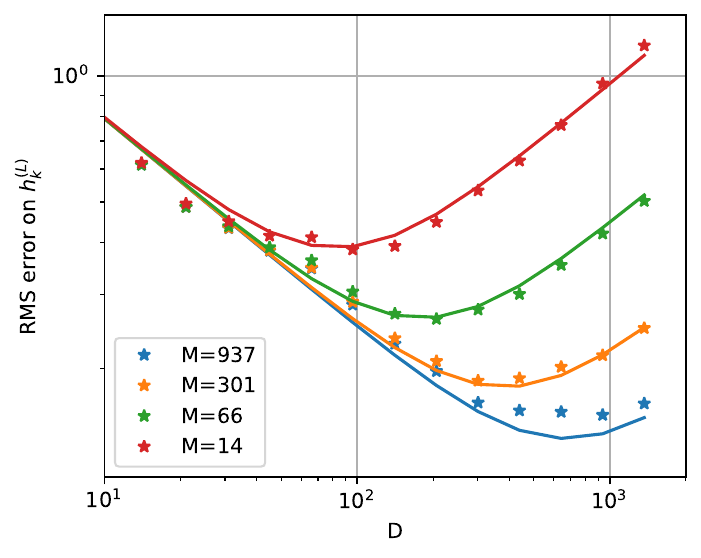}
\caption{Error on $\th^{(L)}_k$ vs $D$ ($L=50$)}
\end{subfigure}%
\begin{subfigure}{0.32\linewidth}
\centering
\includegraphics[scale=0.4]{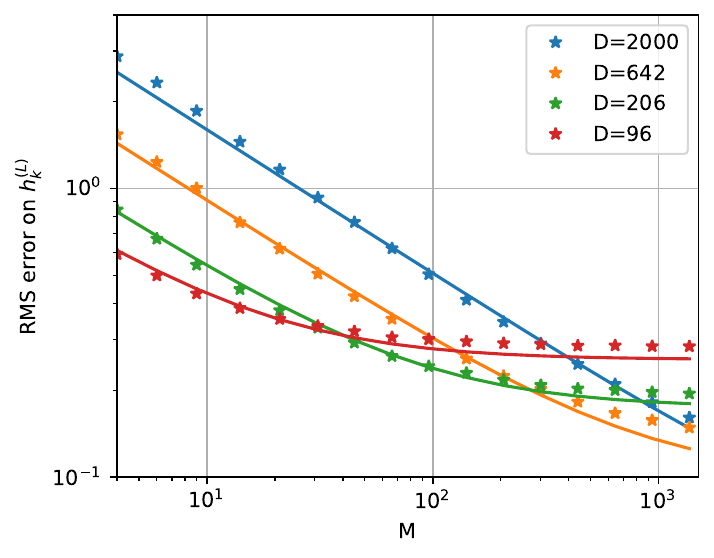}
\caption{Error on $\th^{(L)}_k$ vs $M$ ($L=50$)}
\end{subfigure}
\begin{subfigure}{0.32\linewidth}
\centering
\includegraphics[scale=0.4]{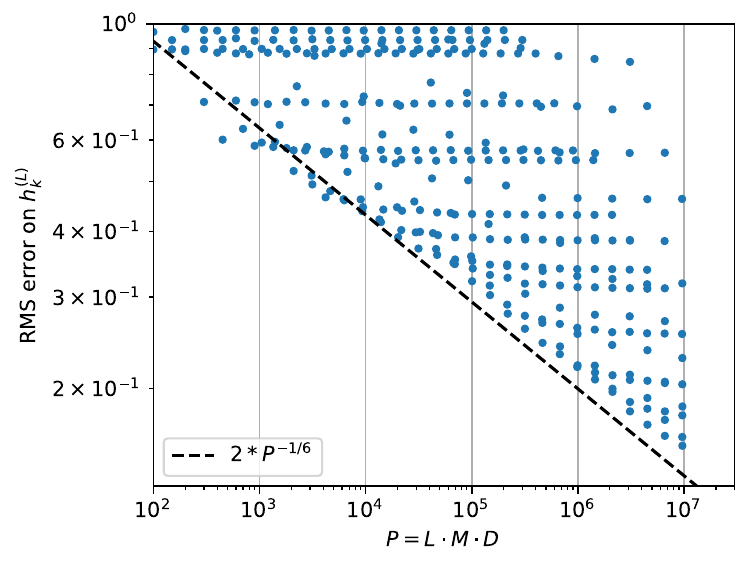}
\caption{Error on $\th^{(L)}_k$ vs $P$}\label{fig:errh_vs_P}
\end{subfigure}
\caption{Comparison of the experimental convergence rate of the hidden representation $h^{L}_k$ with the theoretical upper-bound $\Vert [ \frac{\alpha\sqrt{D}}{\sqrt{ML}}, \frac{\beta}{\sqrt{D}}]\Vert_2$ from Theorem~\ref{thm:main_theorem_intro} with $\alpha=0.8$ and $\beta=2.5$ manually adjusted to fit observations (plain lines). The y-axis shows RMS error (averaged over $5$ random repetitions and over the dataset) on the last hidden state $\th^L_k$ after $k=15$ GD steps. }\label{fig:errH-tanh}
\end{figure}

\subsection{Numerical experiments and tightness of the rate}\label{sec:experiments-tanh}
\paragraph{Rates are tight in embedding space.}
For the hidden representations $\th^{L}_k$, the convergence rate from Theorem~\ref{thm:main_theorem_intro} is empirically tight, as shown by the experimental results in
Figure~\ref{fig:errH-tanh}. There, we plot the RMS error between $\th^{(L)}_k$ for various shapes $(L,M,D)$ (fixing $L=50$)  and $\tilde \th^{(L)}_k$ which is a reference model of size $(L,M,D)=(50,2000,2000)$ acting as a proxy for the limit model. We note that this approximation leads to over-estimating the errors near the threshold $10^{-1}$ which is approximately the error of the reference model; we refer to Section~\ref{ssec:numerical_experiments} for experiments with $\rho(s)=s$ where the limit model can be computed accurately. We compare these observations with the error bound $\Vert [ \frac{\alpha\sqrt{D}}{\sqrt{ML}}, \frac{\beta}{\sqrt{D}}]\Vert_2$ from Theorem~\ref{thm:main_theorem_intro}, with manually adjusted coefficients $(\alpha,\beta)=(0.8,2.5)$. We here ignore the $O(1/L)$ error term, which is not visible in these plots because its prefactor happens to be small in this setting, and refer to~\cite{chizat2025hidden} for experiments showing the tightness of the rate in $L$. We use the $\ell_2$ norm to aggregate the two terms of the bound because, on the one hand, this is strictly equivalent to the bound from Theorem~\ref{thm:main_theorem_intro} and on the other this leads to better fits when the two error terms are asymptotically independent, which appears to be the case here.

The setting used here is $\dimin=\dimout=5$ with $N=10$ input/output drawn from a standard Gaussian,  the mean-square loss, the activation $\rho=\tanh$, and we plot the error after $k=15$ GD iterations\footnote{The code for all the numerical experiments can be found online at \url{https://github.com/lchizat/2026-resnets-of-all-shapes-and-sizes}.}. The randomness of the embedding matrices is coupled between $\th$ and  $\tilde \th$, as required by Theorem~\ref{thm:main_theorem_intro}.

\begin{figure}
\centering
\begin{subfigure}{0.32\linewidth}
\centering
\includegraphics[scale=0.4]{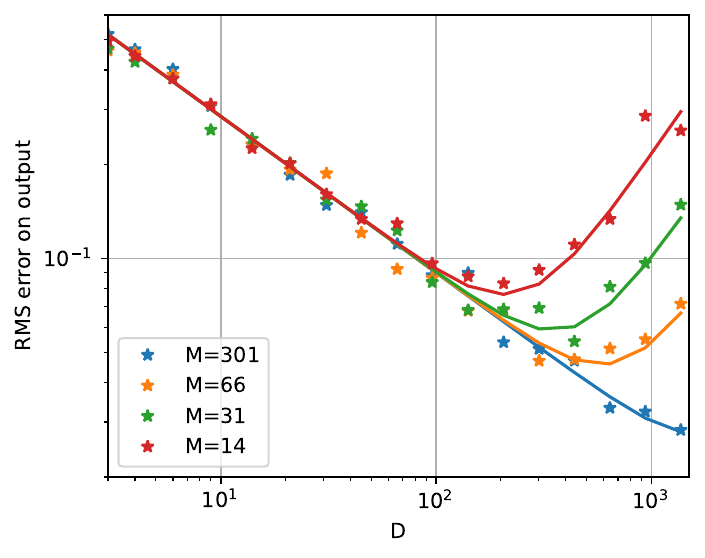}
\caption{Error on $\ty_k$ vs $D$}
\end{subfigure}%
\begin{subfigure}{0.32\linewidth}
\centering
\includegraphics[scale=0.4]{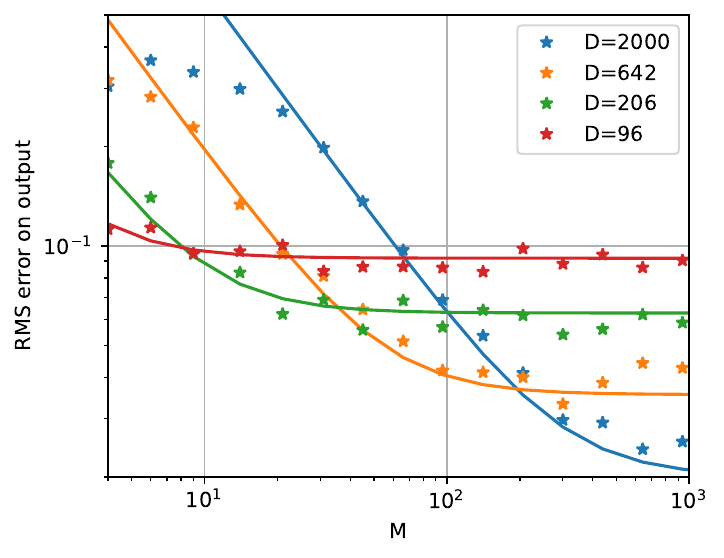}
\caption{Error on $\ty_k$ vs $M$}
\end{subfigure}
\begin{subfigure}{0.32\linewidth}
\centering
\includegraphics[scale=0.4]{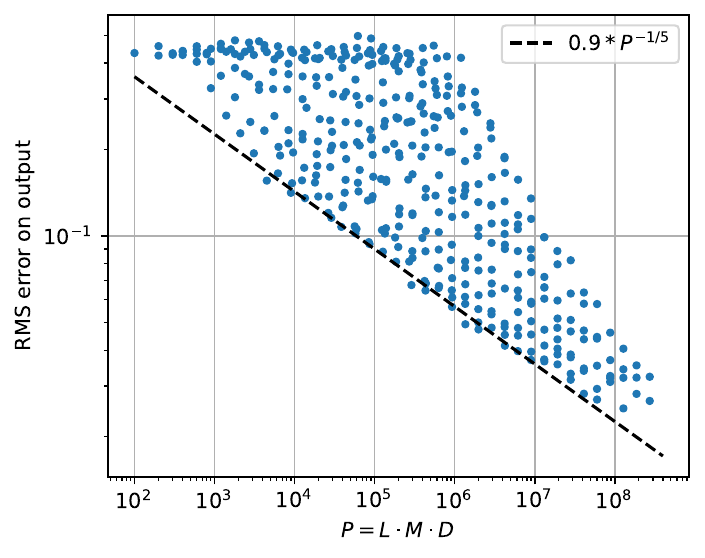}
\caption{Error on $\ty_k$ vs $P$}\label{fig:erry_vs_P}
\end{subfigure}
\caption{Comparison of the experimental convergence rate of the output $\ty_k$ with the conjectured rate $\Vert [ \frac{\alpha D}{ML}, \frac{\beta}{\sqrt{D}}]\Vert_2$ (which is smaller than the rate from Theorem~\ref{thm:main_theorem_intro}) with coefficients $\alpha=0.15$ and $\beta=0.9$ manually adjusted to fit observations (plain lines). The y-axis shows RMS error (averaged over $5$ random repetitions and over the dataset) on the output $\ty_k$ after $k=15$ GD steps.}\label{fig:erry-tanh}
\end{figure}

\paragraph{A conjecture for the rate in output space.} On Figure~\ref{fig:erry-tanh}, we plot the empirical error on the output $\Delta_k^y$. For this quantity, the bound from Theorem~\ref{thm:main_theorem_intro} is loose and we conjecture that the tight rate for $\Delta_k^y$ is instead 
\begin{equation}\label{eq:conjecture-erry}
   \Delta_k^y = \Theta\Big(\frac{1}{L} + \frac{D}{ML} + \frac{1}{\sqrt{D}}\Big ),
\end{equation}
that is, the middle term looses its square-root. The plain lines on Figure~\ref{fig:erry-tanh} show the conjectured rate (without the first term, negligible in our setting) 
$\Vert [ \frac{\alpha D}{ML}, \frac{\beta}{\sqrt{D}}]\Vert_2$, with manually adjusted coefficients $(\alpha,\beta)=(0.15,0.9)$, demonstrating a very good match. The setting is exactly the same as in the previous paragraph, except that we do not couple the initialization's randomness; this is not needed here since the limit $y_k$ is deterministic.

Our heuristic explanation for this conjectured rate is the following: in embedding space, the error term of size $\Theta(\sqrt{D}/\sqrt{ML})$ is due, at leading order, to \emph{centered} fluctuation terms generated at each forward/backward pass by the stochastic approximation of the Mean ODE, see~\cite{chizat2025hidden}. Taken individually, these centered terms become of size $\Theta(1/\sqrt{ML})$ after multiplication by $\frac{1}{D}(\tWout)^\top$ since the latter is akin to taking an average, noting that $\tWout$ is asymptotically independent from these fluctuations. This error term is negligible in front of $1/\sqrt{D}$ under our shape constraints. However, these fluctuation terms are not independent from each other and their pairwise interactions lead to noncentered terms of size $\Theta(D/(ML))$. This argument leads to the conjectured tight bound for $\Delta_k^y$. We leave it to future work to confirm and formalize this discussion.

\paragraph{Optimal shape and convergence rate.}
For a given number of parameters $P=\Theta(LMD)$, the upper bound from Theorem~\ref{thm:main_theorem_intro} is minimized with $LM =\Theta(P^{2/3})$ and $D=\Theta(P^{1/3})$, in which case we obtain the convergence rate $\Theta(P^{-1/6})$ as long as $L=\Omega(P^{1/6})$. This $P^{-1/6}$ rate can be observed on Figure~\ref{fig:errh_vs_P} and it is empirically tight if one is concerned with error in embedding space. 

If one is instead concerned with the error in output space, then one should do this computation with the conjectured rate~\eqref{eq:conjecture-erry} since Theorem~\ref{thm:main_theorem_intro} is not tight. A direct computation shows that the rate~\eqref{eq:conjecture-erry} is minimized with $LM=\Theta(P^{3/5})$ and $D=\Theta(P^{2/5})$, in which case we obtain the (conjectured) convergence rate $\Theta(P^{-1/5})$ as long as $L=\Omega(P^{1/5})$. This $P^{-1/5}$ rate can be observed on Figure~\ref{fig:erry_vs_P}. 

Finally, for a given error level, it is computationally preferable to choose the shape with the smallest depth $L$ so as to maximize parallelism in the computation of the forward or backward pass. This leads to the shape prescription $L=\Theta(P^{1/5})$ and $M,D=\Theta(P^{2/5})$, that is the widths ($M$ and $D$) should scale like the square of the depth $L$. Interestingly, we recover the common practice to scale $M$ and $D$ proportionally and at a faster rate than the depth $L$, see e.g.~\cite{grattafiori2024llama}.

\subsection{Related work}

\paragraph{Bridging Mean-field and Neural ODE analyses.}
The first infinite-dimensional analyses of neural networks training dynamics appeared in three forms: the Neural ODE framework~\cite{weinan2017proposal,chen2019neuralordinarydifferentialequations, lu2020meanfieldanalysisdeepresnet} dealing with $L\to \infty$, the mean-field analysis of two-layer perceptrons~\cite{nitanda2017stochastic,chizat2018globalconvergencegradientdescent, MeiMontanari2018mftwolayernetworks,sirignano2019meanfieldanalysisneural,Rotskoff2022mflimit} dealing with $M\to \infty$, and the Neural Tangent Kernel (NTK)~\cite{du2019gradient,jacot2020neuraltangentkernelconvergence} dealing with more general limits but in a linearized regime. Soon after, it was observed that the infinite-depth ($L \to \infty$) and infinite-width ($M \to \infty$) limits could be combined~\cite{lu2020meanfieldanalysisdeepresnet, ding2022mfresnets}. These works consider the joint limit $L \to \infty$ and $M \to \infty$, with fixed $D$ and, consistently with the above discussion, obtain the \emph{Mean ODE} dynamics in the limit. The tighter analysis in~\cite{chizat2025hidden} shows that it is sufficient to take $L\to \infty$ to converge to the same limit, and demonstrates via the bound~\eqref{eq:hidden-width-informal} that the Mean ODE framework is indeed relevant to describe practical architectures where $M=\Theta(D)$. There is a mathematically rich and growing literature studying properties of the Mean ODE, such as its continuous-time limit and long-time behavior~\cite{jabir2019mean,isobe2023convergence,barboni2023globalconvergenceresnetsfinite, bonnet2023measure,barboni2025understandingtraininginfinitelydeep, daudin2025genericity,gassiat2025gradient}. However, to the best of our knowledge, no prior work has considered the large $D$ behavior of the Mean ODE dynamics, as we do in the present paper.

\paragraph{Large-width analyses for MLPs.}  
The theoretical tractability of the NTK limit stems from a choice of scaling that makes the model asymptotically linear in its parameters. This regime can be studied in many architectures~\cite{allen2019learning,arora2019exact, yang2020tensorii}. The key role of the initialization scale (or explicit scaling factors) in determining the asymptotic training regime was first emphasized in~\cite{chizat2019lazy}, which also argued that this lazy/kernel regime is suboptimal due to the absence of feature learning. The classification of HP scalings was then extended to finite depth MLPs~\cite{geiger2020disentangling, yang2021tensorprogramsIV}, with the latter identifying the maximal (feature) update (MU) regime that enables transfer of optimal HPs between models of different widths~\cite{yang2021tuning}. In the MU regime, the infinite-width limit of MLPs training dynamics was obtained in~\cite{yang2021tensorprogramsIV} using a framework called \emph{Tensor Programs}, see also~\cite{chizat2024infinite} for a non-Gaussian and quantitative approach using the method of moments when $\rho=\mathrm{id}$ and~\cite{bordelon2022selfconsistentdynamicalfieldtheory} for a physics-inspired derivation of the limit model in continuous-time. Other works have also studied the large-depth limit of MLPs at initialization~\cite{pennington2017resurrecting, hanin2018neural, hanin2018start, hayou2019impact,jelassi2023depth}, often noticing that MLPs in the large-depth limit present some fundamental degeneracies. This motivates to turn our attention to ResNets architectures for the large-depth limit.

\paragraph{Large width-and-depth scaling of ResNets.} HP scalings in terms of depth for ResNets have been studied in~\cite{bordelon2023depthwisehyperparametertransferresidual,yang2023tensorprogramsvifeature}, with criteria that singled-out a residual-block scaling of \(\Theta_{M,D}\Big(\frac{1}{\sqrt{L}}\Big)\). In particular,~\cite{bordelon2023depthwisehyperparametertransferresidual} derives a limit model in the  $D=M \to \infty$ and $L\to \infty$ limit with a heuristic discussion on the convergence rates. However, these works deal with depth-one residual blocks, which present a different phenomenology than depth-two blocks used in practice and studied in the present work. In fact, those works noticed that the $\Theta_{M,D}(1/\sqrt{L})$ residual scaling leads to a linearization of residual blocks and experiments in~\cite{dey2026dontlazycompletepenables} suggest that this behavior might be empirically suboptimal when the blocks are depth-two. For such blocks, they proposed instead the residual scale \(\Theta\Big(\frac{1}{L\sqrt{D}}\Big)\) under the assumption \( M = \Theta(D)\), which is the same scaling in $L$ as the Neural Mean ODE literature which uses the residual scale $\Theta_D\Big(\frac{1}{ML}\Big)$. Bridging these various streams of works, \cite{chizat2025hidden} showed that a necessary and sufficient condition for maximal \emph{local} feature updates (MLU) when $D=O(LM)$ is the residual scale \(O\big(\frac{\sqrt{D}}{LM}\big)\). This is the scaling which we adopt here (the factor $1/(ML)$ is inserted in~\eqref{eq:fully_finite_resnet} while the factor $\sqrt{D}$ is absorbed in the STD of the initialization of $\tU^{\ell}$ in~\eqref{eq:resnet-GD}). 

In terms of the object of study, the prior work closest to ours is~\cite{bordelon2024infinitelimitsmultiheadtransformer} which derives the limit training dynamics in a setting that also covers the regime considered here (and additionally studies other residual scalings leading to SDE-type limits). Their analysis relies on a formal DMFT argument based on the path-integral method (see Section~\ref{sec:proof-strategy-cavity} for an alternative heuristic derivation of the limit using the cavity method). A key difference is that the present work provides a fully rigorous analysis and establishes quantitative convergence rates. Moreover, the perspectives of the two works differ:~\cite{bordelon2024infinitelimitsmultiheadtransformer} considers a sequential limit $M=D\to \infty$  followed by $L\to \infty$, whereas our approach treats the joint limit and does not require the proportional scaling $M=\Theta(D)$. When $M=\Theta(D)$, our error bound is $O\Big(\frac{1}{\sqrt{L}}+\frac{1}{\sqrt{D}}\Big)$ which shows that the limits in $D$ and $L$ commute (this extends the commutativity result from~\cite{hayou2023widthdepthlimitscommute} dealing with the first forward pass, albeit in a slightly different setting) and, in particular, that our limit coincides with the $D\to \infty$ then $L\to \infty$ sequential limit.

\paragraph{DMFT, rigorous limits and cavity method.}
Dynamical Mean Field theory (DMFT) was originally introduced in the context of spinglass physics \cite{Sompolinsky1982spinglass,Cugliandolo1993spinglass}. The first rigorous infinite-width limit for fixed-depth neural network in~\cite{yang2021tensorprogramsIV} is expressed in terms of a so-called \emph{Tensor Program} limit which relies on the Gaussian conditioning technique~\cite{bayati2011dynamics} from DMFT. While originally limited to Gaussian initializations, a non-Gaussian extension of the Tensor Program framework was obtained in \cite{golikov2022nongaussiantensorprograms}. Some recent efforts have aimed at making the Tensor Program framework quantitative \cite{golikov2025deep}, but the available results currently cover only the first forward pass. As a consequence, existing rigorous analyses of width–depth limits are limited to sequential limits, first $D\to\infty$ and then $L\to\infty$. For instance, \cite{yang2023tensorprogramsvifeature} proves such a result for depth-one residual blocks with $\rho(s)=s$, which, to the best of our knowledge, remains the only rigorous result addressing the regime $L,D\to\infty$ (in that setting there is no parameter $M$ since the blocks have depth one). Our approach follows a different route: we first analyze the limit $L\to\infty$ and obtain a sharp convergence rate, which subsequently enables us to take joint limits.

Our proof relies on the use of the cavity method, also from DMFT. This technique originates from the statistical physics literature \cite{Mezard1987spinglass}, where it has been widely used (see, e.g., \cite{Mezard2015cavityMP,Braunstein2023cavity} and many others). To the best of our knowledge, our work is among the first to implement this technique in a fully rigorous and quantitative manner. Closely related in spirit is the independent work \cite{dandi2026rigorouscavity}, which also establishes a DMFT limit via the cavity method for the high-dimensional limit of generalized linear model dynamics. Both their and our works rely on a Lindeberg central limit argument, but beyond this methodological similarity the settings and systems under study are substantially different. We also mention several recent works developing rigorous DMFT analyses in machine learning, including \cite{dandi2024benefits,gerbelot2024rigorous, fan2025dynamical,montanari2025dynamicaldecouplinggeneralizationoverfitting}.

\subsection{Notations} \label{sec:basic_notation}

This section contains most notations introduced in this work, and shall serve as a reference for the reader.

\paragraph{General objects:}
\begin{itemize}
    \item For $a,b\in\R$, $[a,b] \subseteq\R$ denotes a real valued interval, while for $a,b\in\mathbb{Z}$, $[a:b] = \{a,a+1,\dots, b\} \subseteq \mathbb{Z}$ is a \textit{discrete} interval. 
    \item Given a sequence $(a_i)_{i=0}^N$, for all $n\in[0:N]$ we set $a_{\wedge n} := (a_i)_{i=0}^n$ and so $\|a_{\wedge n}\|_{\infty} = \max_{i\in[0:n]} |a_i|$. 
    \item We denote by \(\|\cdot\|_{p,n}\) the $p$-norm in $\R^n$ -- for instance, $\|\cdot\|_{2, n}$ denotes the Euclidean norm. In particular, we also denote by $\|\cdot\|_{op, n\times n}$ the operator norm in $\R^{n \times n}$, and $\|\cdot\|_{2, n\times n}$ the Frobenius norm.
    
    \item We denote by $\langle \cdot, \cdot\rangle_{\ndbar}$ the RMS (or \textit{normalized}) inner product in $\R^D$, defined as $\langle x, y\rangle_{\ndbar} := \frac{1}{D}\sum_{d=1}^D x_d y_d$, which yields the RMS norm $\|x\|_{\ndbar} = \sqrt{\langle x,x\rangle_{\ndbar}} = \frac{1}{\sqrt{D}}\|x\|_{2,D}$. 
    
    \item We abuse notation by considering, for matrices $A\in \R^{D\times n}$, 
    \[\|A\|_{\ndbar} := \frac{1}{\sqrt{D}}\|A\|_{2, D\times n} = \sqrt{\frac{1}{D}\sum_{d=1}^D \|A_{d,\cdot}\|_{2,n}^2} = \sqrt{\frac{1}{D}\sum_{d=1}^D \sum_{j=1}^n A_{d,j}^2},\] 
    which is the rescaled Frobenius norm in $\R^{D\times n}$. This norm satisfies $\|Ax\|_{\ndbar} \leq \|A\|_{\ndbar}\|x\|_{2,n}$. An analogous convention is used for matrices $A\in \R^{n\times D}$. We may also write $\|A\|_{\overline{p}, D\times n} := \Big(\frac{1}{D}\sum_{d=1}^D \|A_{d,\cdot}\|_{2, n}^p\Big)^{1/p}$ for higher values of $p \geq 2$.
    \item To avoid ambiguity, we write $h^{(D)} := ( h^d )_{d=1}^D \in \R^D$. That is, $h^d$ denotes a single coordinate of the vector $h^{(D)}\in \R^D$. To alleviate notation, we may omit the parenthesis when writing vectors as superscripts. Namely, $A^{h^{D}}$ should be understood as $A^{h^{(D)}}$ unless stated otherwise.
    \item We consider vectors $\bfz^h = ( z^h_k )_{k=0}^{K-1} \in \R^K$, which are indexed from $0$, consistently with a \textit{path-dependent} view. 
    For $k \in[0:K-1]$, we say that a function $F_k (\bfz^h)$ is \emph{non-anticipative} if it only depends on the first $k$ components of $\bfz^h$. 
    We use the bold notation for vectors in $\R^{K}$; 
    in particular, only the first $k$ components of $\bnabla_{\bfz^h}F_k$ can be non-zero for a non-anticipative $F_k$.
    \item With a slight abuse of notation, for \(\bfz^h, \bfz^b \in \R^K\) and $p\geq 1$ (possibly $p=\infty$), we shall write \(\|(\bfz^h, \bfz^b)\|_{p, 2k} = \big( \sum_{i=0}^{k-1} |z_i^h|^p + |z_i^b|^p\big)^{1/p}\) for the norm taken on the first $k$ coordinates of $(\bfz^h, \bfz^b)$; and similarly for \(\|\bfz^h\|_{p, k}\) or even \(\|\mathbf{M}\|_{p, k\times k}\) for $\mathbf{M}\in\R^{K\times K}$, which is taken on the first $k\times k$ block of the matrix.
    \item We write $\L( X )$ for the law of a random variable $X$.
    \item For $\alpha>0$ and a real-valued random variable $X$, we define the Orlicz norm  
    \[ \|X\|_{\psi_\alpha} := \inf \Big\{ c>0: \E\Big[\exp\Big(\Big(\frac{|X|}{c}\Big)^{\alpha}\Big)\Big] \leq 2 \Big\}.\]
    In particular, for $\alpha\in \{1,2\}$, this corresponds to the \textit{subexponential} and \textit{subgaussian} norms, respectively. We say that a random variable $X$ is $\alpha$-\textit{subweibull} if $\|X\|_{\psi_\alpha} <\infty$.
\end{itemize}

\paragraph{Recurring quantities and bounds:}
\begin{itemize}
    \item We refer to both the \textit{embedding} and \textit{unembedding} matrix as \textit{embedding matrices}. In the finite $D$ setting, we write $\Win^{(D)} = ((\Win^d)^\top)_{d=1}^D \in \R^{D\times \dimin}$, where $\Win^d \in \R^{\dimin}$ represents the $d$-th row of the matrix. Similarly, $\Wout^{(D)} = ((\Wout^d)^\top)_{d=1}^D \in \R^{D\times \dimout}$. When the rows $(\Win^d, \Wout^d)$ are i.i.d.~subgaussian, we denote their variance proxies by $\vpin$ and $\vpout$, respectively.
    
    \item $\Finit = \sigma(\Win,W_{\mathrm{out}})$ denotes the sigma-field generated by the embedding matrices.

    \item The \textit{initialization} of the weights of the ResNet \ref{eq:original_finite_D_system} is based on $U^{(D)}, V^{(D)}$. When the coordinates are i.i.d.~subgaussian, we 
    denote their variance proxies by $\vpu^2$ and $\vpv^2$; and their variances by $\varu^2$ and $\varv^2$, respectively. By standard arguments, we have $\varu \leq \vpu$, $\varv \leq \vpv$.

    \item We denote by $\cstW_k$ a generic $\Finit$-measurable (random) constant that solely depends on $\Vert \Win^{(D)} \Vert_{\overline{D}}$, $\Vert \Wout^{(D)} \Vert_{\overline{D}}$, and possibly $\|\Win^{(D)}\|_{\overline{p}, D\times \dimin}$ and \(\|\Wout^{(D)}\|_{\overline{p}, D\times \dimout}\), for $p \in\{3,4\}$.

    \item For a  $\C([0,1], \R^{2K})$-valued random variable \((\bfH, \bfB)\), 
    we define the covariance kernels
    \[\boldsymbol{\Gamma}^H(s,t) := \E[\bfH(s) \bfH(t)^\top] = (\E[H_i(s) H_j(s)])_{i,j=0}^{K-1} \in \R^{K\times K},\quad \boldsymbol{\Gamma}^B(s,t) := \E[\bfB(s) \bfB(t)^\top].\]
    For any $s\in[0,1]$ and $k\in[0:K-1]$, we set
    \[\boldsymbol{\Gamma}^H_{\wedge k}(s) := \E[\mathbf{H}_{\wedge k}(s) \mathbf{H}_{\wedge k}(s)^\top] = (\E[H_i(s) H_j(s)])_{i,j=0}^{k+1} \in \R^{(k+1)\times(k+1)}, 
    \]
    \[\vec{\boldsymbol{\Gamma}}^H_k(s) := \E[H_k(s) \mathbf{H}_{\wedge k-1}(s)] = (\E[H_k(s) H_i(s)])_{i=0}^{k-1} \in \R^k,\] 
     and similarly for $\boldsymbol{\Gamma}^B_{\wedge k}(s)$ and $\vec{\boldsymbol{\Gamma}}^B_k(s)$. In particular, we have
     \(\boldsymbol{\Gamma}_{\wedge k}^H = \begin{pmatrix} \boldsymbol{\Gamma}_{\wedge k-1}^H &\vec{\boldsymbol{\Gamma}}_k^H\\(\vec{\boldsymbol{\Gamma}}_k^H)^\top & \Gamma^H_{k,k}\end{pmatrix}\).

    \item Similarly, $\cstH_k$ denotes a constant that depends on $k$ and $\boldsymbol{\Gamma}^H_{\wedge k}, \boldsymbol{\Gamma}^B_{\wedge k}$, as well as other fixed parameters of the problem, such as $\varu, \varv, \eta_u,\eta_v$. 
    For $s\in[0,1]$, we write $\llnW_k(s)$ to denote a \textit{subexponential LLN concentration term} as introduced in Definition~\ref{def:lln_terms_RW}. That is,
    \[\llnW_k(s) = \Big\|\frac{1}{D}\sum_{d=1}^D X^d_k(s)\Big\|_{\H},\]
    for $(X^d_k(s))_{d\geq 1}$ an i.i.d.~sequence of centered, subexponential and $\Finit$-measurable random variables defined on a finite-dimensional vector space $\H$
    such that $\|\sup_{s\in[0,1]} \| X^1_k(s)\|_{\H} \|_{\psi_1} \leq c\vpk^2$, for some absolute constant $c>0$ and $\vpk^2$ the centered variance proxy of $\|(\bfH_{\wedge k-1}, \bfB_{\wedge k-1})\|_{\infty}$.

    \item Throughout this work, we make use of the non-anticipative skeleton vectors \(\vec{\mathbf{v}}_k^F\) and \(\vec{\mathbf{v}}_k^G\) that only depend on \(\boldsymbol{\Gamma}^H_{\wedge k-1}(s)\) and \(\boldsymbol{\Gamma}^B_{\wedge k-1}(s)\), and are introduced in Definition~\ref{def:linear_skeleton_meanfield}. These are defined by recurrence: \(\vec{\mathbf{v}}_1^F(s, \bfz^h, \bfz^b) = \rho'(z_0^h)z_0^b\), and \(\vec{\mathbf{v}}_1^G(s, \bfz^h, \bfz^b) = \rho(z_0^h)\) and, for $k> 1$, 
\begin{equation}
  \vec{\mathbf{v}}_{k}^F(s, \bfz^h, \bfz^b) = \begin{pmatrix}
    \vec{\mathbf{v}}_{k-1}^F(s, \bfz^h, \bfz^b)\\
    \rho'\left(p_{k-1}(s, \bfz^h, \bfz^b)\right) q_{k-1}(s, \bfz^h, \bfz^b)
\end{pmatrix}, \quad \vec{\mathbf{v}}_{k}^G(s, \bfz^h, \bfz^b) = \begin{pmatrix}
    \vec{\mathbf{v}}_{k-1}^G(s, \bfz^h, \bfz^b)\\
    \rho\left(p_{k-1}(s, \bfz^h, \bfz^b)\right).
\end{pmatrix},
\end{equation}
where \(p_k(s, \bfz^h, \bfz^b) := z_k^h - \eta_u \vec{\mathbf{v}}_{k}^F(s, \bfz^h, \bfz^b)^\top \vec{\boldsymbol{\Gamma}}^H_k(s)\) and \(q_k(s, \bfz^h, \bfz^b) := z_k^b - \eta_v \vec{\mathbf{v}}_{k}^G(s, \bfz^h, \bfz^b)^\top \vec{\boldsymbol{\Gamma}}^B_k(s)\).
\end{itemize}

\section{Main results on the large-$D$ limit of Mean ODEs} \label{sec:setting_and_main_results}

Let us introduce the dynamics obtained in \cite{chizat2025hidden} in the large-depth $L \rightarrow + \infty$ limit. The Mean ODE model reads, for finite $D\geq 1$, at training time $k \in[0:K-1]$ with $K\geq 1$, for an input $x \in \R^{d_{\mathrm{in}}}$ and some $w\in \R^{\dimout}$,
\begin{multline}\label{eq:original_finite_D_system}
\begin{cases}
h^{(D)}_k (0,x) = \Win^{(D)} x, \\
\partial_s h_k^{(D)} (s,x) = \E \big[ \rho \big( \langle U_k^{(D)} (s), h_k^{(D)} (s,x)\rangle_{\ndbar} \big)  V_k^{(D)} (s)  |\Finit\big], \\
\hat{y}_k^D (x) = \frac{1}{D}(\Wout^{(D)})^\top h_k^{(D)} (1,x),\\
\partial_s b_k^{(D)} (s,x, w) =  - \E \big[ \rho' \big( \langle U_k^{(D)} (s), h_k^{(D)} (s,x)\rangle_{\ndbar} \big) \langle V_k^{(D)} (s), b_k^{(D)} (s,x,w)\rangle_{\ndbar} U_k^{(D)} (s) |\Finit\big], \\
b_k^{(D)} (1,x,w) = \Wout^{(D)}w,
\end{cases}
\end{multline}
where, \(\Win^{(D)} := ((\Win^d)^{\top})_{d=1}^D \in \R^{D \times \dimin}\) and \(\Wout^{(D)} := ((\Wout^d)^\top)_{d=1}^D \in \R^{D \times \dimout}\) are random matrices, which we refer to as the \textit{embedding} matrices and $\Finit = \sigma(\Win, \Wout)$. 

Consider a fixed dataset $(x_i)_{i=1}^N \in (\R^{\dimin})^N$ and a family of loss functions $\ell_i :\R^{d_{\mathrm{out}}} \to \R$ with $i\in[1:N]$, and let \(w_i :=\nabla \ell_i (\hat{y}^D_k(x_i))\). The training dynamics for the Mean ODE model reads
\begin{equation}\label{eq:defn_update_rule_resnet}
    \begin{cases}
        U_{k+1}^{(D)} (s) &= U_k^{(D)} (s) - \frac{\eta_u}{N} \sum_{i=1}^N\rho' \big( \langle U_k^{(D)} (s), h_k^{(D)} (s,x_i)\rangle_{\ndbar}\big) \langle V_k^{(D)} (s),b_k^{(D)} (s,x_i, w_i)\rangle_{\ndbar} h_k^{(D)} (s,x_i), \\
    V_{k+1}^{(D)}(s) &= V_k^{(D)} (s) - \frac{\eta_V}{N} \sum_{i=1}^N\rho \big( \langle U_k^{(D)} (s), h_k^{(D)} (s,x_i)\rangle_{\ndbar} \big) b_k^{(D)} (s,x_i, w_i).
    \end{cases}
\end{equation}
For convenience, we decompose the updates as $(U_k^{(D)}(s),V_k^{(D)}(s)) = (U_0^{(D)} + \Delta U_k^{(D)}(s), V_0^{(D)} + \Delta V_k^{(D)}(s))$, with $U_0^{(D)}$ and $V_0^{(D)}$ being the random vectors in $\R^D$ representing the initialization.

\begin{rem}[Training the embedding matrices]\label{rem:training_embedding_matrices_finiteD}
In practice, the training dynamics also involve updating the embedding matrices. If we write $\big((\hat{W}_{\mathrm{in}}^{(D)})_k, (\hat{W}_{\mathrm{out}}^{(D)})_k\big)_{k\geq 0}$ for the iterates of these matrices over training, then we have, for each $k\in[0:K-1]$,
\[\begin{cases}
    (\hat{W}_{\mathrm{in}}^{(D)})_{k+1} = (\hat{W}_{\mathrm{in}}^{(D)})_{k} - \frac{\eta_{\mathrm{in}}}{N} \sum_{i=1}^N b_{k}^{(D)}(0, x_i, w_i) x_i^\top,\\
    (\hat{W}_{\mathrm{out}}^{(D)})_{k+1} = (\hat{W}_{\mathrm{out}}^{(D)})_{k} - \frac{\eta_{\mathrm{out}}}{N} \sum_{i=1}^N h_{k}^{(D)}(1, x_i)\nabla\ell(\hat{y}^D_{k}(x_i))^\top,
\end{cases}\]
where the dynamics are initialized as \((\hat{W}_{\mathrm{in}}^{(D)})_0\), \((\hat{W}_{\mathrm{out}}^{(D)})_0\) having $D$ independent rows, which are $\vpin^2$ and $\vpout^2$-subgaussian random vectors, respectively. 
In \eqref{eq:original_finite_D_system}, \(\Win^{(D)}\) and \(\Wout^{(D)}\) should then be replaced by \((\hat{W}_{\mathrm{in}}^{(D)})_k\), \((\hat{W}_{\mathrm{out}}^{(D)})_k\) for each $k\in[0:K-1]$.

This additional update rule adds no conceptual difficulty to our analysis, as discussed in Appendix~\ref{sec:extensions_of_results}.
However, for simplicity and brevity, we restrict ourselves to the case of fixed \(\Win^{(D)}\) and \(\Wout^{(D)}\) over training iterations. 
\end{rem}

In the following we restrict ourselves to $N=1$, see Remark~\ref{rem:skeleton_maps_for_N_ge_1} for further insight on the $N>1$ case. In particular, we have a single input $x\in\R^{\dimin}$ and loss $\ell :\R^{d_{\mathrm{out}}} \to \R$, which we consistently omit in our notation. Namely, we write \(h_{k}^{(D)}(s) := h_k^{(D)}(s, x)\) and \(b_{k}^{(D)}(s) := b_k^{(D)}(s, x, \nabla\ell(\hat{y}^D_k(x)))\).

\subsection{Skeleton maps}\label{sec:definition_skeletons}
When sending $D\to \infty$, we will have to deal with the asymptotics of terms such as
\begin{equation} \label{eq:TermDiff}
\langle \sqrt{D}U^{(D)} + \Delta U_k^{(D)} (s), h_k^{(D)} (s)\rangle_{\ndbar} = \frac{1}{\sqrt{D}}\sum_{d=1}^D U^d h_k^d(s)
+ \frac{1}{D} \sum_{d=1}^D \Delta U_k^d(s) h_k^d(s).
\end{equation}
 Remark that, conditional on $\Finit$, the only randomness in the system comes from the initialization $U_0^{(D)}, V_0^{(D)}$. Furthermore, for a solution of the Mean ODE \eqref{eq:original_finite_D_system}, $(\bfh^{(D)}, \bfb^{(D)}) \in \C([0,1], (\R^D)^{2K})$, the random variables
\begin{equation}\label{eq:definition_clt_type_sums}
    ( \bfS^{\mathbf{h}^{D}}(s), \bfS^{\mathbf{b}^{D}}(s)) := \bigg( \bigg( \frac{1}{\sqrt{D}} \sum_{d=1}^D U^{d} h^{d}_i(s) \bigg)_{0 \leq i \leq K-1} , \bigg( \frac{1}{\sqrt{D}} \sum_{d=1}^D V^{d} b^{d}_i(s) \bigg)_{0 \leq i \leq K-1} \bigg) \in \R^{2K}
\end{equation}
are the only quantities through which the initial randomness of $(U^{(D)}_0,V^{(D)}_0)$ affects the dynamics. In particular, at each stage $k\in[0:K-1]$, writing \eqref{eq:defn_update_rule_resnet} in terms of \(\bfS^{\mathbf{h}^{D}}, \bfS^{\mathbf{b}^{D}}\), we notice that
\[ (\Delta U_k^{(D)}(s), \Delta V_k^{(D)}(s) ) = (f_k^{\bfh^{D},\bfb^{D}},g_k^{\bfh^{D},\bfb^{D}})( s, \bfS^{\bfh^{D}}, \bfS^{\bfb^{D}}),\]
for some explicit $\Finit$-measurable functions \(f_k^{\bfh^D,\bfb^D},g_k^{\bfh^D,\bfb^D}\) that do not depend on the randomness of $U_0^{(D)}, V_0^{(D)}$. This motivates the following definition.
\begin{definition}[Finite-dimensional skeleton maps]\label{def:finite_skeleton_maps}
Given some integer $D \geq 1$ and two $\Finit$-measurable 
random variables \((\mathbf{h}^{(D)}, \mathbf{b}^{(D)})\) taking values in $\C([0,1], (\R^D)^{2K})$
we define the skeleton maps associated to $(\mathbf{h}^{(D)}, \mathbf{b}^{(D)})$, as the functions $\mathbf{f}^{\bfh^D,\bfb^D},\mathbf{g}^{\bfh^D,\bfb^D} : [0,1] \times \R^K \times \R^K \rightarrow (\R^D)^K$ that satisfy: $f^{\mathbf{h}^D,\mathbf{b}^D}_{0} = g^{\mathbf{h}^D,\mathbf{b}^D}_{0} \equiv 0$ and, for $k\in[0:K-1]$, $s\in[0,1]$, $\bfz^h, \bfz^b \in \R^K$,
\begin{equation}\label{eq:def_skeleton_finiteD}
\begin{cases}
f^{\mathbf{h}^D,\mathbf{b}^D}_{k+1} = f^{\mathbf{h}^D,\mathbf{b}^D}_{k} - \eta_u \rho' ( z^h_k + \langle h_k^{(D)}, f^{\mathbf{h}^D,\mathbf{b}^D}_{k}\rangle_{\ndbar} ) [ z^b_k + \langle b_k^{(D)}, g^{\mathbf{h}^D,\mathbf{b}^D}_k\rangle_{\ndbar} ] h_k^{(D)}, \\
g^{\mathbf{h}^D,\mathbf{b}^D}_{k+1} = g^{\mathbf{h}^D,\mathbf{b}^D}_{k} - {\eta}_v \rho ( z^h_k + \langle h_k^{(D)},  f^{\mathbf{h}^D,\mathbf{b}^D}_{k} \rangle_{\ndbar} ) b_k^{(D)},
\end{cases}
\end{equation}
where we underly the evaluation of $f^{\mathbf{h}^D,\mathbf{b}^D}_{k}(s,\bfz^h,\bfz^b)$, $g^{\mathbf{h}^D,\mathbf{b}^D}_{k}(s,\bfz^h,\bfz^b)$ and $h_k^{(D)}(s)$, $b_k^{(D)}(s)$.%
\end{definition}

\begin{rem}[Non-anticipative skeleton maps]\label{rem:nonanticipative_skeleton}
    We note that the skeleton maps \(\mathbf{f}^{\bfh^D,\bfb^D}\), \(\mathbf{g}^{\mathbf{h}^D,\mathbf{b}^D}\) are non-anticipative, in the sense that for all $k\in[0:K-1]$, $f^{\bfh^D,\bfb^D}_k ( s , \bfz^h,\bfz^b )$ only depends on the first $k$ components of $\bfz^h, \bfz^b \in\R^{K}$, which represent a \emph{path} along the training time. 
    The \textbf{bold} notation $\bfz^h = ( z^h_j )_{j=0}^{K-1}$ is devoted to vectors of size $K$, which are indexed from $0$ consistently with the path-dependent view.
\end{rem}
\begin{rem}[Fixed-point formulation]
    Using this path-dependent construction, the solution of the Mean ODE \eqref{eq:original_finite_D_system} can be seen as the unique fixed-point of the map $(\bfh^{(D)}, \bfb^{(D)}) \mapsto (\mathbf{U}^{(D)}, \mathbf{V}^{(D)}) := (\mathbf{f}^{\mathbf{h}^D,\mathbf{b}^D}, \mathbf{g}^{\mathbf{h}^D,\mathbf{b}^D})$. 
    This motivates our construction for the limit system in Section~\ref{sec:construction_of_limit_law}. 
\end{rem}

As $D\to \infty$, we will show that the coordinates of $(\bfh^{(D)}, \bfb^{(D)}, \mathbf{f}^{\bfh^D,\bfb^D}, \mathbf{g}^{\bfh^D,\bfb^D})$ behave as independent and $\Finit$-measurable  realizations of a common limit law,  
and the RMS inner products $\langle \cdot, \cdot\rangle_{\ndbar}$ in \eqref{eq:def_skeleton_finiteD} become expectations over the randomness from $\Finit$ in the $D \rightarrow + \infty$ limit. This motivates the following definition.
\begin{definition}[Mean-Field Skeleton Maps]\label{def:mf_skeleton_maps}
Given a $\Finit$-measurable random variable $(\mathbf{H}$, $\mathbf{B})$ taking values in $\C([0,1],\R^{2K})$, we define the associated mean-field skeleton maps as the $\Finit$-measurable random maps $\mathbf{F}^{\bfH,\bfB}, \mathbf{G}^{\bfH,\bfB} : [0,1] \times \R^K \times \R^K \rightarrow \R^K$ given by
\begin{equation}\label{eq:def_skeleton_meanfield}
    \forall k\in[0:K-1],\quad\begin{cases}
F_{k+1}^{\mathbf{H},\mathbf{B}} = F_{k}^{\mathbf{H},\mathbf{B}} - \eta_u \rho' ( z^h_k + \E [ H_k F_k^{\mathbf{H},\mathbf{B}} ] ) [ z^b_k + \E [ B_k G_k^{\mathbf{H},\mathbf{B}} ]] H_k, \\
G_{k+1}^{\mathbf{H},\mathbf{B}} = G_{k}^{\mathbf{H},\mathbf{B}} - {\eta}_v \rho ( z^h_k + \E [ H_k F_{k}^{\mathbf{H},\mathbf{B}}  ] ) B_k,
\end{cases}
\end{equation}
with $F_0^{\mathbf{H},\mathbf{B}} = G_0^{\mathbf{H},\mathbf{B}} \equiv 0$. As in Definition~\ref{def:finite_skeleton_maps}, we consistently underly evaluation at $(s,\bfz^h,\bfz^b)$.    
\end{definition}

We notice that $(\mathbf{F}^{\mathbf{H},\mathbf{B}}, \mathbf{G}^{\mathbf{H},\mathbf{B}})$ are non-anticipative, as detailed in Remark~\ref{rem:nonanticipative_skeleton}. These maps depend on both the law and the realization of $(\mathbf{H},\mathbf{B})$. 
\begin{definition}[Skeleton vectors]\label{def:linear_skeleton_meanfield}
    From $(\bfH, \bfB)$ we define the deterministic skeleton vectors \(\vec{\mathbf{v}}_k^F, \vec{\mathbf{v}}_k^G:[0,1]\times \R^k\times \R^k \to \R^k\) given by the recursion: \(\vec{\mathbf{v}}_1^F(s, \bfz^h, \bfz^b) = \rho'(z_0^h)z_0^b\), and \(\vec{\mathbf{v}}_1^G(s, \bfz^h, \bfz^b) = \rho(z_0^h)\) and, for $k> 1$,
\begin{equation}\label{eq:explicit_definition_vFkvGk}
  \vec{\mathbf{v}}_{k}^F(s, \bfz^h, \bfz^b) = \begin{pmatrix}
    \vec{\mathbf{v}}_{k-1}^F(s, \bfz^h, \bfz^b)\\
    \rho'\left(p_{k-1}(s, \bfz^h, \bfz^b)\right) q_{k-1}(s, \bfz^h, \bfz^b)
\end{pmatrix}, \quad \vec{\mathbf{v}}_{k}^G(s, \bfz^h, \bfz^b) = \begin{pmatrix}
    \vec{\mathbf{v}}_{k-1}^G(s, \bfz^h, \bfz^b)\\
    \rho\left(p_{k-1}(s, \bfz^h, \bfz^b)\right).
\end{pmatrix},
\end{equation}
where \(p_k(s, \bfz^h, \bfz^b) := z_k^h - \eta_u \vec{\mathbf{v}}_{k}^F(s, \bfz^h, \bfz^b)^\top \vec{\boldsymbol{\Gamma}}^H_k(s)\) and \(q_k(s, \bfz^h, \bfz^b) := z_k^b - \eta_v \vec{\mathbf{v}}_{k}^G(s, \bfz^h, \bfz^b)^\top \vec{\boldsymbol{\Gamma}}^B_k(s)\). Note that they are non-anticipative and only depend on \(\boldsymbol{\Gamma}^H_{\wedge k-1}(s)\) and \(\boldsymbol{\Gamma}^B_{\wedge k-1}(s)\).
\end{definition}
Inductively developing \eqref{eq:def_skeleton_meanfield}, we get, for all $k \in[1:K-1]$, $s\in[0,1]$, $\bfz^h, \bfz^b \in\R^K$,
\begin{equation}\label{eq:cleaner_system_Fk_definition}
F_{k}^{\mathbf{H},\mathbf{B}} (s, \bfz^h, \bfz^b)  = -\eta_u \vec{\mathbf{v}}_k^F(s, \bfz^h, \bfz^b)^\top \mathbf{H}_{\wedge k-1}(s) , \;\;
G_{k}^{\mathbf{H},\mathbf{B}}(s, \bfz^h, \bfz^b)  = -\eta_v \vec{\mathbf{v}}_k^G(s, \bfz^h, \bfz^b)^\top \mathbf{B}_{\wedge k-1}(s).
\end{equation}
\begin{rem}[Skeleton Maps for $N>1$]\label{rem:skeleton_maps_for_N_ge_1}
We note that Definition~\ref{def:finite_skeleton_maps} exactly corresponds to the update rule \eqref{eq:defn_update_rule_resnet} for $N=1$ when formally identifying $\Delta U_k$ with $f_k^{\bfh^D, \bfb^D}$.

To handle the case $N>1$, Definition~\ref{def:finite_skeleton_maps} should be adapted for $\mathbf{h}^D, \mathbf{b}^D : [0,1] \rightarrow ( \R^D )^{K\times N}$, by introducing $\mathbf{f}^{\mathbf{h}^D,\mathbf{b}^D},\mathbf{g}^{\mathbf{h}^D,\mathbf{b}^D} : [0,1] \times \R^{K\times N} \times \R^{K\times N} \rightarrow (\R^D)^{K}$ with $f^{\mathbf{h}^D,\mathbf{b}^D}_{0} = g^{\mathbf{h}^D,\mathbf{b}^D}_{0} \equiv 0$, satisfying, for $k\in[0:K-1]$, 
\begin{equation}
\begin{cases}
f^{\mathbf{h}^D,\mathbf{b}^D}_{k+1} = f^{\mathbf{h}^D,\mathbf{b}^D}_{k} - \frac{\eta_u}{N} \sum_{i=1}^N\rho' ( z^h_{k,i} + \langle h_{k,i}^D, f^{\mathbf{h}^D,\mathbf{b}^D}_{k}\rangle_{\ndbar} ) [ z^b_{k,i} + \langle b_{k,i}^D, g^{\mathbf{h}^D,\mathbf{b}^D}_k\rangle_{\ndbar} ] h_{k,i}^D, \\
g^{\mathbf{h}^D,\mathbf{b}^D}_{k+1} = g^{\mathbf{h}^D,\mathbf{b}^D}_{k} - \frac{\eta_v}{N}\sum_{i=1}^N \rho \big( z^h_{k,i} + \langle h_{k,i}^D,  f^{\mathbf{h}^D,\mathbf{b}^D}_{k} \rangle_{\ndbar} \big) b_{k,i}^D,
\end{cases}
\end{equation}
and similarly in Definition~\ref{def:mf_skeleton_maps}. 
This does not conceptually complicate our setting, beyond producing heavier notations and dependencies on $KN$ instead of $K$.
For simplicity and brevity, we stick to the case $N=1$.
\end{rem}

\subsection{Limit dynamics and quantitative convergence}

We recall that we assume $N=1$ throughout the remaining of this article. We now define the limit system and refer to Section~\ref{sec:proof-strategy-cavity} for a quick heuristic derivation. 

\paragraph{Well-posedness.}
Let $(\Omega,\F,\P)$ be a probability space that carries at least two subgaussian random vectors $W_{\mathrm{in}}\in \R^{\dimin}$ and $W_{\mathrm{out}} \in \R^{\dimout}$. We say that  $(\mathbf{H}, \mathbf{B})$ is a solution of the limit system if it is a $\Finit$-measurable random variable in $L^2 ( \Omega, \C ( [0,1], \R^{2 K}) ) $ that a.s. satisfies, for $k \in[0:K-1]$,
\begin{equation} \label{eq:LimH}
\begin{cases}
{H}_k(0) = W_{\mathrm{in}}\cdot x, \\
P_k (s) := Z_k^{h} (s) + \E [ H_k (s) F^{\mathbf{H},\mathbf{B}}_k (s, \mathbf{Z}^{h},\mathbf{Z}^{b}) \vert  \mathbf{Z}^{h},\mathbf{Z}^{b} ], \\
\partial_s {H}_k (s) = \varv^2 \E [ \rho'( P_k (s) ) \E[ H_k (s) \bnabla_{\bfz^b} F_k^{\mathbf{H},\mathbf{B}} (s,  \mathbf{Z}^{h},\mathbf{Z}^{b} ) \vert  \mathbf{Z}^{h},\mathbf{Z}^{b} ] \,]\cdot \mathbf{B}_{\wedge k-1}(s) \\
\qquad\qquad\qquad\qquad\qquad+ \E [ \rho ( P_k (s) ) G^{\mathbf{H},\mathbf{B}}_k (s, \mathbf{Z}^{h},\mathbf{Z}^{b}) \vert \Finit], \qquad s \in [0,1],  \\
y_k := \E [ H_k(1) W_{\mathrm{out}}], 
\end{cases}
\end{equation}
as well as
\begin{multline} \label{eq:LimB}
\begin{cases}
Q_k (s) = Z_k^{b} (s) + \E [ B_k (s) G^{\mathbf{H},\mathbf{B}}_k (s, \mathbf{Z}^{h},\mathbf{Z}^{b}) \vert  \mathbf{Z}^{h},\mathbf{Z}^{b} ],\\
\partial_s {B}_k(s) = - \E \big[ \rho'( P_k(s) ) F_k^{\mathbf{H},\mathbf{B}} (s, \mathbf{Z}^{h},\mathbf{Z}^{b}) Q_k(s)  \vert \Finit\big] \\
\qquad \qquad -\varu^2 \E \big[ \rho' \big( P_k(s) ) \E [ B_k(s) \nabla_{\bfz^h} G_k^{\mathbf{H},\mathbf{B}} ( s, \mathbf{Z}^{h},\mathbf{Z}^{b}) \vert  \mathbf{Z}^{h},\mathbf{Z}^{b} ] \big]\cdot \mathbf{H}_{\wedge k-1}(s) \\
\qquad\qquad- \varu^2 \E \big\{ \rho''( P_k(s) )Q_k(s) \big[ H_k(s) + \E [ H_k(s) \bnabla_{\bfz^h} F_k^{\mathbf{H},\mathbf{B}} (s,  \mathbf{Z}^{h},\mathbf{Z}^{b}) \vert  \mathbf{Z}^{h},\mathbf{Z}^{b} ] \cdot \mathbf{H}_{\wedge k-1}(s) \big] \vert \Finit \big\},\\ 
{B}_k(1) = W_{\mathrm{out}}\cdot \nabla \ell ( y_k ),
\end{cases}
\end{multline}
where $(\mathbf{Z}^{h},\mathbf{Z}^b)$ is a centered Gaussian field in $L^2 ( [0,1],\R^{2K} )$ that has the same covariance operator as $(U \mathbf{H},V \mathbf{B})$, for centered and independent random variables $(U,V)$ with variance $(\varu^2,\varv^2)$ that are further independent of $(W_{\mathrm{in}},W_{\mathrm{out}})$ -- and thus of $(\mathbf{H},\mathbf{B})$.
We abuse notation here by writing $F_k^{\mathbf{H},\mathbf{B}} (s,\mathbf{Z}^h,\mathbf{Z}^b)$ instead of $F_k^{\mathbf{H},\mathbf{B}} (s,\mathbf{Z}^h(s),\mathbf{Z}^b(s))$ and similarly for the rest.

We recall that $F_k^{\mathbf{H},\mathbf{B}} (s,\mathbf{Z}^h,\mathbf{Z}^b)$ and $G_k^{\mathbf{H},\mathbf{B}} (s,\mathbf{Z}^h,\mathbf{Z}^b)$ only depend on $(\mathbf{H}_{\wedge k-1},\mathbf{B}_{\wedge k-1},\mathbf{Z}^h_{\wedge k-1},\mathbf{Z}^b_{\wedge k-1})$.
Yet, \eqref{eq:LimH} still depends on $\L( H_k)$ through $Z^h_k$ and the expectation involving $H_k$, and similarly for \eqref{eq:LimB} on $\L(B_k)$.
System \eqref{eq:LimH}-\eqref{eq:LimB} is thus of McKean-Vlasov type, with the dependence on $\L(\bfH, \bfB)$ only being through the correlation matrices $\boldsymbol{\Gamma}^H$ and \(\boldsymbol{\Gamma}^B\), as the following remark makes explicit.

\begin{rem}[Linear mean-field structure] \label{rem:linStruc}
Using the formulation and notation\footnote{We again abuse notation by writing \(\vec{\mathbf{v}}_{k}^F ( s, \mathbf{Z}^{h},\mathbf{Z}^{b})\) instead of \(\vec{\mathbf{v}}_{k}^F ( s, \mathbf{Z}^{h}(s),\mathbf{Z}^{b}(s))\) and similarly for the rest.} from Definition~\ref{def:linear_skeleton_meanfield}, we can write \eqref{eq:LimH}-\eqref{eq:LimB} explicitly, underlying evaluation at $s\in[0,1]$,
\begin{equation} \label{eq:LimHdecoupledexplicit}
\begin{cases}
{H}_k(0) = W_{\mathrm{in}}\cdot x, \qquad P_k = p_k(\mathbf{Z}^h, \mathbf{Z}^b), \\
\partial_s {H}_k = -\left(\eta_u \varv^2 \E\left[ \rho'(P_k) \bnabla_{\bfz^b} \vec{\mathbf{v}}_{k}^F(\mathbf{Z}^h, \mathbf{Z}^b)\right]\vec{\boldsymbol{\Gamma}}^H_k +\eta_v \E\left[ \rho(P_k) \vec{\mathbf{v}}_{k}^G(\mathbf{Z}^h, \mathbf{Z}^b)\right]  \right) \cdot \mathbf{B}_{\wedge k-1},  \\
y_k := \E [ H_k(1) W_{\mathrm{out}}], 
\end{cases}
\end{equation}
and
\begin{equation} \label{eq:LimBdecoupledexplicit}
\begin{cases}
{B}_k(1) = W_{\mathrm{out}}\cdot \nabla \ell ( y_k ), \qquad Q_k = q_k(\mathbf{Z}^h, \mathbf{Z}^b),\\
\partial_s {B}_k =  \left(\eta_v \varu^2 \E\left[ \rho'(P_k) \bnabla_{\bfz^h} \vec{\mathbf{v}}_{k}^G(\mathbf{Z}^h, \mathbf{Z}^b)\right] \vec{\boldsymbol{\Gamma}}^B_k(s) + \eta_u \E\left[ \rho'(P_k)Q_k\vec{\mathbf{v}}_{k}^F(\mathbf{Z}^h, \mathbf{Z}^b)\right]\right)\cdot \mathbf{H}_{\wedge k-1} \\ \qquad\qquad + \varu^2\left(\eta_u \E\left[\rho''(P_k)Q_k \bnabla_{\bfz^h} \vec{\mathbf{v}}_{k}^F( \mathbf{Z}^h, \mathbf{Z}^b)\right]\vec{\boldsymbol{\Gamma}}^H_k\right) \cdot \mathbf{H}_{\wedge k-1} -\varu^2 \E[\rho''(P_k)Q_k] H_k.
\end{cases}
\end{equation}
By induction \((\bfH, \bfB)\) can thus be written as a linear function of $(\Win, \Wout)$ and so, if the latter is a Gaussian vector, then the former is a Gaussian process. See Proposition~\ref{prop:explicit_dependence_on_WeWu} for details and proofs.
\end{rem}
We now prove the well-posedness of the limit dynamics, under the following assumption.
\begin{assumption}\label{assumption:rho_and_ell} $\phantom{a}$
\begin{enumerate}[label=(\roman*),ref=(\roman*)]
    \item The activation function $\rho:\R \to \R$ satisfies, for a given constant $C_\rho >0$, that $|\rho(0)| \leq C_\rho$, and $\|\rho'\|_\infty \vee\|\rho''\|_\infty \leq C_\rho$. In other words, $\rho$ has bounded derivatives up to second order. 
    \item The loss function $\ell: \R^{\dimout}\to\R$ has a $C_\ell$-Lipschitz gradient $\nabla\ell$, for some constant $C_\ell >0$.
    In particular, \(\|\nabla \ell(y)\|_{2,{\dimout}}\) has at most linear growth.
\end{enumerate}
All the constants involved in our proofs implicitly depend on $C_\rho$ and $C_\ell$.
\end{assumption}
\begin{restatable}[Well-posedness]{theorem}{wellposedness}
\label{thm:well_posedness_main}
There exists a unique solution $(\mathbf{H}, \mathbf{B})$ to the system \eqref{eq:LimH}-\eqref{eq:LimB} in $L^2 ( \Omega, \C ( [0,1], \R^{2 K}) )$.
\end{restatable}
\begin{rem}[Training embedding matrices]\label{rem:training_embedding_matrices_meanfield}
    To cover the training of the embedding matrices, as mentioned in Remark~\ref{rem:training_embedding_matrices_finiteD}, system \eqref{eq:LimH}-\eqref{eq:LimB}, should also involve the update rules
    \begin{equation*}
    (W_{\mathrm{in}})_{k+1} = (W_{\mathrm{in}})_{k} - \eta_{\mathrm{in}}B_k(0)x,\quad\text{and}\quad
    (W_{\mathrm{out}})_{k+1} = (W_{\mathrm{out}})_{k} - \eta_{\mathrm{out}}H_k(1)\nabla\ell(y_k).
\end{equation*}
In particular, as explained in Appendix~\ref{sec:extensions_of_results}, a slight modification of the proof of Theorem~\ref{thm:well_posedness_main}  allows us to prove that the modified system admits a unique solution as well.
\end{rem}

\begin{rem}[DMFT-type closure]
    Equations \eqref{eq:LimH} and \eqref{eq:LimB} only require us to evaluate $F^{\mathbf{H},\mathbf{B}}_k$, $G^{\mathbf{H},\mathbf{B}}_k$ and their gradient at $(  \mathbf{Z}^{h}(s),\mathbf{Z}^{b}(s))$.
We may therefore get rid of these functions by introducing the random variables
\[
\begin{cases}
(F_k,G_k) := (F^{\mathbf{H},\mathbf{B}}_k ( \mathbf{Z}^{h},\mathbf{Z}^{b}), G^{\mathbf{H},\mathbf{B}}_k( \mathbf{Z}^{h},\mathbf{Z}^{b})), \\ 
(\bnabla_{\bfz^\mathfrak{j}} F_k,\bnabla_{\bfz^\mathfrak{j}} G_k) := (\bnabla_{\bfz^\mathfrak{j}} F^{\mathbf{H},\mathbf{B}}_k ( \mathbf{Z}^{h},\mathbf{Z}^{b}), \bnabla_{\bfz^\mathfrak{j}}G^{\mathbf{H},\mathbf{B}}_k( \mathbf{Z}^{h},\mathbf{Z}^{b})),\quad \mathfrak{j}\in\{h,b\},
\end{cases}
\]
together with the update rules \({F}_{k+1} =  {F}_{k}  - \eta_u \rho' ( P_k) Q_k H_k\), \(
 {G}_{k+1} =  {G}_{k} - {\eta}_v \rho ( P_k) B_k\) and
\begin{align*}
    \bnabla_{\bfz^\mathfrak{j}} F_{k+1}& = \begin{pmatrix}
        \bnabla_{\bfz^\mathfrak{j}} F_{k}\\ 0
    \end{pmatrix} - \eta_u\bigg( \rho''(P_k)Q_k\begin{pmatrix}
        \E[H_k\bnabla_{\bfz^\mathfrak{j}} F_{k}|\bfZ^h, \bfZ^b]\\ \mathbbm{1}_{\mathfrak{j}=h}
    \end{pmatrix} + \rho'(P_k)\begin{pmatrix}
        \E[B_k \bnabla_{\bfz^\mathfrak{j}} G_{k}|\bfZ^h, \bfZ^b]\\ \mathbbm{1}_{\mathfrak{j}=b}
    \end{pmatrix}\bigg)H_k,\\
    \bnabla_{\bfz^\mathfrak{j}} G_{k+1}& = \begin{pmatrix}
        \bnabla_{\bfz^\mathfrak{j}} G_{k}\\ 0
    \end{pmatrix} - \eta_v \rho'(P_k)\begin{pmatrix}
        \E[H_k\bnabla_{\bfz^\mathfrak{j}} F_{k}|\bfZ^h, \bfZ^b]\\ \mathbbm{1}_{\mathfrak{j}=h}
    \end{pmatrix}B_k,\qquad\text{for}\;\; \mathfrak{j}\in\{h,b\},
\end{align*}
where we systematically underly evaluation at $s\in[0,1]$ and we posed \(P_k := Z^{h}_k +  \E [ H_k F_k \vert  \mathbf{Z}^{h},\mathbf{Z}^{b} ]\) and \(Q_k := Z^{b}_k +  \E [ B_k G_k \vert  \mathbf{Z}^{h},\mathbf{Z}^{b} ]\).
The variables $(W_{\mathrm{in}},W_{\mathrm{out}},{\mathbf{H}},{\mathbf{B}},{\mathbf{F}},{\mathbf{G}},\bnabla {\mathbf{F}},\bnabla {\mathbf{G}})$
allow for a concise re-writing of \eqref{eq:LimH}-\eqref{eq:LimB}, which no more needs the skeleton maps.
We notice that the variables $(\bnabla {\mathbf{F}},\bnabla {\mathbf{G}})$ describe how $({\mathbf{F}},{\mathbf{G}})$ is affected by perturbations at previous training times. In that sense, these functions are similar to the \emph{linear response functions} that are customary in the setting of DMFT~\cite{bordelon2022selfconsistentdynamicalfieldtheory, montanari2025dynamicaldecouplinggeneralizationoverfitting}.
\end{rem}

\paragraph{Quantitative convergence.}
We now state the central technical theorem of this paper, which proves quantitative convergence of the solution of the Mean ODE \eqref{eq:original_finite_D_system} towards the limit system at a rate $O(D^{-1/2})$ as $D\to \infty$. Consider an i.i.d.~sequence of random vectors $(W^d_e,W^d_u)_{d \geq 1}$ and let \((\mathbf{H}^d,\mathbf{B}^d)_{d\geq 1}\) be an i.i.d.~sequence of copies of the (unique) solution of \eqref{eq:LimH}-\eqref{eq:LimB} with initial conditions $(W_{\mathrm{in}}^d, W_{\mathrm{out}}^d)$ for each $d\geq 1$.
For every $D\geq 1$, consider the solution \(\mathbf{h}^{(D)},\mathbf{b}^{(D)}\) of the Mean ODE \eqref{eq:original_finite_D_system}, initialized using the coupled matrices $W_{\mathrm{in}}^{(D)}$ and $W_{\mathrm{out}}^{(D)}$, of rows $(W_{\mathrm{in}}^d)_{d=1}^D$ and $(W_{\mathrm{out}}^d)_{d=1}^D$, respectively; and with parameter initialization $U_0^{(D)}$, $V_0^{(D)}$. 

In order to prove the convergence result, we complement Assumption~\ref{assumption:rho_and_ell} with the following assumption.
\begin{assumption}\label{assumption:quantitative_convergence} $\phantom{a}$
    \begin{enumerate}[label=(\roman*),ref=(\roman*)]
     \item The activation function $\rho$ further has $\|\rho^{(3)}\|_{\infty}\vee\|\rho^{(4)}\|_{\infty}\vee \|\rho^{(5)}\|_{\infty} \leq C_\rho$, i.e. bounded derivatives up to fifth order.
     \item The rows $(\Win^d)_{d\geq 1}$ are i.i.d.~$\vpin^2$-subgaussian random vectors; and similarly $(\Wout^d)_{d\geq 1}$ are $\vpout^2$-subgaussian.  
    In general, we do not assume $\Win$ and $\Wout$ to be independent from each other, nor their coordinates to be independent within each row.
    \item The \textit{parameter initializations} $U_0^{(D)}$ and $V_0^{(D)}$ are given by $U_0^{(D)} := \sqrt{D}U^{(D)}$ and $V_0^{(D)} := \sqrt{D}V^{(D)}$, for $(U^{d})_{d\geq 1}, (V^{d})_{d\geq 1}$ i.i.d.~families of centered subgaussian random variables, independent from $\Finit:=\sigma(\Win, \Wout)$ and from each other, with variance proxies $\vpu^2$ and $\vpv^2$, respectively; and variances $\varu^2$ and $\varv^2$, respectively.
    \end{enumerate}
    All the constants involved in our proofs implicitly depend on $\eta_u$, $\eta_v$ $\vpu$, $\vpv$, $\varu$, $\varv$, $\vpin$ and $\vpout$ unless explicitly stated.
\end{assumption}

Under Assumptions~\ref{assumption:rho_and_ell} and \ref{assumption:quantitative_convergence} we obtain the desired quantitative convergence result. This is the central technical result of this paper from which Theorem~\ref{thm:main_theorem_intro} follows, as we show in Section~\ref{sec:proof_of_thms_quant_cvg}..
\begin{restatable}[Quantitative convergence]{theorem}{quantitativeconvergencethm}\label{thm:quantitative_convergence_main}
    There exists a constant $\cstH_K>0$ that only depends on $K$ and $\boldsymbol{\Gamma}^H$, $\boldsymbol{\Gamma}^B$, such that, for any $D\geq 1$ and any $\delta\in(e^{-D},1)$, with probability at least $1-\delta$ it holds
    \[ \sup_{\substack{s\in[0,1]\\ k\in[0:K-1]}}\| h^{(D)}_k (s) - H^{(D)}_k (s) \|_{\ndbar}\vee \| b^{(D)}_k (s) - B^{(D)}_k (s) \|_{\ndbar} \leq \cstH_K\bigg(\frac{1 + \log(1/\delta)}{\sqrt{D}}\bigg), \]
    \[\max_{k\in[0:K-1]}\|\hat{y}_k^D - y_k\|_{2, \dimout} \leq \cstH_K \Big(\frac{1 + \log(1/\delta)}{\sqrt{D}}\Big)\]
If $(\Win, \Wout)$ are a.s. bounded, the above holds for any $\delta\in(0,1)$.
\end{restatable}
Integrating the tails, we get the following corollary.
\begin{corollary}[$L^2$ convergence for bounded inputs]\label{cor:quantitative_convergence_L2_main}
    If $(W_{\mathrm{in}}, W_{\mathrm{out}})$ are a.s. bounded, then there exists a constant $\cstH_K>0$ depending solely on $K$ and $\boldsymbol{\Gamma}^H$, $\boldsymbol{\Gamma}^B$, such that, for any $D\geq 1$,
\[\E \bigg[ \sup_{\substack{s\in[0,1]\\ k\in[0:K-1]}}\| h^{(D)}_k (s) - H^{(D)}_k (s) \|^2_{\ndbar} \vee \| b^{(D)}_k (s) - B^{(D)}_k (s) \|^2_{\ndbar}\vee \| \hat{y}^{D}_k - y_k \|^2_{2,\dimout} \bigg] \leq \frac{\cstH_K}{D}.\]
\end{corollary}

\subsection{Explicit limit in the linear case and numerical illustrations}\label{ssec:numerical_experiments}
In the case of a \textit{linear} network where $\rho(x) = ax$ for some fixed $a\in \R$, the dynamics greatly simplifies. 
The recurrence from Definition~\ref{def:linear_skeleton_meanfield} simply reads \(\vec{\mathbf{v}}_1^F(s, \bfz^h, \bfz^b) = a\,z_0^b\), \(\vec{\mathbf{v}}_1^G(s, \bfz^h, \bfz^b) = a\,z_0^h\) and for $k\in[0:K-1]$,
\[\vec{\mathbf{v}}_{k+1}^F(s, \bfz^h, \bfz^b) = \begin{pmatrix}
    \vec{\mathbf{v}}_{k}^F(s, \bfz^h, \bfz^b)\\
    a\, q_k(s, \bfz^h, \bfz^b)
\end{pmatrix},\; \;\;\vec{\mathbf{v}}_{k+1}^G(s, \bfz^h, \bfz^b) = \begin{pmatrix}
    \vec{\mathbf{v}}_{k}^G(s, \bfz^h, \bfz^b)\\
    a\,p_k(s, \bfz^h, \bfz^b)
\end{pmatrix}\]
with $p_k$, $q_k$ unchanged. It follows that \(\vec{\mathbf{v}}_{k}^F(s, \mathbf{0}, \mathbf{0}) = 0\) and \(\vec{\mathbf{v}}_{k}^G(s, \mathbf{0}, \mathbf{0}) = 0\); and that for $\mathfrak{i} \in\{h,b\}$,
    $\bnabla_{\bfz^{\mathfrak{i}}} \vec{\mathbf{v}}_{k}^F(s, \bfz^h, \bfz^b)$ and $\bnabla_{\bfz^{\mathfrak{i}}} \vec{\mathbf{v}}_{k}^G(s, \bfz^h, \bfz^b)$ do not depend on $\bfz^h, \bfz^b$. In particular,
    \(\vec{\mathbf{v}}_{k}^F(s, \bfz^h, \bfz^b)\) and \(\vec{\mathbf{v}}_{k}^G(s, \bfz^h, \bfz^b)\) are linear functions of $\bfz^h, \bfz^b$, given by \(\vec{\mathbf{v}}_{k}^F(s, \bfz^h, \bfz^b) = \bnabla_{\bfz^{h}} \vec{\mathbf{v}}_{k}^F(s)^\top \bfz^h + \bnabla_{\bfz^{b}} \vec{\mathbf{v}}_{k}^F(s)^\top\bfz^b\),
    and similarly for \(\vec{\mathbf{v}}_{k}^G(s, \bfz^h, \bfz^b)\).
Using the linearity of $\vec{\mathbf{v}}_{k}^F$, $\vec{\mathbf{v}}_{k}^G$, \eqref{eq:cleaner_system_Fk_definition} and the explicit expression of the covariance of \((\bfZ^h, \bfZ^b)\), 
we can write \eqref{eq:LimH}-\eqref{eq:LimB} explicitly as
\begin{equation}\label{eq:linear_system_explicit}
\begin{cases}
H_k(0) &= W_{\mathrm{in}}\cdot x, \\
\partial_s H_k (s) &= \vec{\boldsymbol{\Gamma}}_k^H(s)^{\top}\mathbf{M}^{\mathbf{H},\mathbf{B}}_k(s)\mathbf{B}_{\wedge k-1}(s),\\
y_k &= \E [ H_k(1) W_{\mathrm{out}}], 
\end{cases} \quad  \begin{cases}
B_k(1) &= W_{\mathrm{out}}\cdot \nabla\ell(y_k), \\
\partial_s B_k (s) 
&= -\vec{\boldsymbol{\Gamma}}_k^B(s)^{\top}\mathbf{M}^{\mathbf{H},\mathbf{B}}_k(s)^\top\mathbf{H}_{\wedge k-1}(s),
\end{cases}
\end{equation}
\vspace{-8mm}
\begin{multline}\label{eq:linear_case_matrixM_definition}
    \text{where} \quad\mathbf{M}^{\mathbf{H},\mathbf{B}}_k(s) = -a\big[\eta_u\varv^2 \bnabla_{\bfz^{b}} \vec{\mathbf{v}}_{k}^F(s)^{\top}\big( \mathrm{Id}_k -a\eta_v \boldsymbol{\Gamma}_{\wedge k-1}^B (s)\bnabla_{\bfz^{b}} \vec{\mathbf{v}}_{k}^G(s) \big) \\+ \eta_v\varu^2 \big(\mathrm{Id}_k -a\eta_u\bnabla_{\bfz^{h}} \vec{\mathbf{v}}_{k}^F(s)^{\top}\boldsymbol{\Gamma}_{\wedge k-1}^H (s)\big) \bnabla_{\bfz^{h}} \vec{\mathbf{v}}_{k}^G(s)\big].
\end{multline}
All terms in \eqref{eq:linear_system_explicit} are expressed explicitly in terms of $\boldsymbol{\Gamma}_{\wedge k}^H$, and $\boldsymbol{\Gamma}_{\wedge k}^B$, without involving the Gaussian field $(\bfZ^h, \bfZ^b)$. See Appendix~\ref{sec:linear_case_explicit_system} for additional precisions.

\paragraph{Numerical Illustrations.}
From Remark~\ref{rem:linStruc}, we know that for centered Gaussian embeddings $(W_{\mathrm{in}}, W_{\mathrm{out}})$, \((\bfH, \bfB)\) is a centered Gaussian process, and is therefore characterized by its covariance structure. 

In the linear case, starting from \eqref{eq:linear_system_explicit}, we can explicitly describe the ODEs which govern $(\boldsymbol{\Gamma}^H(s), \boldsymbol{\Gamma}^B(s))$, as is done in Appendix~\ref{sec:linear_case_explicit_system} for the case of $(W_{\mathrm{in}}, W_{\mathrm{out}})$ centered, independent and with i.i.d.~coordinates. Using numerical ODE solvers (e.g.~4th-order Runge-Kutta method), we can also compute them with very high accuracy. 
Figure~\ref{fig:linear_forwardbackward_histograms} shows that the resulting covariance structure accurately describes the behaviour of a large finite ResNet over training iterations. 
Across layers $\ell\in[0:L]$, the coordinates of the preactivations $((\mathtt{h}_k^{(\ell)})^{d})_{d=1}^D$ 
approximately behave as independent samples of the Gaussian process $\bfH$ described in \eqref{eq:linear_system_explicit}.

Since we have access to the true limit system in the linear case, we can test the tightness of the convergence rates from Theorem~\ref{thm:main_theorem_intro} without resorting to a proxy for the limit as done in Section~\ref{sec:experiments-tanh}.
We compute $\Delta_k^h$ for finite ResNets of various shapes and sizes, and obtain the results in Figure~\ref{fig:tight_cvg_rate_linear}, from where we see that the rates are indeed tight in embedding space.

In all of our experiments for the linear setting, we have $\dimin = \dimout=3$ and consider a single randomly sampled datapoint $(x, y^*)\in\R^{\dimin}\times \R^{\dimout}$, with $\ell(\cdot) := \frac{1}{2}\|\cdot-y^*\|_{2,\dimout}^2$.

\begin{figure}[h]
        \centering
        \begin{subfigure}{0.67\linewidth}
        \centering
        \includegraphics[width=\linewidth]{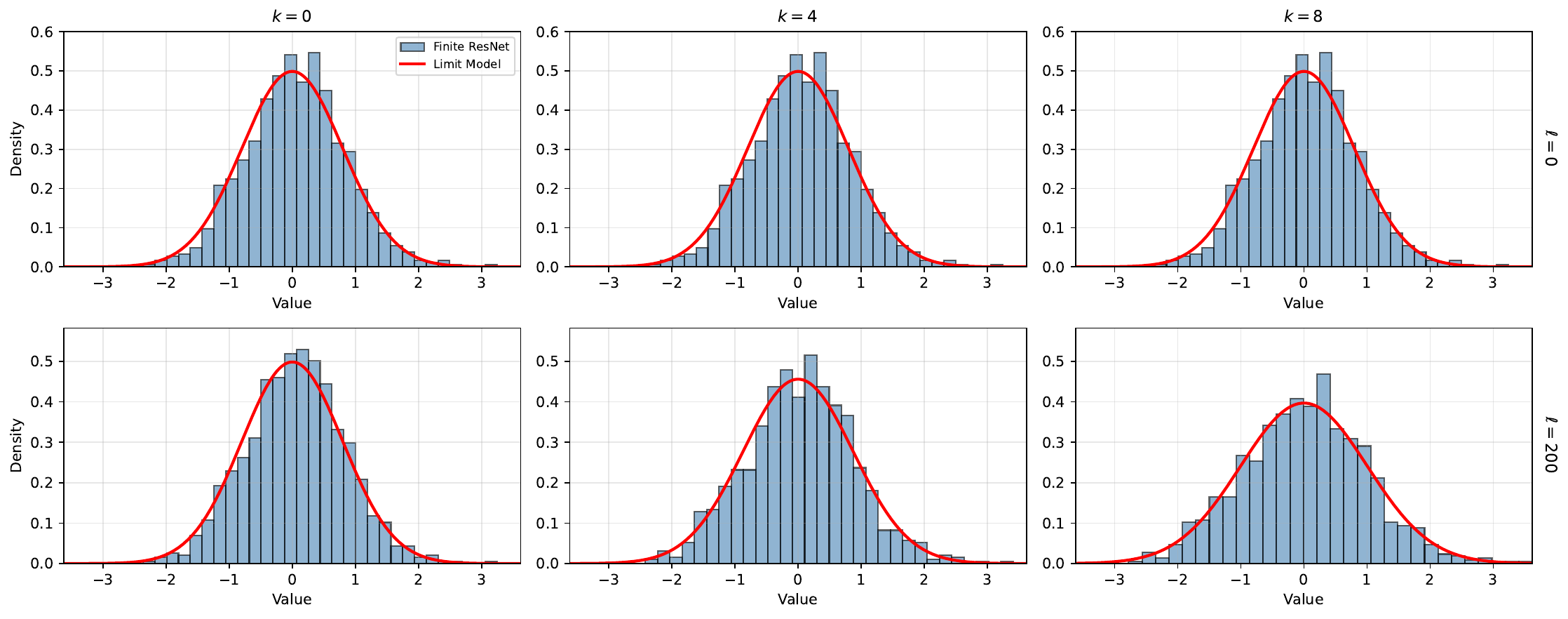}
        \caption{Histogram of coordinate values for $((\th^{(\ell)}_k)^d)_{d=1}^D$}
        \end{subfigure}%
        \begin{subfigure}{0.33\linewidth}
        \centering
        \includegraphics[width=\textwidth]{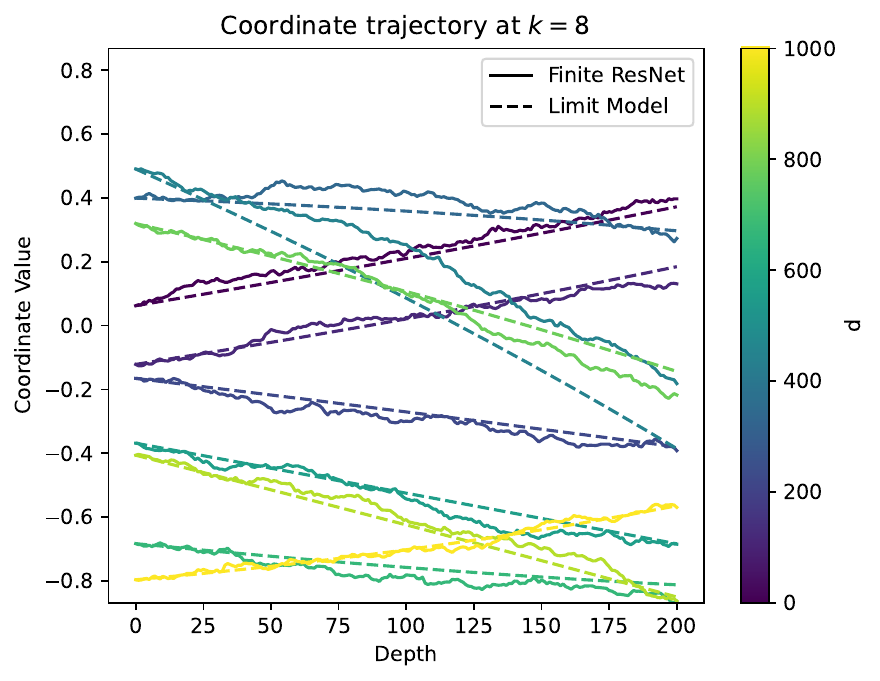}
        \caption{Some trajectories of $(\th^{(\cdot)}_k)^d$}
        \end{subfigure}

        \caption{Histograms for the coordinates of the forward pass 
        in a finite linear Resnet ($L=200$, $M=D=1000$), compared to the limit model, at training iterations $k \in\{0,4,8\}$.
        The red lines represent the pdf of the Gaussian random variable $H_k(\ell/L)$. The right plot shows values trajectories across depth of a few embedding coordinates $d$, comparing them to the coupled samples of the limit stochastic process $\ell\mapsto H_k^d(\ell/L)$.}
        \label{fig:linear_forwardbackward_histograms}
\end{figure}
    
\begin{figure}[h]
\centering
\begin{subfigure}{0.71\linewidth}
\centering
\includegraphics[width=\textwidth]{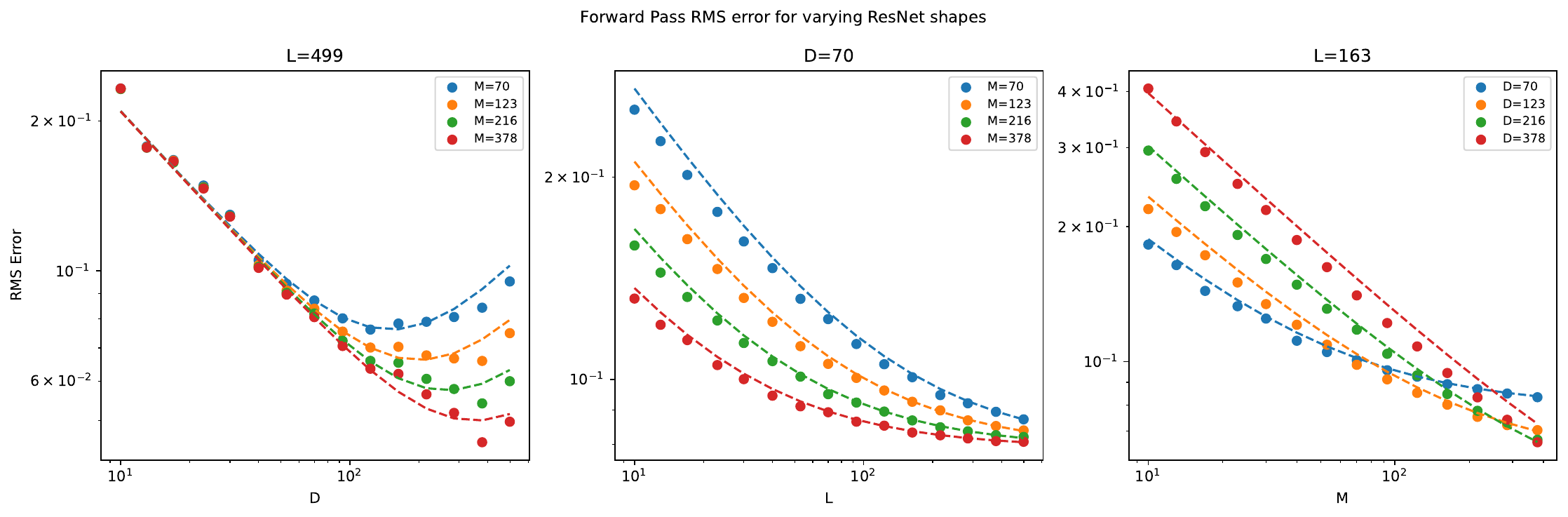}
\caption{Error on $\th^{L}_k$ vs $D$}
\end{subfigure}%
\begin{subfigure}{0.29\linewidth}
\centering
\includegraphics[width=\textwidth]{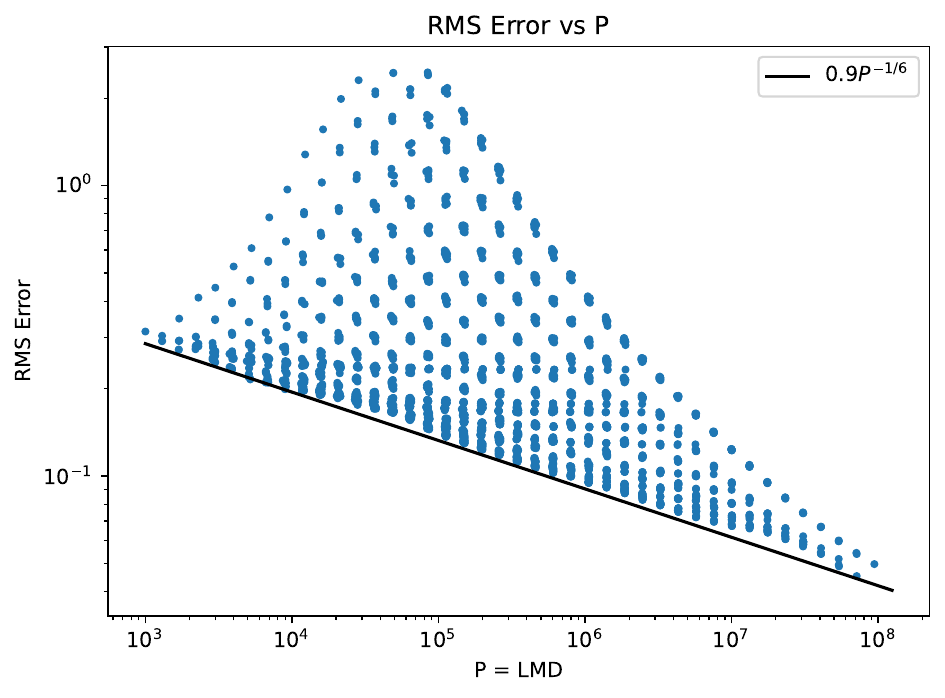}
\caption{Error on $\th^{L}_k$ vs $P$}
\end{subfigure}
\caption{Forward Pass RMS error after $k=5$ GD steps, $\Delta_5^h$, in the linear case $\rho(s)=s$, between finite ResNets of variable size, and the true limit model. Dotted lines display the function $\|[0.67\sqrt{\frac{D}{ML}}, 0.44\frac{1}{\sqrt{D}}]\|_2$ obtained from fitting our error rate to the data. The error term $O(\frac{1}{L})$ is negligible in our regime. The y-axis shows RMS error (averaged over $60$ random repetitions with a single datapoint). In the last plot, we overlay the predicted scaling of $\Theta(P^{-1/6})$}  
\label{fig:tight_cvg_rate_linear}
\end{figure}

\subsection{Heuristic derivation of the limit model via the cavity method}\label{sec:proof-strategy-cavity}

A key technical challenge is controlling the convergence of terms like \eqref{eq:TermDiff}, which exhibit a CLT-type scaling\footnote{There are very few instances in the literature of quantitative mean-field limits for particle systems that involve interactions with CLT-type scaling; we refer to \cite{erny2023strong} for a quite specific example of such a result.}.
Specifically, closing the $D \rightarrow +\infty$ dynamics of the $d$-th coordinate $h^{d}_k$ from \eqref{eq:original_finite_D_system} requires computing the limit of $\E [ \rho ( \langle U_k^{(D)}, h^{(D)}_k\rangle_{\ndbar} ) \sqrt{D} V^{d} |\Finit ]$ -- a similar difficulty arises for $b^{d}_k$.
Because of the $\sqrt{D}$ factor, obtaining the asymptotic expression of this quantity, or even just showing that it does not blow up, is not immediate. Indeed, while $V^d$ is centered, correlations with $\langle U_k^{(D)}, h^{(D)}_k\rangle_{\ndbar}$ are inherited from the past and require an explicit computation.

As $D \rightarrow + \infty$, the mean-field structure suggests that $V^d$ becomes independent of $h^{(D)}_k$.
However, we still have to keep track of the dependence of $\Delta U_k^{(D)}$ on $V^d$. The skeleton structure $\Delta U_k^{(D)} = f^{\mathbf{h}^D,\mathbf{b}^D}_k ( \bfS^{\bfh^D} , {\bfS}^{\bfb^D} )$ developed in Section~\ref{sec:definition_skeletons} shows that $\Delta U_k^{(D)}$ only depends on $V^d$ through the rescaled sums $( \bfS^{\bfh^D} , {\bfS}^{\bfb^D} )$ from \eqref{eq:definition_clt_type_sums}.
In particular, the order of magnitude for this dependence is $1/\sqrt{D}$, allowing for a first-order expansion with respect to $D$ that will precisely compensate the $\sqrt{D}$ factor in $\E [ \rho ( \langle U_k^{(D)}, h^{(D)}_k\rangle_{\ndbar} ) \sqrt{D} V^{d} |\Finit ]$.
Since the work \cite{Mezard1987spinglass} from statistical physics, this idea of capturing correlations by expanding with respect to a single fixed coordinate is well-known as the \emph{cavity method}. 
In our case, this expansion formally reads
\begin{align*}
    \E [ \rho ( \langle &U_k^D, h^D_k\rangle_{\ndbar} ) \sqrt{D} V^{d} {|\Finit} ] = \E [ \rho ( S^{\bfh^D}_k + \langle f^{\mathbf{h}^D,\mathbf{b}^D}_k ( \bfS^{\bfh^D} , \bfS^{\bfb^D} ), h^{D}_k \rangle_{\ndbar} ) \sqrt{D} V^{d} {|\Finit}]\\ &\simeq
\E [ \rho ( S^{\bfh^D}_k + \langle f^{\mathbf{h}^D,\mathbf{b}^D}_k ( \bfS^{\bfh^D} , \overline{\bfS}^{\bfb^D} ), h^D_k\rangle_{\ndbar} ) \sqrt{D} V^{d} {|\Finit}] \\
&\quad + \E [ \rho' ( S^{\bfh^D}_k + \langle f^{\mathbf{h}^D,\mathbf{b}^D}_k ( \bfS^{\bfh^D} , \overline{\bfS}^{\bfb^D} ), h^D_k\rangle_{\ndbar} ) \langle \bnabla_{\bfz^b} f^{\mathbf{h}^D,\mathbf{b}^D}_k ( \bfS^{\bfh^D} , \overline{\bfS}^{\bfb^D} ) [ V^d \mathbf{b}^{d} ], h^D_k\rangle_{\ndbar} V^d {|\Finit} ] + 
O (D^{-1/2}),
\end{align*}
where we introduced $\overline{\bfS}^{\bfb^D} := \bfS^{\mathbf{b}^D} - V^d \mathbf{b}^{d} / \sqrt{D}$.
Since $( \bfS^{\bfh^D} , \overline{\bfS}^{\bfb^D} )$ is independent of $V^{d}$ conditionally on $\Finit$, the first term vanishes, whereas the second reads 
\[\varv^2 \, \E \bigg[ \rho' ( S^{\bfh^D}_k + \langle f^{\mathbf{h}^D,\mathbf{b}^D}_k ( \bfS^{\bfh^D} , \overline{\bfS}^{\bfb^D} ), h^D_k\rangle_{\ndbar} ) \frac{1}{D}\sum_{d'=1}^D \bnabla_{\bfz^b} (f^{d'})^{\mathbf{h}^D,\mathbf{b}^D}_k ( \bfS^{\bfh^D} , \overline{\bfS}^{\bfb^D} ) h^{d'}_k \bigg\vert \Finit \bigg]\bfb^d. \]
As $D \rightarrow + \infty$, we expect $\bfb^d \rightarrow \bfB$, $\bfh^d \rightarrow \bfH$, $(\mathbf{f}^d)^{\bfh^D,\bfb^D} \rightarrow \mathbf{F}^{\bfH,\bfB}$, and the above term to converge towards
\[ \varv^2 \, \E [ \rho'( P_k ) \E[ H_k \bnabla_{\bfz^b} F_k^{\mathbf{H},\mathbf{B}} ( \mathbf{Z}^{\mathbf{H}} , \mathbf{Z}^{\mathbf{B}} ) \vert \bfZ^H, \bfZ^B ] {|\Finit} \,] \mathbf{B}. \]
The Gaussian fields $(\bfZ^H, \bfZ^B)$ arise as the limit in law of $(\bfS^{\bfh^D} , \overline{\bfS}^{\bfb^D})$, as expected from the CLT. 
A quantitative version of this cavity-type expansion for the forward pass is established in Lemma~\ref{lem:proof_cvg_cavity_TCL_H} below, using a quantitative CLT result proved in Appendix~\ref{sec:quantitative_clt}.
The analogous procedure for the backward pass is done in Lemma~\ref{lem:proof_cvg_cavity_TCL_B}.

\paragraph{Organisation of the proofs.}
Section~\ref{sec:Skel} is devoted to the construction of the skeleton maps introduced in Section~\ref{sec:definition_skeletons}, as well as establishing a priori bounds on them.
Section~\ref{sec:construction_of_limit_law} rigorously builds the limit dynamics \eqref{eq:LimH}-\eqref{eq:LimB}, proving Theorem~\ref{thm:well_posedness_main}.
Section~\ref{sec:CV} is finally devoted to the proof of our main result, Theorem~\ref{thm:quantitative_convergence_main}, establishing the quantitative convergence of \eqref{eq:original_finite_D_system} towards \eqref{eq:LimH}-\eqref{eq:LimB}.
In particular, our quantitative cavity-type method is developed therein.

\section{Estimates on skeleton maps} \label{sec:Skel}

Throughout this section, we provide a priori estimates on the skeleton maps, which shall be very useful in our proofs of Theorems~\ref{thm:well_posedness_main} and \ref{thm:quantitative_convergence_main}.

\subsection{Finite-dimensional setting}\label{subsec:finiteDSkeleton}
We have the following key result.

\begin{lemma}[Finite-dimensional skeleton bounds]\label{lem:Reg_skel_finiteD}
    Let $\mathbf{h}^D, \mathbf{b}^D$ and $\mathbf{f}^{\mathbf{h}^D,\mathbf{b}^D},\mathbf{g}^{\mathbf{h}^D,\mathbf{b}^D}$ be as in Definition~\ref{def:finite_skeleton_maps}. Let $m^D_k := \max_{s\in[0,1]} \|h_k^D(s)\|_{\ndbar} \vee \|b_k^D(s) \|_{\ndbar}$.
    There exist explicit constants $C^{(i)}_k, C^{(ii)}_k, \hat{C}_k > 0$
    that only depend on  $k(\eta_u\vee\eta_v)$ and  $\lVert m^D_{\wedge k-1} \rVert_\infty$, such that for every $\bfz^h,\bfz^b \in \R^K$ with 
    $\max_{j\in[0: k-1]} \vert z^h_j \vert \vee \vert z^b_j \vert \leq M$, 
    for some $M \geq 1$,
\begin{enumerate}[label=(\roman*),ref=(\roman*)]
    \item\label{lemRegSkelfiniteD:L2} $\max_{s\in[0,1]}\|f^{\mathbf{h}^D,\mathbf{b}^D}_k(s,\bfz^h, \bfz^b)\|_{\ndbar} \vee \|g^{\mathbf{h}^D,\mathbf{b}^D}_k(s,\bfz^h, \bfz^b)\|_{\ndbar} \leq C_k^{(i)} M$.
    \item\label{lemRegSkelfiniteD:as} For every $d\in [1:D]$,\[\max_{\substack{j\in[0:k]\\s\in[0,1]}} |(f_j^d)^{\mathbf{h}^D,\mathbf{b}^D}(s,\bfz^h, \bfz^b)| \leq C_k^{(ii)}M\max_{\substack{j\in[0:k-1]\\s\in[0,1]}}|h_k^{d,D}(s)|,\] \[\max_{\substack{j\in[0:k]\\s\in[0,1]}} |(g_j^d)^{\mathbf{h}^D,\mathbf{b}^D}(s,\bfz^h, \bfz^b)| \leq C_k^{(ii)}M\max_{\substack{j\in[0:k-1]\\s\in[0,1]}}|b_k^{d,D}(s)|. \]
    \item\label{lemRegSkelfiniteD:grad} It holds that
    \[\max_{\substack{j\in[0:k]\\s\in[0,1]}} \| \bnabla_{s^h} (f_j^d)^{\mathbf{h}^D,\mathbf{b}^D} ( s , \bfz^h, \bfz^b ) \|_{2,j} \vee \| \bnabla_{s^b} (f_j^d)^{\mathbf{h}^D,\mathbf{b}^D} ( s , \bfz^h, \bfz^b ) \|_{2,j}  \leq e^{\hat{C}_k M}\max_{\substack{j\in[0:k-1]\\s\in[0,1]}}\vert h^{d,D}_j ( s )  \vert, \]
   \[\max_{\substack{j\in[0:k]\\s\in[0,1]}} \| \bnabla_{s^h} (g_j^d)^{\mathbf{h}^D,\mathbf{b}^D} ( s , \bfz^h, \bfz^b ) \|_{2,j} \vee \| \bnabla_{s^b} (g_j^d)^{\mathbf{h}^D,\mathbf{b}^D} ( s , \bfz^h, \bfz^b ) \|_{2,j} \leq e^{\hat{C}_kM} \max_{\substack{j\in[0:k-1]\\s\in[0,1]}}\vert b^{d,D}_j ( s )  \vert. \]
\end{enumerate}
\end{lemma}

\begin{proof}[Proof]
To ease notations, we write $f^{\mathbf{h},\mathbf{b}}_{k} =: f_k = (f_k^d)_{d=1}^D$ and similarly for $g$, $h$ and $b$. 
   For~\ref{lemRegSkelfiniteD:L2}, we take the $\|\cdot\|_{\ndbar}$ norm directly in the definition of the finite-width skeleton maps \eqref{def:finite_skeleton_maps}, underlying evaluation at $(s, \bfz^h, \bfz^b)$:
    \[\begin{cases}
\|f_{k+1}\|_{\ndbar} \leq \|f_{k}\|_{\ndbar} + \eta_u |\rho' ( z^h_k + \langle h_k , f_{k}\rangle_{\ndbar} )|  |z^b_k + \langle b_k, g_k\rangle_{\ndbar} | \|h_k\|_{\ndbar}, \\
\|g_{k+1}\|_{\ndbar} \leq \|g_{k}\|_{\ndbar} +{\eta}_v |\rho ( z^h_k + \langle h_k, f_{k}\rangle_{\ndbar} ) \|b_k\|_{\ndbar}.
\end{cases}\]

Now, using the linear growth of $\rho$, the boundedness of its derivative, the Cauchy-Schwarz inequality (C-S), and the boundedness of $\bfz^h$, $\bfz^b$
, we get
\[\begin{cases}
\|f_{k+1}\|_{\ndbar} \leq \|f_{k}\|_{\ndbar} + \eta_u C_\rho (M + \|b_k\|_{\ndbar}\|g_k\|_{\ndbar})\|h_k\|_{\ndbar}, \\
\|g_{k+1}\|_{\ndbar} \leq \|g_{k}\|_{\ndbar} +{\eta}_v C_\rho (1 + M + \|h_k\|_{\ndbar}\|f_k\|_{\ndbar}) \|b_k\|_{\ndbar}.
\end{cases}\]

Setting 
$u_k := \max_{s\in[0,1]} \|f_k(s)\|_{\ndbar} \vee \|g_k(s) \|_{\ndbar}$, taking maximum over $s \in [0,1]$ and $j \in [0: k]$, we can join both inequalities to obtain:
\begin{align*}
u_{k+1} 
&\leq (1+ (\eta_u\vee \eta_v) C_\rho \lVert m_{\wedge k} \rVert_\infty^2) u_k + (\eta_u\vee \eta_v) C_\rho 2M \lVert m_{\wedge k} \rVert_\infty\\
&\leq (1+ (\eta_u\vee \eta_v) C_\rho (1\vee \lVert m_{\wedge k} \rVert_\infty^2)) u_k + (\eta_u\vee \eta_v) C_\rho 2M(1\vee \lVert m_{\wedge k} \rVert_\infty),
\end{align*}
so that applying the discrete Grönwall lemma\footnote{A slight generalization: if \( u_{k+1} \le (1 + \alpha_k)u_k + \beta_k \) for \( k \ge 0 \) and \( (\alpha_k)_k, (\beta_k)_k > 0 \) increasing sequences, then $\forall k\in\mathbb{N},$  
\(
u_k \le (1+\alpha_{k-1})^k\left(u_0 + \frac{\beta_{k-1}}{\alpha_{k-1}}\right) - \frac{\beta_{k-1}}{\alpha_{k-1}}\le e^{k\alpha_{k-1}}\left(u_0 + \frac{\beta_{k-1}}{\alpha_{k-1}}\right) - \frac{\beta_{k-1}}{\alpha_{k-1}}.
\)} -- noting that $u_0=0$ -- we obtain
\[\max_{s\in[0,1]} \|f_k(s)\|_{\ndbar} \vee \|g_k(s) \|_{\ndbar} =u_{k} \leq C_k^{(i)} M, \]
where
\[ C_k^{(i)} := \frac{2}{1\vee \lVert m_{\wedge k-1} \rVert_\infty} [ \exp [ C_\rho (1\vee \lVert m_{\wedge k-1} \rVert_\infty^2) k(\eta_u\vee \eta_v)]-1 ]. \]
For any $d\in [1:D]$, consider the skeleton map from Definition~\ref{def:finite_skeleton_maps} and apply the properties of $\rho$ similarly to obtain
\[
\begin{cases}
|f^{d}_{k+1}| \leq |f^{d}_{k}| + \eta_u C_\rho  (M + \|b_k\|_{\ndbar} \|g_k\|_{\ndbar}) |h_k^d|, \\
|g^{d}_{k+1}| \leq |g^{d}_{k}| + {\eta}_v C_\rho (1 + M + \|h_k\|_{\ndbar} \|f_k\|_{\ndbar}) |b_k^d|.
\end{cases}
\]
We can plug in our previous estimate to get
\[
\max_{s\in[0,1]}|f^{d}_{k+1}(s)| \leq \max_{s\in[0,1]}|f^{d}_{k}(s)| + \eta_u C_\rho  M(1 + \lVert m_{\wedge k} \rVert_\infty C_k^{(i)}) \max_{\substack{j\in[0:k-1]\\s\in[0,1]}}|h_j^d(s)|.
\] And thus, since $f_0 \equiv 0$ and $C_k^{(i)}$ is increasing in $k$, by developing the recursion we obtain
\[\max_{s\in[0,1]}|f^{d}_{k+1}(s)| \leq (k+1)\eta_u C_\rho  (1 + \lVert m_{\wedge k} \rVert_\infty C_k^{(i)})  M\max_{\substack{j\in[0:k-1]\\s\in[0,1]}}|h_j^d(s)| \leq C_k^{(ii)} M,\]
where
\[ C_k^{(ii)} :=   C_\rho k(\eta_u\vee\eta_v)   [ 2 + \lVert m_{\wedge k-1} \rVert_\infty C_{k-1}^{(i)} ]. \]
The same follows for $g$ via analogous arguments. 

To prove~\ref{lemRegSkelfiniteD:grad}, we reason as before, noticing the following recurrence that defines the Jacobian matrices \(\bnabla_{\bfz^\mathfrak{j}} f_{k}(s, \bfz^h, \bfz^b), \bnabla_{\bfz^\mathfrak{j}} g_{k}(s, \bfz^h, \bfz^b)\in \R^{k\times D} \). Let $\hat{p}_k := z_k^h + \langle h_k, f_k\rangle_{\ndbar}$ and $\hat{q}_k := z_k^b + \langle b_k, g_k\rangle_{\ndbar}$. 
For $\mathfrak{j}\in\{h,b\}$, we have
\begin{align*}
    \bnabla_{\bfz^\mathfrak{j}} f_{k+1}& = \begin{pmatrix}
        \bnabla_{\bfz^\mathfrak{j}} f_{k}\\ \vec{0}^\top
    \end{pmatrix} - \eta_u\left( \rho''(\hat{p}_k)\hat{q}_k\begin{pmatrix}
        \frac{1}{D}\bnabla_{\bfz^\mathfrak{j}} f_{k}\cdot h_k\\ \mathbbm{1}_{\mathfrak{j}=h}
    \end{pmatrix} + \rho'(\hat{p}_k)\begin{pmatrix}
        \frac{1}{D}\bnabla_{\bfz^\mathfrak{j}} g_{k}\cdot b_k\\ \mathbbm{1}_{\mathfrak{j}=b}
    \end{pmatrix}\right)h_k^\top,\\
    \bnabla_{\bfz^\mathfrak{j}} g_{k+1}& = \begin{pmatrix}
        \bnabla_{\bfz^\mathfrak{j}} g_{k}\\ \vec{0}^\top
    \end{pmatrix} - \eta_v \rho'(\hat{p}_k)\begin{pmatrix}
        \frac{1}{D}\bnabla_{\bfz^\mathfrak{j}} f_{k}\cdot h_k\\ \mathbbm{1}_{\mathfrak{j}=h}
    \end{pmatrix}b_k^\top.
\end{align*}
From this expression, we find the following estimate for $v_k^\mathfrak{j} := \max_{s\in[0,1]}\|\bnabla_{\bfz^\mathfrak{j}} f_{k}(s)\|_{\ndbar} \vee \|\bnabla_{\bfz^\mathfrak{j}} g_{k}(s)\|_{\ndbar}$,
\begin{align*}
    v_{k+1}^\mathfrak{j} &\leq v_k^\mathfrak{j} + (\eta_u\vee \eta_v)C_\rho \lVert m_{\wedge k} \rVert_\infty 2 M \lVert m_{\wedge k} \rVert_\infty C_k^{(i)}(1 + \lVert m_{\wedge k} \rVert_\infty v_k^\mathfrak{j}) =: (1+\alpha_k M)v_k^\mathfrak{j} + \beta_k M,
\end{align*}
Up to redefining them, we can assume that $(\alpha_k)_{k \geq 0}, (\beta_k)_{k \geq 0}$ are increasing, so that the discrete Grönwall lemma yields
\[
v_{k}^\mathfrak{j} \leq (1 + 2C_\rho C_k^{(i)}(1\vee \lVert m_{\wedge k-1} \rVert_\infty^3) (k(\eta_u\vee \eta_v))M)^k \leq \exp(\hat{C}^{h,b}_k M),
\]
where
\[ \hat{C}^{h,b}_k := 2C_\rho C_k^{(i)}(1\vee \lVert m_{\wedge k-1} \rVert_\infty^3) k(\eta_u\vee \eta_v)M. \]

From our expression for the gradients, we can now  study the corresponding $\| \cdot \|_{2, k+1}$ norm. We can observe that for any $d\in [1:D]$,
\begin{align*}
    \|\nabla_{\bfz^\mathfrak{j}} f_{k+1}^d \|_2& = \|\nabla_{\bfz^\mathfrak{j}} f_{k}^d\|_2 \\&\;\;\;\;+ C_\rho\eta_u\left((M + \|b_k\|_{\ndbar}\|g_k\|_{\ndbar}) (1 + \|\nabla_{\bfz^j} f_{k}\|_{\ndbar} \|h_k\|_{\ndbar}) + (1 + \|\nabla_{\bfz^\mathfrak{j}} g_{k}\|_{\ndbar} \|b_k\|_{\ndbar})\right)|h_k^d|, \\
    \|\nabla_{\bfz^\mathfrak{j}} g_{k+1}^d \|_2& = \|\nabla_{\bfz^\mathfrak{j}} g_{k}^d \|_2 + \eta_v C_\rho (1 + \|\nabla_{\bfz^\mathfrak{j}} f_{k}\|_{\ndbar} \|h_k\|_{\ndbar}) |b_k^d|.
\end{align*}
Taking maximum over $s \in [0,1]$, $j \in [0: k]$ and rearranging terms, a direct induction yields
\[\max_{s\in[0,1]}  \| \nabla_{s^\mathfrak{j}} f^{d}_k ( s , \bfz^h, \bfz^b ) \|_{2,k} \leq e^{\hat{C}_k^{h,b}M} \max_{\substack{j\in[0:k-1]\\s\in[0,1]}}\vert h^d_k ( s )  \vert, \]
where again $\hat{C}_k^{h,b}$ is a constant depending solely on $k(\eta_u\vee\eta_v)$ and $\lVert m_{\wedge k-1} \rVert_\infty$. Proceeding analogously for $g$ concludes.
\end{proof}

\begin{rem}[$k$-dependent polynomial growth] \label{rem:polynomial_skeleton_derivative_finiteD}
    In the proof of Lemma~\ref{lem:Reg_skel_finiteD}, we could replace the 
    exponential constants $\exp(\hat{C}_kM)$  by $C_k^{(iii)} M^{k+1}$ for a constant $C_k^{(iii)}$ satisfying the same requirements. 
\end{rem}

\subsection{Mean-field skeleton}\label{sec:mf_skeleton}

As in the finite-$D$ setting, we can prove the following estimates for the mean-field skeletons.

\begin{lemma}[Mean-field skeleton bounds] \label{lem:RegSkel}
Let $(\mathbf{H},\mathbf{B}) = ( H_k, B_k )_{0 \leq k \leq K-1}$ be a random variable in $\C([0,1];\R^{2K})$ 
such that $M_k:=\max_{s\in[0,1]}\E^{1/2}[| H_k(s)|^2]\vee\E^{1/2}[|B_k(s)|^2]$ are finite for all $k$. Then, there exist constants $C^{(i)}_k, C^{(ii)}_k, \hat{C}_k > 0$ that only depend on $k(\eta_u\vee \eta_v)$ and $\lVert M_{\wedge k-1} \rVert_\infty$, such that for every $\bfz^h,\bfz^b \in \R^K$ with $\max_{j\in[0:k-1]} \vert z^h_j \vert \vee \vert z^b_j \vert \leq M$, $M \geq 1$, 
\begin{enumerate}[label=(\roman*),ref=(\roman*)]
    \item\label{lemRegSkel:L2} $\max_{s\in[0,1]} \E^{1/2}[|F^{\mathbf{H},\mathbf{B}}_k ( s , \bfz^h, \bfz^b )|^2] \vee \E^{1/2}[|G^{\mathbf{H},\mathbf{B}}_k ( s , \bfz^h, \bfz^b )|^2] \leq C_k^{(i)} M$.
    \medskip
    
    The following bounds hold almost surely:
    \item\label{lemRegSkel:as}  $\max_{\substack{j\in[0:k]\\s\in[0,1]}} \vert F^{\mathbf{H},\mathbf{B}}_j ( s , \bfz^h, \bfz^b ) \vert \leq C_k^{(ii)} M \max_{\substack{j\in[0:k-1]\\s\in[0,1]}} \vert H_j ( s )  \vert $, \\
    $\max_{\substack{j\in[0:k]\\s\in[0,1]}} \vert G^{\mathbf{H},\mathbf{B}}_j ( s , \bfz^h, \bfz^b ) \vert \leq C_k^{(ii)} M \max_{\substack{j\in[0:k-1]\\s\in[0,1]}} \vert B_j ( s )  \vert $.
    \item\label{lemRegSkel:Grad} $\max_{\substack{j\in[0:k]\\s\in[0,1]}}  \| \bnabla_{s^h} F^{\mathbf{H},\mathbf{B}}_j ( s , \bfz^h, \bfz^b ) \|_{2,j} \vee \| \bnabla_{s^b} F^{\mathbf{H},\mathbf{B}}_j ( s , \bfz^h, \bfz^b ) \|_{2,j} \leq e^{\cstHhat_kM} \max_{\substack{j\in[0:k-1]\\s\in[0,1]}} \vert H_j ( s )  \vert $, \\
    $\max_{\substack{j\in[0:k]\\s\in[0,1]}} \| \bnabla_{s^h} G^{\mathbf{H},\mathbf{B}}_j ( s , \bfz^h, \bfz^b ) \|_{2,j} \vee \| \bnabla_{s^b} G^{\mathbf{H},\mathbf{B}}_j ( s , \bfz^h, \bfz^b ) \|_{2,j}  \leq e^{\cstHhat_k M} \max_{\substack{j\in[0:k-1]\\s\in[0,1]}} \vert B_j ( s )  \vert $.
    \item\label{lemRegSkel:Hess} $\max_{\substack{j\in[0:k]\\s\in[0,1]\\ \mathfrak{i}, \mathfrak{j}\in\{h,b\}} } \| \bnabla_{s^{\mathfrak{i}}, s^{\mathfrak{j}}}^2 F^{\mathbf{H},\mathbf{B}}_j ( s , \bfz^h, \bfz^b ) \|_{2,j\times j} \leq e^{\cstHhat_k M} \max_{\substack{j\in[0:k-1]\\s\in[0,1]}} \vert H_j ( s )  \vert $, \\
    $\max_{\substack{j\in[0:k]\\s\in[0,1]\\ \mathfrak{i}, \mathfrak{j}\in\{h,b\}} } \| \bnabla_{s^{\mathfrak{i}}, s^{\mathfrak{j}}}^2 G^{\mathbf{H},\mathbf{B}}_j ( s , \bfz^h, \bfz^b ) \|_{2,j\times j} \leq e^{\cstHhat_kM} \max_{\substack{j\in[0:k-1]\\s\in[0,1]}} \vert B_j ( s )  \vert $, \\
    where $\|\cdot\|_{2,j\times j}$ denotes the Frobenius norm in $\R^{j\times j}$. 
    \item If we assume that the derivatives of $\rho$ up to the $(m+1)$-th order are bounded by $C_\rho$, then analogous bounds hold for the $m$-th order derivatives of $F^{\mathbf{H},\mathbf{B}}$ and $G^{\mathbf{H},\mathbf{B}}$:
    \[\max_{\substack{j\in[0:k]\\s\in[0,1]\\ \mathfrak{i}_1,..., \mathfrak{i}_m\in\{h,b\}} }  \| \bnabla_{s^{\mathfrak{i}_1}, ..., s^{\mathfrak{i}_m}}^m F^{\mathbf{H},\mathbf{B}}_j ( s , \bfz^h, \bfz^b ) \|_{2,j^m} \leq e^{\cstHhat_kM} \max_{\substack{j\in[0:k-1]\\s\in[0,1]}} \vert H_j ( s )  \vert,\]
    where $\|\cdot\|_{2,j^m}$ denotes the Frobenius norm in $(\R^{j})^m$.
\end{enumerate}
\end{lemma}

The analogous of Remark~\ref{rem:polynomial_skeleton_derivative_finiteD} also holds for these estimates: the exponential dependence could be improved to polynomial if we allow dependences other than $K \eta$ for $\cstHhat_k$.

\begin{proof}
    We notice that the definition of $F_{k}^{\mathbf{H}, \mathbf{B}}$ and $G_{k}^{\mathbf{H}, \mathbf{B}}$ exactly correspond to those of $f_{k}^{\mathbf{h}^{(D)}, \mathbf{b}^{(D)}}$ and $g_{k}^{\mathbf{h}^{(D)}, \mathbf{b}^{(D)}}$, up to replacing inner products in $\R^D$ $\langle\cdot, \cdot\rangle_{\ndbar}$ by inner products in $L^2(\L(W_{\mathrm{in}}, W_{\mathrm{out}}))$. Thus, we follow the exact same reasoning as in Lemma~\ref{lem:Reg_skel_finiteD}, replacing $\frac{1}{D}\sum_{d=1}^D [\cdot]$ by $\E[\cdot]$ to obtain the desired results.

    The only part that differs from the proof of Lemma~\ref{lem:Reg_skel_finiteD} is the estimate for the Hessian and higher order derivatives for the skeleton maps. 
    For $\mathfrak{i}, \mathfrak{j} \in \{h,b\}$, underlying evaluation at $(s, \bfz^h, \bfz^b)$, we can write the following recursion for the Hessians in $\R^{(k+1)\times(k+1)}$:
    \begin{align*}
        \bnabla^2_{\bfz^\mathfrak{i}, \bfz^\mathfrak{j}} &F_{k+1} =\begin{pmatrix}
            \bnabla^2_{\bfz^\mathfrak{i}, \bfz^\mathfrak{j}} F_{k} &0\\
            0&0
        \end{pmatrix} \\&-\eta_u H_k\bigg[\rho''(p_k)\Big( \nabla_{\bfz^{\mathfrak{i}}} q_k(\nabla_{\bfz^{\mathfrak{j}}} p_k)^{\top}  +\nabla_{\bfz^{\mathfrak{i}}} p_k(\nabla_{\bfz^{\mathfrak{j}}} q_k)^{\top}\Big) +\rho'''(p_k)q_k \nabla_{\bfz^{\mathfrak{i}}} p_k(\nabla_{\bfz^{\mathfrak{j}}} p_k)^{\top} \\
        &\qquad\qquad\qquad+\rho'(p_k)\begin{pmatrix}
            \E[B_k \bnabla^2_{\bfz^\mathfrak{i}, \bfz^\mathfrak{j}} G_{k} ]&0\\
            0&0
        \end{pmatrix}  + \rho''(p_k)q_k\begin{pmatrix}
            \E[H_k \bnabla^2_{\bfz^\mathfrak{i}, \bfz^\mathfrak{j}} F_{k} ]&0\\
            0&0
        \end{pmatrix}\bigg],
    \end{align*}
    \begin{align*}
        \bnabla^2_{\bfz^\mathfrak{i}, \bfz^\mathfrak{j}} &G_{k+1} =\begin{pmatrix}
            \bnabla^2_{\bfz^\mathfrak{i}, \bfz^\mathfrak{j}} G_{k} &0\\
            0&0
        \end{pmatrix} -\eta_v B_k\bigg[\rho''(p_k)\nabla_{\bfz^{\mathfrak{i}}} p_k(\nabla_{\bfz^{\mathfrak{j}}} p_k)^{\top} +\rho'(p_k)\begin{pmatrix}
            \E[H_k \bnabla^2_{\bfz^\mathfrak{i}, \bfz^\mathfrak{j}} F_{k} ]&0\\
            0&0
        \end{pmatrix}\bigg],
    \end{align*}
    where $p_k(s, \bfz^h, \bfz^b) = z_k^h + \E[H_k(s)F_k(s, \bfz^h, \bfz^b)]$, $q_k(s, \bfz^h, \bfz^b) = z_k^b + \E[B_k(s)G_k(s, \bfz^h, \bfz^b)]$ so that, for $\mathfrak{i}\in\{h,b\}$, 
    \[\nabla_{\bfz^{\mathfrak{i}}} p_k(s, \bfz^h, \bfz^b) = \begin{pmatrix}
        \E[H_k(s)\nabla_{\bfz^{\mathfrak{i}}}F_k(s,\bfz^h, \bfz^b)]\\
        \mathbbm{1}_{\mathfrak{i} = h}
    \end{pmatrix},\;\; \nabla_{\bfz^{\mathfrak{i}}} q_k(s, \bfz^h, \bfz^b) = \begin{pmatrix}
        \E[B_k(s)\nabla_{\bfz^{\mathfrak{i}}}G_k(s,\bfz^h, \bfz^b)]\\
        \mathbbm{1}_{\mathfrak{i} = b}
    \end{pmatrix}. \]
    We now define $u_k := \max_{s\in[0,1]} \E^{1/2}[\|\bnabla^2_{\bfz^\mathfrak{i}, \bfz^\mathfrak{j}} F_{k}\|_{2,k\times k}^2] \vee \E^{1/2}[\|\bnabla^2_{\bfz^\mathfrak{i}, \bfz^\mathfrak{j}} G_{k}\|_{2,k\times k}^2]$ and observe that, for a constant $\hat{C}_k$ that may change from line to line but keeping the required dependencies, 
    \[u_{k+1} \leq u_k + (\eta_u\vee \eta_v)\hat{C}_kM(e^{2\hat{C}_kM} + u_k).\]
    The discrete Grönwall lemma then yields $u_k \leq \exp(\hat{C}_kM)$.

    The bound on $u_k$ allows us to bound $\|\bnabla^2_{\bfz^\mathfrak{i}, \bfz^\mathfrak{j}} F_{k+1}\|_{2,(k+1)^2}$ and $\|\bnabla^2_{\bfz^\mathfrak{i}, \bfz^\mathfrak{j}} G_{k+1}\|_{2,(k+1)^2}$ via
    \[\|\bnabla^2_{\bfz^\mathfrak{i}, \bfz^\mathfrak{j}} F_{k+1}\|_{2,(k+1)^2} \leq \|\bnabla^2_{\bfz^\mathfrak{i}, \bfz^\mathfrak{j}} F_{k}\|_{2,k^2} + \eta_u \max_{\substack{j\in[0:k]\\ t\in[0,1]}} |H_j(t)| \hat{C}_kM(e^{2\hat{C}_kM} + u_k), \]
    which precisely yields the desired bound by arguing as in~\ref{lemRegSkel:Grad}. 
    The arguments for the higher derivatives of $F_k$ and $G_k$ are completely similar, and thus omitted.
\end{proof}

Following the same proof as that of Lemma~\ref{lem:RegSkel}, this result translates to a priori bounds on \(\vec{\mathbf{v}}_{k}^F\), \(\vec{\mathbf{v}}_{k}^G\) and their gradients.
\begin{lemma}\label{lem:explicit_regskel}
    Let $k \geq 1$, and $M_k$ as in Lemma~\ref{lem:RegSkel}. There exists a constant $C^{H,B}_k$ depending only on $k(\eta_u\vee \eta_v)$ and $\|M_{\wedge k-1}\|_{\infty}$ such that, for any $\bfz^h,\bfz^b \in \R^K$ with $\max_{i\in[0:k-1]} \vert z^h_i \vert \vee \vert z^b_i \vert \leq M$, $M \geq 1$, it holds that
    \begin{enumerate}[label=(\roman*),ref=(\roman*)]
        \item \(\|\vec{\mathbf{v}}_{k}^F(s, \bfz^h, \bfz^b)\|_{2, k} \vee \|\vec{\mathbf{v}}_{k}^G(s, \bfz^h, \bfz^b)\|_{2, k} \leq C_k^{H,B} M\), 
        \item \(\max_{\mathfrak{i}\in\{h,b\}}\|\bnabla_{\bfz^{\mathfrak{i}}}\vec{\mathbf{v}}_{k}^F(s, \bfz^h, \bfz^b)\|_{2, k\times k} \vee \|\bnabla_{\bfz^{\mathfrak{i}}}\vec{\mathbf{v}}_{k}^G(s, \bfz^h, \bfz^b)\|_{2, k\times k} \leq e^{C_k^{H,B}M}.\)
    \end{enumerate}
\end{lemma}

\subsection{A priori controls on the Mean ODE} \label{ssec:Dfixed}
From the setting in Section~\ref{sec:setting_and_main_results}, recall that $U^{(D)} := ({U}^d)_{d \geq 1}$, $V^{(D)} :=({V}^d)_{d \geq 1}$ are sequences of i.i.d.~centered subgaussian variables with respective variance proxys $\vpu^2$, and $\vpv^2$. Similarly, the embedding matrices $W_{\mathrm{in}}^{(D)} \in \R^{D\times \dimin}$ and $W_{\mathrm{out}}^{(D)} \in \R^{D\times \dimout}$ have i.i.d.~rows which are subgaussian vectors of variance proxy $\vpin^2$ and $\vpout^2$, respectively. Also, recall $\F_W = \sigma ( W_{\mathrm{in}},W_{\mathrm{out}})$.
Consider a solution to the Mean ODE \eqref{eq:original_finite_D_system}, $(\bfh^{(D)}, \bfb^{(D)}) \in \C([0,1], (\R^D)^{2K})$, and define
\[\begin{cases}
    P_k^{\mathbf{h}^{(D)}, \mathbf{b}^{(D)}}(s):= \langle \sqrt{D}U^{(D)} + \Delta U_k^{(D)}(s), h_k^{(D)}(s)\rangle_{\ndbar},\\
    Q_k^{\mathbf{h}^{(D)}, \mathbf{b}^{(D)}}(s):= \langle \sqrt{D}V^{(D)} + \Delta V_k^{(D)}(s), b_k^{(D)}(s)\rangle_{\ndbar},
\end{cases}\]
recalling that we decompose the weight updates as $U_k^{(D)}(s) = \sqrt{D}U^{(D)} + \Delta U^{(D)}_k(s)$ and similarly for $V_k^{(D)}(s)$.
We now provide a priori estimates for the solution of the Mean ODE, following \cite[Proposition 6.1]{chizat2025hidden} while keeping track of the influence of $\Win$ and $\Wout$.
\begin{lemma}[A priori bounds Mean ODE]\label{lem:embeddingcontrol}
    For all $k \in[0:K-1]$, there exists a $\Finit$-measurable random constant $\cstW_k$ that only depends on $k$, $(\eta_u\vee\eta_v)$, 
    $\varu$,
    $\varv$, $\|W_{\mathrm{in}}^D\|_{\ndbar}$ and $\|W_{\mathrm{out}}^D\|_{\ndbar}$ such that
    \[\sup_{s\in[0,1]}\|h_k^{(D)} (s)\|_{\ndbar}  \leq \cstW_k, \; \sup_{s\in[0,1]}\|b_k^{(D)} (s)\|_{\ndbar}  \leq \cstW_k, \]
    \[\sup_{s\in[0,1]}\E^{1/2}[P^{\mathbf{h}^{(D)}, \mathbf{b}^{(D)}}_k(s)^2| \F_W ]  \leq \cstW_k, \; \sup_{s\in[0,1]}\E^{1/2}[Q^{\mathbf{h}^{(D)}, \mathbf{b}^{(D)}}_k(s)^2| \F_W ]  \leq \cstW_k, \]
    \[\sup_{s\in[0,1]} \E^{1/2}[\|\Delta U_k(s)\|_{\ndbar}^2| \F_W ]  \leq \cstW_k, \; \sup_{s\in[0,1]}\E^{1/2}[\|\Delta V_k(s)\|_{\ndbar}^2| \F_W ]  \leq \cstW_k. \]
    In particular, if $\|W_{\mathrm{in}}^{(D)}\|_{\ndbar}$ and $\|W_{\mathrm{out}}^{(D)}\|_{\ndbar}$ are bounded, then all the above quantities are bounded as well.
\end{lemma}

\begin{rem}[Super-exponential growth of the constants] \label{rem:SuperExpGrowth}
    The constants involved in this lemma depend on $k$, as they are built by induction over the training steps.
    Each induction step $k$ involves closing a Grönwall argument to bound $h_k^{(D)}$, which produces an additional exponential dependence on $C_{k-1}^{\overline{W}}$.
    Iterating this, the resulting constant $C_{k}^{\overline{W}}$ grows like the composition of $k$ exponentials. 
    This rough bound is sufficient for our purpose where the number $K$ of training steps is fixed and finite, and its improvement is beyond the scope of the present article.
\end{rem}

\begin{proof}
    We closely follow the proof of \cite[Proposition 6.1] {chizat2025hidden}, but carefully keeping track of the dependencies on $\|W_{\mathrm{in}}^{(D)}\|_{\ndbar}$, $\|W_{\mathrm{out}}^{(D)}\|_{\ndbar}$.
    We proceed by induction.

    \paragraph{Initial.} For $k=0$, we see that for all $s\in[0,1]$, we have $h_0^{(D)} (s) = W_{\mathrm{in}}^{(D)} x$ so that \(\sup_{s\in[0,1]}\|h_0^{(D)} (s)\|_{\ndbar} \leq \|W_{\mathrm{in}}^{(D)}\|_{\ndbar} \|x\|_{2,{\dimin}}\). Similarly, since $b_0^{(D)} (s) = W_{\mathrm{out}}^{(D)}\nabla\ell(\hat{y}_0)$ for all $s\in[0,1]$, we can use the linear growth of $\nabla \ell$ to obtain \[\sup_{s\in[0,1]}\|b_0^{(D)} (s)\|_{\ndbar} \leq \|W_{\mathrm{out}}^{(D)}\|_{\ndbar} \|\nabla\ell(\hat{y}_0)\|_{2,{\dimout}} \leq \|W_{\mathrm{out}}^{(D)}\|_{\ndbar}C_\ell [ 1+\|W_{\mathrm{out}}^{(D)}\|_{\ndbar}\|W_{\mathrm{in}}^{(D)}\|_{\ndbar}\|x\|_{2,\dimin} ]. \] 
    As in the proof of \cite[Proposition 6.1] {chizat2025hidden}, we have that 
    \[ \begin{cases}
    \sup_{s\in[0,1]} \E^{1/2} [ \vert P^{\mathbf{h}^{(D)}, \mathbf{b}^{(D)}}_0(s)|^2 \vert \F_W ]  \leq \varu\|W_{\mathrm{in}}^{(D)}\|_{\ndbar} \|x\|_{2,\dimin}, \\
    \sup_{s\in[0,1]} \E^{1/2} [ \vert Q^{\mathbf{h}^{(D)}, \mathbf{b}^{(D)}}_0(s)\vert^2 \vert \F_W ]  \leq \varv\|W_{\mathrm{out}}^{(D)}\|_{\ndbar} \|\nabla\ell(\hat{y}_0)\|_{2,\dimout},
     \end{cases}
     \] 
    which we bound as above. Since $\Delta U_0(s) = \Delta V_0 (s) = 0$ 
    for all $s\in[0,1]$, we conclude the initial step by defining $\cstW_0$ as the maximum among the above constants, which only depends on the required variables.

    \paragraph{Inductive step.} We introduce the following quantities
    \[\mathfrak{h}_k:=\sup_{s\in[0,1]}\|h_k^{(D)} \|_{\ndbar}  , \quad \mathfrak{p}_k := \sup_{s\in[0,1]} \E^{1/2} [ \vert P^{\mathbf{h}^{(D)}, \mathbf{b}^{(D)}}_k(s) \vert^2 \vert \F_W ], \quad \mathfrak{u}_k:=\sup_{s\in[0,1]}\E [ \|\Delta U_k (s) \|_{\ndbar}^2 \vert \F_W ], 
    \] and similarly for $\mathfrak{b}_k, \mathfrak{q}_k, \mathfrak{v}_k$. 
    With this, following the steps in \cite[Proposition 6.1]{chizat2025hidden}, we get the following estimates
    \[ \mathfrak{u}_k \leq \mathfrak{u}_{k-1} + C_\rho \eta_u \mathfrak{h}_{k-1} \mathfrak{b}_{k-1} (\varv + \mathfrak{v}_{k-1}), \quad \mathfrak{v}_k \leq \mathfrak{v}_{k-1} + C_\rho \eta_v(1\vee\mathfrak{h}_{k-1}) \mathfrak{b}_{k-1} (1 + \varu + \mathfrak{u}_{k-1}).
    \]
    Using our induction hypothesis to bound $1\vee \lVert \mathfrak{h}_{\wedge k-1} \rVert_\infty \vee \lVert \mathfrak{b}_{\wedge k-1} \rVert_\infty \leq C_{k-1}^{\overline{W}}$, we can further fuse the bounds on $\mathfrak{u}_k$ and $\mathfrak{v}_k$ to obtain the recurrence
    \begin{align*}
    \mathfrak{u}_{k} \vee \mathfrak{v}_{k} 
    &\leq (1+C_\rho (\eta_u\vee\eta_v) (C_{k-1}^{\overline{W}})^2)(\mathfrak{u}_{k-1} \vee \mathfrak{v}_{k-1}) + (1+(\varu\vee\varv))C_\rho (\eta_u\vee\eta_v) (C_{k-1}^{\overline{W}})^2).
    \end{align*}
    Using the discrete Grönwall inequality, this finally yields
    \[\mathfrak{u}_{k} \vee \mathfrak{v}_{k} \leq (1+ \varu \vee \varv)\exp \big[ (C_\rho(C_{k-1}^{\overline{W}})^2 (k(\eta_u\vee\eta_v))\big] =: \cstW_k. \]
    
    Following \cite[Proposition 6.1]{chizat2025hidden}, we now control $\mathfrak{h}_{k}$, $\mathfrak{b}_{k}$ using Grönwall's lemma on the differential equations that define $h_k$ and $b_k$, yielding
    \begin{align*}
        \mathfrak{h}_k &\leq \left(\|W_{\mathrm{in}}^{(D)}\|_{\ndbar}\|x\|_{2,\dimin} + C_\rho(\varv + \mathfrak{v}_{k})\right)\exp[C_\rho(\varu + \mathfrak{u}_k)(\varv + \mathfrak{v}_k)], \\
        \mathfrak{b}_k &\leq \left(\|W_{\mathrm{out}}^{(D)}\|_{\ndbar}\|\nabla \ell(\hat{y}_k)\|_{2, \dimout}\right)\exp[C_\rho(\varu + \mathfrak{u}_k)(\varv + \mathfrak{v}_k)] \\
        &\leq \left(\|W_{\mathrm{out}}^{(D)}\|_{\ndbar}C_\ell(1+\|W_{\mathrm{out}}^{(D)}\|_{\ndbar}\mathfrak{h}_k)\right)\exp[C_\rho(\varu + \mathfrak{u}_k)(\varv + \mathfrak{v}_k)],
    \end{align*}
    so that
    $\mathfrak{p}_k \leq (\varu + \mathfrak{u}_k)\mathfrak{h}_k$ and $\mathfrak{q}_k \leq (\varv + \mathfrak{v}_k)\mathfrak{b}_k$. From our bound on $\mathfrak{u}_{k}\vee \mathfrak{v}_{k}$, this implies the desired bounds on $\mathfrak{h}_k, \mathfrak{b}_k, \mathfrak{p}_k, \mathfrak{q}_k$, which have the required dependencies.
\end{proof}

These estimates on the Mean ODE lead to estimates on the associated skeleton maps, as shown by the following result.
\begin{lemma}\label{lem:finite_skeleton_bounded_ndbar}
Let $(\bfh^{(D)}, \bfb^{(D)})$ be the solution to the Mean ODE \eqref{eq:original_finite_D_system}, and let \(\mathbf{f}^{\bfh^D,\bfb^D},\mathbf{g}^{\bfh^D,\bfb^D}\) be the skeleton maps associated to $(\bfh^{(D)}, \bfb^{(D)})$. Then,
for any $k\in[0:K-1]$, there exists a $\Finit$-measurable random constant $\cstW_k$ as in Lemma~\ref{lem:embeddingcontrol}, such that for any $\bfz^h, \bfz^b \in \R^K$ and $p\geq 1$,
    \[\big\| f^{(D)}_k (s, \bfz^h, \bfz^b)\big\|_{\ndbar}\vee \big\| g^{(D)}_k (s, \bfz^h, \bfz^b)\big\|_{\ndbar} \leq \cstW_k\big(1+ \|(\bfz^h, \bfz^b)\|_{p, 2k}\big), \]
    where the norm is taken on the first $k$ coordinates: \(\|(\bfz^h, \bfz^b)\|_{p, 2k} = \big( \sum_{i=0}^{k-1} |z_i^h|^p + |z_i^b|^p\big)^{1/p}\).
\end{lemma}
\begin{proof}
    This follows from the skeleton bounds from Lemma~\ref{lem:Reg_skel_finiteD}-\ref{lemRegSkelfiniteD:L2}, namely
\[\big\| f^{(D)}_k (s, \bfz^h, \bfz^b)\big\|_{\ndbar} \leq C_k^{(ii)} ( 1+\max_{j\in[0:k-1]} |z_j^{h}|\vee |z_j^{b}|)\leq C_k^{(ii)} ( 1+\|(\bfz^h, \bfz^b)\|_{p, 2k}),\] where we used that $\|\cdot\|_{\infty} \leq \|\cdot\|_p$ for any $p\geq 1$, and where $C_k^{(ii)}$ is an increasing function of $\max_{i\in[0:k-1]} \max_{s\in[0,1]} \|h_i^{(D)}(s)\|_{\ndbar}\vee\|b_i^{(D)}(s)\|_{\ndbar}$. Using Lemma~\ref{lem:embeddingcontrol}, it is thus an increasing function of $\|W_{\mathrm{in}}^{(D)}\|_{\ndbar}\vee\|W_{\mathrm{out}}^{(D)}\|_{\ndbar}$, and so we denote it by $\cstW_k$. The same argument applies to $g^{(D)}_k$. Joining both bounds concludes the proof.
\end{proof}

\section{Construction of the limit dynamics}\label{sec:construction_of_limit_law}

Let us now rigorously build the limit system.
Let $(\Omega,\F,\P)$ be a probability space that carries two subgaussian random vectors $W_{\mathrm{in}}\in \R^{\dimin}$ and $W_{\mathrm{out}} \in \R^{\dimout}$.
We want to construct a $\F_W$-measurable random variable $(\mathbf{H}, \mathbf{B})$ in $L^2 ( \Omega, \C ( [0,1], \R^{2 K}) ) $ that a.s. satisfies, the system \eqref{eq:LimH}-\eqref{eq:LimB}. That is, for $k \in[0:K-1]$,
\begin{equation} 
\begin{cases}
{H}_k(0) = W_{\mathrm{in}}\cdot x, \\
P_k (s) := Z_k^{h} (s) + \E [ H_k (s) F^{\mathbf{H},\mathbf{B}}_k (s, \mathbf{Z}^{h},\mathbf{Z}^{b}) \vert  \mathbf{Z}^{h},\mathbf{Z}^{b} ], \\
\partial_s {H}_k (s) = \varv^2 \E [ \rho'( P_k (s) ) \E[ H_k (s) \bnabla_{\bfz^b} F_k^{\mathbf{H},\mathbf{B}} (s,  \mathbf{Z}^{h},\mathbf{Z}^{b} ) \vert  \mathbf{Z}^{h},\mathbf{Z}^{b} ] \,] \cdot\mathbf{B}_{\wedge k-1}(s) \\
\qquad\qquad\qquad\qquad\qquad+ \E [ \rho ( P_k (s) ) G^{\mathbf{H},\mathbf{B}}_k (s, \mathbf{Z}^{h},\mathbf{Z}^{b}) \vert \Finit], \qquad s \in [0,1],  \\
y_k := \E [ H_k(1) W_{\mathrm{out}}], 
\end{cases}
\end{equation}
as well as
\begin{multline} 
\begin{cases}
Q_k (s) = Z_k^{b} (s) + \E [ B_k (s) G^{\mathbf{H},\mathbf{B}}_k (s, \mathbf{Z}^{h},\mathbf{Z}^{b}) \vert  \mathbf{Z}^{h},\mathbf{Z}^{b} ],\\
\partial_s {B}_k(s) = - \E \big[ \rho'( P_k(s) ) F_k^{\mathbf{H},\mathbf{B}} (s, \mathbf{Z}^{h},\mathbf{Z}^{b}) Q_k(s)  \vert \Finit\big] \\
\qquad \qquad -\varu^2 \E \big[ \rho' \big( P_k(s) ) \E [ B_k(s) \nabla_{\bfz^h} G_k^{\mathbf{H},\mathbf{B}} ( s, \mathbf{Z}^{h},\mathbf{Z}^{b}) \vert  \mathbf{Z}^{h},\mathbf{Z}^{b} ] \big] \cdot \mathbf{H}_{\wedge k-1}(s) \\
\qquad\qquad- \varu^2 \E \big\{ \rho''( P_k(s) )Q_k(s) \big[ H_k(s) + \E [ H_k(s) \bnabla_{\bfz^h} F_k^{\mathbf{H},\mathbf{B}} (s,  \mathbf{Z}^{h},\mathbf{Z}^{b}) \vert  \mathbf{Z}^{h},\mathbf{Z}^{b} ]\cdot \mathbf{H}_{\wedge k-1}(s) \big] \vert \Finit \big\},\\ 
{B}_k(1) = W_{\mathrm{out}}\cdot \nabla \ell ( y_k ),
\end{cases}
\end{multline}
where $(\mathbf{Z}^{h},\mathbf{Z}^b)$ is a centered Gaussian field in $L^2 ( [0,1],\R^{2K} )$ that has the same covariance operator as $(U \mathbf{H},V \mathbf{B})$, for centered and independent random variables $(U,V)$ of variances $\varu^2, \varv^2$, independent of $(W_{\mathrm{in}},W_{\mathrm{out}})$ -- and thus of $(\mathbf{H},\mathbf{B})$.
We also recall that $F_k^{\mathbf{H},\mathbf{B}}$ and $G_k^{\mathbf{H},\mathbf{B}} (s,\mathbf{Z}^h,\mathbf{Z}^b)$ are non-anticipative, and that \eqref{eq:LimH}-\eqref{eq:LimB} is of McKean-Vlasov type.

In the sections that follow, we rigorously prove Theorem~\ref{thm:well_posedness_main}, which we recall here.
\wellposedness*
For this, we first study the properties of the Gaussian Field $(\mathbf{Z}^{h},\mathbf{Z}^b)$, in Section~\ref{ssec:GaussianFields}. In Section~\ref{ssec:main_proof_well_posedness} we define a map whose fixed-point solves \eqref{eq:LimH}-\eqref{eq:LimB}, we provide a priori estimates for it and we conclude by proving Theorem~\ref{thm:well_posedness_main}. Finally, in Section~\ref{ssec:properties_of_mf_limit} we study some useful properties of the limit system.

\subsection{Gaussian fields}\label{ssec:GaussianFields}

In this section, we recall some useful properties of Gaussian  fields in infinite-dimensional spaces.
The proof of the following results can be found in the classical textbook \cite{bogachev1998gaussian}, or in the more concise introduction \cite[Section 3]{hairer2023introductionstochasticpdes}.

Given a random variable $(\mathbf{H}, \mathbf{B})$ in $L^2 ( \Omega, \C ( [0,1], \R^{2 K}) ) $ and some independent centered random real-valued variables $U,V$ independent of $(\mathbf{H}, \mathbf{B})$ with variances $\varu^2, \varv^2$, we define the covariance kernel
\begin{equation}\label{eq:covariance_kernel_gaussian}
    K_{U\mathbf{H}, V\mathbf{B}}(s,t) := \E[(U\mathbf{H}, V\mathbf{B})(s)(U\mathbf{H}, V\mathbf{B})(t)^\top] = \begin{pmatrix}
        \varu^2\boldsymbol{\Gamma}^H(s,t) & 0\\ 0 &\varv^2\boldsymbol{\Gamma}^B(s,t)
    \end{pmatrix}, \quad (s,t) \in[0,1]^2,
\end{equation}
where we recall that \(\boldsymbol{\Gamma}^H(s,t) =\E[\mathbf{H}(s)\mathbf{H}(t)^\top]\) and similarly for \(\boldsymbol{\Gamma}^B(s,t)\), and we used the independence and centeredness of $U$, $V$.
This continuous function defines the following covariance operator
\begin{align*}
    C_{U\mathbf{H}, V\mathbf{B}}:
    \begin{cases}
    L^2([0,1],\R^{2K}) &\to L^2([0,1],\R^{2K}), \\
   \phantom{abcdef}\varphi &\mapsto \Big[s \mapsto \int_{0}^1 K_{U\mathbf{H}, V\mathbf{B}}(s,t)\varphi(t) \d t\Big].
   \end{cases}
\end{align*}
The operator $C_{U\mathbf{H}, V\mathbf{B}}$ is bounded, self-adjoint and {trace-class}, and only depends on the random variables $U$ and $V$ through their variances $\varu^2$ and $\varv^2$.
We first recall a classical existence result.
\begin{proposition}\label{prop:construct_Gaussian_field_L2}
    There exists a {unique} centered Gaussian measure over $L^2([0,1],\R^{2K})$ with covariance operator $C_{U\mathbf{H}, V\mathbf{B}}$.
\end{proposition}

This allows us to build the Gaussian field $(\mathbf{Z}^h, \mathbf{Z}^b)$ $L^2 ( \Omega, L^2 ( [0,1], \R^{2 K}) ) $. 
To ensure that $(\mathbf{Z}^h, \mathbf{Z}^b)$ has a {continuous} modification, we have to construct a Gaussian measure over $\C([0,1],\R^{2K})$ with the given covariance structure. 

\begin{theorem}[Kolmogorov-Chentsov]\label{thm:Kolmogorov-Chentsov}
Let 
\(
K : [0,1] \times [0,1]\to \R^{k\times k}
\)
be a symmetric positive semidefinite covariance operator. 
We assume that there exist \(\alpha, C > 0\) such that
\[ \forall (s,t) \in [0,1]^2, \qquad
\mathrm{Tr} [ K(s, s) + K(t, t) - 2 K(s, t) ] \le C |s - t|^{2\alpha}.
\]
Then, there exists a unique centred Gaussian measure \(\mu\) on \(\C([0,1], \mathbb{R}^k)\) such that
\[ \forall (s,t) \in [0,1]^2, \qquad
\int_{\C([0,1], \mathbb{R}^k)} f(s) f(t)^{\top} \, \mu(df) = K(s, t),
\]
Furthermore,
\(
\mu\big(\C^\beta([0,1], \mathbb{R}^k)\big) = 1
\) for every $\beta \in (0,\alpha)$, where \(\C^\beta([0,1], \mathbb{R}^k)\) denotes the space of \(\beta\)-Hölder continuous functions.
\end{theorem}

To apply this result in our case, it is sufficient that $(U\mathbf{H}(s), V\mathbf{B}(s))_{s \in [0,1]} \in \C^\alpha ( [0,1], L^2 (\Omega, \R^{2K})) $. 

\begin{proposition}\label{prop:construct_Gaussian_field}
    If $(\mathbf{H(s)}, \mathbf{B}(s))_{s \in [0,1]} \in \C^\alpha ( [0,1], L^2 (\Omega, \R^{2K}))$ for some $\alpha >0$, then there exists a Gaussian random variable $(\tilde{\mathbf{Z}}^h, \tilde{\mathbf{Z}}^b)$ on $\C([0,1];\R^{2K})$ such that $\forall s,t\in[0,1]$, 
    \[\E[(\tilde{\mathbf{Z}}^h, \tilde{\mathbf{Z}}^b)(s)(\tilde{\mathbf{Z}}^h, \tilde{\mathbf{Z}}^b)(t)^\top] = K_{U\mathbf{H}, V\mathbf{B}}(s,t). \]
\end{proposition}

In the following, we only use the law of $(\mathbf{Z}^h, \mathbf{Z}^b)$.
Without loss of generality, we shall thus assume that $(\mathbf{Z}^h, \mathbf{Z}^b)$ has a.s. continuous trajectories.

\begin{proof} From 
    \(
        K_{U\mathbf{H}, V\mathbf{B}}(s,t) 
        = \begin{pmatrix}
        \varu^2\boldsymbol{\Gamma}^H(s,t) & 0\\ 0 &\varv^2\boldsymbol{\Gamma}^B(s,t)
    \end{pmatrix},
    \)
    it follows that, \[ \mathrm{Tr}[K_{U\mathbf{H}, V\mathbf{B}}(s,t))] = \varu^2\mathrm{Tr}[\boldsymbol{\Gamma}^H(s,t)] + \varv^2 \mathrm{Tr}[\boldsymbol{\Gamma}^B(s,t)]. \] We notice that $\mathrm{Tr}[\E[\mathbf{H}(s)\mathbf{H}(t)^\top]] = \E[\langle \mathbf{H}(s), \mathbf{H}(t)\rangle]$ and similarly for $\mathbf{B}$, and so
    \begin{multline*}
        \mathrm{Tr}[K_{U\mathbf{H}, V\mathbf{B}}(s,s) + K_{U\mathbf{H}, V\mathbf{B}}(t,t) -2K_{U\mathbf{H}, V\mathbf{B}}(s,t)] \\= \varu^2\mathrm{Tr}[\boldsymbol{\Gamma}^H(s,s) + \boldsymbol{\Gamma}^H(t,t) -2\boldsymbol{\Gamma}^H(s,t)] + \varv^2\mathrm{Tr}[\boldsymbol{\Gamma}^B(s,s) + \boldsymbol{\Gamma}^B(t,t) -2\boldsymbol{\Gamma}^B(s,t)]\\=
        \E[\|\mathbf{H}(s) - \mathbf{H}(t)\|^2_{2,k}] + \E[\|\mathbf{B}(s) - \mathbf{B}(t)\|^2_{2,k}] \leq C|t-s|^{2\alpha},
    \end{multline*}
    where we used our assumption for the last inequality. 
    Applying Theorem~\ref{thm:Kolmogorov-Chentsov} now constructs the desired Gaussian measure. 
\end{proof}

\begin{rem}[Independence structure]\label{rem:gaussian_independence}
    From the covariance structure defined by $C_{U\bfH, V\bfB}$, we get that $\bfZ^h$ is independent from $\bfZ^b$, even if $\bfH$ and $\bfB$ are not themselves independent.
\end{rem}

\begin{rem}[Exponential Moments]\label{rem:GaussianHasExponentialMoments}
From Fernique's theorem, the constructed Gaussian measure $\mu := \L((\mathbf{Z}^h, \mathbf{Z}^b)) \in\mathcal{P}(\C([0,1],\R^{2K}))$ has exponential moments of all orders. 
More precisely, for every $t\in \R$,
\[ \int\exp(t\|z\|_{\infty})d\mu(z) < C_t, \] 
for some $C_t > 0 $ that only depends on $t$ and the covariance operator of $\mu$. 
\end{rem}

\begin{rem}[Non-anticipative bounds with respect to $k$]\label{rem:progressively_build_Gaussian_field}
In the following, we require bounds on exponential moments of the Gaussian field that only depend on the past.
To get this property, we notice that the first $k$ components $(\mathbf{Z}^h_{\wedge k-1}, \mathbf{Z}^b_{\wedge k-1})$ of the Gaussian field $(\mathbf{Z}^h, \mathbf{Z}^b)$ themselves form a Gaussian field.
If further $(\mathbf{H}_{\wedge k-1}, \mathbf{B}_{\wedge k-1})_{s \in [0,1]} \in \C^\alpha([0,1],L^2(\Omega,\R^{2k}))$, then we can apply Remark~\ref{rem:GaussianHasExponentialMoments} to bound $\E[\exp(t\|(\mathbf{Z}^h_{\wedge k-1}, \mathbf{Z}^b_{\wedge k-1})\|_{\infty})]$ by a constant that only depends on $t$ and $C_{U\mathbf{H}_{\wedge k-1}, V\mathbf{B}_{\wedge k-1}}$, which is precisely the covariance associated to $(\mathbf{Z}^h_{\wedge k-1}, \mathbf{Z}^b_{\wedge k-1})$.
\end{rem}

We close this section by recalling the Lipschitz-regularity property of the {Gaussian projection}, see e.g. \cite[Theorem 2.2]{cuestaalbertos1996L2Wasserstein}.

\begin{theorem}\label{teo:GaussianW2control}
    Let $\mu, \nu$ be Radon probability measures of defined on a Hilbert space $\H$, such that $\int_\H\|x\|_{\H}^2 d\mu(x)$ and $\int_\H\|x\|_{\H}^2 d\nu(x)$ are finite. Let $G_\mu$ be the Gaussian probability measure defined on $\H$ with the same mean and covariance operator as $\mu$. Then, the following holds:
    \[W_2(\mu, \nu) \geq W_{2}(G_\mu, G_\nu) = \|m_\mu - m_\nu\|^2  + \mathrm{tr}(C_\mu + C_\nu - 2(C_\mu^{1/2}C_\nu C_\mu^{1/2})^{1/2})\]
\end{theorem}

This gives us the following estimate.

\begin{corollary}\label{cor:GaussianW2inequality}
    Let $(\mathbf{Z}^h, \mathbf{Z}^b)$, $(\mathbf{Z}^{\star,h}, \mathbf{Z}^{\star,b})$ be Gaussian in $L^2([0,1]; \R^{2K})$ with covariance structure given by $C_{U\mathbf{H}, V\mathbf{B}}$, $C_{U\mathbf{H}^{\star}, V\mathbf{B}^{\star}}$ respectively. Then, for every $s\in[0,1]$, 
    \[ W_2^2 ( (\L(\mathbf{Z}^h \vert_{[0,s]}),\L(\mathbf{Z}^{\star,h}\vert_{[0,s]})) \leq W_2^2 ( \L(U \mathbf{H} \vert_{[0,s]}), \L(U \mathbf{H}^\star \vert_{[0,s]} )) \leq \varu^2 \int_0^s \E [ \| \mathbf{H} (r) - \mathbf{H}^{\star} (r) \|^2 ] \d r, \]
    \[ W_2^2 ( (\L(\mathbf{Z}^b \vert_{[s,1]}),\L(\mathbf{Z}^{\star,b}\vert_{[s,1]})) \leq W_2^2 ( \L(V \mathbf{B} \vert_{[s,1]}), \L(V\mathbf{B}^\star \vert_{[s,1]}) ) \leq \varv^2 \int_s^1 \E [ \| \mathbf{B} (r) - \mathbf{B}^{\star} (r) \|^2 ] \d r, \]
\end{corollary}
\begin{proof}
    We notice see that $(\mathbf{Z}^h|_{[0,s]}, \mathbf{Z}^b|_{[0,s]})$ is still a Gaussian random variable, with covariance structure given by that of $(U\mathbf{H}|_{[0,s]}, V\mathbf{B}|_{[0,s]})$, and similarly for $\mathbf{Z}^\star$. 
    The first inequality follows from Theorem~\ref{teo:GaussianW2control}, and the second by choosing the trivial coupling to get
    \[W_2^2 ( \L(U \mathbf{H} \vert_{[0,s]}), \L(U \mathbf{H}^\star \vert_{[0,s]}) ) \leq \E [ \| U\mathbf{H} (r) - U\mathbf{H}^{\star} (r) \|_{L^2[0,s]}^2 ]. \] 
    The independence of $U$ from $(\mathbf{H}, \mathbf{H}^\star)$ yields the result.
\end{proof}

\subsection{Well-posedness for the limit system}\label{ssec:main_proof_well_posedness}

Leveraging the dependence of $(F_k^{\mathbf{H},\mathbf{B}},G_k^{\mathbf{H},\mathbf{B}})$ on the past, we are going to build $(\mathbf{H}_{\wedge k},\mathbf{B}_{\wedge k})$ by finite induction on $k$.
A fixed-point argument with respect to $s$ is needed at each step to handle the non-linearity in $\L(H_k,B_k)$.
Let us first establish a few preliminary estimates.

Given $k \in [1:K-1]$, and a $\F_W$-measurable random variable $(\mathbf{H},\mathbf{B}) \in L^2 ( \Omega, \C([0,1],\R^{2(k+1)}))$, we define a $\C([0,1],\R^{2(k+1)})$-valued random variable $(\overline{\mathbf{H}},\overline{\mathbf{B}})$ by $(\overline{\mathbf{H}}_{\wedge k-1},\overline{\mathbf{B}}_{\wedge k-1}) := ({\mathbf{H}}_{\wedge k-1},{\mathbf{B}}_{\wedge k-1})$ and $s \in [0,1] \mapsto (\overline{H}_{k}(s),\overline{B}_{k}(s))$ is defined by
\begin{multline} \label{eq:lemLimH}
\overline{H}_{k} (s) := W_{\mathrm{in}}\cdot x + \int_0^s \varv^2 \E [ \rho'( P_k (r) ) \E[ H_k (r) \bnabla_{\bfz^b} F_k^{\mathbf{H},\mathbf{B}} (r,  \mathbf{Z}^{h},\mathbf{Z}^{b} ) \vert  \mathbf{Z}^{h},\mathbf{Z}^{b} ] \,] \cdot\mathbf{B}_{\wedge k-1}(r) \\
\qquad\qquad\qquad\qquad\qquad+ \E [ \rho ( P_k (r) ) G^{\mathbf{H},\mathbf{B}}_k (r, \mathbf{Z}^{h},\mathbf{Z}^{b}) \vert \F_W] \d r, 
\end{multline}\vspace{-1cm}
\begin{multline} \label{eq:lemLimB}
\overline{B}_{k} (s) := W_{\mathrm{out}}\cdot \nabla \ell ( y_k ) - \int_s^1 \big\{ \E \big[ \rho'( P_k(r) ) F_k^{\mathbf{H},\mathbf{B}} (r, \mathbf{Z}^{h},\mathbf{Z}^{b}) Q_k(r) \vert \F_W \big] \\
\qquad\qquad+ \varu^2 \E \big[ \rho''( P_k (r))Q_k(r) ( H_k(r) + \E [ H_k(r) \bnabla_{\bfz^h} F_k^{\mathbf{H},\mathbf{B}} ( r, \mathbf{Z}^{h},\mathbf{Z}^{b}) \vert  \mathbf{Z}^{h},\mathbf{Z}^{b} ] \cdot\mathbf{H}_{\wedge k-1}(r) )  \vert \F_W\big] \\
\qquad \qquad \qquad +\varu^2 \E \big[ \rho' \big( P_k(r) ) \E [ B_k(r) \bnabla_{\bfz^h} G_k^{\mathbf{H},\mathbf{B}} ( r, \mathbf{Z}^{h},\mathbf{Z}^{b}) \vert  \mathbf{Z}^{h},\mathbf{Z}^{b} ] \big] \cdot\mathbf{H}_{\wedge k-1}(r) \big\} \d r,
\end{multline}
where
\[
\begin{cases}
P_k (s) := Z_k^{h} (s) + \E [ H_k (s) F^{\mathbf{H},\mathbf{B}}_k (s, \mathbf{Z}^{h},\mathbf{Z}^{b}) \vert  \mathbf{Z}^{h},\mathbf{Z}^{b} ], \\
Q_k (s) := Z_k^{b} (s) + \E [ B_k (s) G^{\mathbf{H},\mathbf{B}}_k (s, \mathbf{Z}^{h},\mathbf{Z}^{b}) \vert  \mathbf{Z}^{h},\mathbf{Z}^{b} ], \\
y_k := \E[H_k(1) W_{\mathrm{out}}].
\end{cases}
\]
In the above definitions, $(\mathbf{Z}^{h},\mathbf{Z}^{b})$ is a centred Gaussian field in $L^2([0,1],\R^{2k})$ with the same covariance operator as $(U \mathbf{B},V \mathbf{H})$, for some independent centered $U,V$ with variance $\varu^2$, $\varv^2$ and independent of $\F_W$, using Propositions~\ref{prop:construct_Gaussian_field_L2}-\ref{prop:construct_Gaussian_field}.  
Up to replacing $\Omega$ by some larger product space $\Omega \times \tilde\Omega$, we can assume that $(\mathbf{Z}^{h},\mathbf{Z}^{b})$ is built on the same probability space $\Omega$.
We assume it to be defined on $\Omega$ and independent of $(\mathbf{H},\mathbf{B})$.
We notice that this definition only depends on $\mathbf{H},\mathbf{B}$ through $\L(\mathbf{H},\mathbf{B})$. 
To alleviate notations, we write \(F_k^{\mathbf{H},\mathbf{B}} ( s, \mathbf{Z}^{h},\mathbf{Z}^{b})\) instead of \(F_k^{\mathbf{H},\mathbf{B}} ( s, \mathbf{Z}^{h}(s),\mathbf{Z}^{b}(s))\) and so on. We also omit the dependence on $k$ for $(\mathbf{H},\mathbf{B},\overline{\mathbf{H}},\overline{\mathbf{B}})$ -- this will be clear from context.

In the next subsection we provide key a priori estimates for the map $(\bfH, \bfB) \mapsto (\overline{\bfH}, \overline{\bfB})$.
\subsubsection{A priori estimates}

\begin{lemma}[A priori bounds] \label{lem:LimApriori}
If $s \mapsto ( \mathbf{H}_{\wedge k -1}(s),\mathbf{B}_{\wedge k -1}(s)) \in \C^{1/2}([0,1], L^2(\Omega,\R^{2k}))$, then:
\begin{enumerate}[label=(\roman*),ref=(\roman*)]
\item\label{lemLim:boundH} There exists $C_k > 0$ that only depends on $\L( \mathbf{H}_{\wedge k -1},\mathbf{B}_{\wedge k -1})$ such that
\[ \forall s \in [0,1], \qquad \E\big[\sup_{r\in[0,s]}\overline{H}^2_k(r)\big] \leq C_k + C_k\int_0^s \E\big[\sup_{r\in[0,t]}H_k^2(r)\big] \d t. \]
\item\label{lemLim:boundB} 
There exists $C^H_k >0$ that only depends on $\L( \mathbf{H},\mathbf{B}_{\wedge k -1})$ -- now including $H_k$ -- such that
\[ \forall s \in [0,1], \qquad \E\big[\sup_{r\in[s,1]}\overline{B}_k^2(r)\big] \leq C_k^H + C_k^H\int_s^1 \E\big[\sup_{r\in[t,1]}B_k^2(r)\big] \d t. \]
\end{enumerate}
\end{lemma}

\begin{proof}
In the following, the constants $C_k$ and $C^H_k$ may change from line to line, but always keeping the required dependence. We consider  $M_Z(r) := 1\vee\max_{i \in [0,k-1]} \vert Z^h_i (r) \vert \vee \vert Z^b_i (r) \vert$, which is a random variable that only depends on $(\mathbf{H}_{\wedge k-1}, \mathbf{B}_{\wedge k-1})$ through its law. 
From Remark~\ref{rem:progressively_build_Gaussian_field}, $\E[e^{tM_Z(r)}] \leq C_t$ for constant $C_t > 0$ that only depends on $\L(\mathbf{H}_{\wedge k-1}, \mathbf{B}_{\wedge k-1})$ and $t \in \R$. 

\medskip

\ref{lemLim:boundH}
We bound the derivative of $\overline{H}_k$, which is 
well-defined from \eqref{eq:lemLimB}. For $r \in [0,1]$,
\begin{align*}
    \big|\E [ \rho'(P_k(r)) \E[ H_k &(r) \bnabla_{\bfz^b} F_k^{\mathbf{H},\mathbf{B}} (r,  \mathbf{Z}^{h},\mathbf{Z}^{b} ) \vert  \mathbf{Z}^{h},\mathbf{Z}^{b} ] ] \cdot\mathbf{B}_{\wedge k-1}(r)\big| \\&\leq C_k \E [ |H_k (r)| \|\bnabla_{\bfz^b} F_k^{\mathbf{H},\mathbf{B}} (r,  \mathbf{Z}^{h},\mathbf{Z}^{b} ) \|_{2,k} ] \sup_{t\in[0,1]}\|\mathbf{B}_{\wedge k-1}(t)\|_{2,k} \\
    &\leq C_k \E \big[ |H_k (r)| e^{\cstHhat_k  M_Z(r)}\max_{\substack{j\in[0:k-1]\\t\in[0,1]}} |H_j(t)| \big] \sup_{t\in[0,1]}\|\mathbf{B}_{\wedge k-1}(t)\|_{2,k}\\
    &\leq C_k \E \big[ |H_k (r)| \max_{\substack{j\in[0:k-1]\\t\in[0,1]}} |H_j(t)| \big] \E\big[e^{\cstHhat_k  M_Z(r)}\Big]\sup_{t\in[0,1]}\|\mathbf{B}_{\wedge k-1}(t)\|_{2,k}\\
    &\leq C_k \E^{1/2} [ |H_k (r)|^2] \E^{1/2}\big[\max_{\substack{j\in[0:k-1]\\t\in[0,1]}} |H_j(t)|^2 \big] \E[e^{\cstHhat_k  M_Z(r)}]\sup_{t\in[0,1]}\|\mathbf{B}_{\wedge k-1}(t)\|_{2,k}\\
    &\leq C_k \sup_{t\in[0,1]}\|\mathbf{B}_{\wedge k-1}(t)\|_k \E^{1/2} [ |H_k (r)|^2],
\end{align*}
where we used the uniform boundedness of $\rho'$, the estimates from Lemma~\ref{lem:RegSkel}-\ref{lemRegSkel:Grad}, the independence of $\F_W$ and the Gaussian field $(\mathbf{Z}^h,\mathbf{Z}^b)$, the Cauchy-Schwarz inequality (C-S), and we absorbed terms from the past into the constant $C_k$.

Similarly, using the linear growth of $\rho$ and Lemma~\ref{lem:RegSkel}-\ref{lemRegSkel:as},
\begin{align*}
    \big|\E [ \rho ( P_k (r) ) G_k^{\mathbf{H},\mathbf{B}} (r,  \mathbf{Z}^{h},\mathbf{Z}^{b} ) \vert  \Finit ]\big| &\leq C_k\E [ (1 +|P_k (r)| )M_Z(r) \max_{\substack{j\in[0:k-1]\\t\in[0,1]}}|B_j(t)| \vert \F_W]\\
    &\leq C_k \max_{\substack{j\in[0:k-1]\\t\in[0,1]}}|B_j(t)| \{ \E [ M_Z(r) ] + \E [|P_k (r)|M_Z(r) ] \},
\end{align*}
where we used the independence of $\F_W$ and $(\mathbf{Z}^h,\mathbf{Z}^b)$ for the last line. From the definition of $P_k$ and the estimates from Lemma~\ref{lem:RegSkel}-\ref{lemRegSkel:as}, we see that $|P_k(r)| \leq C_k[ |Z_k^h(r)| + M_Z(r)\E^{1/2}[H_k^2(r)] ]$, so that 
\begin{multline*}
    \big|\E [ \rho ( P_k (r) ) G_k^{\mathbf{H},\mathbf{B}} (r,  \mathbf{Z}^{h},\mathbf{Z}^{b} ) \vert  \Finit ]\big|\\\leq C_k \max_{\substack{j\in[0:k-1]\\t\in[0,1]}}|B_j(t)| \{ 1 + \E [|Z^h_k (r)|M_Z(r) ] + \E^{1/2}[H_k^2(r)]\E [M_Z^2(r) ] \}\\
    \leq C_k \max_{\substack{j\in[0:k-1]\\t\in[0,1]}}|B_j(t)| \{ 1 + \E^{1/2} [|Z^h_k (r)|^2] + \E^{1/2}[H_k^2(r)] \},
\end{multline*}
where the last line used C-S and absorbed terms from the past into the constant $C_k$.

Taking squares and integrating over $s \in [0,1]$, 
\begin{align*}
\int_{0}^s |\partial_r {\overline{H}}_k (r)|^2\d r & \leq  C_k\sup_{t\in[0,1]}\|\mathbf{B}_{\wedge k-1}(t)\|_{2,k}^2\int_0^s\E [ |H_k (r)|^2]\d r \\&\qquad+ C_k \max_{\substack{j\in[0:k-1]\\t\in[0,1]}}|B_j(t)|^2 \left( 1 + \int_0^s\E [|Z^h_k (r)|^2] + \E[H_k^2(r)] \d r\right)\\
&\leq C_k\big(\sup_{t\in[0,1]}\|\mathbf{B}_{\wedge k-1}(t)\|_{2,k}^2 + \max_{\substack{j\in[0:k-1]\\t\in[0,1]}}|B_j(t)|^2 \big)\left( 1 + \int_0^s\E [|Z^h_k (r)|^2] + \E[H_k^2(r)] \d r\right).
\end{align*}
We now use the fact that $\int_0^s\E [|Z^h_k (r)|^2]\d r = \E[\|Z_k^h\|_{L^2(0,s)}^2] = \varu^2\E[\|H_k\|_{L^2(0,s)}^2]$ to obtain
\[\int_{0}^s |\partial_r {\overline{H}}_k(r)|^2\d r \leq C_k\bigg(\sup_{t\in[0,1]}\|\mathbf{B}_{\wedge k-1}(t)\|_{2,k}^2 + \max_{\substack{j\in[0:k-1]\\t\in[0,1]}}|B_j(t)|^2 \bigg)\left( 1 + \int_0^s\E[H_k^2(r)] \d r\right). \]
Taking squares in \eqref{eq:lemLimH}, Jensen's inequality now yields, for every $s\in[0,1]$,
\[|\overline{H}_k(s)|^2  \leq \|W_{\mathrm{in}}\|_{\dimin}^2\|x\|_{\dimin}^2  + C_k\bigg(\sup_{t\in[0,1]}\|\mathbf{B}_{\wedge k-1}(t)\|_{2,k}^2 + \max_{\substack{j\in[0:k-1]\\t\in[0,1]}}|B_j(t)|^2 \bigg)\left( 1 + \int_0^s \E[H_k^2(r)] \d r\right).\]
This bound is {increasing} in $s\in[0,1]$. Taking supremum over $s$ and expectations thus yields
\begin{align*}
\E\left[\sup_{r\in[0,s]}\overline{H}^2_k(r)\right]  
& \leq C_k \Big(1 + \int_0^s \E[H_k^2(r)] \d r\Big) \leq C_k \left(1 + \int_0^s \E\left[\sup_{r\in[0,t]}H_k^2(r)\right] \d t\right),
\end{align*}
where we absorbed terms from the {past} into $C_k$. This is the desired bound.
\medskip

\ref{lemLim:boundB} 
Recalling that $\rho'$ and $\rho''$ are uniformly bounded, the following terms from \eqref{eq:lemLimB},
\[\big|\E \big[ \rho'( P_k(r) ) Q_k(r)F_k^{\mathbf{H},\mathbf{B}} (r, \mathbf{Z}^{h},\mathbf{Z}^{b}) \vert \Finit \big], \]
\[ \big|\varu^2 \E \big[ \rho' \big( P_k(r) ) \E [ B_k(r) \nabla_{\bfz^h} G_k^{\mathbf{H},\mathbf{B}} ( r, \mathbf{Z}^{h},\mathbf{Z}^{b}) \vert  \mathbf{Z}^{h},\mathbf{Z}^{b} ] \big]\cdot \mathbf{H}_{\wedge k-1}(r) \big|, \]
can be handled as in \ref{lemLim:boundH},
by bounding their $\lVert \cdot \rVert^2_{L^2(0,s)}$-norm by a term of type 
\[ C_k\Big(\sup_{t\in[0,1]}\|\mathbf{H}_{\wedge k-1}(t)\|_{2,k}^2 + \max_{\substack{j\in[0:k-1]\\t\in[0,1]}}|H_j(t)|^2 \Big)\left( 1 + \int_s^1 \E[B_k^2(r)] \d r\right). \]

We now bound the remaining term
\begin{align*}
    \big|\varu^2& \E \big[ \rho''( P_k (r))Q_k(r) ( H_k(r) + \E [ H_k(r) \bnabla_{\bfz^h} F_k^{\mathbf{H},\mathbf{B}} ( r, \mathbf{Z}^{h},\mathbf{Z}^{b}) \vert  \mathbf{Z}^{h},\mathbf{Z}^{b} ] \cdot\mathbf{H}_{\wedge k-1}(r) )  \vert \Finit \big] \big|\\
    & \leq C_k \E[|Q_k(r)|]|H_k(r)| + C_k \E\Big[|Q_k(r)| \E[|H_k(r)| \|\bnabla_{\bfz^h}F_k^{\mathbf{H},\mathbf{B}} ( r, \mathbf{Z}^{h},\mathbf{Z}^{b})\|_{2,k}|\mathbf{Z}^h, \mathbf{Z}^b]\Big] \|\mathbf{H}_{\wedge k-1}(r)\|_{2,k}\\
    & \leq C_k \E[|Q_k(r)|]|H_k(r)| + C_k \E\big[|H_k(r)|\max_{\substack{j\in[0:k-1]\\t\in[0,1]}}|H_j(t)|\big]\E\big[|Q_k(r)|e^{\cstHhat_k  M_Z(r)}\big] \|\mathbf{H}_{\wedge k-1}(r)\|_{2,k},
\end{align*}
where we used the triangle and C-S inequalities, as well as the independence of $\F_W$ and $(\mathbf{Z}^h,\mathbf{Z}^b)$ and the estimates from Lemma~\ref{lem:RegSkel}-\ref{lemRegSkel:Grad}.
From the definition of $Q_k$ and Lemma~\ref{lem:RegSkel}-\ref{lemRegSkel:as}, we further have $|Q_k(r)| \leq C_k(|Z_k^b(r)| + M_Z(r)\E^{1/2}[B_k^2(r)])$, so that 
\[ \E[|Q_k(r)|] \leq C_k(\E[|Z_k^b(r)|] + \E[M_Z(r)]\E^{1/2}[B_k^2(r)]) \leq C_k\{ \E^{1/2}[|Z_k^b(r)|^2] + \E^{1/2}[B_k^2(r)] \},\] 
as well as
\begin{align*}
    \E[|Q_k(r)|e^{\cstHhat_k  M_Z(r)}] &\leq C_k \E\big[|Z_k^b(r)|e^{\cstHhat_k  M_Z(r)}\big] + C_k \E[M_Z(r)e^{\cstHhat_k  M_Z(r)}]\E^{1/2}[B_k^2(r)]\\
    &\leq C_k \E^{1/2}\big[|Z_k^b(r)|^2\big] + C_k \E^{1/2}[B_k^2(r)].
\end{align*}
Inserting this in our previous estimate,
\begin{align}
\big|\varu^2 \E \big[ \rho''&( P_k (r))Q_k(r)( H_k(r) + \E [ H_k(r) \bnabla_{\bfz^h} F_k^{\mathbf{H},\mathbf{B}} ( r, \mathbf{Z}^{h},\mathbf{Z}^{b}) \vert  \mathbf{Z}^{h},\mathbf{Z}^{b} ]\cdot \mathbf{H}_{\wedge k-1}(r) )  \vert \Finit \big] \big|\\
    \qquad &\leq C_k \{ \E^{1/2}[|Z_k^b(r)|^2] + \E^{1/2}[B_k^2(r)] \} |H_k(r)| \\&\qquad+ C_k \E^{1/2}[|H_k(r)|^2]\Big\{\E^{1/2}\big[|Z_k^b(r)|^2\big] + \E^{1/2}[B_k^2(r)]\big\} \|\mathbf{H}_{\wedge k-1}(r)\|_{2,k}\\
    &\leq C_k \big\{ \E^{1/2}\big[|Z_k^b(r)|^2\big] + \E^{1/2}[B_k^2(r)] \big\} \sup_{t\in[0,1]} \big\{ |H_k(t)| + \|\mathbf{H}_{\wedge k-1}(t)\|_{2,k}\E^{1/2}[|H_k(t)|^2] \big\}.
\end{align}
Taking squares in \eqref{eq:lemLimB} and integrating over $s\in[0,1]$, we get
\begin{align*}
    \int_s^1|\partial_r {\overline{B}}_k(r)|^2 \d r  &\leq C_k\Bigg(\sup_{t\in[0,1]}\|\mathbf{H}_{\wedge k-1}(t)\|_{2,k}^2 + \max_{\substack{j\in[0:k-1]\\t\in[0,1]}}|H_j(t)|^2 \Bigg)\left( 1 + \int_s^1 \E[B_k^2(r)] \d r\right) \\&+ C_k\sup_{t\in[0,1]} \Big(|H_k(t)|^2 + \|\mathbf{H}_{\wedge k-1}(t)\|^2_{2,k}\E[|H_k(t)|^2]\Big)\int_s^1  \Big(\E\big[|Z_k^b(r)|^2\big] + \E[B_k^2(r)]\Big)\d r\\
    &\leq C_k\Big((1+ \sup_{t\in[0,1]} \E[|H_k(t)|^2])\sup_{t\in[0,1]}\|\mathbf{H}_{\wedge k-1}(t)\|_{2,k}^2 \\& \qquad\qquad\qquad+ \max_{\substack{j\in[0:k-1]\\t\in[0,1]}}|H_j(t)|^2 +\sup_{t\in[0,1]} |H_k(t)|^2 \Big)\left( 1 + \int_s^1 \E[B_k^2(r)] \d r\right),
\end{align*}
where we used previous bounds and $\E[\|Z_k^b\|_{L^2([s,1])}^2] = \varv^2\E[\|B_k\|_{L^2([s,1])}^2]$.
From Jensen's inequality, 
\[|\overline{B}_k(s)|^2 \leq \|W_{\mathrm{out}} \|^2_{\dimout}\|\nabla \ell(y_k)\|^2_{\dimout} + \int_s^1|\partial_r {\overline{B}}_k(r)|^2 \d r. \] 
Taking supremum over $s$ and expectations, we eventually get
\[\E\big[\sup_{r\in[s,1]}\overline{B}_k^2(r)\big] \leq C_k^H + C_k^H\int_s^1 \E \big [B_k^2(r) \big] \d r \leq C_k^H + C_k^H\int_s^1 \E\big[\sup_{r\in[t,1]}B_k^2(r)\big] \d t, \]
where we used 
\[ \|y_k\|_{\dimout} \leq \E[|H_k(1)|\|W_{\mathrm{out}}\|_{\dimout}] \leq \E^{1/2}[|H_k(1)|^2]\E^{1/2}[\|W_{\mathrm{out}}|_{\dimout}^2] \leq C_k^H, \]
and we absorbed the relevant quantities -- which may now depend on $H_k$ -- into the constant $C_k^H$. This concludes the proof.
\end{proof}

\begin{lemma}[Local Lipschitz estimates] \label{lem:LimlocLip}
Let $(\mathbf{H},\mathbf{B})$, $(\mathbf{H}^\star,\mathbf{B}^\star)$ be $\F_W$-measurable random variables in $L^2 ( \Omega, \C( [0,1], \R^{2(k+1)}))$ with $(\mathbf{H}_{\wedge k-1},\mathbf{B}_{\wedge k-1}) =(\mathbf{H}^\star_{\wedge k-1},\mathbf{B}^\star_{\wedge k-1})$. 
Assume that $s \mapsto ( \mathbf{H}_{\wedge k -1}(s),\mathbf{B}_{\wedge k -1}(s))$ lies in $\C^{1/2}([0,1], L^2(\Omega,\R^{2k}))$, and similarly for $(\mathbf{H}^{\star}_{\wedge k -1},\mathbf{B}^{\star}_{\wedge k -1})$.
Let $(\overline{\mathbf{H}},\overline{\mathbf{B}})$, $(\overline{\mathbf{H}}^\star,\overline{\mathbf{B}}^\star)$ be defined from them as in \eqref{eq:lemLimH}-\eqref{eq:lemLimB}. Then: 
\begin{enumerate}[label=(\roman*),ref=(\roman*)]
\item\label{lemLim:LipH} There exists $C_k > 0$ that only depends on $\L( \mathbf{H}_{\wedge k -1},\mathbf{B}_{\wedge k -1})$, such that for every $s \in [0,1]$,
\[\E\big[ \sup_{r\in[0,s]}\vert \overline{H}_k (r) - \overline{H}^\star_k (r) \vert^2 \big] \leq C_k \big\{ 1+\E\big[\sup_{t\in[0,s]} |H_k (t)|^2\big] \big\} \int_0^s \E\big[\sup_{r\in[0,t]}|H_k(r) - H_k^\star(r)|^2\big] \d t.\]
\item\label{lemLim:LipB} If further $H_k = H_k^\star$ -- i.e. $\mathbf{H} = \mathbf{H}^\star$ --, then there exists $C^{H}_k >0$ that only depends on $\L( \mathbf{H},\mathbf{B}_{\wedge k -1})$ such that
\[\E\big[\sup_{r\in[s,1]}\vert \overline{B}_k (r) - \overline{B}^\star_k (r) \vert^2\big] \leq  C_k^H\int_s^1 \E\big[ \sup_{r\in[t,1]}|B_k(r) - B_k^\star(r)|^2\big] \d t. \]
\end{enumerate}
\end{lemma}

\begin{proof}
In the following, $C_k$ and $C^H_k$ may change from line to line, and we define $M_Z(r) := 1\vee\max_{i \in [0,k-1]} \vert Z^h_i (r) \vert \vee \vert Z^b_i (r) \vert$ as previously, which satisfies $\E[e^{tM_Z(r)}] \leq C_t$ for $C_t >0$ that only depends on $t$ and $\L(\mathbf{H}_{\wedge k-1}, \mathbf{B}_{\wedge k-1})$. 

Since $(\mathbf{H}_{\wedge k-1},\mathbf{B}_{\wedge k-1}) =(\mathbf{H}^\star_{\wedge k-1},\mathbf{B}^\star_{\wedge k-1})$, we recall that $F_k^{\mathbf{H}, \mathbf{B}} = F_k^{\mathbf{H}^\star, \mathbf{B}^{\star}}$ and $G_k^{\mathbf{H}, \mathbf{B}} = G_k^{\mathbf{H}^\star, \mathbf{B}^\star}$, and so on for their gradients. 

\medskip

\ref{lemLim:LipH} For every $s \in [0,1]$, we write
\begin{align*}
    \vert \overline{H}_k (s) &- \overline{H}^\star_k (s) \vert = \bigg|\int_0^s \varv^2 \E [ \rho'( P_k (r) ) \E[ H_k (r) \bnabla_{\bfz^b} F_k^{\mathbf{H},\mathbf{B}} (r,  \mathbf{Z}^{h},\mathbf{Z}^{b} ) \vert  \mathbf{Z}^{h},\mathbf{Z}^{b} ] \,] \mathbf{B}(r) \\
    &\qquad\qquad\qquad\qquad -\varv^2 \E [ \rho'( P^{\star}_k (r) ) \E[ H^{\star}_k (r) \bnabla_{\bfz^b} F_k^{\mathbf{H},\mathbf{B}} (r,  \mathbf{Z}^{h},\mathbf{Z}^{b} ) \vert  \mathbf{Z}^{h},\mathbf{Z}^{b} ] \,] \cdot \mathbf{B}_{\wedge k-1}(r)\\
&\qquad\qquad\qquad+ \E [ \rho ( P_k (r) ) G^{\mathbf{H},\mathbf{B}}_k (r, \mathbf{Z}^{h},\mathbf{Z}^{b}) \vert \Finit ] - \E [ \rho ( P^\star_k (r) ) G^{\mathbf{H},\mathbf{B}}_k (r, \mathbf{Z}^{h},\mathbf{Z}^{b}) \vert \Finit ] \d r\bigg|\\
&= \bigg|\int_0^s \varv^2 (\E [  \E[ (\rho'( P_k (r) )H_k (r) - \rho'( P_k^\star (r) )H_k^\star (r)) \bnabla_{\bfz^b} F_k^{\mathbf{H},\mathbf{B}} (r,  \mathbf{Z}^{h},\mathbf{Z}^{b} ) \vert  \mathbf{Z}^{h},\mathbf{Z}^{b} ] \,] )\cdot \mathbf{B}_{\wedge k-1}(r)\\
&\qquad+ \E [ (\rho ( P_k (r) ) - \rho ( P_k^\star (r) ))G^{\mathbf{H},\mathbf{B}}_k (r, \mathbf{Z}^{h},\mathbf{Z}^{b}) \vert \Finit ]\d r\bigg|. 
\end{align*}
Using that $\rho$ and $\rho'$ are $C_\rho$-Lipschitz,
\begin{align*}
\vert \overline{H}_k (s) &- \overline{H}^\star_k (s) \vert \leq C_\rho  \int_0^s \E [ \vert P_k (r) - P^\star_k (r) \vert \vert G^{\mathbf{H},\mathbf{B}}_k (r, \mathbf{Z}^{h},\mathbf{Z}^{b})\vert  \vert \Finit ] \d r \\
&+ \varv^2 C_\rho \int_0^s \E [ \vert \E [ | H_k (r) - H^\star_k (r) | \|\boldsymbol\nabla_{\bfz^b} F_k^{\mathbf{H},\mathbf{B}}(r,\mathbf{Z^h,\mathbf{Z}^b})\|_{2,k} \vert \mathbf{Z}^h,\mathbf{Z}^b ] \vert ] \|\mathbf{B}_{\wedge k-1}(r)\|_{2,k} \d r \\
&+ \varv^2  C_\rho \int_0^s \E [ \vert P_k (r) - P^\star_k (r) \vert  \E [ |H_k (r)| \|\boldsymbol\nabla_{\bfz^b} F_k^{\mathbf{H},\mathbf{B}}(r,\mathbf{Z^h,\mathbf{Z}^b})\|_{2,k} \vert \mathbf{Z}^h,\mathbf{Z}^b ] \vert ] \|\mathbf{B}_{\wedge k-1}(r) \|_{2,k} \d r.
\end{align*} 
We then use Lemma~\ref{lem:RegSkel}-\ref{lemRegSkel:as},\ref{lemRegSkel:Grad} to bound
\begin{align*}
    \vert \overline{H}_k (s) - &\overline{H}^\star_k (s) \vert \leq  C_k \max_{\substack{j\in[0:k-1]\\t\in[0,1]}}|B_j(t)|\int_0^s \E [ \vert P_k (r) - P^\star_k (r) \vert M_Z(r) ] \d r \\
&+ C_k \sup_{t\in[0,1]}\|\mathbf{B}_{\wedge k-1}(t)\|_{2,k} \int_0^s \E \big[ | H_k (r) - H^\star_k (r) | \max_{\substack{j\in[0:k-1]\\t\in[0,1]}}|H_j(t)|\big] \d r \\
&+ C_k \sup_{t\in[0,1]}\|\mathbf{B}_{\wedge k-1}(t) \|_{2,k}  \int_0^s \E \big[ |H_k (r)| \max_{\substack{j\in[0:k-1]\\t\in[0,1]}}|H_j(t)|\big]\E [ \vert |P_k (r) - P^\star_k (r)| e^{C_k M_Z(r)} ]\d r,
\end{align*}
where we absorbed terms from the {past} 
into the constant $C_k$.
We further apply C-S twice and rearrange terms, absorbing constants in $C_k$, to obtain
\begin{align*}
    \vert \overline{H}_k (s) - \overline{H}^\star_k (s) \vert &\leq  C_k \max_{\substack{j\in[0:k-1]\\t\in[0,1]}}|B_j(t)|\int_0^s \E^{1/2} [ \vert P_k (r) - P^\star_k (r) \vert^2] \d r \\
&+ C_k \sup_{t\in[0,1]}\|\mathbf{B}_{\wedge k-1}(t)\|_{2,k} \int_0^s \E^{1/2} [ | H_k (r) - H^\star_k (r) |^2 ] \d r \\
&+ C_k \sup_{t\in[0,1]}\|\mathbf{B}_{\wedge k-1}(t) \|_{2,k}  \E^{1/2}\big[\sup_{t\in[0,s]} |H_k (t)|^2\big] \int_0^s \E^{1/2}[ |P_k (r) - P^\star_k (r)|^2]\d r.
\end{align*}
Let us now bound
\begin{align*}
    |P_k (r) - P^\star_k (r)| & \leq |Z_k^h(r) - Z_k^{h^\star}(r)| + \E[|H_k(r) - H_k^\star(r)| |F_k^{\mathbf{H}, \mathbf{B}}(r, \mathbf{Z}^h, \mathbf{Z}^b)| |\mathbf{Z}^h, \mathbf{Z}^b]\\
    & \leq |Z_k^h(r) - Z_k^{h^\star}(r)| + C_k\E\bigg[|H_k(r) - H_k^\star(r)| \max_{\substack{j\in[0:k-1]\\t\in[0,1]}} |H_j(t)|\bigg] M_Z(r)\\
    & \leq |Z_k^h(r) - Z_k^{h^\star}(r)| + C_k\E^{1/2}[|H_k(r) - H_k^\star(r)|^2] M_Z(r),
\end{align*}
where we used again Lemma~\ref{lem:RegSkel}-\ref{lemRegSkel:as} and C-S.
We thus have 
\[\E^{1/2}[|P_k (r) - P^\star_k (r)|^2] \leq \E^{1/2}[|Z_k^h(r) - Z_k^{h^\star}|^2] + C_k \E^{1/2}[|H_k(r) - H_k^\star(r)|^2].\] 
Taking squares in our bound on \(\vert \overline{H}_k (s) - \overline{H}^\star_k (s) \vert\), this implies
\begin{align*}
   \vert \overline{H}_k (s) - \overline{H}^\star_k (s) \vert^2 &\leq  C_k \big\{ \max_{\substack{j\in[0:k-1]\\t\in[0,1]}}|B_j(t)|^2 +\sup_{t\in[0,1]}\|\mathbf{B}_{\wedge k-1}(t) \|^2_{2,k} \big\}\\&\qquad\cdot\big\{ 1+\E\big[\sup_{t\in[0,s]} |H_k (t)|^2\big]\big\} \int_0^s \E[|Z_k^h(r) - Z_k^{h^\star}|^2] + \E[|H_k(r) - H_k^\star(r)|^2] \d r.
\end{align*}
Taking supremum over $s$, expectations, and absorbing constants in $C_k$,
\begin{multline*}
   \E\big[\sup_{r\in[0,s]}\vert \overline{H}_k (r) - \overline{H}^\star_k (r) \vert^2\big] \leq  C_k \big\{ 1+\E\big[\sup_{t\in[0,s]} |H_k (t)|^2\big] \big\} \\ \cdot \int_0^s \big\{ \E[|Z_k^h(r) - Z_k^{h^\star}|^2] + \E[|H_k(r) - H_k^\star(r)|^2]\big\} \d r.
\end{multline*}
We now use that the definition of $(\mathbf{\overline{H}},\mathbf{\overline{H}}^\star) \in L^2([0,1],\R^{2k})$ only depends on $(\mathbf{Z}^h,\mathbf{Z}^{h^\star})$ through $(\L(\mathbf{Z}^h),\L(\mathbf{Z}^{h^\star}))$.
We thus have a degree of freedom in the choice of the coupling $(\mathbf{Z}^h,\mathbf{Z}^{h^\star})$, and we can minimise over it. 
By choosing an optimal coupling for $(Z^h_k, Z^{h^\star}_k)$, we get
\[ \E\Big[\int_0^s|Z^h_k(r) -Z^{h^\star}_k(r)|^2\d r\Big] = W_2^2 ( (\L(Z_k^h \vert_{[0,s]}),\L(Z_k^{h^\star}\vert_{[0,s]})) \le \varu^2 \int_0^s \E [ \vert H_k (r) - H^{\star}_k (r) \vert^2 ] \d r, \]
where the last inequality results from Corollary~\ref{cor:GaussianW2inequality}. This eventually yields
\begin{align*}
   \E\big[\sup_{r\in[0,s]}\vert \overline{H}_k (r) - \overline{H}^\star_k (r) \vert^2\big] &\leq  C_k \big\{ 1+\E\big[\sup_{t\in[0,s]} |H_k (t)|^2\big] \big\}\int_0^s \E[|H_k(r) - H_k^\star(r)|^2] \d r\\
   &\leq  C_k \big(1+\E\big[\sup_{t\in[0,s]} |H_k (t)|^2\big]\big)\int_0^s \E\big[\sup_{r\in[0,t]}|H_k(r) - H_k^\star(r)|^2\big] \d t,
\end{align*}
which is the desired bound.

\medskip

\ref{lemLim:LipB} By assumption, we have $\mathbf{H} = \mathbf{H}^\star$. This implies that $\mathbf{Z}^h$ and $\mathbf{Z}^{\star, h}$ have the same law, and so do \(P_k\) and \(P_k^\star\). 
Since $P_k(r)$ and $P_k^\star(r)$ are only involved within expectations, we can thus replace $P_k^\star(r)$ by $P_k(r)$ in the definition of \(\overline{B}^\star_k\). For $s\in[0,1]$, this yields
\begin{align*}
    \vert \overline{B}_k (s)& - \overline{B}^\star_k (s) \vert = \bigg|\int_s^1 \E \big[ \rho'( P_k(r) ) Q_k(r)F_k^{\mathbf{H},\mathbf{B}} (r, \mathbf{Z}^{h},\mathbf{Z}^{b}) \vert \Finit \big]\\
    &\qquad\qquad\qquad\qquad-\E \big[ \rho'( P_k(r) ) Q^\star_k(r)F_k^{\mathbf{H},\mathbf{B}} (r, \mathbf{Z}^{h},\mathbf{Z}^{b}) \vert \Finit \big]\\
    & \qquad\qquad\qquad+\varu^2 \E \big[ \rho' \big( P_k(r) ) \E [ B_k(r) \nabla_{\bfz^h} G_k^{\mathbf{H},\mathbf{B}} ( r, \mathbf{Z}^{h},\mathbf{Z}^{b}) \vert  \mathbf{Z}^{h},\mathbf{Z}^{b} ] \big]\cdot \mathbf{H}_{\wedge k-1}(r)\\
    &\qquad\qquad\qquad -\varu^2 \E \big[ \rho' \big( P_k(r) ) \E [ B^\star_k(r) \nabla_{\bfz^h} G_k^{\mathbf{H},\mathbf{B}} ( r, \mathbf{Z}^{h},\mathbf{Z}^{b}) \vert  \mathbf{Z}^{h},\mathbf{Z}^{b} ] \big] \cdot \mathbf{H}_{\wedge k-1}(r)\\
    &+\varu^2 \E \big[ \rho''( P_k (r))Q_k(r) {(} H_k(r) + \E [ H_k(r) \bnabla_{\bfz^h} F_k^{\mathbf{H},\mathbf{B}} ( r, \mathbf{Z}^{h},\mathbf{Z}^{b}) \vert  \mathbf{Z}^{h},\mathbf{Z}^{b} ] \cdot \mathbf{H}_{\wedge k-1}(r) {)}  \vert \Finit \big]\\
    &-\varu^2 \E \big[ \rho''( P_k (r))Q^\star_k(r) ( H_k(r) + \E [ H_k(r) \bnabla_{\bfz^h} F_k^{\mathbf{H},\mathbf{B}} ( r, \mathbf{Z}^{h},\mathbf{Z}^{b}) \vert  \mathbf{Z}^{h},\mathbf{Z}^{b} ] \cdot\mathbf{H}_{\wedge k-1}(r) )  \vert \Finit \big]
    \d r \bigg|.
\end{align*}
We use the triangle inequality and the boundedness of $\rho$, $\rho'$ to obtain
\begin{align*}
    \vert &\overline{B}_k (s) - \overline{B}^\star_k (s) \vert \leq C_k\int_s^1 \E \big[ |F_k^{\mathbf{H},\mathbf{B}} (r, \mathbf{Z}^{h},\mathbf{Z}^{b})| |Q_k(r) - Q_k^\star(r)| \vert \Finit \big]\\
    & \qquad\qquad+ \E \big[  \E [ |B_k(r) - B_k^\star(r)| \|\nabla_{\bfz^h} G_k^{\mathbf{H},\mathbf{B}} ( r, \mathbf{Z}^{h},\mathbf{Z}^{b})\|_{2,k} \vert  \mathbf{Z}^{h},\mathbf{Z}^{b} ] \big] \|\mathbf{H}_{\wedge k-1}(r)\|_{2,k}\\
    &+ \E \big[( |H_k(r)| + \E [ |H_k(r)| \|\bnabla_{\bfz^h} F_k^{\mathbf{H},\mathbf{B}} ( r, \mathbf{Z}^{h},\mathbf{Z}^{b})\|_{2,k} \vert  \mathbf{Z}^{h},\mathbf{Z}^{b} ] \|\mathbf{H}_{\wedge k-1}(r)\|_{2,k} ) |Q_k(r) - Q_k^\star(r)| \vert \Finit \big]
    \d r
\end{align*}
We now use our bounds from Lemma~\ref{lem:RegSkel}-\ref{lemRegSkel:as}-\ref{lemRegSkel:Grad}, the $\Finit$-measurability of $(\mathbf{H},\mathbf{B})$, the independence of $(\mathbf{Z}^h,\mathbf{Z}^{b})$ from $\F_W$, and the C-S inequality to get
\begin{align*}
    \vert &\overline{B}_k (s) - \overline{B}^\star_k (s) \vert \leq C_k\int_s^1 \E \big[ M_Z(r)  |Q_k(r) - Q_k^\star(r)|\big]\max_{\substack{j\in[0:k-1]\\t\in[0,1]}}|H_j(t)|\\
    & \qquad\qquad+ \E \big[ |B_k(r) - B_k^\star(r)| \max_{\substack{j\in[0:k-1]\\t\in[0,1]}}|B_j(t)| \big] \|\mathbf{H}_{\wedge k-1}(r)\|_{2,k}\\
    &\qquad\qquad+\big\{ |H_k(r)| + \E \big[ |H_k(r)| \max_{\substack{j\in[0:k-1]\\t\in[0,1]}}|H_j(t)| \big]  \|\mathbf{H}_{\wedge k-1}(r)\|_{2,k} \big\} \E \big[ e^{C_k M_Z(r)}|Q_k(r) - Q_k^\star(r)| \big]
    \d r\\
    &\leq C_k \max_{\substack{j\in[0:k-1]\\t\in[0,1]}}|H_j(t)|\int_s^1 \E^{1/2} \big[ |Q_k(r) - Q_k^\star(r)|^2\big]\d r \\&\qquad\qquad\qquad+ C_k \sup_{t\in[0,1]}\|\mathbf{H}_{\wedge k-1}(t)\|_{2,k} \int_s^1\E^{1/2} [ |B_k(r) - B_k^\star(r)|^2] \d r \\
    &\qquad\qquad\qquad\qquad+C_k \sup_{t\in[s,1]} |H_k(t)| + \E^{1/2} [ |H_k(t)|^2] \|\mathbf{H}_{\wedge k-1}(t)\|_{2,k} \int_s^1\E^{1/2} \big[|Q_k(r) - Q_k^\star(r)|^2 \big]
    \d r.
\end{align*}
As in the proof of \ref{lemLim:LipH}, we have
\[\E^{1/2}[|Q_k (r) - Q^\star_k (r)|^2] \leq \E^{1/2}[|Z_k^b(r) - Z_k^{b^\star}|^2] + C_k \E^{1/2}[|B_k(r) - B_k^\star(r)|^2],\] and thus, so that taking squares and Jensen's inequality yield
\begin{align*}
    \vert &\overline{B}_k (s) - \overline{B}^\star_k (s) \vert^2 \leq  C_k \max_{\substack{j\in[0:k-1]\\t\in[0,1]}}|H_j(t)|^2\int_s^1 \E \big[ |Q_k(r) - Q_k^\star(r)|^2\big]\d r \\&+ C_k\sup_{t\in[0,1]}\|\mathbf{H}_{\wedge k-1}(t)\|_{2,k} ^2\int_s^1\E[ |B_k(r) - B_k^\star(r)|^2] \d r \\
    &\qquad\qquad+C_k\sup_{t\in[s,1]}\Big(|H_k(t)|^2 + \E[ |H_k(t)|^2] \|\mathbf{H}_{\wedge k-1}(t)\|_{2,k}^2 \Big)\int_s^1\E \big[|Q_k(r) - Q_k^\star(r)|^2 \big]
    \d r]\\
    &\leq  C_k \bigg\{ \max_{\substack{j\in[0:k-1]\\t\in[0,1]}}|H_j(t)|^2 + \sup_{t\in[0,1]}\|\mathbf{H}_{\wedge k-1}(t)\|_{2,k} ^2 + \sup_{t\in[s,1]} |H_k(t)|^2 + \E [ |H_k(t)|^2] \|\mathbf{H}_{\wedge k-1}(t)\|_{2,k}^2 \bigg\} \\&\qquad\qquad\qquad\cdot\int_s^1 \E \big[ |Z^b_k(r) - Z_k^{b^\star}(r)|^2\big] + \E[ |B_k(r) - B_k^\star(r)|^2] \d r.
\end{align*}
Taking supremum over $s$ and expectations,
\[\E\big[\sup_{r\in[s,1]}\vert \overline{B}_k (r) - \overline{B}^\star_k (r) \vert^2\big] \leq  C_k^H\int_s^1 \E [ |Z^b_k(r) - Z_k^{b^\star}(r)|^2 ] + \E[ |B_k(r) - B_k^\star(r)|^2] \d r,\]
where relevant terms were absorbed into $C_k^H$. As previously in \ref{lemLim:LipH}, we replace the square $L^2$-norm of  $Z^b_k - Z^{b^\star}_k$ in $L^2 ( \Omega, L^2 ( [s,1],\R^2))$ by $W_2^2 ( (\L(Z_k^b \vert_{[s,1]}),\L(Z_b^{h^\star}\vert_{[s,1]}))$, we use Corollary~\ref{cor:GaussianW2inequality}, and finally obtain that
\[\E\big[\sup_{r\in[s,1]}\vert \overline{B}_k (r) - \overline{B}^\star_k (r) \vert^2\big] \leq  C_k^H\int_s^1 \E[ |B_k(r) - B_k^\star(r)|^2] \d r \leq  C_k^H\int_s^1 \E\big[ \sup_{r\in[t,1]}|B_k(r) - B_k^\star(r)|^2\big] \d t, \]
which concludes the proof.
\end{proof}

\subsubsection{Proof of Theorem~\ref{thm:well_posedness_main}}
\begin{proof}[Proof of Theorem~\ref{thm:well_posedness_main}]
We show by induction on $k \in [0:K-1]$ that $(\mathbf{H}_{\wedge k},\mathbf{B}_{\wedge k})$ is uniquely defined in $L^2 (\Omega,\C ( [0,1],\R^{2(k+1)} ))$.
Since the proof of Lemma~\ref{lem:LimApriori} shows that $(\mathbf{H}_{\wedge k},\mathbf{B}_{\wedge k})$ actually belongs to the Sobolev space $L^2 (\Omega, H^1 ( [0,1],\R^{2(k+1)} ))$, we perform the induction directly in this space.

\paragraph{Base case.} Using that $F^{\mathbf{H},\mathbf{B}}_0 = G^{\mathbf{H},\mathbf{B}}_0 \equiv 0$, we get that $H_0 (s) = W_{\mathrm{in}}\cdot x$.
Similarly, we get \begin{multline*}
    \partial_s B_0(s) = \E[\rho''(Z_0^h(s)) Z_0^b H_0(s)|\Finit] = \E[\rho''(Z_0^h(s)) Z_0^b]H_0(s) = \E[\rho''(Z_0^h(s)) ]\E[Z_0^b]H_0(s) = 0,
\end{multline*}
where we used the $\Finit$-measurability of $H_0$, and that the centered variables $(Z_0^h, Z_0^b)$ are independent from each other and from $\Finit$ using Remark~\ref{rem:gaussian_independence}. That is, for all $s\in[0,1]$, $B_0(s) = \Wout\cdot \nabla \ell(y_0)$, with $y_0 = \E[\Wout \Win^\top]x$. We thus have a unique solution $(H_0, B_0)$ in $L^2(\Omega, H^1([0,1], \R^{2}))$.

\paragraph{Inductive step.}
For $k \in [1: K-1]$, let us assume that $(\mathbf{H}_{\wedge k -1},\mathbf{B}_{\wedge k-1})$ is uniquely well-defined in $L^2(\Omega, H^1([0,1]; \R^{2k}))$. In particular, 
\[\E^{1/2}[\|(\mathbf{H}_{\wedge k-1}, \mathbf{B}_{\wedge k-1})(s) - (\mathbf{H}_{\wedge k-1}, \mathbf{B}_{\wedge k-1})(t)\|_{2,k}^2] \leq |t-s|^{1/2} \E^{1/2} [ \|(\mathbf{H}_{\wedge k-1}, \mathbf{B}_{\wedge k-1})\|^2_{ H^1([0,1], \R^{2k}) } ],\] 
and so that $s \mapsto ( \mathbf{H}_{\wedge k -1}(s),\mathbf{B}_{\wedge k -1}(s)) \in \C^{1/2}([0,1], L^2(\Omega,\R^{2k}))$, as required by Lemmas~\ref{lem:LimApriori}-\ref{lem:LimlocLip}.
To build $H_k$, we use a fixed-point argument on the map
\[ \Phi_k : H \in L^2 (\Omega,L^2([0,1],\R)) \mapsto \overline{H}_k \in L^2 (\Omega,L^2([0,1],\R)), \]
where $\overline{\mathbf{H}}$ is defined by \eqref{eq:lemLimH} from $\mathbf{H} := (\mathbf{H}_{\wedge k-1},H)$ and any $\mathbf{B} \in L^2 ( \Omega, \C^{2(k+1)})$ such that $\mathbf{B}_{\wedge k-1}$ coincides with the induction hypothesis -- indeed, \eqref{eq:lemLimH} does not involve $B_k$.
The solutions of \eqref{eq:lemLimH} are exactly the fixed-points of $\Phi_k$.

We consider an arbitrary starting point $H^{(0)}$, and we inductively define $H^{(n+1)} = \Phi_k( H^{(n)})$.
Setting $v_n(s) :=\E [\sup_{r\in[0,s]} \vert H^{(n)}(r) \vert^2 ]$, Lemma~\ref{lem:LimApriori} yields
\[ \forall s \in [0,1], \qquad v_{n+1}(s) \leq C_k + C_k\int_0^s v_{n}(r) \d r. \] 
Iterating this argument classically yields
\[v_{n+1}(s) \leq \underbrace{C_k \sum_{\ell=0}^{n}  \frac{(C_ks)^\ell}{\ell!}}_{\xrightarrow[n\to\infty]{}\,C_k\exp(C_ks)} + \underbrace{\frac{(C_ks)^{n+1}}{(n+1)!}\|H^{(0)}\|^2_{L^2(\Omega, \C ([0,1],\R))}}_{\xrightarrow[n\to\infty]{}\,0}. \]
As a consequence, there exists $n_0\in\mathbb{N}$ depending only on $C_k$ and $\|H^{(0)}\|^2_{L^2}$ such that $ \|H^{(n)} \|^2_{L^2(\Omega, \C([0,1],\R))} \leq R_k^2 :=2C_k\exp(C_k)$ for $n \geq n_0$.
Setting $u_m(s) := \E[\sup_{r\in[0,s]} |H^{(n_0+m+1)}-H^{(n_0+m)}|^2(r)]$, Lemma~\ref{lem:LimlocLip} yields
\[u_m(s)\leq C_k \big[ 1+ \|H_{m+n_0}\|^2_{L^2(\Omega, \C([0,1],\R))} \big] \int_0^s u_{m-1}(r) \d r \leq C_k [ 1+ R_k^2 ] \int_0^s u_{m-1}(r)\d r. \]
We then iterate this argument, and we bound the initial value by
\[u_0(s) \leq 2 \E[\sup_{r\in[0,s]} |H^{(n_0+1)}(r)|^2]+ 2 \E[\sup_{r\in[0,s]}|H^{(n_0)}(r)|^2]) \leq 4R_k^2.\] 
We deduce the bound
\(u_m(s)\leq 4R_k^2C_k^m(1+ R_k^2)^m\frac{s^m}{m!}\), which implies that $\sum_{m \in \mathbb{N}} u_m(1)$ is finite.
This proves that $(H^{(n)} )_{n\geq 0}$ is a Cauchy sequence in the complete space $L^2(\Omega, L^2 ([0,1],\R))$, and thus converges, showing existence of a fixed-point $H_k$.
If $H_k$ and $\tilde{H}_k$ are two fixed-points of $\Phi_k$, then the iterates $\Phi^{(n)}_k$ of $\Phi_k$ on $H_k$ and $\tilde{H}_k$ are uniformly bounded by the above reasoning, and we get that
\[ \lVert H_k - \tilde{H}_k \Vert_{L^2} = \lVert \Phi^{(n)}_k ( H_k ) -  \Phi^{(n)}_k ( \tilde{H}_k ) \Vert_{L^2} < \lVert H_k - \tilde{H}_k \Vert_{L^2}, \]
for large enough $n$, proving uniqueness. 
Thus, $\mathbf{H}_{\wedge k} := (\mathbf{H}_{\wedge k-1},H_k)$ is uniquely defined in $ L^2(\Omega, L^2([0,1]; \R^{2(k+1)}))$. One shall also immediately see that, from the proof of Lemma~\ref{lem:LimApriori}-\ref{lemLim:boundH}, $H_k$ satisfies $H_k \in L^2(\Omega; H^1([0,1]; \R))$. 

Now knowing that $(\mathbf{H}_{\wedge k},\mathbf{B}_{\wedge k-1})$ is uniquely defined, we construct $B_k$ similarly by leveraging Lemma~\ref{lem:LimApriori}-\ref{lemLim:boundB} and Lemma~\ref{lem:LimlocLip}-\ref{lemLim:LipB} to prove existence and uniqueness for the fixed-point of
\[ B \in L^2 (\Omega,L^2([0,1],\R))) \mapsto \overline{B}_k \in L^2 (\Omega,L^2([0,1],\R))), \]
where $\overline{\mathbf{B}}$ is defined by \eqref{eq:lemLimB} from $\mathbf{H} = \mathbf{H}_{\wedge k}$ and $\mathbf{B} := (\mathbf{B}_{\wedge k-1},B)$. Using the previous arguments and the bounds from Lemma~\ref{lem:LimApriori}-\ref{lemLim:boundB}, we further deduce that $\mathbf{H}_{\wedge k}$ actually lies in $L^2 (\Omega, H^1 ( [0,1],\R^{2(k+1)} ))$ as required to complete the induction.
\end{proof}

\subsection{Properties of the limit dynamics}\label{ssec:properties_of_mf_limit}

From Theorem~\ref{thm:well_posedness_main}, we know that the system \eqref{eq:LimH}-\eqref{eq:LimB} admits a unique solution in $L^2(\Omega, \C([0,1], \R^{2K}))$. 
From Remark~\ref{rem:linStruc}, we can explicitly rewrite the system as follows, 
underlying evaluation at $s\in[0,1]$,
\begin{equation}
\begin{cases}
{H}_k(0) = W_{\mathrm{in}}\cdot x, \qquad P_k = p_k(\mathbf{Z}^h, \mathbf{Z}^b), \\
\partial_s {H}_k = -\big[ \eta_u \varv^2 \E [ \rho'(P_k) \bnabla_{\bfz^b} \vec{\mathbf{v}}_{k}^F(\mathbf{Z}^h, \mathbf{Z}^b) ] \vec{\boldsymbol{\Gamma}}^H_k +\eta_v \E [ \rho(P_k) \vec{\mathbf{v}}_{k}^G(\mathbf{Z}^h, \mathbf{Z}^b) ]  \big] \cdot \mathbf{B}_{\wedge k-1},  \\
y_k := \E [ H_k(1) W_{\mathrm{out}}], 
\end{cases}
\end{equation}
\begin{equation}
\begin{cases}
{B}_k(1) = W_{\mathrm{out}}\cdot \nabla \ell ( y_k ), \qquad Q_k = q_k(\mathbf{Z}^h, \mathbf{Z}^b),\\
\partial_s {B}_k = \big[ \eta_v \varu^2 \E [ \rho'(P_k) \bnabla_{\bfz^h} \vec{\mathbf{v}}_{k}^G(\mathbf{Z}^h, \mathbf{Z}^b) ] \vec{\boldsymbol{\Gamma}}^B_k(s) + \eta_u \E [ \rho'(P_k)Q_k\vec{\mathbf{v}}_{k}^F(\mathbf{Z}^h, \mathbf{Z}^b) ] \big] \cdot \mathbf{H}_{\wedge k-1} \\ \qquad\qquad + \varu^2 \big[ \eta_u \E [\rho''(P_k)Q_k \bnabla_{\bfz^h} \vec{\mathbf{v}}_{k}^F( \mathbf{Z}^h, \mathbf{Z}^b) ]\vec{\boldsymbol{\Gamma}}^H_k \big] \cdot \mathbf{H}_{\wedge k-1} -\varu^2 \E[\rho''(P_k)Q_k] H_k.
\end{cases}
\end{equation}

The system is of McKean-Vlasov type, because it depends on its own law.

\begin{proposition}\label{prop:explicit_dependence_on_WeWu}
For any $k\in[0:K-1]$, we can write:
\[\mathbf{H}_{\wedge k}(s) = \boldsymbol{\Phi}_k(s).\begin{pmatrix}
    W_{\mathrm{in}}\\W_{\mathrm{out}}
\end{pmatrix},\;\;\;\mathbf{B}_{\wedge k}(s) = \boldsymbol{\Psi}_k(s).\begin{pmatrix}
    W_{\mathrm{in}}\\W_{\mathrm{out}}
\end{pmatrix},\] for some deterministic (and continuous) functions $\boldsymbol{\Phi}_k, \boldsymbol{\Psi}_k : [0,1]\to \R^{(k+1)\times (\dimin + \dimout)}$, which only depend on \(\boldsymbol{\Gamma}_{\wedge k}^H, \boldsymbol{\Gamma}_{\wedge k}^B\). In particular:
\begin{enumerate}[label=(\roman*),ref=(\roman*)]
    \item If $W_{\mathrm{in}}, W_{\mathrm{out}}$ are centered, then so is $(\mathbf{H}, \mathbf{B})$.
    \item If $(W_{\mathrm{in}}, W_{\mathrm{out}})^\top$ is a Gaussian vector, then $(\mathbf{H}, \mathbf{B})$ is a Gaussian process.
    \item Almost surely,
    \[\|(\mathbf{H}, \mathbf{B})\|_\infty \leq \sup_{s\in[0,1]}[ \|\boldsymbol{\Phi}_{K-1}(s)\|_{op}+ \|\boldsymbol{\Psi}_{K-1}(s)\|_{op} ] [ \|W_{\mathrm{in}} \|_{2,\dimin} + \|W_{\mathrm{out}} \|_{2,\dimout} ]. \]
    Namely, if $W_{\mathrm{in}}, W_{\mathrm{out}}$ are {a.s. bounded} random vectors, then $(\mathbf{H}, \mathbf{B})$ is a.s. bounded as well.
\end{enumerate}
    In any of these cases, we get that $(\mathbf{H}, \mathbf{B})$ has moments of all orders as a random variable on $L^2(\Omega, \C([0,1], \R^{2K}))$.
\end{proposition}
\begin{proof}
    This is immediate by induction from
    \(H_0(s) = W_{\mathrm{in}}\cdot x,\;\; B_0(s) = W_{\mathrm{out}}\cdot \nabla \ell(f_0)\) and integrating equations \eqref{eq:LimHdecoupledexplicit}-\eqref{eq:LimBdecoupledexplicit}. All the consequences follow directly from the resulting expression. Namely:
    \begin{enumerate}[label=(\roman*),ref=(\roman*)]
    \item Immediate.
        \item For any $n\in\mathbb{N}$, $a\in\R^n$ and $t\in[0,1]^n$ we have that \(\sum_{i=1}^n a_i (\mathbf{H}(t_i), \mathbf{B}(t_i))\) is a linear combination of the coordinates of the Gaussian vector $(W_{\mathrm{in}}, W_{\mathrm{out}})^\top$, hence Gaussian. It follows that $(\mathbf{H}, \mathbf{B})$ is a Gaussian process.
        \item Follows directly from the fact that \(\|(\mathbf{H}, \mathbf{B})\|_\infty = \sup_{s\in[0,1]} \|(\mathbf{H}, \mathbf{B})(s)\|_{\R^{2K}}\) and from our explicit expressions for $\mathbf{H}(s)$ and $\mathbf{B}(s)$.
    \end{enumerate}
\end{proof}

We also get the following subgaussianity properties.

\begin{proposition}\label{prop:skeleton_subgaussianity}
    Assume that $W_{\mathrm{in}}$ and $W_{\mathrm{out}}$ are subgaussian vectors.
    Then, for all $k\in [0:K-1]$, there exists $\vpk >0$ depending only on \(\boldsymbol{\Gamma}_{\wedge k}^H, \boldsymbol{\Gamma}_{\wedge k}^B\), 
    $\|W_{\mathrm{in}}\|_{vp}$, $\|W_{\mathrm{out}}\|_{vp}$, $\dimin$ and $\dimout$, such that:
    \begin{enumerate}[label=(\roman*),ref=(\roman*)]
        \item \(\| (\bfH_{\wedge k}, \bfB_{\wedge k})\|_{\infty}\) is $\vpk^2$-subgaussian.
        \item There is an absolute constant $c>0$, such that
        \[\Big\|\sup_{s\in[0,1]}\|\bfH_{\wedge k}(s) \bfH_{\wedge k}(s)^\top - \E[\bfH_{\wedge k}(s) \bfH_{\wedge k}(s)^\top]\|_{2, (k+1)\times (k+1)}\Big\|_{\psi_1} \leq c\vpk^2,\]
        and similarly for other combinations of $\bfH_{\wedge k}$ and $\bfB_{\wedge k}$.
    \end{enumerate}
\end{proposition}
\begin{proof}
We note that, since $\Win, \Wout$ are subgaussian vectors, each of their coordinates are subgaussian as well.
\begin{enumerate}[label=(\roman*),ref=(\roman*)]
    \item From Proposition~\ref{prop:explicit_dependence_on_WeWu}, there exists $\cstH_k > 0$ depending only on \(\boldsymbol{\Gamma}_{\wedge k}^H, \boldsymbol{\Gamma}_{\wedge k}^B\) 
    such that 
    \[\|(\mathbf{H}_{\wedge k}, \mathbf{B}_{\wedge k})\|_{\infty} \leq \cstH_k [ \|W_{\mathrm{in}}\|_{2,\dimin} + \|W_{\mathrm{out}}\|_{2,\dimout} ]. \]
    From usual properties of subgaussian vectors, we have
    \[\|\|W_{\mathrm{in}}\|_{2, \dimin}\|_{\psi_2}^2 = \|\|W_{\mathrm{in}}\|_{2, \dimin}^2\|_{\psi_1} \leq \sum_{j=1}^{\dimin} \|(W_{\mathrm{in}})_j^2\|_{\psi_1} = \sum_{j=1}^{\dimin} \|(W_{\mathrm{in}})_j\|_{\psi_2}^2<+\infty,  \]
    and similarly for $W_{\mathrm{out}}$. That is, both \(\|W_{\mathrm{in}}\|_{2, \dimin}\) and \(\|W_{\mathrm{out}}\|_{2, \dimin}\) are subgaussian random variables whose variance proxy depends solely on $\|\Win\|_{vp}$, $\|\Wout\|_{vp}$, $\dimin$ and $\dimout$. Using the monotonicity of the Orlicz norm,
    \[
        \|\|(\mathbf{H}_{\wedge k}, \mathbf{B}_{\wedge k})\|_{\infty}\|_{\psi_2} \leq \cstH_k [ \|\|W_{\mathrm{in}}\|_{2,\dimin}\|_{\psi_2} + \|\|W_{\mathrm{out}}\|_{2,\dimout}\|_{\psi_2} ] < +\infty,
    \]
    so that \(\|(\mathbf{H}_{\wedge k}, \mathbf{B}_{\wedge k})\|_{\infty}\) is subgaussian. Since \[\|\|(\mathbf{H}_{\wedge k}, \mathbf{B}_{\wedge k})\|_{\infty}-\E[\|(\mathbf{H}_{\wedge k}, \mathbf{B}_{\wedge k})\|_{\infty}]\|_{\psi_2} \leq c \|\|(\mathbf{H}_{\wedge k}, \mathbf{B}_{\wedge k})\|_{\infty}\|_{\psi_2}\] by the centering lemma \cite[Lemma 2.7.8]{vershynin2018highdimensionalproba}, it follows that \(\|(\mathbf{H}_{\wedge k}, \mathbf{B}_{\wedge k})\|_{\infty}\) is $\vpk^2$-subgaussian with a variance proxy $\vpk$ that only depends on $\cstH_k$, $\|W_{\mathrm{in}}\|_{vp}$, $\|W_{\mathrm{out}}\|_{vp}$, $\dimin$ and $\dimout$.

    \item We note that
    \[
       \sup_{s\in[0,1]}\|\bfH_{\wedge k}(s) \bfH_{\wedge k}(s)^\top\|_{2, (k+1)\times (k+1)} \leq \sup_{s\in[0,1]} \|\bfH_{\wedge k}(s)\|_{2,k+1}^2 \leq  \|(\mathbf{H}_{\wedge k}, \mathbf{B}_{\wedge k})\|_{\infty}^2,
    \]
    so that
    \[\|\sup_{s\in[0,1]}\|\bfH_{\wedge k}(s) \bfH_{\wedge k}(s)^\top\|_{2, (k+1)\times (k+1)}\|_{\psi_1} \leq  \|\|(\mathbf{H}_{\wedge k}, \mathbf{B}_{\wedge k})\|_{\infty}^2\|_{\psi_1} = \|\|(\mathbf{H}_{\wedge k}, \mathbf{B}_{\wedge k})\|_{\infty}\|_{\psi_2}^2.\]
    Thus, \(\sup_{s\in[0,1]}\|\bfH_{\wedge k}(s) \bfH_{\wedge k}(s)^\top\|_{2, (k+1)\times (k+1)}\) is subexponential, with norm depending solely on $\vpk^2$. Modulo modifying $\vpk$ by an absolute constant, the same holds for \[\sup_{s\in[0,1]}\|\bfH_{\wedge k}(s) \bfH_{\wedge k}(s)^\top - \E[\bfH_{\wedge k}(s) \bfH_{\wedge k}(s)^\top]\|_{2, (k+1)\times (k+1)}\] and also for analogous terms combining $\bfH_{\wedge k}$ and $\bfB_{\wedge k}$.
\end{enumerate}
\end{proof}

\section{Quantitative convergence} \label{sec:CV}

In this section, we rigorously prove Theorem~\ref{thm:quantitative_convergence_main}, which we recall here.
\quantitativeconvergencethm*

In Theorem~\ref{thm:well_posedness_main} we established that the mean-field system \eqref{eq:LimH}-\eqref{eq:LimB} has a unique solution in $L^2(\Omega, H^1([0,1],\R^{2K}))$. 
We now consider an i.i.d.~sequence of random subgaussian vectors $(W^d_e,W^d_u)_{d \geq 1}$ of variance proxies $\vpin^2$ and $\vpout^2$ respectively. Let \((\mathbf{H}^d,\mathbf{B}^d)_{d\geq 1}\) be an i.i.d.~sequence of copies of the (unique) solution of \eqref{eq:LimH}-\eqref{eq:LimB} with initial conditions $(W_{\mathrm{in}}^d, W_{\mathrm{out}}^d)$ for each $d\geq 1$.
That is, for $d\geq 1$, we have
\[
\begin{cases}
\partial_s H^{d}_k (s) =  \overline\H^{d}_k ( s), \\
H^{d}_k (0) = W_{\mathrm{in}}^d \cdot x,\\
y_k :=\E[H_k^{d}(1)W_{\mathrm{out}}^d],
\end{cases}
\qquad
\begin{cases}
\partial_s B^{d}_k (s) =  \overline\B^{d}_k ( s), \\
B^{d}_k (1) = W_{\mathrm{out}}^d \cdot \nabla\ell(y^{\infty}_k),
\end{cases}
\]
where \(\overline\H^{d}_k ( s)\) and \(\overline\B^{d}_k ( s)\) correspond to the right hand side of \eqref{eq:LimH} and \eqref{eq:LimB} respectively, with the added marker $d\geq 1$.
We notice that $y_k$ is unambiguously defined independently of $d$, because the variables  $(\bfH^d, \bfB^d)_{d\geq 1}$ are i.i.d.~Also, let $\mathbf{F}^{\bfH^d, \bfB^d}$, $\mathbf{G}^{\bfH^d, \bfB^d}$ be the skeleton maps associated to $(\bfH^d, \bfB^d)$ as in Definition~\ref{def:mf_skeleton_maps}. For simplicity throughout this section, for $D\geq 1$, we use the notation $\mathbf{F}^{(D)} := (\mathbf{F}^d := \mathbf{F}^{\bfH^d, \bfB^d})_{d=1}^D$ and similarly for \(\mathbf{G}^{(D)}\).

We couple these copies of the limit system to the finite-dimensional one, by sharing the same initial conditions. Let $D\geq 1$, and consider the solution \(\mathbf{h}^{(D)},\mathbf{b}^{(D)}\) of the Mean ODE \eqref{eq:original_finite_D_system}, initialized using the matrices $W_{\mathrm{in}}^{(D)}$ and $W_{\mathrm{out}}^{(D)}$, of rows $(W_{\mathrm{in}}^d)_{d=1}^D$ and $(W_{\mathrm{out}}^d)_{d=1}^D$, respectively; and with parameter initialization given by the i.i.d.~sequence $(U^d,V^d)_{d\geq 1}$ of independent, centered, subgaussian random variables, of variance proxies $\vpu^2$, $\vpv^2$ and variances \(\varu^2, \varv^2\), also independent from $\Finit$.
Also, let $\mathbf{f}^{\mathbf{h}^D,\mathbf{b}^D}, \mathbf{g}^{\mathbf{h}^D,\mathbf{b}^D}$ be the skeleton maps associated to \(\mathbf{h}^{(D)},\mathbf{b}^{(D)}\), as in Definition~\ref{def:finite_skeleton_maps}, which we here denote by \(\mathbf{f}^{(D)}, \mathbf{g}^{(D)}\) for simplicity. 

\medskip
In other words, consider the finite-dimensional system, for $D\geq 1$ and $d\in[1:D]$,
\[
\begin{cases}
\partial_s h^{d}_k (s) =  \H^{d}_k ( s, \mathbf{h}^{(D)}), \\
h^{d}_k (0) = W_{\mathrm{in}}^d \cdot x,\\
\hat{y}_k^{D} := \frac{1}{D} \sum_{d=1}^D h_k^{d}(1)W_{\mathrm{out}}^d,
\end{cases}
\qquad
\begin{cases}
\partial_s b^{d}_k (s) =  \B^{d}_k ( s, \mathbf{b}^{(D)}), \\
b^{d,D}_k (1) = W_{\mathrm{out}}^d \cdot \nabla \ell(\hat{y}_k^{D}),
\end{cases}
\]
where we define
\begin{align*}
\H^{d}_k ( s, \bfh^{(D)} ) 
&:= \E [ \rho ( P_k^{\bfh^D, \mathbf{b}^D}) [ \sqrt{D} V^d + g^{d}_k (s, \bfS^{\bfh^D}, \bfS^{\mathbf{b}^D} ) ] \vert \Finit],\\
\B^{d}_k ( s, \bfb^{(D)} ) 
&:= -\E [ \rho' (P_k^{\mathbf{h}^D, \mathbf{b}^D})Q_k^{\mathbf{h}^D, \mathbf{b}^D} [ \sqrt{D} U^d + f^{d}_k (s, \bfS^{\mathbf{h}^D}, \bfS^{\mathbf{b}^D} ) ] \vert \Finit ],
\end{align*}
for
\[\begin{cases}
    P_k^{\bfh^D, \mathbf{b}^D} :=S^{\bfh^D}_k + \langle f^{(D)}_k (s, \bfS^{\bfh^D}, \bfS^{\mathbf{b}^D}  ) , h^{(D)}_k (s)\rangle_{\ndbar}, \\
    Q_k^{\mathbf{h}^D, \mathbf{b}^D} := S^{\mathbf{b}^D}_k + \langle g^{(D)}_k (s, \bfS^{\mathbf{h}^D}, \bfS^{\mathbf{b}^D}  ), b^{(D)}_k (s)\rangle_{\ndbar},
\end{cases}\]
and, for $s\in[0,1]$,
\[ ( \bfS^{\mathbf{h}^D}(s), \bfS^{\mathbf{b}^D}(s)) := \bigg( \bigg( \frac{1}{\sqrt{D}} \sum_{d=1}^D U^{d} h^{d}_i(s) \bigg)_{0 \leq i \leq K-1} , \bigg( \frac{1}{\sqrt{D}} \sum_{d=1}^D V^{d} b^{d}_i(s) \bigg)_{0 \leq i \leq K-1} \bigg). \]
We similarly define $(\bfS^{\mathbf{H}^D}(s),\bfS^{\mathbf{B}^D}(s))$. As explained in Section~\ref{sec:basic_notation}, we write \(\H^{(D)}_k ( s, \bfh^{(D)} ) = (\H^{d}_k ( s, \bfh^{(D)} ))_{d=1}^D\), and similarly for $\B^{(D)}_k ( s, \bfb^{(D)} )$.

\subsection{A priori estimates}

\subsubsection{Conditional exponential moments}

Only for this subsection, consider $A^{(D)} = (A^d)_{d=1}^D$ a generic $\Finit$-measurable random vector on $\R^{D}$. Also, consider $(U^d)_{d\geq 1}$ an i.i.d.~sequence of independent, centered, subgaussian random variables, of variance proxy $\vpu^2$ and variance $\varu^2$; independent from $\Finit$. We now study the properties of the CLT-type sum: \[S^{A^{(D)}} := \frac{1}{\sqrt{D}}\sum_{d=1}^D U^dA^d. \]

\begin{lemma} \label{lem:Magic}
It holds that:
\[
\E [ \vert S^{A^{(D)}}\vert^2 \vert \Finit ] = \varu^2 \| A^{(D)} \|^2_{\ndbar}
\]
\end{lemma}
\begin{proof}
    It suffices to see that
    \[
\E [ \vert S^{A^{(D)}}\vert^2 \vert \Finit] = \frac{1}{D}\sum_{d, d'=1}^D\E [ U^dA^d U^{d'}A^{d'} \vert \Finit] =\frac{\E[ (U^1)^2]}{D} \sum_{d=1}^D \vert A^d \vert^2 = \varu^2 \| A^{(D)} \|_{\ndbar}^2,
\]
where we used the independence and centeredness of $(U^d, V^d)_{d\geq 1}$. 
\end{proof}
We further prove a conditional subgaussianity property for our CLT sums.
\begin{lemma}\label{lem:exp_moment_magic}
Conditionally on $\Finit$, $S^{A^{(D)}}$ is a centered \(\vpu^2 \|A^{(D)}\|_{\ndbar}^2\)-subgaussian random variable. As a consequence,
    \[ \forall t \in \R, \qquad
\E[ \exp(t|S^{A^{(D)}}|) \vert \Finit] \leq 2\exp\bigg(\frac{t^2\vpu^2}{2}\|A^{(D)}\|^2_{\ndbar}\bigg). \]
\end{lemma}
\begin{proof}
    Let $t\in \R$. Since $U:=(U^d)_{d\geq 1}$ is independent from $\Finit$, conditionally on $\Finit$ each variable is still  $\vpu^2$-subgaussian.
    Conditionally on $\Finit$, $S^{A^{(D)}} = \frac{1}{\sqrt{D}}\sum_{d=1}^D A^dU^d$ is a weighted sum of independent and centered subgaussian random variables. 
By standard properties of subgaussian random variables, $S^{A^{(D)}}$ is a centered subgaussian random variable, with variance proxy $\tau^2 := \vpu^2 \frac{1}{D}\sum_{d=1}^D |A^d|^2 =\vpu^2 \|A^{(D)}\|_{\ndbar}^2$. The conclusion follows.
\end{proof}
As a consequence of conditional subgaussianity, we get conditional moments estimates.
\begin{lemma}\label{lem:moment_magic}
    For every $p\geq 2$, there exists $C_p >0$ that only depends on $p$, such that
    \[
\E^{1/p} [ \vert S^{A^{(D)}}\vert^p \vert \Finit] \leq C_p \vpu\| A^{(D)} \|_{\ndbar}. \]
\end{lemma}

\begin{rem}[Conditional moments of all orders]\label{rem:conditional_moments_of_all_orders}
The proof in the following sections
extensively uses Lemmas~\ref{lem:Magic},  \ref{lem:exp_moment_magic} and \ref{lem:moment_magic}, applied to 
\[ A^{(D)} \in \{ h_k^{(D)}(s), H_k^{(D)}(s), h_k^{(D)}(s)-H_k^{(D)}(s), b_k^{(D)}(s),B_k^{(D)}(s), b_k^{(D)}(s)-B_k^{(D)}(s) \}, \] 
for $(k,s) \in[0:K-1] \times [0,1]$.
These allow us to control arbitrary moments of $S_k^{\bfh^D}$, $S_k^{\bfH^D}$, $(S_k^{\bfh^D} -S_k^{\bfH^D})$, $S_k^{\bfb^D}$, $S_k^{\bfB^D}$ and $(S_k^{\bfb^D} -S_k^{\bfB^D})$.
\end{rem}

\subsubsection{Bounds in RMS norm for the limit system} 

Similarly to Lemmas~\ref{lem:embeddingcontrol} and \ref{lem:finite_skeleton_bounded_ndbar} we can control the RMS norm of the vectors $(H_k^{(D)}(s), B_k^{(D)}(s))$, for all $k\in[0:K-1]$ and $s\in[0,1]$.

\begin{lemma}\label{lem:MF_skeleton_bounded_ndbar} There exists a constant $\cstH_k$ depending only on $k$ and $\boldsymbol{\Gamma}_{\wedge k}^H, \boldsymbol{\Gamma}_{\wedge k}^B$, such that
    \[\bigg\| \sup_{\substack{s\in[0,1]\\j\in[0:k-1]}}|H_j^{(D)}(s)|\vee|B_j^{(D)}(s)|\bigg\|_{\ndbar} \leq \cstH_k [ \|W_{\mathrm{in}}^{(D)}\|_{\ndbar} + \|W_{\mathrm{out}}^{(D)}\|_{\ndbar} ] =: \cstW_k \;\; \text{a.s.}, \]
    and for $\bfz^h,\bfz^b\in\R^K$ such that $\max_{j\in[0:k-1]} \vert z^h_j \vert \vee \vert z^b_j \vert \leq M$, $M \geq 1$, we have
    \[\bigg\| \sup_{\substack{s\in[0,1]\\j\in[0:k]}}|F^{(D)}_k(s, \bfz^h, \bfz^b)|\vee |G_k^{(D)}(s, \bfz^h, \bfz^b)| \bigg\|_{\ndbar}  \leq \cstH_k M[\|W_{\mathrm{in}}^{(D)}\|_{\ndbar} + \|W_{\mathrm{out}}^{(D)}\|_{\ndbar} ] =\cstW_k M \;\; \text{a.s.}, \]
    where we denote by $\cstW_k$ a generic $\Finit$-measurable random constant, depending only on $\cstH_k$ and $\|W_{\mathrm{in}}^{(D)}\|_{\ndbar}$, $\|W_{\mathrm{in}}^{(D)}\|_{\ndbar}$, independently of $D$.
\end{lemma}
\begin{proof}
As before, let $\cstH_k$ denote a constant depending only on $k$ and $\boldsymbol{\Gamma}_{\wedge k}^H, \boldsymbol{\Gamma}_{\wedge k}^B$. From Proposition~\ref{prop:explicit_dependence_on_WeWu}, we know that for all $d\in [1:D]$, $s\in[0,1]$ and $j\in[0:k]$, we have $|H_j^d(s)|\vee|B_j^d(s)| \leq \cstH_k[\|W_{\mathrm{in}}^d\|_{2, \dimin} + \|W_{\mathrm{out}}^d\|_{2, \dimout}]$ a.s. Taking $\sup_{j,s}$ and RMS norm, we get
    \begin{multline*}
        \bigg\| \sup_{\substack{s\in[0,1]\\j\in[0:k]}}|H_j^{(D)}(s)|\vee|B_j^{(D)}(s)|\bigg\|_{\ndbar} \leq \cstH_k\sqrt{\frac{1}{D}\sum_{d=1}^D \|W_{\mathrm{in}}^d\|_{2, \dimin}^2 + \|W_{\mathrm{out}}^d\|_{2, \dimout}^2} \\\leq \cstH_k[ \|W_{\mathrm{in}}^{(D)}\|_{\ndbar} + \|W_{\mathrm{out}}^{(D)}\|_{\ndbar} ].
    \end{multline*}
    From here, using Lemma~\ref{lem:RegSkel}-\ref{lemRegSkel:as}, we obtain the second bound.
\end{proof}

\subsubsection{Concentration estimates}

From this point onward, we consider the following LLN terms, which play a key role in the proof of Theorem~\ref{thm:quantitative_convergence_main}.

\begin{definition}[Subexponential LLN concentration terms]\label{def:lln_terms_RW}
    Consider the solution \((\bfH, \bfB)\) of \eqref{eq:LimH}-\eqref{eq:LimB}. A \textit{subexponential LLN concentration term}, denoted by $\llnW_k(s)$ for $s\in[0,1]$, is a quantity of the form
    \[\llnW_k(s) = \bigg\|\frac{1}{D}\sum_{d=1}^D X^d_k(s)\bigg\|_{\H},\]
    where $(X^d_k(s))_{d\geq 1}$ is an i.i.d.~sequence of centered, subexponential and $\Finit$-measurable random variables defined on a finite-dimensional vector space $\H$ 
    such that $\|\sup_{s\in[0,1]} \| X^1_k(s)\|_{\H} \|_{\psi_1} \leq c\vpk^2$, for some absolute constant $c>0$ and $\vpk^2$ being the centered variance proxy of $\|(\bfH_{\wedge k-1}, \bfB_{\wedge k-1})\|_{\infty}$.
    For brevity, we also absorb finite sums of such terms into the notation $\llnW_k(s)$.
\end{definition}

\begin{rem}[Examples of subexponential LLN concentration terms]\label{rem:examples_llnW_terms}
    Using the subexponential property from Proposition~\ref{prop:skeleton_subgaussianity}, the following are  subexponential LLN concentration terms:
    \[\begin{cases}
        \llnW_k(s) := \big\| \frac{1}{D}\sum_{d=1}^{D}  \bfH_{\wedge k-1}^d(s) H_k^d(s) - \E[\bfH_{\wedge k-1}^1(s) H_k^1(s)]\big\|_{2,k},\\
        \llnW_k(s) := \big\| \frac{1}{D}\sum_{d=1}^{D}  \bfH_{\wedge k-1}^d(s) \bfH_{\wedge k-1}^d(s)^{\top} - \E[\bfH_{\wedge k-1}^1(s) \bfH_{\wedge k-1}^1(s)^{\top}]\big\|_{2,k\times k},\\
        \llnW_k(s) := \big\| \frac{1}{D}\sum_{d=1}^{D}  \Win \Wout^{\top} - \E[\Win \Wout^{\top}]\big\|_{2,\dimin\times\dimout},
    \end{cases}\]
    as well as many other variants involving $\bfH$, $\bfB$, $\Win$ and $\Wout$.
\end{rem}

We have the following concentration result for such terms.
\begin{proposition}\label{prop:bernstein_control_llnW}
    Let $\llnW_k(s)$ be a subexponential LLN concentration term as in Definition~\ref{def:lln_terms_RW}.
    For some absolute constant $c>0$, it holds that for any $k\in[0:K-1]$, $s\in[0,1]$, $D\geq 1$ and $\delta\in(0,1)$,
    \[\P\Bigg(\llnW_k(s) \leq c\vpk^2\bigg(\frac{\log(2/\delta)}{D} + \sqrt{\frac{\log(2/\delta)}{D}} \bigg) \Bigg) \geq 1-\delta,\]
    where $\vpk^2$ is the variance proxy of $\|(\bfH_{\wedge k-1}, \bfB_{\wedge k-1})\|_{\infty}$.
    More importantly,
    \[\P\Bigg(\int_0^1\llnW_k(s)\d s \leq c\vpk^2\bigg(\frac{\log(2/\delta)}{D} + \sqrt{\frac{\log(2/\delta)}{D}} \bigg) \Bigg) \geq 1-\delta.\]
    In particular,
    \[\P\Bigg(\max_{\substack{j\in[0:k]}}\int_0^1\llnW_j(s)\d s \leq c\vpk^2\bigg(\frac{\log(2k/\delta)}{D} + \sqrt{\frac{\log(2k/\delta)}{D}} \bigg) \Bigg) \geq 1-\delta.\]
\end{proposition}
\begin{proof}
    The first result follows from the usual Bernstein concentration for empirical means of sequences of i.i.d.~subexponential random variables \cite[Theorem 2.9.1]{vershynin2018highdimensionalproba}.  

    For the integral control, we recall from Definition~\ref{def:lln_terms_RW} that a subexponential LLN concentration term $\llnW_k(s)$ is of the form
    \(\llnW_k(s) = \|\frac{1}{D}\sum_{d=1}^D X^d_k(s) \|_{\H}\)
    for $(X^d_k(s))_{d\geq 1}$ i.i.d.~centered subexponential, $\Finit$-measurable random variables, with $\|\sup_{s\in[0,1]} \| X^1_k(s)\|_{\H} \|_{\psi_1} \leq c\vpk^2$.
    To obtain the desired result, we use Pinelis' Theorem \cite[Theorem 3.3]{pinelis1994BanachConcentration}, applied on the Hilbert space $L^2([0,1], \H)$. 
    We can see that $(X_k^d)_{d\geq 1} :=(s\mapsto X_k^d(s))_{d\geq 1}$ is an i.i.d.~sequence on this space, which satisfies \[\|X_k^1(s)\|^2_{L^2([0,1], \H)} = \int_0^1 \|X_k^1(s)\|_{\H}^2 ds \leq \sup_{s\in[0,1]} \|X_k^1(s)\|_{\H}^2.\]
    Thus, for any $p\geq 2$,
    \[\E[\|X_k^1(s)\|_{L^2([0,1],\H)}^p]\leq \E\Big[\sup_{s\in[0,1]} \|X_k^1(s)\|_{\H}^p\Big] \leq p! \Big\| \sup_{s\in[0,1]} \|X_k^1(s)\|_{\H}\Big\|_{\psi_1}^p \leq p!c^p\vpk^{2p},\]
    and so we can directly apply the Bernstein concentration from Pinelis' Theorem to obtain, for a potentially different absolute constant $c>0$,
    \[\P\Bigg(\bigg\| \frac{1}{D}\sum_{d=1}^D X^d_k\bigg\|_{L^2([0,1], \H)}\leq c\vpk^2\bigg(\frac{\log(2/\delta)}{D} + \sqrt{\frac{\log(2/\delta)}{D}} \bigg) \Bigg) \geq 1-\delta.\]
    In particular, since
    \[\int_0^1\llnW_k(s)\d s \leq \bigg(\int_0^1|\llnW_k(s)|^2\d s \bigg)^{1/2} = \bigg(\int_0^1\Big\|\frac{1}{D}\sum_{d=1}^D X^d_k(s)\Big\|^2_\H\d s \bigg)^{1/2}= \bigg\| \frac{1}{D}\sum_{d=1}^D X^d_k\bigg\|_{L^2([0,1], \H)},\]
    we finally obtain
    \[\P\Bigg(\int_0^1\llnW_k(s)\d s \leq c\vpk^2\bigg(\frac{\log(2/\delta)}{D} + \sqrt{\frac{\log(2/\delta)}{D}} \bigg) \Bigg) \geq 1-\delta.\]
    The control for \(\max_{j\in[0:k-1]} \int_0^1\llnW_j(s)\d s\) is obtained from performing a classical union bound.
\end{proof}

We eventually state the following generic exponential concentration bound.
\begin{lemma}\label{lem:subweibull_concetration}
    Let $p\geq 2$. For some absolute constant $c>0$, it holds that, for any $\delta \in (0,1)$, 
    \[
\P\Bigg( \bigg| \frac{1}{D}\sum_{d=1}^D \|\Win^d\|_{2,\dimin}^p -\E[\|\Win^1\|_{2,\dimin}^p] \bigg| \leq  c \|\|\Win\|_{2, \dimin}\|_{\psi_2}^p \bigg( \sqrt{\frac{\log(2/\delta)}{D}} \vee \Big( \frac{\log(2/\delta)}{D} \Big)^{p/2} \bigg) \Bigg) \geq 1 - \delta,
\] and similarly for $\Wout$.
\end{lemma}
\begin{proof}
    Let $p\geq 2$.
    Using standard arguments, $\|\Win\|_{2,\dimin}$ and $\|\Wout\|_{2,\dimout}$ are subgaussian random variables, from where $\|\Win\|_{2,\dimin}^p$ and $\|\Wout\|_{2,\dimout}^p$ are $(2/p)$-subweibull random variables \cite[Definition 2.2]{kuchibhotla2022subweibull}. It follows that \(\frac{1}{D}\sum_{d=1}^D \|\Win^d\|_{2,\dimin}^p -\E[\|\Win^1\|_{2,\dimin}^p]\) is an empirical average of centered i.i.d.~subweibull terms. By a usual concentration result \cite[Theorem 3.1]{kuchibhotla2022subweibull}, we obtain
    \[
\P\Bigg( \bigg| \frac{1}{D}\sum_{d=1}^D \|\Win^d\|_{2,\dimin}^p -\E[\|\Win^1\|_{2,\dimin}^p] \bigg| \leq  c \|\|\Win\|_{2, \dimin}\|_{\psi_2}^p \bigg( \sqrt{\frac{\log(2/\delta)}{D}} \vee \Big( \frac{\log(2/\delta)}{D} \Big)^{p/2} \bigg) \Bigg) \geq 1 - \delta,
\] and similarly for $\Wout$.
\end{proof}

\subsubsection{Propagation of chaos for skeleton maps}

The following result controls propagation of chaos for the skeleton maps in terms of the one for $\bfh^{(D)}$ and $\bfb^{(D)}$.

\begin{proposition}[Propagation of chaos for skeleton maps]\label{prop:propcaos_Dbar_whp}
    Let $D\geq 1$. Let $\mathbf{f}^{(D)}, \mathbf{g}^{(D)}$ be the skeleton maps associated to $\bfh^{(D)}, \bfb^{(D)}$ and $\mathbf{F}^{(D)}$, $\mathbf{G}^{(D)}$ those associated to $\bfH^{(D)}$, $\bfB^{(D)}$, as introduced previously. There exists a random constant $\cstW_k$ that only depends on $k$, $\boldsymbol{\Gamma}_{\wedge k}^H$, $\boldsymbol{\Gamma}_{\wedge k}^B$, $\|W_{\mathrm{in}}^{(D)}\|_{\ndbar}\vee\|W_{\mathrm{out}}^{(D)}\|_{\ndbar}$, such that for every $s\in[0,1]$ and $\bfz^h,\bfz^b \in \R^K$ with $\max_{j\in[0:k-1]} \vert z^h_j \vert \vee \vert z^b_j \vert \leq M$, for some $M \geq 1$,
\begin{multline*}
\max_{\substack{j\in[0:k]}} \|f^{(D)}_j (s , \bfz^h, \bfz^b) - F^{(D)}_j (s , \bfz^h, \bfz^b) \|_{\ndbar}\vee \| g^{(D)}_j(s , \bfz^h, \bfz^b) - G^{(D)}_j (s , \bfz^h, \bfz^b) \|_{\ndbar} \\
\leq e^{\cstW_k M} \Bigg( \max_{\substack{j\in[0:k-1]}} \| h^{(D)}_j (s ) - H^{(D)}_j (s ) \|_{\ndbar} \vee \| b^{(D)}_j (s ) - B^{(D)}_j (s ) \|_{\ndbar} +  \max_{\substack{j\in[0:k-1]}}\llnW_j(s) \Bigg)\,\text{a.s.} 
\end{multline*} 
\end{proposition}

\begin{proof}
        To ease notation, we drop the $(D)$ superscript. Recalling Definitions~\ref{def:finite_skeleton_maps} and \ref{def:mf_skeleton_maps} we have, underlying evaluation at \((s, \bfz^h, \bfz^b )\), 
    \begin{align*}
        \|f_{k+1} - F_{k+1}\|_{\ndbar} &\leq  \|f_{k} - F_{k}\|_{\ndbar}  + \eta_u \big\| \rho'(z^h_k + \langle h_k, f_k\rangle_{\ndbar}) \left( z^b_k + \langle b_k, g_k\rangle_{\ndbar} \right) h_k \\ &\qquad\qquad\qquad\qquad\qquad\qquad\qquad\qquad -  \rho' ( z^h_k + \E [ H^1_k F_k^1 ] ) ( z^b_k + \E [ B_k^1 G_k^1 ]) H_k \big\|_{\ndbar} \,,\\
        \|g_{k+1} - G_{k+1}\|_{\ndbar} &\leq  \|g_{k} - G_{k}\|_{\ndbar}  + \eta_v \big\| \rho(z^h_k + \langle h_k, f_k\rangle_{\ndbar}) b_k -   \rho ( z^h_k + \E [ H^1_k F_k^1 ] ) B_k \big\|_{\ndbar}.
    \end{align*}
    where we replaced $\E[H^d_kF^d_k] = \E[H^1_kF^1_k]$ and $\E[B^d_kG^d_k] = \E[B^1_kG^1_k]$, since the variables are i.i.d.~
    We separate the right-most term in both equations, using the linear growth of $\rho$ and the boundedness of $\rho'$ and $\rho''$, to obtain
\begin{align*}
    \|f_{k+1} - F_{k+1}\|_{\ndbar} &\leq  \|f_{k} - F_{k}\|_{\ndbar}  + \eta_u C_\rho\big[ |\langle h_k, f_k\rangle_{\ndbar} - \E[H^1_kF^1_k]| |z^b_k + \E[B_k^1G_k^1]| \|h_k\|_{\ndbar }\\ &\qquad\qquad\qquad\qquad+\| h_k\|_{\ndbar} |\langle b_k, g_k\rangle_{\ndbar}-\E [ B_k^1 G_k^1 ]| + |z^b_k + \E[B_k^1G_k^1]|\|h_k -H_k \|_{\ndbar} \big], \\
    \|g_{k+1}-G_{k+1}\|_{\ndbar} &\leq  \|g_{k} - G_{k}\|_{\ndbar}  + \eta_v C_\rho\big[ |\langle h_k, f_k\rangle_{\ndbar} - \E[H^1_kF^1_k]| \|b_k\|_{\ndbar }\\ &\qquad\qquad\qquad\qquad +(1+|z^h_k + \E[H_k^1 F^1_k]|)\|b_k -B_k \|_{\ndbar} \big].
\end{align*}
Using Lemma~\ref{lem:RegSkel}, we bound all the terms of the form
\[ |z^b_k + \E[B_k^1G_k^1]| \leq M+\E^{1/2}[B_k^2]\E^{1/2}[G_k^2] \leq M [1+M_k\cstH_k] =: \cstH_k M, \]
where \(\cstH_k\) is a constant that depends on the desired variables and may change from line to line.
We now study 
\begin{align*}
    \bigg\vert \frac{1}{D}\sum_{d=1}^D h_k^df_k^d - \E[H_k^d F_k^d ] \bigg\vert \leq \underbrace{\bigg| \frac{1}{D}\sum_{d=1}^D (h_k^df_k^d -  H_k^dF_k^d) \bigg|}_{\text{(I)}} + \underbrace{\bigg|\frac{1}{D}\sum_{d=1}^D H_k^dF_k^d - \E[H_k^d F_k^d ] \bigg|}_{\text{(II)}}.
\end{align*}
From \eqref{eq:cleaner_system_Fk_definition} and Lemma~\ref{lem:explicit_regskel}, we have that
\begin{multline*}
    \text{(II)} = \eta_u \Big|\vec{\mathbf{v}}_k^F(s, \bfz^h, \bfz^b)^\top \Big(\frac{1}{D}\sum_{d=1}^D H_k^d\mathbf{H}^d_{\wedge k-1}(s) - \E[H_k^d \bfH_{\wedge k-1}^d ] \Big) \Big| \\
    \leq \eta_u \big\|\vec{\mathbf{v}}_k^F(s, \bfz^h, \bfz^b)\big\|_{2,k} \Big\|\frac{1}{D}\sum_{d=1}^D H_k^d\mathbf{H}^d_{\wedge k-1}(s) - \E[H_k^d \bfH_{\wedge k-1}^d(s) ] \Big\|_{2,k} \leq \cstH_k M \llnW_k(s),
\end{multline*}
where in the last inequality we identified a subexponential LLN concentration term as in Definition~\ref{def:lln_terms_RW}.
Using C-S, we further bound
\begin{align*}
    \text{(I)}&\leq  \|h_k^{(D)}\|_{\ndbar}\|f_k^{(D)} - F_k^{(D)}\|_{\ndbar} + \|h_k^{(D)}-  H_k^{(D)}\|_{\ndbar} \|F_k^{(D)}\|_{\ndbar}.
\end{align*}
Combining this with the analogous bound for $|\langle b_k, g_k\rangle_{\ndbar} - \E[B^1_kG^1_k]|$ yields
\begin{multline*}
    \|f_{k+1} - F_{k+1}\|_{\ndbar} \leq  \|f_{k} - F_{k}\|_{\ndbar}  + \eta_u \cstH_kM \{ \|h_k\|_{\ndbar }[ \|h_k\|_{\ndbar}\|f_k - F_k\|_{\ndbar} + \|F_k\|_{\ndbar}\|h_k - H_k\|_{\ndbar} ] \\
    + \cstH_k M\llnW_k +\|h_k\|_{\ndbar }[ \|b_k\|_{\ndbar}\|g_k - G_k\|_{\ndbar} + \|G_k\|_{\ndbar}\|b_k - B_k\|_{\ndbar} + \cstH_k M\llnW_k ] +\|h_k -H_k \|_{\ndbar} \},
\end{multline*}
and similarly for $g$. From Lemma~\ref{lem:embeddingcontrol}, it holds that $m_k:=\sup_{s\in[0,1]}\|h_k^{(D)}(s)\|_{\ndbar} \vee \|b_k^{(D)}(s)\|_{\ndbar} \leq \cstW_k$ a.s. for a random constant $\cstW_k$ that depends on $\|W_{\mathrm{in}}^{(D)}\|_{\ndbar}\vee\|W_{\mathrm{out}}^{(D)}\|_{\ndbar}$ and may change from line to line. 
Similarly, from Lemma~\ref{lem:MF_skeleton_bounded_ndbar} we have \(\|F_k\|_{\ndbar}\vee \|G_k\|_{\ndbar} \leq \cstH_k\cstW_kM\) a.s., so that
\begin{multline}\label{eq:as_propcaos_gronwall}
    \|f_{k+1} - F_{k+1}\|_{\ndbar} \leq  \|f_{k} - F_{k}\|_{\ndbar}  + \eta_u \cstH_k\cstW_kM \big[ \|f_k - F_k\|_{\ndbar} + \|g_k - G_k\|_{\ndbar} \big] \\
    + \eta_u \cstH_k\cstW_ke^{\cstH_k M} \big[ \|h_k - H_k\|_{\ndbar} +\|b_k - B_k\|_{\ndbar} + \llnW_k \big]\;\; \text{a.s.},
\end{multline}
and further
\[\|f_{k+1} - F_{k+1}\|_{\ndbar}\leq v_k(s) [ 1  + 2\eta_u \cstH_k\cstW_kM ] 
    + 2\eta_u \cstH_k\cstW_ke^{\cstH_k M} [ x_k(s)+\llnW_k ]\;\; \text{a.s.}, \]
where we have introduced
\[v_k(s) := \|f_{k} - F_{k}\|_{\ndbar}\vee\|g_{k} - G_{k}\|_{\ndbar} \;\text{ and }\; x_k(s) := \|h_{k} - H_{k}\|_{\ndbar}\vee\|b_{k} - B_{k}\|_{\ndbar}.\]
Following the same steps, we obtain an analogous recurrence relation for $\|g_{k} - G_{k}\|_{\ndbar}$, which in turn yields
\[v_{k+1}(s)\leq v_k(s) [ 1  + (\eta_u\vee \eta_v) \cstW_kM ] 
    + (\eta_u\vee \eta_v) e^{\cstW_k M} \big[ \|x_{\wedge k}\|_{\infty}(s)+\max_{j\in[0:k]}\llnW_j(s)\big] \;\; \text{a.s.}, \]
where we absorbed constant terms into $\cstW_k$ and we have set $\|x_{\wedge k}\|_{\infty}(s) := \max_{j\in[0:k]} x_j(s)$. Applying the discrete Grönwall inequality, since $v_0(s) \equiv 0$, we obtain for all $s\in[0,1]$ and $k\in[0:K-1]$,
\begin{multline*}
    v_k(s) \leq \frac{(\eta_u\vee \eta_v) e^{\cstW_k M} \big[ \|x_{\wedge k-1}\|_{\infty}(s)+\max_{j\in[0:k-1]}\llnW_j(s)\big]}{(\eta_u\vee \eta_v) \cstW_kM}\big[ \exp(k(\eta_u\vee \eta_v) \cstW_kM) -1\big] \\
    \leq e^{\cstW_k M} \big[ \|x_{\wedge k-1}\|_{\infty}(s)+\max_{j\in[0:k-1]}\llnW_j(s) \big] \;\;\text{a.s.},
\end{multline*}
which gives the desired bound.
\end{proof}
\subsubsection{Decorrelation lemmas}
The convergence proof will rely on the following elementary but crucial lemma, which corresponds verbatim to \cite[Lemma 6.1]{chizat2025hidden}:
\begin{lemma}[Decorrelation Lemma from \cite{chizat2025hidden}]\label{lem:decorrelation_chizat}
Let $(X,Y)\in \mathbb{R}^D \times \mathbb{R}$ be a pair of random variables with  
$X \in L^q$ and $Y \in L^2$ for some $q \ge 2$.
\begin{enumerate}[label=(\roman*),ref=(\roman*)]
    \item\label{lemdecchizat:nonindep} Then:
    \( \| \mathbb{E}[XY] \|_{q,D} \;\le\; \E^{1/q}[ \|X\|_{q,D}^q ] \cdot \E^{1/2}[|Y|^2]. \)

    \item\label{lemdecchizat:indep} If moreover the coordinates of $X$ are centered independent with
    variance bounded by $\sigma^2$ -- but not necessarily independent of $Y$ --, then
    \[   \| \mathbb{E}[XY] \|_{q,D} \;\le\; \sigma \, \E^{1/2}[|Y|^2].
    \]
\end{enumerate}

In particular,
\( \| \mathbb{E}[XY] \|_{\ndbar} \;\le\; \sqrt{\sigma D}\, \E^{1/2}[|Y|^2]\).
\end{lemma}
The above result still works with conditional expectations.

\begin{lemma}[Higher Order Decorrelation Lemma]\label{lem:fine_decorrelation_lemma}
    Let $X$ be a random vector in $\R^D$, independent from $\Finit$, with centered and i.i.d.~coordinates, variance bounded by $\sigma^2$ and fourth moment $\mu_4$. Let $Y$ be a random variable (possibly correlated to $X$) such that $|Y| \leq \sum_{i=1}^r|Z^{(1)}_i||Z^{(2)}_i|$ for some random variables $(Z^{(1)}_i, Z^{(2)}_i)_{i=1}^r$ -- not necessarily independent to $X$. Then:
    \[\|\E[XY|\Finit]\|_{2,D} \leq (\mu_4 + 6\sigma^4)^{1/4} \sum_{i=1}^r\E^{1/2}[(Z^{(1)}_i)^2|\Finit] \E^{1/4}[(Z^{(2)}_i)^4|\Finit]. \]
\end{lemma}
\begin{proof}
We use the dual definition of the $2$-norm and develop from there:
    \begin{align*}
        \|\E[XY|\Finit] \|_2 &= \sup_{\|a\|_2\leq 1} \E[Y a^\top X|\Finit] \\&\leq \sup_{\|a\|_2\leq 1} \E[|Y|| a^\top X||\Finit]
        \\&\leq \sup_{\|a\|_2\leq 1} \sum_{i=1}^r\E[|Z^{(1)}_i||Z^{(2)}_i|| a^\top X||\Finit]\\
        &\leq \sup_{\|a\|_2\leq 1} \sum_{i=1}^r\E^{1/2}[|Z^{(1)}_i|^2|\Finit]\E^{1/4}[|Z^{(2)}_i|^4|\Finit] \E^{1/4}[| a^\top X|^4|\Finit]\\
        &\leq \bigg[\sup_{\|a\|_2\leq 1} \E^{1/4}[| a^\top X|^4|\Finit] \bigg]\sum_{i=1}^r\E^{1/2}[|Z^{(1)}_i|^2|\Finit]\E^{1/4}[|Z^{(2)}_i|^4|\Finit], 
    \end{align*}
    where we have used the hypothesis on $Y$, and the Cauchy-Schwarz inequality twice. To conclude, it suffices to note that \(\E[| a^\top X|^4|\Finit] = \E[| a^\top X|^4]\) due to independence, and that
    \begin{align*}
        \E[| a^\top X|^4] &=\sum_{d_1, d_2, d_3, d_4} a_{d_1}a_{d_2}a_{d_3}a_{d_4} \E[X_{d_1}X_{d_2}X_{d_3}X_{d_4}]\\
        &= \mu_4 \sum_{d=1}^D a_d^4 + 6 \sigma^4\Big( \sum_{d=1}^D a_d^2\Big)^2\\
        &\leq \left(\mu_4  + 6 \sigma^4\right) \|a\|_2^4,
    \end{align*}
    where we used that \( \sum_{d=1}^D a_d^4 \leq (\sum_{d=1}^D a_d^2)^2 = \|a\|_2^4 \). In conclusion,
    \begin{align*}
        \|\E[XY|\Finit] \|_2 
        &\leq \left(\sup_{\|a\|_2\leq 1} \left(\mu_4  + 6 \sigma^4\right)^{1/4} \|a\|_2 \right)\sum_{i=1}^r\E^{1/2}[|Z^{(1)}_i|^2|\Finit]\E^{1/4}[|Z^{(2)}_i|^4|\Finit]\\
        &\leq \left(\mu_4  + 6 \sigma^4\right)^{1/4}\sum_{i=1}^r\E^{1/2}[|Z^{(1)}_i|^2|\Finit]\E^{1/4}[|Z^{(2)}_i|^4|\Finit],
    \end{align*}
    as desired.
\end{proof}
\subsection{Proofs of Theorems~\ref{thm:quantitative_convergence_main} and \ref{thm:main_theorem_intro}}\label{sec:proof_of_thms_quant_cvg}

To prove Theorem~\ref{thm:quantitative_convergence_main}, we first establish the following intermediate result.

\begin{theorem}[Subexponential chaos]\label{thm:whp_recurrence relation}
Let $D\geq 1$. For all $k \in [0:K-1]$, it holds that
\[ \sup_{s\in[0,1]}\| h^{(D)}_k (s) - H^{(D)}_k (s) \|_{\ndbar} \leq \cstWh_k\Big(\frac{1}{\sqrt{D}} + \sum_{j=0}^k \int_0^1\llnW_j(s) \d s\Big)\;\;\text{a.s.}, \]
\[ \sup_{s\in[0,1]}\| b^{(D)}_k (s) - B^{(D)}_k (s) \|_{\ndbar} \leq \cstWh_k\Big(\frac{1}{\sqrt{D}} + \sum_{j=0}^k \int_0^1\llnW_j(s) \d s \Big)\;\;\text{a.s.} \]
where $\llnW_j(s)$ are subexponential LLN concentration terms as in Definition~\ref{def:lln_terms_RW}.
\end{theorem}
We first show that Theorem~\ref{thm:quantitative_convergence_main} follows from Theorem~\ref{thm:whp_recurrence relation}.
\begin{proof}[Proof of Theorem~\ref{thm:quantitative_convergence_main}]
Let $D\geq 1$. From Theorem~\ref{thm:whp_recurrence relation}, we get that
\[\max_{\substack{s\in[0,1]\\k\in[0:K-1]}}\| h^{(D)}_k (s) - H^{(D)}_k (s) \|_{\ndbar}\vee \| b^{(D)}_k (s) - B^{(D)}_k (s) \|_{\ndbar} \leq \cstWh_K\Big(\frac{1}{\sqrt{D}} + \sum_{j=0}^{K-1} \int_0^1\llnW_j(s) \d s\Big)\;\;\text{a.s.}\]
Let $\delta \in(0,1)$. From Proposition~\ref{prop:bernstein_control_llnW}, we know that
    \begin{equation}\label{eq:final_thm_whp_concentration_llnW}
    \P\Bigg(\sum_{j=0}^{K-1}\int_0^1\llnW_j(s)\d s \leq K\cstH_K\sigma_K^2\bigg(\frac{\log(4K/\delta)}{D} + \sqrt{\frac{\log(4K/\delta)}{D}} \bigg) \Bigg) \geq 1-\delta/2.\end{equation}
On the other hand, we need to prove concentration for the terms
    \[M_{\mathrm{in},p} :=\frac{1}{D}\sum_{d=1}^D \|\Win^d\|_{2,\dimin}^p \;\text{and}\;\;M_{\mathrm{out},p} := \frac{1}{D}\sum_{d=1}^D \|\Wout^d\|_{2,\dimout}^p, \quad p\in\{2,3,4\}.\]
    From Lemma~\ref{lem:subweibull_concetration}, 
    \[
\P\Bigg( \big| M_{\mathrm{in},p}  \big| \leq \E[\|\Win\|_{2,\dimin}^p] +C \|\|\Win\|_{2, \dimin}\|_{\psi_2}^p \bigg( \sqrt{\frac{\log(2/\delta)}{D}} \vee \Big( \frac{\log(2/\delta)}{D} \Big)^{p/2} \bigg) \Bigg) \geq 1 - \delta,
\]
and similarly for \(M_{\mathrm{out},p}\). Using a union bound, we get with probability higher than $1-\delta/2$ that for any $D\geq 1$,
\begin{multline*}
    \max_{p\in\{2,3,4\}} \bigg(\frac{1}{D}\sum_{d=1}^D \|\Win^d\|_{2,\dimin}^p\bigg)\vee\bigg(\frac{1}{D}\sum_{d=1}^D \|\Wout^d\|_{2,\dimout}^p\bigg) \\\leq \max_{p\in\{2,3,4\}} (\vpin\vee\vpout)^p\bigg(1+\sqrt{\frac{\log(12/\delta)}{D}} \vee \max_{p\in\{2,3,4\}}\Big( \frac{\log(12/\delta)}{D} \Big)^{p/2}\bigg).
\end{multline*}
In particular, if $\delta \geq e^{-D}$ (i.e. $ \frac{\log(12/\delta)}{D} = \frac{\log(12)}{D} + \frac{\log(1/\delta)}{D}\leq \log(12) + 1$), it holds that, for some absolute constant $c>0$,
\begin{equation}\label{eq:final_thm_whp_concentration_normpW}
    \max_{p\in\{2,3,4\}} \bigg(\frac{1}{D}\sum_{d=1}^D \|\Win^d\|_{2,\dimin}^p\bigg)\vee\bigg(\frac{1}{D}\sum_{d=1}^D \|\Wout^d\|_{2,\dimout}^p\bigg) \leq c\max_{p\in\{2,3,4\}} (\vpin\vee\vpout)^p.
\end{equation}
Injecting \eqref{eq:final_thm_whp_concentration_normpW} into the constant \(\cstWh_K\), we obtain a fully deterministic constant $\cstH_K$. Furthermore, if $(\Win,\Wout)$ are a.s. bounded, then \eqref{eq:final_thm_whp_concentration_normpW} holds for all $\delta \in(0,1)$.

Combining this into an union bound with \eqref{eq:final_thm_whp_concentration_llnW}, we obtain that with probability higher than $1-\delta$, it holds
\[ \sup_{\substack{s\in[0,1]\\ k\in[0:K-1]}}\| h^{(D)}_k (s) - H^{(D)}_k (s) \|_{\ndbar}\vee \| b^{(D)}_k (s) - B^{(D)}_k (s) \|_{\ndbar} \leq \cstH_K\bigg(\frac{1 + \log(4K/\delta)}{\sqrt{D}}\bigg), \]
where we have used that $\log(4K/\delta) \geq 1$ to keep only the leading term \(\log(4K/\delta)\geq\sqrt{\log(4K/\delta)}\).
We can further absorb \(\log(4K) \geq 1\) into $\cstH_K$ to conclude.

To obtain the claim on the output $\hat{y}_k$ as a consequence of Theorem~\ref{thm:whp_recurrence relation}, we follow a direct computation that we do not write here to avoid repetitions, as it exactly corresponds to \eqref{eq:proof_cvg_lemB_init} below in the proof of Theorem \ref{thm:whp_recurrence relation}.
This yields
\[\max_{k\in[0:K-1]}\|\hat{y}_k^D - y_k\|_{2, \dimout} \leq
        \cstWh_k\Big(\frac{1}{\sqrt{D}} + \sum_{j=0}^{K-1}\int_0^1\llnW_j(u)\d u \Big).\] From \eqref{eq:final_thm_whp_concentration_llnW} and \eqref{eq:final_thm_whp_concentration_normpW} it follows that \(\max_{k\in[0:K-1]}\|\hat{y}_k^D - y_k\|_{2, \dimout} \leq \cstH_K \frac{1 + \log(1/\delta)}{\sqrt{D}}\) with probability at least $1-\delta$, as desired.
\end{proof}

We now prove Theorem~\ref{thm:main_theorem_intro} using Theorem~\ref{thm:quantitative_convergence_main}.
\begin{proof}[Proof of Theorem~\ref{thm:main_theorem_intro}]
    Let $K\geq 0$. Consider the finite ResNet given by \eqref{eq:fully_finite_resnet} with preactivations $\th_k^{(\ell)}(x_i)$ and output \(\mathtt{y}_k(x_i)\); as well as the finite-$D$ Mean ODE given by \eqref{eq:original_finite_D_system}, with preactivations \(h_k^{(D)}(s, x_i)\) and output $\hat{y}^{D}_k(x_i)$. Let
    \[\begin{cases}
     \tilde{\Delta}_k^y \coloneqq \max_{i\in [1:n]} \Vert \mathtt{y}_k(x_i)-\hat{y}^{D}_k(x_i)\Vert_{2, \dimout},\\
     \tilde{\Delta}_k^h \coloneqq \max_{i\in [1:n]}\max_{\ell\in [1:L]} \Vert \th_k^{(\ell)}(x_i) - h_k^{(D)}(\ell/L,x_i) \Vert_{\ndbar}.
    \end{cases}\] 
    Let $R := c(\vpin\vee\vpout)^2/(C_0 \wedge\sqrt{C_0})$, for $c>0$ an absolute constant as in \eqref{eq:final_thm_whp_concentration_normpW}. If we suppose that $\|W_{\mathrm{in}}^D\|_{\ndbar} \vee \|W_{\mathrm{out}}^D\|_{\ndbar} \leq R$, the assumptions of Theorem~\ref{thm:main_theorem_intro} allow us to apply \cite[Theorem 3]{chizat2025hidden} to obtain constants $c_1^R, c_2^R >0$ that only depend on $k$, $n$ and $R$ such that, for any $\delta \in (e^{-M}, 1)$, if $c_1^R\Big( \frac{1}{L} + \frac{\sqrt{D} + \log(1/\delta)}{\sqrt{LM}}\Big) \leq c_2^R$, then
    \[\P\Big(\max_{k\leq K}\tilde{\Delta}_k^y \vee \tilde{\Delta}_k^h \leq c_1^R\Big( \frac{1}{L} + \frac{\sqrt{D} + \log(1/\delta)}{\sqrt{LM}}\Big)|\Finit\Big) \geq 1-\delta/4.\]
    We consider the events $E = \{\max_{k\leq K}\tilde{\Delta}_k^y \vee \tilde{\Delta}_k^h \leq c_1^R\Big( \frac{1}{L} + \frac{\sqrt{D} + \log(1/\delta)}{\sqrt{LM}}\Big)\}$ and $A =\{\|W_{\mathrm{in}}^D\|_{\ndbar} \vee \|W_{\mathrm{out}}^D\|_{\ndbar} \leq R \}$. We notice that
    \[\P(E) 
    = \E[\mathbbm{1}_{E} \mathbbm{1}_{A}] + \E[\mathbbm{1}_{E} \mathbbm{1}_{A^c}] \geq \E[\mathbbm{1}_{E} \mathbbm{1}_{A}] +0 = \E[\E[\mathbbm{1}_{E} \mathbbm{1}_{A}|\Finit]] = \E[\mathbbm{1}_{A}\P(E|\Finit)],\]
    where we used the fact that $A$ is $\Finit$ measurable.
    Let $\delta \geq e^{-D/C_0}$ ; 
    since $D \leq  C_0 M$, we have $\delta \geq e^{-M}$. 
    Under the event $A$, we can apply \cite[Theorem 3]{chizat2025hidden} and obtain $\P(E|\Finit) \geq 1-\delta/4$, so that \(\P(E) \geq \P(A)(1-\delta/4)\). 
    Since $\delta \geq e^{-D/C_0}$, we have $\frac{\log(1/\delta)}{D} \leq \frac{1}{C_0}$ and, as in the proof of Theorem~\ref{thm:quantitative_convergence_main}, we use Lemma~\ref{lem:subweibull_concetration} to prove that $\P(A) = \P(\|W_{\mathrm{in}}^D\|_{\ndbar} \vee \|W_{\mathrm{out}}^D\|_{\ndbar} \leq R) \geq 1-\delta/4$.
    Hence, \(\P(E) \geq (1-\delta/4)^2 = 1-2\frac{\delta}{4} + \frac{\delta^2}{16}\geq 1-\delta/2 + 0\). 
    That is, we have $c_1^R, c_2^R >0$ that only depend on $k$ and $R$ such that, for any $\delta \in (e^{-D/C_0}, 1)$, if $c_1^R\Big( \frac{1}{L} + \frac{\sqrt{D} + \log(1/\delta)}{\sqrt{LM}}\Big) \leq c_2^R$, 
    \begin{equation}\label{eq:proof_main_thm_intro_lenaic_part}
      \P\Big(\max_{k\leq K}\tilde{\Delta}_k^y \vee \tilde{\Delta}_k^h \leq c_1^R\Big( \frac{1}{L} + \frac{\sqrt{D} + \log(1/\delta)}{\sqrt{LM}}\Big)\Big) \geq 1-\delta/2.  
    \end{equation}
    For appropriately coupled i.i.d.~copies of the limit system, with preactivations $H_k^{(D)}(s,x)$ and output $y_k$, let
    \[\begin{cases}
     \hat{\Delta}_k^y \coloneqq \max_{i\in [1:n]} \Vert \hat{y}^{D}_k(x_i) - y_k(x_i)\Vert_{2, \dimout},\\
     \hat{\Delta}_k^h \coloneqq \max_{i\in [1:n]}\max_{\ell\in [1:L]} \Vert h_k^{(D)}(\ell/L,x_i) - H_k^{(D)}(\ell/L, x_i) \Vert_{\ndbar}.
    \end{cases}\] 
    By Theorem~\ref{thm:quantitative_convergence_main}, we have\footnote{Technically, the stament of Theorem~\ref{thm:quantitative_convergence_main} is for $\delta \in (e^{-D}, 1)$. Observing the argument in its proof, modulo modifying $\cstH_{K,N}$ using $C_0$, it holds for $\delta \in (e^{-D/C_0}, 1)$.} for $\delta \in (e^{-D/C_0}, 1)$
    \begin{equation}\label{eq:proof_main_thm_intro_new_part}
    \P\Big(\max_{k\leq K}\hat{\Delta}_k^y \vee \hat{\Delta}_k^h \leq\cstH_{K,N}\Big( \frac{1 + \log(1/\delta)}{\sqrt{D}} \Big) \Big) \geq 1-\delta/2.
    \end{equation}
    We let $c_1 = \cstH_{K,n} \vee c_1^R$ and $c_2 = (c_2^R \wedge \frac{c_1}{2C_0})$. Let $\delta\in(0,1)$. If $c_1\Big( \frac{1}{L} + \frac{\sqrt{D}}{\sqrt{LM}} + \frac{1 + \log(1/\delta)}{\sqrt{D}}\Big) \leq c_2$, then we have $c_1^R\Big( \frac{1}{L} + \frac{\sqrt{D} + \log(1/\delta)}{\sqrt{LM}}\Big) \leq c_2^R$, and also
    $c_1\frac{\log(1/\delta)}{D}\leq c_1\frac{\log(1/\delta)}{\sqrt{D}} \leq \frac{c_1}{2C_0}$, which forces $\log(1/\delta) < \frac{D}{C_0}$ and hence $\delta > e^{-D/C_0}$. From this, both \eqref{eq:proof_main_thm_intro_lenaic_part} and \eqref{eq:proof_main_thm_intro_new_part} hold.  
    A simple union bound and the triangle inequality finally yield
    \[\P\Big( \max_{k\leq K} \; \Delta_k^y \wedge \Delta_k^h \leq c_1\Big(\frac{1}{L}+\frac{\sqrt{D}}{\sqrt{ML}}+\frac{1+\log(1/\delta)}{\sqrt{D}}\Big)\Big) \geq 1-\delta,\] 
    which is the desired result. 
\end{proof}

The proof for Theorem~\ref{thm:whp_recurrence relation} is done by recurrence in the next subsections.

\subsection{Proof of Theorem~\ref{thm:whp_recurrence relation}}\label{ssec:proof_thm_recurrence_relation}
We fix $D\geq 1$. Throughout this proof, \(\cstWh_k\) represents a $\Finit$-measurable random constant that only depends on \(\|\Win^{(D)}\|_{\bar{p},D\times\dimin}, \|\Wout^{(D)}\|_{\bar{p},D\times\dimout}\) for $p\in\{2,3,4\}$ and may change from line to line. Similarly, \(\llnW_k(s)\) denotes a subexponential LLN concentration term as in Definition~\ref{def:lln_terms_RW}, which may also change from line to line.

\subsubsection{Base case}
By definition, for all $d\in[1:D]$ and $s\in[0,1]$,
$h_0^{d}(s) \equiv W_{\mathrm{in}}^d\cdot x \equiv H_0^d(s)$, which implies that $\sup_{s\in[0,1]}\| h^D_0 (s) - H^{D}_0 (s) \|_{\ndbar} = 0=O(D^{-1/2})$.
Similarly, $b_0^{(D)}(s) \equiv W_{\mathrm{out}}^{(D)}\nabla \ell(\hat{y}^D_0)$ and $B_0^{(D)}(s) \equiv W_{\mathrm{out}}^{(D)}\nabla\ell(y_0)$. 
Using that $\nabla\ell$ is Lipschitz,
\begin{align*}
        \|W_{\mathrm{out}}^{(D)} [ \nabla\ell(\hat{y}^{D}_0) -\nabla\ell(y_0) ] \|_{\ndbar}
        &\leq \| W_{\mathrm{out}}^{(D)}\|_{\ndbar}\bigg\| \frac{1}{D}\sum_{d=1}^D h_0^d(1)W_{\mathrm{out}}^d - \E[H_0^d(1)W_{\mathrm{out}}^d]\bigg\|_{2,\dimout}\\
        &\leq \| W_{\mathrm{out}}^{(D)}\|_{\ndbar}\bigg\| \frac{1}{D}\sum_{d=1}^D (W_{\mathrm{in}}^d\cdot x)W_{\mathrm{out}}^d - \E[(W_{\mathrm{in}}^d\cdot x)W_{\mathrm{out}}^d]\bigg\|_{2,\dimout}\\
        &\leq \| W_{\mathrm{out}}^{(D)}\|_{\ndbar}\bigg\| \frac{1}{D}\sum_{d=1}^D W_{\mathrm{out}}^d(W_{\mathrm{in}}^d)^\top - \E[W_{\mathrm{out}}^d(W_{\mathrm{in}}^d)^\top]\bigg\|_{2,\dimout\times \dimin} \|x\|_{2,\dimin}\\
        &=: \cstW_0 \llnW_0,
    \end{align*}
    where we denoted by \(\cstW_0\) a constant depending only on the RMS norm of the initialization, and by \(\llnW_0\) a subexponential LLN concentration term
    that, in this case, does not depend on $s$. Thus, as required, \[\sup_{s\in[0,1]}\|b_0^{(D)}(s) -B_0^{(D)}(s)\|_{\ndbar} \leq \cstW_0\Big(\frac{1}{\sqrt{D}} + \int_0^1\llnW_0(u)\d u\Big).\]
   
\subsubsection{Inductive step}
Let $k\in[1:K-1]$. By induction, assume that for all $j \in[0:k-1]$ it holds that 
\[\sup_{s\in[0,1]}\| h^D_{j} (s) - H^{D}_j (s) \|_{\ndbar} \vee \| b^D_{j} (s) - B^{D}_j (s) \|_{\ndbar}\leq \cstWh_j\Big(\frac{1}{\sqrt{D}} + \sum_{\ell=0}^{j-1} \int_0^1\llnW_\ell(s) \d s\Big)\;\;\text{a.s.},\]
which implies 
\begin{equation}\label{eq:proof_cvg_inductive_hypothesis}
    \max_{\substack{s\in[0,1]\\j\in[0:k-1]}}\| h^D_{j} (s) - H^{D}_j (s) \|_{\ndbar} \vee \| b^D_{j} (s) - B^{D}_j (s) \|_{\ndbar}\leq \cstWh_k\Big(\frac{1}{\sqrt{D}} + \sum_{j=0}^{k-1} \int_0^1\llnW_j(s) \d s\Big)\;\;\text{a.s.}
\end{equation}
We introduce the following notation for pivot terms
\begin{multline*}
\H^{d}_k ( s, \bfH^{(D)} ) := \E \big[ \rho \big( S^{\mathbf{H}^D}_k + \langle F^{(D)}_k (s, \bfS^{\mathbf{H}^D}, \bfS^{\mathbf{B}^D}  ) , H^{(D)}_k (s)\rangle_{\ndbar} \big) \sqrt{D} V^d \vert \Finit \big] \\+\E \big[ \rho \big( S^{\mathbf{H}^D}_k + \langle F^{(D)}_k (s, \bfS^{\mathbf{H}^D}, \bfS^{\mathbf{B}^D}  ) , H^{(D)}_k (s)\rangle_{\ndbar} \big)
G^d_k (s, \bfS^{\mathbf{H}^D}, \bfS^{\mathbf{B}^D} ) \vert \Finit \big]  \\
=: \hat{\H}^{d}_k( s, \bfH^{(D)} ) + \check{\H}^{d}_k( s, \bfH^{(D)} ).
\end{multline*}
We abuse notation here, since $\H^{d}_k ( s, \bfH^{(D)} )$ does not correspond to simply replacing \(\bfH^{(D)}\) in $\H^{d}_k ( s, \bfh^{(D)} )$; the difference between both will be clear from context.
We similarly define $(\hat{\overline{\H}}^d_k,\check{\overline{\H}}^d_k)$ and $(\hat{\B}_k^{d}, \check{\B}_k^{d}, \hat{\overline{\B}}_k^{d}, \check{\overline{\B}}_k^{d})$ .

\paragraph{Forward pass.} 
We start by computing, for $t \in [0,1]$,
    \[h^{(D)}_k (t) - H^{(D)}_k (t) = W_{\mathrm{in}}^{(D)} x + \int_0^t \H^{(D)}_k(s, h_k^{(D)}) \d s - W_{\mathrm{in}}^{(D)} x - \int_0^t \overline{\H}^{(D)}_k(s) \d s.\]
    Since the initial conditions coincide\footnote{In the setting of trainable embedding matrices, we would need to show that an extra term $\|(\hat{W}_e^{(D)})_k - (W_{\mathrm{in}}^{(D)})_k\|_{\ndbar}\|x\|_{2,\dimin}$ is $O(D^{-1/2})$ with high probability. See Appendix~\ref{sec:extensions_of_results} for further details.}, in order to bound \(\| h^D_k (s) - H^{D}_k (s) \|_{\ndbar}\) is bounded if we bound \(\|\H^{(D)}_k(s, \bfh^{(D)}) -\overline{\H}^{(D)}_k(s)\|_{\ndbar}\) uniformly in $s\in[0,1]$. To do so, we introduce $\H^{D}_k (s, \bfH^{(D)})$ as a pivot, and we decompose \(\H^{D}_k (s, \bfH^{(D)}) = \hat{\H}_k^{(D)}(s, \bfH^{(D)}) + \check{\H}_k^{(D)}(s, \bfH^{(D)})\) and \(\overline{\H}^{D}_k (s) = \hat{\overline{\H}}_k^{(D)}(s) + \check{\overline{\H}}_k^{(D)}(s)\). We are thus reduced to studying
    \begin{align}\label{eq:proof_cvg_recurrence_integral_decomposition_h}
        \|h^{(D)}_k (t) - H^{(D)}_k &(t)\|_{\ndbar} \leq \int_0^t\|\H^{(D)}_k(s, \bfh^{(D)}) -\H^{(D)}_k (s, \bfH^{(D)})\|_{\ndbar}\d s \\
        &+ \int_0^t\|\hat{\H}^{(D)}_k (s, \bfH^{(D)})-\hat{\overline{\H}}^{(D)}_k(s)\|_{\ndbar} \d s \notag + \int_0^t\|\check{\H}^{(D)}_k (s, \bfH^{(D)})-\check{\overline{\H}}^{(D)}_k(s)\|_{\ndbar} \d s.
    \end{align}
    To bound these quantities, we use the following lemmas.

    \begin{lemma}[Propagation of chaos -- Forward pass]\label{lem:proof_cvg_prop_caos_H} 
        For all $s\in[0,1]$,
        \begin{multline*}
            \|\H^{(D)}_k(s, h_k^{(D)}) -\H^{(D)}_k(s, H_k^{(D)})\|_{\ndbar} \\\leq \cstWh_k\Big(\| h_k^{(D)}(s) - H_k^{(D)}(s)\|_{\ndbar} +  \sum_{j=0}^{k-1} \llnW_j(s) +\frac{1}{\sqrt{D}} + \sum_{j=0}^{k-1} \int_0^1\llnW_j(u)\d u \Big). 
        \end{multline*}
    \end{lemma}

    \begin{lemma}[Quantitative CLT -- Forward pass]\label{lem:proof_cvg_TCL_H} 
        For all $s\in[0,1]$, 
        \[\|\check{\H}^{(D)}_k ( s, \bfH^{(D)} ) - \check{\overline{\H}}^{(D)}_k (s)\|_{\ndbar} \leq \cstWh_k\Big(\frac{1}{\sqrt{D}}  + \llnW_k(s)\Big).\]
    \end{lemma}

    \begin{lemma}[Cavity method -- Forward pass]\label{lem:proof_cvg_cavity_TCL_H}
    For all $s\in[0,1]$,
    \[\|\hat{\H}^{(D)}_k ( s, \bfH^{(D)} ) - \hat{\overline{\H}}^{(D)}_k (s)\|_{\ndbar}\leq \cstWh_k\Big(\frac{1}{\sqrt{D}}  + \llnW_k(s)\Big).\]
    \end{lemma}

Gathering the results from Lemmas~\ref{lem:proof_cvg_prop_caos_H}, \ref{lem:proof_cvg_TCL_H} and \ref{lem:proof_cvg_cavity_TCL_H}, and replacing them in \eqref{eq:proof_cvg_recurrence_integral_decomposition_h} yields, for $t\in[0,1]$,
\begin{multline*}
        \|h^{(D)}_k (t) - H^{(D)}_k (t)\|_{\ndbar} 
        \\\leq \cstWh_k\bigg(\frac{1}{\sqrt{D}} +\sum_{j=0}^{k-1} \int_0^1\llnW_j(u)\d u + \int_0^t\|h^{(D)}_k (s) - H^{(D)}_k (s)\|_{\ndbar}\d s + \int_0^t \sum_{j=0}^{k-1} \llnW_j(s)\d s+ \int_0^t \llnW_k(s)\d s \bigg)\\
    \leq \cstWh_k\bigg(\frac{1}{\sqrt{D}} +\sum_{j=0}^{k} \int_0^1\llnW_j(u)\d u\bigg) + \cstWh_k\int_0^t \|h^{(D)}_k (s) - H^{(D)}_k (s)\|_{\ndbar} \d s.
\end{multline*}
Applying Grönwall's inequality, we finally obtain
\begin{multline}\label{eq:lemh-Hquantitative}
    \sup_{t\in[0,1]}\|h^{(D)}_k (t) - H^{(D)}_k (t)\|_{\ndbar} \leq \cstWh_k\bigg(\frac{1}{\sqrt{D}} +\sum_{j=0}^{k} \int_0^1\llnW_j(u)\d u\bigg)\exp(\cstWh_k)\\
    \leq \cstWh_k\bigg(\frac{1}{\sqrt{D}} +\sum_{j=0}^{k} \int_0^1\llnW_j(u)\d u\bigg),\quad\text{a.s.},
\end{multline}
which is the bound we wanted.

\paragraph{Backward pass.}
For $t \in [0,1]$,
    \[b^{(D)}_k (t) - B^{(D)}_k (t) = W_{\mathrm{out}}^{(D)} \nabla\ell(\hat{y}^{D}_k) + \int_0^t \B^{(D)}_k(s, \bfb^{(D)}) ds - W_{\mathrm{out}}^{(D)} \nabla\ell(y_k) - \int_0^t \overline{\B}^{(D)}_k(s) ds.\]
    As before, we introduce the pivot term $\B^{D}_k (s, \bfB^{(D)})$ and decompose \(\B^{D}_k (s, \bfB^{(D)}) = \hat{\B}_k^{(D)}(s, \bfB^{(D)}) + \check{\B}_k^{(D)}(s, \bfB^{(D)})\) and \(\overline{\B}^{D}_k (s) = \hat{\overline{\B}}_k^{(D)}(s) + \check{\overline{\B}}_k^{(D)}(s)\), to obtain
    \begin{multline}\label{eq:proof_cvg_recurrence_integral_decomposition_b}
        \|b^{(D)}_k (t) - B^{(D)}_k (t)\|_{\ndbar} \leq \|W_{\mathrm{out}}^{(D)}( \nabla\ell(\hat{y}^{D}_k) -\nabla\ell(y_k))\|_{\ndbar}+ \int_0^t\|\B^{(D)}_k(s, \bfb^{(D)}) -\B^{(D)}_k (s, \bfB^{(D)})\|_{\ndbar}\d s \\+ \int_0^t\|\hat{\B}^{(D)}_k (s, \bfB^{(D)})-\hat{\overline{\B}}^{(D)}_k(s)\|_{\ndbar} \d s+ \int_0^t\|\check{\B}^{(D)}_k (s, \bfB^{(D)})-\check{\overline{\B}}^{(D)}_k(s)\|_{\ndbar} \d s.
    \end{multline}
    Unlike the forward pass, the initial conditions do not coincide here. Despite this, from the Lipschitz property of $\nabla\ell$ and the definition of the outputs $(\hat{y}^D_k, y_k)$, we bound
    \[
        \|W_{\mathrm{out}}^{(D)}( \nabla\ell(\hat{y}^{D}_k) -\nabla\ell(y_k))\|_{\ndbar} \leq \| W_{\mathrm{out}}^{(D)}\|_{\ndbar}\| \nabla\ell(\hat{y}^{D}_k) -\nabla\ell(y_k)\|_{2, \dimout} \leq \cstW_k\|\hat{y}^{(D)}_k - y_k\|_{2, \dimout},
    \]
    where we used the notation $\cstW_k$ as before.
    We focus on the term \(\|\hat{y}^{(D)}_k - y_k\|_{2, \dimout}\): introducing the pivot $\frac{1}{D}\sum_{d=1}^D H_k^{d}(1)W_{\mathrm{out}}^d$, using C-S yields
    \begin{align*}
        \|\hat{y}^{(D)}_k -& y_k\|_{2, \dimout} = \bigg\| \frac{1}{D}\sum_{d=1}^D h_k^{d}(1)W_{\mathrm{out}}^d-\E[H_k^d(1)W_{\mathrm{out}}^d]\bigg\|_{2, \dimout}\\&\leq \frac{1}{D}\sum_{d=1}^D |h_k^{d}(1)-H_k^d(1)|\|W_{\mathrm{out}}^d\|_{2,\dimout}+ \bigg\| \frac{1}{D}\sum_{d=1}^D H_k^{d}(1)W_{\mathrm{out}}^d-\E[H_k^d(1)W_{\mathrm{out}}^d]\bigg\|_{2,\dimout}\\
        &\leq \| W_{\mathrm{out}}^{(D)}\|_{\ndbar} \|h_k^{(D)}(1)-H_k^{(D)}(1)\|_{\ndbar}+ \bigg\| \frac{1}{D}\sum_{d=1}^D H_k^{d}(1)W_{\mathrm{out}}^d-\E[H_k^d(1)W_{\mathrm{out}}^d]\bigg\|_{2,\dimout}\\
        &\leq \cstW_k\|h_k^{(D)}(1)-H_k^{(D)}(1)\|_{\ndbar} + \llnW_k  \\
        &= \cstW_k\|h_k^{(D)}(1)-H_k^{(D)}(1)\|_{\ndbar} + \int_0^1\llnW_k(u)\d u ,
    \end{align*}
    using the notations $\cstW_k$ and $\llnW_k$ as before\footnote{In the setting of trainable embedding matrices, we would have to consider two additional pivot terms: $\frac{1}{D}\sum_{d=1}^D H_k^{d}(1)(\hat{W}_{\mathrm{out}}^d)_k$ and $\frac{1}{D}\sum_{d=1}^D H_k^{d}(1)(W_{\mathrm{out}}^d)_k$. From here, the same two terms from the current analysis would appear, along with a term $\|\frac{1}{D}\sum_{d=1}^D H_k^{d}(1)((\hat{W}_r^d)_k - (W_{\mathrm{out}}^d)_k)\|_{\dimout, 2} \leq \|H_k^{(D)}(1)\|_{\ndbar}\|(\hat{W}_{\mathrm{out}}^{(D)})_k - (W_{\mathrm{out}}^{(D)})_k\|_{\ndbar}$ which can be proved to be $O(D^{-1/2})$ with high probability. See Appendix~\ref{sec:extensions_of_results} for further details.} -- here $\llnW_k$ does not depend on $t\in[0,1]$, but we still write \(\int_0^1 \llnW_k(u) \d u\) for notational consistency. Using \eqref{eq:lemh-Hquantitative}, we get 
    \begin{align}\label{eq:proof_cvg_lemB_init}
        \|\hat{y}^{(D)}_k - y_k\|_{2, \dimout}& \leq
        \cstWh_k\Big(\frac{1}{\sqrt{D}} + \sum_{j=0}^k\int_0^1\llnW_j(u)\d u \Big)\notag\\
        \|W_{\mathrm{out}}^{(D)}( \nabla\ell(\hat{y}^{D}_k) -\nabla\ell(y_k))\|_{\ndbar}, &\leq
        \cstWh_k\Big(\frac{1}{\sqrt{D}} + \sum_{j=0}^k\int_0^1\llnW_j(u)\d u \Big)\;\;\text{a.s.}
    \end{align}
    From there, we conclude by means of the following lemmas.
    \begin{lemma}[Propagation of Chaos -- Backward Pass]\label{lem:proof_cvg_prop_caos_B} For all $s\in[0,1]$,
        \begin{multline*}
            \|\B^{(D)}_k(s, \bfb^{(D)}) -\B^{(D)}_k(s, \bfB^{(D)})\|_{\ndbar} \leq \cstWh_k\bigg( \| b_k^{(D)}(s) - B_k^{(D)}(s)\|_{\ndbar} +\| h_k^{(D)}(s) - H_k^{(D)}(s)\|_{\ndbar} \\+\sum_{j=0}^{k-1} \llnW_j(s)+\frac{1}{\sqrt{D}} + \sum_{j=0}^{k-1} \int_0^1 \llnW_j(u) \d u\bigg).
        \end{multline*}
    \end{lemma}

    \begin{lemma}[Quantitative CLT -- Backward Pass]\label{lem:proof_cvg_TCL_B}
    For all $s\in[0,1]$,
        \[\|\check{\B}^{(D)}_k ( s ) - \check{\overline{\B}}^{(D)}_k (s)\|_{\ndbar} \leq \cstWh_k \Big(\llnW_k(s) + \frac{1}{\sqrt{D}}\Big).\]
    \end{lemma}
    \begin{lemma}[Cavity Method -- Backward Pass]\label{lem:proof_cvg_cavity_TCL_B}
For all $s\in[0,1]$,
\[\|\hat{\B}^{(D)}_k ( s ) - \hat{\overline{\B}}^{(D)}_k (s)\|_{\ndbar}\leq \cstWh_k \Big(\llnW_k(s) + \frac{1}{\sqrt{D}}\Big).\]
\end{lemma}
Injecting the bounds from Lemmas~\ref{lem:proof_cvg_prop_caos_B}, \ref{lem:proof_cvg_TCL_B} and \ref{lem:proof_cvg_cavity_TCL_B} into \eqref{eq:proof_cvg_recurrence_integral_decomposition_b}, as well as \eqref{eq:proof_cvg_lemB_init}, 
\begin{multline*}
        \|b^D_k (t) - B^{D}_k (t)\|_{\ndbar}
        \leq \cstWh_k\Big( \int_0^t \|b^D_k (s) - B^{D}_k (s)\|_{\ndbar}\d s + \sup_{s\in[0,1]} \|h^D_k (s) - H^{D}_k (s)\|_{\ndbar}\Big)\\
        +\cstWh_k\Big(\frac{1}{\sqrt{D}} + \sum_{j=0}^k\int_0^1\llnW_j(u)\d u \Big) 
        + \cstWh_k\Big( \frac{1}{\sqrt{D}} + \sum_{j=0}^{k-1} \int_0^1 \llnW_j(u)\d u\Big) \\+ \cstWh_k \Big( \frac{1}{\sqrt{D}} + \int_0^t \llnW_k(s) \d s\Big) + \cstWh_k \int_0^t\sum_{j=0}^{k-1} \llnW_j(s)\d s\\
    \leq \cstWh_k\bigg(\frac{1}{\sqrt{D}} +\sum_{j=0}^{k} \int_0^1\llnW_j(u)\d u\bigg) + \cstWh_k\int_0^t \|b^D_k (s) - B^{D}_k (s)\|_{\ndbar} \d s,
\end{multline*}
where we further used \eqref{eq:lemh-Hquantitative}.
Applying Grönwall's inequality, we finally obtain
\begin{multline}
    \sup_{t\in[0,1]}\|b^{(D)}_k (t) - B^{(D)}_k (t)\|_{\ndbar} \leq \cstWh_k \exp(\cstWh_k) \bigg(\frac{1}{\sqrt{D}} +\sum_{j=0}^{k} \int_0^1\llnW_j(u)\d u\bigg)\\
    \leq \cstWh_k\bigg(\frac{1}{\sqrt{D}} +\sum_{j=0}^{k} \int_0^1\llnW_j(u)\d u\bigg)\;\;\text{a.s.},
\end{multline}
which concludes the proof.

\subsubsection{Proof of intermediary estimates}

\begin{proof}[Proof of Lemma~\ref{lem:proof_cvg_prop_caos_H}]
        We directly compute: 
            \begin{align*}
        \H^{(D)}_k (s, \bfh^D )& - \H^{(D)}_k (s, \bfH^D) = \E \Big[ \rho ( P_k^{\mathbf{h}^D, \mathbf{b}^D} (s) ) \Big( g^{(D)}_k (s, \bfS^{\mathbf{h}^D}, \bfS^{\mathbf{b}^D}  ) - G^{(D)}_k (s, \bfS^{\mathbf{h}^D}, \bfS^{\mathbf{b}^D}  ) \Big) \big\vert \Finit\Big] \\
        &\qquad + \E \Big[ \rho ( P_k^{\mathbf{h}^D, \mathbf{b}^D} (s) ) \Big( G^{(D)}_k (s, \bfS^{\mathbf{h}^D}, \bfS^{\mathbf{b}^D}  ) - G^{(D)}_k (s, \bfS^{\mathbf{H}^D}, \bfS^{\mathbf{B}^D}  ) \Big) \big\vert \Finit\Big] \\
&\qquad + \E \Big[ \Big( \rho ( P^{\mathbf{h}^D, \mathbf{b}^D}_k (s) ) - \rho ( P^{\mathbf{H}^D, \mathbf{B}^D}_k (s) ) \Big) \sqrt{D} V^{(D)}\big\vert \Finit\Big]\\
&\qquad + \E \Big[ \Big( \rho ( P^{\mathbf{h}^D, \mathbf{b}^D}_k (s) ) - \rho ( P^{\mathbf{H}^D, \mathbf{B}^D}_k (s) ) \Big) G^{(D)}_k (s, \bfS^{\mathbf{H}^D}, \bfS^{\mathbf{B}^D} ) \big\vert \Finit\Big]\\
&=: \blue{(i)} + \blue{(ii)} + \blue{(iii)} + \blue{(iv)},
\end{align*}
where we introduced various pivots and we set
\begin{equation}\label{eq:proof_cvg_prop_caos_definition_P}
        \begin{cases}
     P_k^{\mathbf{h}^D, \mathbf{b}^D}(s) := \bfS^{\mathbf{h}^D}_k(s) + \Big\langle f^{(D)}_k (s, \bfS^{\mathbf{h}^D}(s), \bfS^{\mathbf{b}^D}(s)), h_k^{(D)}(s)\Big\rangle_{\ndbar}, \\  
     P_k^{\mathbf{H}^D, \mathbf{B}^D}(s) := \bfS^{\mathbf{H}^D}_k(s) + \Big\langle F^{(D)}_k (s, \bfS^{\mathbf{H}^D}(s), \bfS^{\mathbf{B}^D}(s)), H_k^{(D)}(s)\Big\rangle_{\ndbar}.
    \end{cases}
    \end{equation}
    We now study the RMS norm of each term separately.
    
    \paragraph*{Bound on \blue{(i)}.}
    We apply Lemma~\ref{lem:decorrelation_chizat}-\ref{lemdecchizat:nonindep} to obtain
    \begin{align*}
        \|\blue{(i)}\|_{\ndbar}\leq \E^{1/2} \Big[ \big( \rho ( P^{\mathbf{h}^D, \mathbf{b}^D}_k (s) ) \big)^2 \big\vert \Finit\Big]\E^{1/2} \Big[ \big\|g^{(D)}_k (s, \bfS^{\mathbf{h}^D}, \bfS^{\mathbf{b}^D}  ) - G^{(D)}_k (s, \bfS^{\mathbf{h}^D}, \bfS^{\mathbf{b}^D}  ) \big\|_{\ndbar}^2\big\vert \Finit\Big]. 
    \end{align*}
    Similarly, using the triangle inequality,
\begin{align*}
    \|\blue{(ii)}\|_{\ndbar}\leq \E \Big[ \big|\rho ( P_k^{\mathbf{h}^D, \mathbf{b}^D} (s) )\big| \big\| G^{(D)}_k (s, \bfS^{\mathbf{h}^D}, \bfS^{\mathbf{b}^D}  ) - G^{(D)}_k (s, \bfS^{\mathbf{H}^D}, \bfS^{\mathbf{B}^D}  )\big\|_{\ndbar}\big\vert \Finit\Big].
\end{align*}
For $p\geq 2$, from the linear growth of $\rho$, we have
\begin{equation}\label{eq:proof_cvg_prop_caos_rho_lin_growth}
    \E^{1/p}\big[\big| \rho ( P^{\mathbf{h}^D, \mathbf{b}^D}_k (s) ) \big|^p | \Finit\big] \leq C_\rho (1 + \E^{1/p}\big[\big| P^{\mathbf{h}^D, \mathbf{b}^D}_k (s)  \big|^p | \Finit\big]).
\end{equation} Furthermore,
\begin{align*}
    \E^{1/p}\big[\big| P^{\mathbf{h}^D, \mathbf{b}^D}_k (s)  \big|^p | \Finit\big] &\leq \E^{1/p}\Big[ \Big| \bfS^{\mathbf{h}^D}_k(s) + \Big\langle f^{(D)}_k (s, \bfS^{\mathbf{h}^D}(s), \bfS^{\mathbf{b}^D}(s)), h_k^{(D)}(s)\Big\rangle_{\ndbar}\Big|^p \big| \Finit \Big]\\
    &\leq C_p\vpu \big\| h^D_k(s) \big\|_{\ndbar} + \E^{1/p}\Big[\big\| f^{(D)}_k (s, \bfS^{\mathbf{h}^D}(s), \bfS^{\mathbf{b}^D}(s))\big\|_{\ndbar}^p \big\|h_k^{(D)}(s)\big\|_{\ndbar}^p| \Finit\Big] \\
    &\leq C_p\Big(\vpu + \E^{1/p}\big[\big\| f^{(D)}_k (s, \bfS^{\mathbf{h}^D}(s), \bfS^{\mathbf{b}^D}(s))\big\|_{\ndbar}^p| \Finit\big]\Big)\big\| h^D_k(s) \big\|_{\ndbar},
\end{align*}
where we have used Lemma~\ref{lem:moment_magic}, C-S, and the $\Finit$-measurability of $\|h_k^{(D)}\|_{\ndbar}$. From Lemma~\ref{lem:finite_skeleton_bounded_ndbar} we also have
\(\big\| f^{(D)}_k (s, \bfS^{\mathbf{h}^D}(s), \bfS^{\mathbf{b}^D}(s))\big\|_{\ndbar} \leq \cstW_k(1+\|(\bfS^{\mathbf{h}^D}(s), \bfS^{\mathbf{b}^D}(s))\|_{p, 2k})\), and thus, taking the $L^p(\Finit)$ norm yields
\begin{align}\label{eq:proof_cvg_prop_caos_bound_f}
    \E^{1/p}\big[\big\| f^{(D)}_k (s, \bfS^{\mathbf{h}^D}(s), \bfS^{\mathbf{b}^D}(s))\big\|_{\ndbar}^p |\Finit \big] &\leq \cstW_k(1+\|\|(\bfS^{\mathbf{h}^D}(s), \bfS^{\mathbf{b}^D}(s))\|_{p, 2k}\|_{L^p(\Finit)}) \notag\\
    &\leq \cstW_k\Big(1+\sum_{j=0}^{k-1}\|S_j^{\mathbf{h}^D}(s)\|_{L^p(\Finit)} + \|S_j^{\mathbf{b}^D}(s)\|_{L^p(\Finit)}\Big)\notag\\
    &\leq \cstW_k\Big(1+\sum_{j=0}^{k-1}C_p\vpu \|h_j^{(D)}(s)\|_{\ndbar} + C_p\vpv \|b_j^{(D)}(s)\|_{\ndbar}\Big)\notag\\
    &\leq \cstW_k\Big(1+\sum_{j=0}^{k-1}\cstW_k\Big) \leq k\cstW_k,
\end{align}
where the third inequality comes from Lemma~\ref{lem:moment_magic}, and the fourth one from the fact that $\max_{j\in[0:k-1]}\|h_j^{(D)}(s)\|_{\ndbar} \leq \cstW_k$ from Lemma~\ref{lem:embeddingcontrol}.
Absorbing $k$ and other constants into $\cstW_k$, gathering terms yields
\begin{equation}\label{eq:proof_cvg_prop_caos_bound_P}
    \E^{1/p}\big[\big| P^{\mathbf{h}^D, \mathbf{b}^D}_k (s)  \big|^p | \Finit\big] \leq \cstW_k  \;\; \text{ and } \;\;\E^{1/p}\big[\big| \rho ( P^{\mathbf{h}^D, \mathbf{b}^D}_k (s) ) \big|^p | \Finit\big]  \leq \cstW_k \;\; \text{a.s.}
\end{equation} Replacing in our bound for \blue{(i)} yields
\begin{align*}
    \|\blue{(i)}\|_{\ndbar} 
    &\leq \cstW_k\E^{1/2}\Big[\big\|g^{(D)}_k (s, \bfS^{\mathbf{h}^D}, \bfS^{\mathbf{b}^D}  ) - G^{(D)}_k (s, \bfS^{\mathbf{h}^D}, \bfS^{\mathbf{b}^D}  ) \big\|_{\ndbar}^2|\Finit\Big].
\end{align*}
The right hand-side is then handled using Proposition~\ref{prop:propcaos_Dbar_whp}, using the same argument as the one giving \eqref{eq:proof_cvg_prop_caos_bound_prefinal} in the bound on \blue{(iii)} below; we omit it to avoid redundancies.

\paragraph*{Bound on \blue{(ii)}.} For all $d\geq 1$, we have
\begin{multline*}
    \Big|G^{d}_k (s, \bfS^{\mathbf{h}^D}(s), \bfS^{\mathbf{b}^D}(s)  ) - G^{d}_k (s, \bfS^{\mathbf{H}^D}(s), \bfS^{\mathbf{B}^D}(s)) \Big|\\\leq \|\bnabla G_k^d(s, \xi_{\bfS^{\mathbf{h}^D}(s), \bfS^{\mathbf{H}^D}(s)}, \bfS^{\mathbf{b}^D}(s)) \|_{k,2} \| \bfS^{\mathbf{h}^D}(s)- \bfS^{\mathbf{H}^D}(s)\|_{k,2}\\
    \;\; + \|\bnabla G_k^d(s, \bfS^{\mathbf{h}^D}(s), \xi_{\bfS^{\mathbf{b}^D}(s), \bfS^{\mathbf{B}^D}(s)})) \|_{k,2} \| \bfS^{\mathbf{b}^D}(s)- \bfS^{\mathbf{B}^D}(s)\|_{k,2},
\end{multline*}
where \(\xi_{\bfS^{\mathbf{h}^D}(s), \bfS^{\mathbf{H}^D}(s)}\) represents a point in the line segment between \(\bfS^{\mathbf{h}^D}(s)\) and \(\bfS^{\mathbf{H}^D}(s)\), and similarly for \(\xi_{\bfS^{\mathbf{b}^D}(s), \bfS^{\mathbf{B}^D}(s)}\). From Lemma~\ref{lem:RegSkel}-\ref{lemRegSkel:Grad}, it follows that
\begin{multline*}
    \big|G^{d}_k (s, \bfS^{\mathbf{h}^D}(s), \bfS^{\mathbf{b}^D}(s)  ) - G^{d}_k (s, \bfS^{\mathbf{H}^D}(s), \bfS^{\mathbf{B}^D}(s)) \big|\\\leq e^{\cstHhat_kM_S}\big( \| \bfS^{\mathbf{h}^D}(s)- \bfS^{\mathbf{H}^D}(s)\|_{k,2} + \| \bfS^{\mathbf{b}^D}(s)- \bfS^{\mathbf{B}^D}(s)\|_{k,2}\big) \max_{\substack{i\in[0:k-1]\\t\in[0,1]}}|B_i^d(t)|,
\end{multline*}
where $M_S := 1\vee\max_{i\in[0:k-1]} |S_i^{\mathbf{h}^D}(s)|\vee |S_i^{\mathbf{b}^D}(s)| \vee |S_i^{\mathbf{H}^D}(s)|\vee |S_i^{\mathbf{B}^D}(s)|$. Taking the RMS norm further yields
\begin{multline}\label{eq:proof_cvg_LipschitzBoundG}
\big\|G^{(D)}_k (s, \bfS^{\mathbf{h}^D}(s), \bfS^{\mathbf{b}^D}(s)  ) - G^{(D)}_k (s, \bfS^{\mathbf{H}^D}(s), \bfS^{\mathbf{B}^D}(s)) \big\|_{\ndbar}\\
\leq e^{\cstHhat_k M_S}\big( \| \bfS^{\mathbf{h}^D}(s)- \bfS^{\mathbf{H}^D}(s)\|_{k,2} + \| \bfS^{\mathbf{b}^D}(s)- \bfS^{\mathbf{B}^D}(s)\|_{k,2}\big) \bigg\|\max_{\substack{i\in[0:k-1]\\t\in[0,1]}}|B_i^{(D)}(t)|\bigg\|_{\ndbar},
\end{multline}
and using Lemma~\ref{lem:MF_skeleton_bounded_ndbar}, we obtain
\begin{align*}
    \|\blue{(ii)}\|_{\ndbar}
    &\leq \E \big[ |\rho ( P_k^{\mathbf{h}^D, \mathbf{b}^D} (s) )| e^{\cstHhat_k M_S}\big( \| \bfS^{\mathbf{h}^D}(s)- \bfS^{\mathbf{H}^D}(s)\|_{k,2} + \| \bfS^{\mathbf{b}^D}(s)- \bfS^{\mathbf{B}^D}(s)\|_{k,2}\big) \big\vert \Finit\big] \cstW_k.
\end{align*}
We apply C-S twice on the conditional expectation, to obtain using \eqref{eq:proof_cvg_prop_caos_bound_P} that
\[\|\blue{(ii)}\|_{\ndbar}\leq \cstW_k\E^{1/4}\big[e^{4\cstHhat_kM_S} |\Finit\big] \E^{1/2}\big[\| \bfS^{\mathbf{h}^D}(s)- \bfS^{\mathbf{H}^D}(s)\|_{k,2}^2 + \| \bfS^{\mathbf{b}^D}(s)- \bfS^{\mathbf{B}^D}(s)\|^2_{k,2} |\Finit\big].\]
From Lemma~\ref{lem:Magic}, we have 
\begin{multline*}
    \E^{1/2}\big[\| \bfS^{\mathbf{h}^D}(s)- \bfS^{\mathbf{H}^D}(s)\|_{k,2}^2 + \| \bfS^{\mathbf{b}^D}(s)- \bfS^{\mathbf{B}^D}(s)\|^2_{k,2} |\Finit\big] \\\leq \sum_{i=0}^{k-1}(\vpu \|h_i^{(D)}(s) - H_i^{(D)}(s) \|_{\ndbar} + \vpv \|b_i^{(D)}(s) - B_i^{(D)}(s) \|_{\ndbar}).
\end{multline*}
On the other hand, since \begin{multline*}
    \exp(\cstHhat_kM_S)\leq e^{\cstHhat_k}+ \sum_{j=0}^{k-1}(\exp(\cstHhat_k|S_j^{\bfh^D}(s)|) +\exp(\cstHhat_k|S_j^{\bfb^D}(s)|) )  \\+ \sum_{j=0}^{k-1}\exp(\cstHhat_k|S_j^{\bfH^D}(s)|)  + \exp(\cstHhat_k|S_j^{\bfB^D}(s)|)),
\end{multline*} from Lemma~\ref{lem:exp_moment_magic} we have that $\E[\exp(\cstHhat_kM_S)|\Finit]$ is a.s. bounded by terms depending on $\max_{\substack{j\in[0:k-1]\\s\in[0,1]}}\|h_j^{(D)}(s)\|_{\ndbar}\vee\|b_j^{(D)}(s)\|_{\ndbar}\vee\|H_j^{(D)}(s)\|_{\ndbar}\vee\|B_j^{(D)}(s)\|_{\ndbar}$. That is, it can be absorbed into $\cstW_k$ by means of Lemmas~\ref{lem:embeddingcontrol} and \ref{lem:MF_skeleton_bounded_ndbar}, thus yielding
\begin{equation}\label{eq:proof_cvg_prop_caos_bound_ii}
    \|\blue{(ii)}\|_{\ndbar}\leq \cstW_k\sum_{i=0}^{k-1} \|h_i^{(D)}(s) - H_i^{(D)}(s) \|_{\ndbar} + \|b_i^{(D)}(s) - B_i^{(D)}(s) \|_{\ndbar}.
\end{equation}

\paragraph*{Bound on \blue{(iv)}.}Using the triangle inequality and C-S, we obtain
\begin{align*}
        \|\blue{(iv)}\|_{\ndbar}&\leq \E \big[ \big| \rho ( P^{\mathbf{h}^D, \mathbf{b}^D}_k (s) ) - \rho ( P^{\mathbf{H}^D, \mathbf{B}^D}_k (s) ) \big| \big\|G^{(D)}_k (s, \bfS^{\mathbf{H}^D}_k, \bfS^{\mathbf{B}^D} ) \big\|_{\ndbar}\big\vert \Finit\big]\\
        &\leq \E^{1/2}\big[ \big|\rho ( P^{\mathbf{h}^D, \mathbf{b}^D}_k (s) ) - \rho ( P^{\mathbf{H}^D, \mathbf{B}^D}_k (s) )\big|^2 \big\vert \Finit \big] \E^{1/2}\big[\big\|G^{(D)}_k (s, \bfS^{\mathbf{H}^D}, \bfS^{\mathbf{B}^D} ) \big\|_{\ndbar}^2 \big\vert \Finit \big]\\
        &\leq \cstW_k\E^{1/2}\big[ \big|\rho ( P^{\mathbf{h}^D, \mathbf{b}^D}_k (s) ) - \rho ( P^{\mathbf{H}^D, \mathbf{B}^D}_k (s) )\big|^2 \vert \Finit \big] \E^{1/2}\big[\check{M}_S^2 \vert \Finit \big]\\
        &\leq \cstW_k\E^{1/2}\big[ \big|\rho ( P^{\mathbf{h}^D, \mathbf{b}^D}_k (s) ) - \rho ( P^{\mathbf{H}^D, \mathbf{B}^D}_k (s) )\big|^2 \vert \Finit \big],
    \end{align*}
where the third inequality used Lemma~\ref{lem:MF_skeleton_bounded_ndbar}, and we defined $\check{M}_S := 1\vee\max_{i\in[0:k-1]} |S_i^{\mathbf{H}^D}(s)|\vee |S_i^{\mathbf{B}^D}(s)|$. The last inequality follows from absorbing \(\E^{1/2}\big[\check{M}_S^2 \big\vert \Finit \big]\) into \(\cstW_k\), arguing as in \blue{(i)} using Lemmas~\ref{lem:moment_magic} and \ref{lem:MF_skeleton_bounded_ndbar}.
We now prove an estimate on \(| \rho ( P^{\mathbf{h}^D, \mathbf{b}^D}_k (s) ) - \rho ( P^{\mathbf{H}^D, \mathbf{B}^D}_k (s) ) |\). Since $\rho$ is Lipschitz, it suffices to study 
\begin{multline*}
    \big|  P^{\mathbf{h}^D, \mathbf{b}^D}_k (s) - P^{\mathbf{H}^D, \mathbf{B}^D}_k (s) \big|\leq \big|  \bfS^{\mathbf{h}^D}_k(s) - \bfS^{\mathbf{H}^D}_k(s)\big|\\ +  \big|\big\langle f^{(D)}_k (s, \bfS^{\mathbf{h}^D}(s), \bfS^{\mathbf{b}^D}(s)), h_k^{(D)}(s)\big\rangle_{\ndbar} - \big\langle F^{(D)}_k (s, \bfS^{\mathbf{H}^D}(s), \bfS^{\mathbf{B}^D}(s)), H_k^{(D)}(s)\big\rangle_{\ndbar} \big|.
\end{multline*}
We introduce a pivot term and use C-S to obtain
\begin{align*}
    \big| P^{\mathbf{h}^D, \mathbf{b}^D}_k& (s) - P^{\mathbf{H}^D, \mathbf{B}^D}_k (s) \big|\\&\leq \big|  \bfS^{\mathbf{h}^D}_k(s) - \bfS^{\mathbf{H}^D}_k(s)\big| + \big|\big\langle f^{(D)}_k (s, \bfS^{\mathbf{h}^D}(s), \bfS^{\mathbf{b}^D}(s))- F^{(D)}_k (s, \bfS^{\mathbf{H}^D}(s), \bfS^{\mathbf{B}^D}(s)), h_k^{(D)}(s)\big\rangle_{\ndbar} \big|\\
    &\qquad\qquad+ \big| \big\langle F^{(D)}_k (s, \bfS^{\mathbf{H}^D}(s), \bfS^{\mathbf{B}^D}(s)), h_k^{(D)}(s) - H_k^{(D)}(s)\big\rangle_{\ndbar} \big|\\
    &\leq \big|  \bfS^{\mathbf{h}^D}_k(s) - \bfS^{\mathbf{H}^D}_k(s)\big| +  \big\| F^{(D)}_k (s, \bfS^{\mathbf{H}^D}(s), \bfS^{\mathbf{B}^D}(s))\big\|_{\ndbar} \big\|h_k^{(D)}(s) - H_k^{(D)}(s)\big\|_{\ndbar}\\&\qquad\qquad\qquad\qquad +  \big\| f^{(D)}_k (s, \bfS^{\mathbf{h}^D}(s), \bfS^{\mathbf{b}^D}(s))- F^{(D)}_k (s, \bfS^{\mathbf{H}^D}(s), \bfS^{\mathbf{B}^D}(s))\big\|_{\ndbar}  \big\|h_k^{(D)}(s)\big\|_{\ndbar}\\
    &\leq \big|  \bfS^{\mathbf{h}^D}_k(s) - \bfS^{\mathbf{H}^D}_k(s)\big| +  \cstW_k\check{M}_S \big\|h_k^{(D)}(s) - H_k^{(D)}(s)\big\|_{\ndbar}\\&\qquad\qquad+  \cstW_k\big(\big\| f^{(D)}_k (s, \bfS^{\mathbf{h}^D}(s), \bfS^{\mathbf{b}^D}(s))- F^{(D)}_k (s, \bfS^{\mathbf{h}^D}(s), \bfS^{\mathbf{b}^D}(s))\big\|_{\ndbar}\\
    &\qquad\qquad\qquad\qquad\qquad +  \big\| F^{(D)}_k (s, \bfS^{\mathbf{h}^D}(s), \bfS^{\mathbf{b}^D}(s))- F^{(D)}_k (s, \bfS^{\mathbf{H}^D}(s), \bfS^{\mathbf{B}^D}(s))\big\|_{\ndbar} \big),
\end{align*}
where in the last inequality we used Lemmas~\ref{lem:embeddingcontrol} and \ref{lem:MF_skeleton_bounded_ndbar}, absorbed terms into $\cstW_k$, and introduced yet another pivot.
Using Lemma~\ref{lem:MF_skeleton_bounded_ndbar} as in equation \eqref{eq:proof_cvg_LipschitzBoundG} we further obtain
\begin{multline*}
\big\|F^{(D)}_k (s, \bfS^{\mathbf{h}^D}(s), \bfS^{\mathbf{b}^D}(s)) - F^{(D)}_k (s, \bfS^{\mathbf{H}^D}(s), \bfS^{\mathbf{B}^D}(s)) \big\|_{\ndbar}\\
\leq \cstW_k e^{\cstHhat_kM_S}\left( \| \bfS^{\mathbf{h}^D}(s)- \bfS^{\mathbf{H}^D}(s)\|_{k,2} + \| \bfS^{\mathbf{b}^D}(s)- \bfS^{\mathbf{B}^D}(s)\|_{k,2}\right).
\end{multline*}
Gathering these estimates and using C-S we obtain
\begin{multline}\label{eq:proof_cvg_prop_caos_lip_P}
    \E^{1/2}\big[\big|\rho ( P^{\mathbf{h}^D, \mathbf{b}^D}_k (s) ) - \rho ( P^{\mathbf{H}^D, \mathbf{B}^D}_k (s) )\big|^2\big|\Finit\big] \\
    \leq C_\rho \E^{1/2}\big[ \big| \bfS^{\mathbf{h}^D}_k(s) - \bfS^{\mathbf{H}^D}_k(s)\big|^2\big|\Finit \big] + \cstW_k \E^{1/2}\big[|\check{M}_S|^2|\Finit\big]\big\|h_k^{(D)}(s) - H_k^{(D)}(s)\big\|_{\ndbar} \\\qquad + \cstW_k \E^{1/2}\big[\big\| f^{(D)}_k (s, \bfS^{\mathbf{h}^D}(s), \bfS^{\mathbf{b}^D}(s))- F^{(D)}_k (s, \bfS^{\mathbf{h}^D}(s), \bfS^{\mathbf{b}^D}(s))\big\|_{\ndbar}^2\big|\Finit  \big]\\
    \qquad + \cstW_k \E^{1/4}\big[e^{4\cstHhat_k M_S}|\Finit\big] \E^{1/4}\big[\| \bfS^{\mathbf{h}^D}(s)- \bfS^{\mathbf{H}^D}(s)\|_{k,2}^4|\Finit \big] 
    \\
    \qquad +\cstW_k \E^{1/4}\big[e^{4\cstHhat_k M_S}|\Finit\big]\E^{1/4}\big[ \| \bfS^{\mathbf{b}^D}(s)- \bfS^{\mathbf{B}^D}(s)\|_{k,2}^4|\Finit\big] 
    \\
    \leq \cstW_k \bigg(\big\|h_k^{(D)}(s) - H_k^{(D)}(s)\big\|_{\ndbar}
    + \sum_{i=0}^{k-1} \|h_i^{(D)}(s) - H_i^{(D)}(s)\|_{\ndbar} + \|b_i^{(D)}(s) - B_i^{(D)}(s)\|_{\ndbar}\\
    +\E^{1/2}\big[\big\| f^{(D)}_k (s, \bfS^{\mathbf{h}^D}(s), \bfS^{\mathbf{b}^D}(s))- F^{(D)}_k (s, \bfS^{\mathbf{h}^D}(s), \bfS^{\mathbf{b}^D}(s))\big\|_{\ndbar}^2\big|\Finit  \big]\bigg),
\end{multline}
where we argue as in \eqref{eq:proof_cvg_prop_caos_bound_ii}, using Lemma~\ref{lem:exp_moment_magic} to absorb \(\E^{1/4}[e^{4\cstHhat_k M_S}|\Finit] \) into $\cstW_k$, and Lemma~\ref{lem:moment_magic} to control $\E^{1/4}[\| \bfS^{\mathbf{h}^D}(s)- \bfS^{\mathbf{H}^D}(s)\|_{k,2}^4|\Finit]$ by $\cstW_k\sum_{i=0}^{k-1} \|h_i^{(D)}(s) - H_i^{(D)}(s)\|_{\ndbar}$, and similarly for $b$. Gathering terms, we obtain
\begin{multline*}
    \|\blue{(iv)}\|_{\ndbar}\leq \cstW_k \Big(\big\|h_k^{(D)}(s) - H_k^{(D)}(s)\big\|_{\ndbar}
    + \sum_{i=0}^{k-1} \|h_i^{(D)}(s) - H_i^{(D)}(s)\|_{\ndbar} + \|b_i^{(D)}(s) - B_i^{(D)}(s)\|_{\ndbar}\\
    +\E^{1/2}\big[\big\| f^{(D)}_k (s, \bfS^{\mathbf{h}^D}(s), \bfS^{\mathbf{b}^D}(s))- F^{(D)}_k (s, \bfS^{\mathbf{h}^D}(s), \bfS^{\mathbf{b}^D}(s))\big\|_{\ndbar}^2\big|\Finit  \big]\Big).
\end{multline*}

\paragraph*{Bound on \blue{(iii)}.} We use the decorrelation estimate from Lemma~\ref{lem:decorrelation_chizat} to obtain
\begin{multline*}
    \|\blue{(iii)}\|_{\ndbar} = \big\|\E \big[ \big( \rho ( P^{\mathbf{h}^D, \mathbf{b}^D}_k (s) ) - \rho ( P^{\mathbf{H}^D, \mathbf{B}^D}_k (s) ) \big) \sqrt{D}V^{(D)} \big\vert \Finit\big]\big\|_{\ndbar}\\
        \leq  \varv \E^{1/2}\big[ \big| \rho ( P^{\mathbf{h}^D, \mathbf{b}^D}_k (s) ) - \rho ( P^{\mathbf{H}^D, \mathbf{B}^D}_k (s) )\big|^2 |\Finit\big]. 
\end{multline*}
From \eqref{eq:proof_cvg_prop_caos_lip_P}, it follows that \(\|\blue{(iii)}\|_{\ndbar}\) satisfies the same a.s. bound as \(\|\blue{(iv)}\|_{\ndbar}\). 
Gathering all terms, we obtain
\begin{multline*}
        \|\H^{(D)}(s, \bfh^{(D)}) - \H^{(D)}(s, \bfH^{(D)})\|_{\ndbar} \leq \cstW_k \Big(\big\|h_k^{(D)}(s) - H_k^{(D)}(s)\big\|_{\ndbar} \\+ \sum_{i=0}^{k-1} \|h_i^{(D)}(s) - H_i^{(D)}(s)\|_{\ndbar} + \|b_i^{(D)}(s) - B_i^{(D)}(s)\|_{\ndbar}\\
    +\E^{1/2}\big[\big\| g^{(D)}_k (s, \bfS^{\mathbf{h}^D}(s), \bfS^{\mathbf{b}^D}(s))- G^{(D)}_k (s, \bfS^{\mathbf{h}^D}(s), \bfS^{\mathbf{b}^D}(s))\big\|_{\ndbar}^2\big|\Finit  \big] \\
    +\E^{1/2}\big[\big\| f^{(D)}_k (s, \bfS^{\mathbf{h}^D}(s), \bfS^{\mathbf{b}^D}(s))- F^{(D)}_k (s, \bfS^{\mathbf{h}^D}(s), \bfS^{\mathbf{b}^D}(s))\big\|_{\ndbar}^2\big|\Finit  \big]\Big).
\end{multline*}
The last two terms are handled using Proposition~\ref{prop:propcaos_Dbar_whp}, which yields
\begin{multline*}
 \|f^{(D)}_k (s , \bfS^{\mathbf{h}^D}(s), \bfS^{\mathbf{b}^D}(s)) - F^{(D)}_k (s , \bfS^{\mathbf{h}^D}(s), \bfS^{\mathbf{b}^D}(s)) \|_{\ndbar} \\
\leq e^{\cstW_k M_S} \Big( \max_{\substack{j\in[0:k-1]}} \| h^{(D)}_j (s ) - H^{(D)}_j (s ) \|_{\ndbar} \vee \| b^{(D)}_j (s ) - B^{(D)}_j (s ) \|_{\ndbar} +  \max_{\substack{j\in[0:k-1]}}\llnW_j(s) \Big),
\end{multline*}
with $M_S$ as before. Taking the conditional $L^2$ norm we obtain
\begin{multline}\label{eq:proof_cvg_prop_caos_bound_prefinal}
 \E^{1/2}\Big[\|f^{(D)}_k (s , \bfS^{\mathbf{h}^D}(s), \bfS^{\mathbf{b}^D}(s)) - F^{(D)}_k (s , \bfS^{\mathbf{h}^D}(s), \bfS^{\mathbf{b}^D}(s)) \|_{\ndbar}^2|\Finit\Big] \\
\leq \E^{1/2}\Big[e^{2\cstW_k M_S}|\Finit\Big] \Bigg( \max_{\substack{j\in[0:k-1]}} \| h^{(D)}_j (s ) - H^{(D)}_j (s ) \|_{\ndbar} \vee \| b^{(D)}_j (s ) - B^{(D)}_j (s ) \|_{\ndbar} +  \max_{\substack{j\in[0:k-1]}}\llnW_j(s) \Bigg)\\
\leq \cstW_k \Bigg( \sum_{j=0}^{k-1} \| h^{(D)}_j (s ) - H^{(D)}_j (s ) \|_{\ndbar} + \| b^{(D)}_j (s ) - B^{(D)}_j (s ) \|_{\ndbar} +  \sum_{j=0}^{k-1}\llnW_j(s) \Bigg),
\end{multline}
where we used the $\Finit$-measurability of $(\bfh, \bfH, \bfb, \bfB)$ and $\llnW_k$, as well as Lemma~\ref{lem:exp_moment_magic}.

\paragraph*{Gathering bounds.}
This finally yields
\begin{multline*}
        \|\H^{(D)}(s, \bfh^{(D)}) - \H^{(D)}(s, \bfH^{(D)})\|_{\ndbar} \\\leq \cstW_k \Big(\big\|h_k^{(D)}(s) - H_k^{(D)}(s)\big\|_{\ndbar}+ \sum_{j=0}^{k-1}\llnW_j(s)+ \sum_{j=0}^{k-1} \|h_j^{(D)}(s) - H_j^{(D)}(s)\|_{\ndbar} + \|b_j^{(D)}(s) - B_j^{(D)}(s)\|_{\ndbar}\Big)\\
        \leq \cstWh_k \Big(\big\|h_k^{(D)}(s) - H_k^{(D)}(s)\big\|_{\ndbar}+ \sum_{j=0}^{k-1}\llnW_j(s)+  \frac{1}{\sqrt{D}} + \sum_{j=0}^{k-1} \int_0^1\llnW_j(s)\d s\Big),
\end{multline*}
where the last line follows from our induction hypothesis \eqref{eq:proof_cvg_inductive_hypothesis}. This concludes the proof.
\end{proof}

\begin{proof}[Proof of Lemma~\ref{lem:proof_cvg_TCL_H}]
    Recalling our definitions for the involved terms, we have for $s\in[0,1]$,
    \begin{multline*}
        \check{\H}^{(D)}_k ( s, \bfH^{(D)} ) - \check{\overline{\H}}^{(D)}_k (s) =\E[\rho(P_k^{\mathbf{H}^D, \mathbf{B}^D}(s)) G_k^{(D)}(s, \bfS^{\mathbf{H}^D}(s), \bfS^{\mathbf{B}^D}(s))|\Finit] \\\qquad- \E[\rho(P_k^{\mathbf{Z}^H, \mathbf{Z}^B}(s)) G_k^{(D)}(s, \mathbf{Z}^{H}(s), \mathbf{Z}^{B}(s))|\Finit]\\
        =\E[\rho(P_k^{\mathbf{H}^D, \mathbf{B}^D}(s)) G_k^{(D)}(s, \bfS^{\mathbf{H}^D}(s), \bfS^{\mathbf{B}^D}(s))|\Finit] - \E[\rho(\tilde{P}_k^{\mathbf{Z}^H, \mathbf{Z}^B}(s)) G_k^{(D)}(s, \mathbf{Z}^{H}(s), \mathbf{Z}^{B}(s))|\Finit] \\+ \E[(\rho(\tilde{P}_k^{\mathbf{Z}^H, \mathbf{Z}^B}(s)) - \rho(P_k^{\mathbf{Z}^H, \mathbf{Z}^B}(s))) G_k^{(D)}(s, \mathbf{Z}^{H}(s), \mathbf{Z}^{B}(s))|\Finit]
        =:\blue{(\mathfrak{i})} + \blue{(\mathfrak{ii})},
    \end{multline*}
    where we have introduced a pivot term and we set
    \[\begin{cases}
        P_k^{\mathbf{H}^D, \mathbf{B}^D}(s) = \bfS^{\mathbf{H}^D}_k(s) + \big\langle F^{(D)}_k (s, \bfS^{\mathbf{H}^D}(s), \bfS^{\mathbf{B}^D}(s)), H_k^{(D)}(s)\big\rangle_{\ndbar} ,\\
        \tilde{P}_k^{\mathbf{Z}^H, \mathbf{Z}^B}(s) = \mathbf{Z}^{H}_k(s) + \big\langle F^{(D)}_k (s, \mathbf{Z}^{H}(s), \mathbf{Z}^{B}(s)), H_k^{(D)}(s)\big\rangle_{\ndbar},\\
        P_k^{\mathbf{Z}^H, \mathbf{Z}^B}(s) = \mathbf{Z}^{H}_k(s) + \E\big[ F^1_k (s, \mathbf{Z}^{H}(s), \mathbf{Z}^{B}(s))H_k^{1}(s)\big \vert \mathbf{Z}^{H}, \mathbf{Z}^{B}\big]. 
    \end{cases}\]

    \paragraph*{Bound on \(\blue{(\mathfrak{ii})}\).} 
    We define the $\Finit$-measurable function $\varphi^{(D)}:\R^{k+1}\times\R^{k}\to \R^D$, given by
    \[\varphi^{(D)}(\bfz^h, \bfz^b) := \rho\big(z_k^h + \big\langle F^{(D)}_k (s, \bfz^{h}, \bfz^{b}), H_k^{(D)}(s)\big\rangle_{\ndbar}\big)G_k^{(D)}(s, \bfz^{h}, \bfz^{b}).\]
    From the bounds in Lemmas~\ref{lem:RegSkel} and \ref{lem:MF_skeleton_bounded_ndbar}, we have that, for $k \in \{1,2,3 \}$ and $d\in[1:D]$,
    \[\|D^k\varphi^{d}(\bfz^h, \bfz^b)\|_{op} \leq \cstW_k\max_{\substack{t\in[0,1]\\j\in[0:k-1]}}(|H_j^{d}(t)| + |B_j^{d}(t)|)\exp(\cstHhat_k\|(\bfz^h, \bfz^b)\|_{2k+1,2}),\]
where $\cstHhat_k>0$ and $\cstW_k$ depends solely on $\|W_{\mathrm{in}}^{(D)}\|_{\ndbar}$, $\|W_{\mathrm{out}}^{(D)}\|_{\ndbar}$, as before. This exponential growth allows us to apply the quantitative CLT result from Corollary~\ref{cor:quantitative_noniid_clt_conditional}. Hence, we obtain a constant $\cstWh_k>0$ depending solely on $\cstHhat_k, k, \vpu, \vpv$, $\|W_{\mathrm{in}}^{(D)}\|_{\ndbar}$, $\|W_{\mathrm{out}}^{(D)}\|_{\ndbar}$, \(\|W_{\mathrm{in}}^{(D)}\|_{\overline{3}, D}\) and \(\|W_{\mathrm{out}}^{(D)}\|_{\overline{3}, D}\), such that
    \begin{align}\label{eq:proof_cvg_TCL_bound_ii}
        \|\blue{(\mathfrak{ii})}\|_{\ndbar}
        &= \|\E[\varphi^{(D)}(\bfS^{\mathbf{H}^D}(s), \bfS^{\mathbf{B}^D}(s))|\Finit] - \E[\varphi^{(D)}(\mathbf{Z}^{H}(s), \mathbf{Z}^{B}(s))|\Finit]\|_{\ndbar}\notag
        \\&\leq \cstW_k\|(\max_{t\in[0,1]}(|H_k^{d}(t)| + |B_k^{d}(t)|))_{d=1}^D\|_{\ndbar}\big(\frac{1}{\sqrt{D}} + \llnW_k(s)\big)\notag\\
        &\leq \cstWh_k\big(\frac{1}{\sqrt{D}} + \llnW_k(s)\big),
    \end{align}
    where we used Lemma~\ref{lem:MF_skeleton_bounded_ndbar} to absorb \(\|(\max_{t\in[0,1]}(|H^{d}(t)| + |B^{d}(t)|))_{d=1}^D\|_{\ndbar}\) into $\cstW_k$.
    The term $\llnW_k(s)$ here comes from Corollary~\ref{cor:quantitative_noniid_clt_conditional} and corresponds to
    \begin{multline*}
    \Big\| \frac{1}{D}\sum_{d=1}^D \E[(U^d\bfH^d_{\wedge k}, V^d\bfB^d_{\wedge k-1})(s)(U^d\bfH^d_{\wedge k}, V^d\bfB^d_{\wedge k-1})(s)^\top|\Finit] - K_{U^1\bfH^1_{\wedge k}, V^1\bfB^1_{\wedge k-1}}(s,s)\Big\|_{op, 2k+1\times 2k+1},
    \end{multline*}
    with \(K_{U^1\bfH^1_{\wedge k}, V^1\bfB^1_{\wedge k-1}}(s,s)\) being the covariance of \((\bfZ^\bfH_{\wedge k}(s), \bfZ^\bfB_{\wedge k-1}(s))\), as in \eqref{eq:covariance_kernel_gaussian}. Using the Frobenius norm, which dominates the operator norm, and using Proposition~\ref{prop:skeleton_subgaussianity} as in Remark~\ref{rem:examples_llnW_terms}, this is indeed a subexponential LLN concentration term as in Definition~\ref{def:lln_terms_RW}, which we denote by $\llnW_k(s)$.
    
    \paragraph*{Bound on \(\blue{(\mathfrak{i})}\).} Using that $\rho$ is Lipschitz and the triangle and C-S inequalities,
    \[\|\blue{(\mathfrak{i})}\|_{\ndbar}\leq \E^{1/2}\big[|\tilde{P}_k^{\mathbf{Z}^H, \mathbf{Z}^B}(s) - P_k^{\mathbf{Z}^H, \mathbf{Z}^B}(s)|^2|\Finit \big] \E^{1/2}\big[\|G_k^{(D)}(s, \mathbf{Z}^{H}(s), \mathbf{Z}^{B}(s))\|_{\ndbar}^2|\Finit\big].\]
    By Lemma~\ref{lem:MF_skeleton_bounded_ndbar}, the second factor is bounded by $\cstW_k\E^{1/2}\big[|M_Z(s)|^2 |\Finit\big]$, with $M_Z(s) := 1\vee\max_{j\in[0:k-1]}|Z_j^h(s)|\vee|Z_j^b(s)|$. Furthermore, since $(\bfZ^H, \bfZ^B)$ is independent from $\Finit$ and has exponentially bounded moments -- see Remark~\ref{rem:GaussianHasExponentialMoments} -- we bound this term by a deterministic constant $\cstH_k>0$. That is,
    \[\|\blue{(\mathfrak{i})}\|_{\ndbar}\leq \cstH_k\E^{1/2}\big[|\tilde{P}_k^{\mathbf{Z}^H, \mathbf{Z}^B}(s) - P_k^{\mathbf{Z}^H, \mathbf{Z}^B}(s)|^2|\Finit \big].\]
   Using \eqref{eq:cleaner_system_Fk_definition} and C-S,
    \begin{multline*}
        |\tilde{P}_k^{\mathbf{Z}^H, \mathbf{Z}^B}(s) - P_k^{\mathbf{Z}^H, \mathbf{Z}^B}(s)|\\
     = \left| \frac{1}{D}\sum_{d=1}^{D} F_k^d(s, \mathbf{Z}^H(s), \mathbf{Z}^B(s)) H_k^d(s) - \E[F_k^1(s, \mathbf{Z}^H(s), \mathbf{Z}^B(s)) H_k^1(s)|\mathbf{Z}^H, \mathbf{Z}^B]\right|\\
     \leq \eta_u \big\|\vec{\mathbf{v}}_k^F(s, \bfZ^H, \bfZ^B)\big\|_{2,k} \Big\|\frac{1}{D}\sum_{d=1}^D H_k^d\mathbf{H}^d_{\wedge k-1}(s) - \E[H_k^d \bfH_{\wedge k-1}^d(s) ] \Big\|_{2,k}
    \leq \cstH_k M_Z \llnW_k(s),
    \end{multline*}
    where we used Lemma~\ref{lem:explicit_regskel} and identified a subexponential LLN concentration term.
    Using the independence and arbitrarily bounded moments of $(\bfZ^H, \bfZ^B)$, as well as the $\Finit$-measurability of $\llnW_k$, we obtain 
    \[\E^{1/2}[|\tilde{P}_k^{\mathbf{Z}^H, \mathbf{Z}^B}(s) - P_k^{\mathbf{Z}^H, \mathbf{Z}^B}(s)|^2|\Finit] \leq \cstH_k\E^{1/2}[M_Z^{2}|\Finit]\llnW_k(s) \leq \cstH_k \llnW_k(s).\]
    Gathering all contributions, we obtain the desired result
    \[\|\check{\H}^{(D)}_k ( s, \bfH^{(D)}  ) - \check{\overline{\H}}^{(D)}_k (s)\|_{\ndbar} \leq \|\blue{(\mathfrak{i})}\|_{\ndbar} + \|\blue{(\mathfrak{ii})}\|_{\ndbar}\leq \cstWh_k\Big(\frac{1}{\sqrt{D}} + \llnW_k(s)\Big).\]
\end{proof}

\begin{proof}[Proof of Lemma~\ref{lem:proof_cvg_cavity_TCL_H}]
    We need to compare, for $d\in[1:D]$ and $s\in[0,1]$:
    \begin{multline*}
        |\hat{\H}^d_k ( s, \bfH^{(D)}  ) - \hat{\overline{\H}}^d_k (s)| = \big|\E[\rho(P_k^{\mathbf{H}^D, \mathbf{B}^D}(s))\sqrt{D}V^d|\Finit] \\- \varv^2\E\big[ \rho'(P_k^{\mathbf{Z}^H, \mathbf{Z}^B}(s)) \E[H_k^1(s) \bnabla_{\bfz^{b}}F_k^1(s, \mathbf{Z}^H(s), \mathbf{Z}^B(s))|\mathbf{Z}^H, \mathbf{Z}^B]\big]\cdot \mathbf{B}_{\wedge k-1}^d(s)\big|,
    \end{multline*}
    where we have set
\begin{equation}\label{eq:proof_cvg_TCL_cavity_definition_P}
        \begin{cases}
        P_k^{\mathbf{H}^D, \mathbf{B}^D}(s) = \bfS^{\mathbf{H}^D}_k(s) + \big\langle F^{(D)}_k (s, \bfS^{\mathbf{H}^D}(s), \bfS^{\mathbf{B}^D}(s)), H_k^{(D)}(s)\big\rangle_{\ndbar},\\
        P_k^{\mathbf{Z}^H, \mathbf{Z}^B}(s) = \mathbf{Z}^{H}_k(s) + \E\big[ F^1_k (s, \mathbf{Z}^{H}(s), \mathbf{Z}^{B}(s))H_k^{1}(s)\big \vert \mathbf{Z}^{H}, \mathbf{Z}^{B}\big].
    \end{cases}
    \end{equation}
    
    \paragraph*{Cavity method.} We write a Taylor expansion for the $\Finit$-measurable function 
    \[ \phi:\R^{k+1}\times\R^{k} \to \R, \quad (\bfz^h, \bfz^b) \mapsto \rho(z_k^h + \langle F^{(D)}_k(s, \bfz^h, \bfz^b), H_k^{(D)}(s)\rangle_{\ndbar}). \]
    Since this function is twice continuously differentiable, we obtain, for all $\overline{\bfz}^b \in\R^k$,
    \[\phi(\bfz^h, \bfz^b) = \phi(\bfz^h, \overline{\bfz}^b) + \bnabla_{\bfz^b}\phi(\bfz^h, \overline{\bfz}^b)\cdot(\bfz^b-\overline{\bfz}^b) + r(\bfz^h, \bfz^b, \overline{\bfz}^b),\]
    where \(r(\bfz^h, \bfz^b, \overline{\bfz}^b) = \int_0^1(1-t) \bnabla_{\bfz^b}^2\phi(\bfz^h, t\bfz^b + (1-t)\overline{\bfz}^b) [\bfz^b -\overline{\bfz}^b]^{\otimes2}\d t\) is the integral remainder, which satisfies
    \begin{equation}\label{eq:proof_cvg_TCL_cavity_bound_integral_remainder}
        |r(\bfz^h, \bfz^b, \overline{\bfz}^b)| \leq \frac{1}{2}\sup_{t\in[0,1]} \|\bnabla_{\bfz^b}^2\phi(\bfz^h, t\bfz^b + (1-t)\overline{\bfz}^b)\|_{2,k\times k} \|\bfz^b -\overline{\bfz}^b\|_{2,k}^2.
    \end{equation}
    Computing the gradient and Hessian of $\phi$ explicitly, we obtain
    \[\bnabla_{\bfz^b}\phi(\bfz^h, \bfz^b) = \rho'(\tilde{p}_k(s, \bfz^h, \bfz^b))\frac{1}{D} \bnabla_{\bfz^b}F^{(D)}_k(s, \bfz^h, \bfz^b) H_k^{(D)}(s),\]
    \begin{align*}
        \bnabla_{\bfz^b}^2\phi(\bfz^h, \bfz^b) &= \frac{\rho''(\tilde{p}_k(s, \bfz^h, \bfz^b))}{D^2}(\bnabla_{\bfz^b}F^{(D)}_k(s, \bfz^h, \bfz^b) H_k^{(D)}(s))(\bnabla_{\bfz^b}F^{(D)}_k(s, \bfz^h, \bfz^b) H_k^{(D)}(s))^{\top}\\
        &\qquad + \rho'(\tilde{p}_k(s, \bfz^h, \bfz^b)) \frac{1}{D}\sum_{d=1}^D H_k^{d}(s) \bnabla^2_{\bfz^b}F^{d}_k(s, \bfz^h, \bfz^b), 
    \end{align*}
    where we used the shorthand $\tilde{p}_k(s, \bfz^h, \bfz^b) := z_k^h + \langle F^{(D)}_k(s, \bfz^h, \bfz^b), H_k^{(D)}(s)\rangle_{\ndbar}$. From Lemma~\ref{lem:RegSkel} it follows that, for a constant $\cstHhat_k>0$,
    \begin{equation}\label{eq:proof_cvg_TCL_cavity_bound_hessian_norm}
        \|\bnabla_{\bfz^b}^2\phi(\bfz^h, \bfz^b)\|_{2,k\times k} \leq e^{\hat{C}_k(1 \vee \|(\bfz^h, \bfz^b)\|_{\infty, 2k})} \bigg(\big\|\max_{\substack{j\in[0:k-1]\\t\in[0,1]}} |H_j^{(D)}(t)|\big\|_{\ndbar}^2 + \big\|\max_{\substack{j\in[0:k-1]\\t\in[0,1]}} |H_j^{(D)}(t)|\big\|_{\ndbar}\bigg).
    \end{equation}
    Since \[\E[\rho(P_k^{\mathbf{H}^D, \mathbf{B}^D}(s))\sqrt{D}V^d|\Finit] = \E[\phi(\bfS^{\mathbf{H}^D}, \bfS^{\mathbf{B}^D})\sqrt{D}V^d|\Finit],\] we apply our Taylor expansion around \(\overline{\bfS}^{d,\mathbf{B}^D} = \bfS^{\mathbf{B}^D}-\frac{1}{\sqrt{D}}V^d\mathbf{B}_{\wedge k-1}^d(s)\), which yields
    \begin{align*}
        \E[\phi(\bfS^{\mathbf{H}^D}, \bfS^{\mathbf{B}^D})\sqrt{D}V^d|\Finit] &= \E[\phi(\bfS^{\mathbf{H}^D}, \overline{\bfS}^{d,\mathbf{B}^D})\sqrt{D}V^d|\Finit]\\&\qquad+\E[\bnabla_{\bfz^b}\phi(\bfS^{\mathbf{H}^D}, \overline{\bfS}^{d,\mathbf{B}^D})\cdot\big(\frac{1}{\sqrt{D}}V^d\mathbf{B}_{\wedge k-1}^d(s)\big)\sqrt{D}V^d|\Finit]\\&\qquad + \E[r(\bfS^{\mathbf{H}^D}, \bfS^{\mathbf{B}^D}, \overline{\bfS}^{d,\mathbf{B}^D})\sqrt{D}V^d|\Finit]\\
        &=: \purple{(\mathfrak{i})_d} + \purple{(\mathfrak{ii})_d} + \purple{(\mathfrak{iii})_d}.
    \end{align*}
    Since $V^d$ is centered and independent from \((\bfS^{\mathbf{H}^D}, \overline{\bfS}^{d,\mathbf{B}^D})\), even conditionally to $\Finit$, we get 
    \[\purple{(\mathfrak{i})_d}=0,\;\;\text{ and }\;\;\purple{(\mathfrak{ii})_d} = \varv^2\E[\bnabla_{\bfz^b}\phi(\bfS^{\mathbf{H}^D}, \overline{\bfS}^{d,\mathbf{B}^D})|\Finit]\cdot \mathbf{B}_{\wedge k-1}^d(s).\]
    Using \eqref{eq:proof_cvg_TCL_cavity_bound_integral_remainder}, C-S, Lemma~\ref{prop:moment_bound_subgaussian} for the subgaussian variable $V^d$ and \eqref{eq:proof_cvg_TCL_cavity_bound_hessian_norm}, we get
    \begin{align}
        |\purple{(\mathfrak{iii})_d}|&\leq \frac{\sqrt{D}}{2}\E\Big[\sup_{t\in[0,1]} \|\bnabla_{\bfz^b}^2\phi(\bfS^{\mathbf{H}^D}, t\bfS^{\mathbf{B}^D} + (1-t)\overline{\bfS}^{d,\mathbf{B}^D})\|_{2,k\times k} \big\| \frac{1}{\sqrt{D}}V^d\mathbf{B}_{\wedge k-1}^d(s)\big\|_{2,k}^2 |V^d|\big|\Finit\Big]\notag\\
        &\leq \frac{1}{2\sqrt{D}}\E\Big[\sup_{t\in[0,1]} \|\bnabla_{\bfz^b}^2\phi(\bfS^{\mathbf{H}^D}, t\bfS^{\mathbf{B}^D} + (1-t)\overline{\bfS}^{d,\mathbf{B}^D})\|_{2,k\times k} |V^d|^3\big|\Finit\Big] \big\|\mathbf{B}_{\wedge k-1}^d(s)\big\|_{2,k}^2\notag\\
        &\leq \frac{C \vpv^3}{\sqrt{D}}\E^{1/2}\Big[\sup_{t\in[0,1]} \|\bnabla_{\bfz^b}^2\phi(\bfS^{\mathbf{H}^D}, t\bfS^{\mathbf{B}^D} + (1-t)\overline{\bfS}^{d,\mathbf{B}^D})\|_{2,k\times k}^2 \big|\Finit\Big] \big\|\mathbf{B}_{\wedge k-1}^d(s)\big\|_{2,k}^2\notag\\
        &\leq \frac{C \vpv^3}{\sqrt{D}}\E^{1/2}\Big[\sup_{t\in[0,1]} e^{\hat{C}_k M_{t,S}} \big|\Finit\Big] \Big(1\vee \big\|\max_{\substack{j\in[0:k-1]\\t\in[0,1]}} |H_j^{(D)}(t)|\big\|_{\ndbar}^2 \Big)\big\|\mathbf{B}_{\wedge k-1}^d(s)\big\|_{2,k}^2\notag\\
        &\leq \frac{\cstW_k}{\sqrt{D}}\E^{1/2}\big[\sup_{t\in[0,1]} e^{\hat{C}_k M_{t,S}} \big|\Finit\big]\big\|\mathbf{B}_{\wedge k-1}^d(s)\big\|_{2,k}^2,
    \end{align}
    where we introduced \(M_{t,S} := \max_{j\in[0:k-1]} |S_j^{\mathbf{H}^D}(s)| + t|S_j^{\mathbf{B}^D}(s)| + (1-t)|\overline{S}_j^{d,\mathbf{B}^D}(s)|\), and we used  Lemma~\ref{lem:MF_skeleton_bounded_ndbar} in the last inequality. For any $t\in[0,1]$, \[M_{t,S}\leq M_{S, \overline{S}} := \max_{j\in[0:k-1]} |S_j^{\mathbf{H}^D}(s)| + |S_j^{\mathbf{B}^D}(s)| + |\overline{S}_j^{d,\mathbf{B}^D}(s)|,\] and hence Lemma~\ref{lem:exp_moment_magic} allows us to absorb  $\E^{1/2}\big[e^{C M_{S, \overline{S}}} |\Finit\big]$ into $\cstW_k$, finally yielding \begin{equation}\label{eq:proof_cvg_TCL_cavity_iii}
        |\purple{(\mathfrak{iii})_d}|\leq \frac{\cstW_k}{\sqrt{D}}\big\|\mathbf{B}_{\wedge k-1}^d(s)\big\|_{2,k}^2.
    \end{equation}
    
    \paragraph*{Quantitative CLT.} Gathering terms,
    \begin{multline*}
        |\hat{\H}^d_k ( s, \bfH^{(D)}  ) - \hat{\overline{\H}}^d_k (s)|
        \\\leq \big|\purple{(\mathfrak{ii})_d}
        - \varv^2\E\big[ \rho'(P_k^{\mathbf{Z}^H, \mathbf{Z}^B}(s)) \E[H_k^1(s) \bnabla_{\bfz^{b}}F_k^1(s, \mathbf{Z}^H(s), \mathbf{Z}^B(s))|\mathbf{Z}^H, \mathbf{Z}^B]\big]\cdot \mathbf{B}_{\wedge k-1}^d(s)\big| + \big|\purple{(\mathfrak{iii})_d}\big|\\
        \leq \varv^2\big\|\E[\bnabla_{\bfz^b}\phi(\bfS^{\mathbf{H}^D}(s), \overline{\bfS}^{d,\mathbf{B}^D}(s))|\Finit] - \E[\psi(\mathbf{Z}^H(s), \mathbf{Z}^B(s))]\big\|_{2,k}\|\mathbf{B}_{\wedge k-1}^d(s)\|_{2,k} + |\purple{(\mathfrak{iii})_d}|,
    \end{multline*}
    where in the last line we used C-S and introduced the function $\psi:\R^{k+1}\times\R^k\to\R^k$, given by \[\psi(\bfz^h, \bfz^b) =\rho'(z_k^h + \E[F_k^1(s, \bfz^h, \bfz^b)H_k^1(s)])\E[H_k^1(s)\bnabla_{\bfz^b}F_k^1(s, \bfz^h, \bfz^b)].\]
    Taking the RMS norm, we obtain
    \begin{multline}\label{eq:proof_cvg_TCL_cavity_bound_RMS}
        \Big\|\hat{\H}^{(D)}_k ( s, \bfH^{(D)}  ) - \hat{\overline{\H}}^{(D)}_k (s)\Big\|_{\ndbar}\\\leq \varv^2\Big\|\E[\bnabla_{\bfz^b}\phi(\bfS^{\mathbf{H}^D}(s), \overline{\bfS}^{d,\mathbf{B}^D}(s))|\Finit] - \E[\psi(\mathbf{Z}^H(s), \mathbf{Z}^B(s))]\Big\|_{2,k}\Big\|\|\mathbf{B}_{\wedge k-1}^{(D)}(s)\|_{2,k}\Big\|_{\ndbar} \\+ \frac{\cstW_k}{\sqrt{D}}\Big\|\big\|\mathbf{B}_{\wedge k-1}^{(D)}(s)\big\|_{2,k}^2\Big\|_{\ndbar}\\
        \leq \cstW_k\Big\|\E[\bnabla_{\bfz^b}\phi(\bfS^{\mathbf{H}^D}(s), \overline{\bfS}^{d,\mathbf{B}^D}(s))|\Finit] - \E[\psi(\mathbf{Z}^H(s), \mathbf{Z}^B(s))]\Big\|_{2,k}+ \frac{\cstWh_k}{\sqrt{D}},
    \end{multline}
    where we noted that, from Proposition~\ref{prop:explicit_dependence_on_WeWu},
    \begin{multline*}
        \Big\|\big\|\mathbf{B}_{\wedge k-1}^{(D)}(s)\big\|_{2,k}^2\Big\|_{\ndbar} =\Big( \frac{1}{D}\sum_{d=1}^D \Big(\sum_{j=0}^{k-1} |B_j^d(s)|^2\Big)^2\Big)^{1/2}
        \leq k \Big( \frac{1}{D}\sum_{d=1}^D \max_{j\in[0:k-1]}|B_j^d(s)|^4\Big)^{1/2}\\ \leq \cstH_k\Big( \frac{1}{D}\sum_{d=1}^D \| W_{\mathrm{in}}^d\|^4_{2, \dimin} + \| W_{\mathrm{out}}^d\|^4_{2, \dimout}\Big)^{1/2} =: \cstWh_k, 
    \end{multline*}
    and $\cstWh_k$ now further depends on $\|W_{\mathrm{in}}^{(D)}\|_{\overline{4}, D}$ and $\|W_{\mathrm{out}}^{(D)}\|_{\overline{4}, D}$. In \eqref{eq:proof_cvg_TCL_cavity_bound_RMS} we are reduced to studying
\begin{multline*}
    \Big\|\E[\bnabla_{\bfz^b}\phi(\bfS^{\mathbf{H}^D}(s), \overline{\bfS}^{d,\mathbf{B}^D}(s))|\Finit] - \E[\psi(\mathbf{Z}^H(s), \mathbf{Z}^B(s))]\Big\|_{2,k}\\
    \leq \big\|\E[\bnabla_{\bfz^b}\phi(\bfS^{\mathbf{H}^D}(s), \overline{\bfS}^{d,\mathbf{B}^D}(s))|\Finit] - \E[\psi(\bfS^{\mathbf{H}^D}(s), \overline{\bfS}^{d,\mathbf{B}^D}(s))\big|\Finit] \big\|_{2,k}\\+\big\|\E[\psi(\bfS^{\mathbf{H}^D}(s), \overline{\bfS}^{d,\mathbf{B}^D}(s))\big|\Finit] -\E[\psi(\mathbf{Z}^H(s), \mathbf{Z}^B(s))]\big\|_{2,k}\\
    =: \purple{(i)} + \purple{(ii)},
\end{multline*}
where we introduced the pivot \(\E[\psi(\bfS^{\mathbf{H}^D}(s), \overline{\bfS}^{d,\mathbf{B}^D}(s))\big|\Finit]\).

The term \purple{(ii)} is controlled as in \eqref{eq:proof_cvg_TCL_bound_ii} from the proof of Lemma~\ref{lem:proof_cvg_TCL_H}, using the quantitative CLT result from Corollary~\ref{cor:quantitative_noniid_clt_conditional}. Indeed, $\psi$ is deterministic and has $\cstHhat_k$-exponentially growing derivatives of orders up to 3, so we get a constant $\cstWh_k>0$ depending only on $\cstHhat_k, k, \vpu, \vpv$, $\|W_{\mathrm{in}}^{(D)}\|_{\ndbar}$, $ \|W_{\mathrm{out}}^{(D)}\|_{\ndbar}$, \(\|W_{\mathrm{in}}^{(D)}\|_{\overline{3}, D}\) and \(\|W_{\mathrm{out}}^{(D)}\|_{\overline{3}, D}\), such that
\[\purple{(ii)} =\big\|\E[\psi(\bfS^{\mathbf{H}^D}(s), \overline{\bfS}^{d,\mathbf{B}^D}(s))\big|\Finit] -\E[\psi(\mathbf{Z}^H(s), \mathbf{Z}^B(s))]\big\|_{2,k} \leq \cstWh_k\Big(\frac{1}{\sqrt{D}} + \llnW_k(s)\Big).\]

 \paragraph*{Bounding the remainder.} 
To handle \purple{(i)}, we use the skeleton vectors from Definition~\ref{def:linear_skeleton_meanfield} to obtain, for all $\bfz^h, \bfz^b \in \R^K$, 
\[\bnabla_{\bfz^b}\phi(\bfz^h, \bfz^b) = -\eta_u\rho'\Big(z_k^h -\eta_u \vec{\mathbf{v}}_k^F(s, \bfz^h, \bfz^b)^\top \Big(\frac{1}{D} \sum_{d=1}^D\bfH^d_{\wedge k-1}(s) H_k^d\Big) \Big)\bnabla_{\bfz^b}\vec{\mathbf{v}}_k^F(s, \bfz^h, \bfz^b)\frac{1}{D}\sum_{d=1}^D \bfH_{\wedge k-1}^d(s) H_k^{d}(s),\]
where we used \(\big\langle F_k^{(D)}(s, \bfz^h, \bfz^b), H_k^{(D)}(s)\big\rangle_{\ndbar} =  -\eta_u \vec{\mathbf{v}}_k^F(s, \bfz^h, \bfz^b)^\top \Big(\frac{1}{D} \sum_{d=1}^D\bfH^d_{\wedge k-1}(s) H_k^d\Big)\) and similarly for \(\bnabla_{\bfz^b}F_k^{(D)}(s, \bfz^h, \bfz^b)\). Also, since \(\vec{\mathbf{v}}_k^F\) depends solely on the law of $(\bfH^d, \bfB^d)$, it is the same for all $d\in[1:D]$, since they are i.i.d.~Similarly,
\[\psi(\bfz^h, \bfz^b) =-\eta_u\rho'\big(z_k^h -\eta_u \vec{\mathbf{v}}_k^F(s, \bfz^h, \bfz^b)^\top \E[H_k^1(s)\bfH_{\wedge k-1}^1(s)]\big)\bnabla_{\bfz^b}\vec{\mathbf{v}}_k^F(s, \bfz^h, \bfz^b)\E[H_k^1(s)\bfH_{\wedge k-1}^1(s)].\]
We then define
\begin{equation}\label{eq:proof_cvg_TCL_cavity_def_terms_with_v}
    \begin{cases}
    \hat{P}_k^{\bfH^D, \overline{\bfB}^D}(s) := S_k^{\mathbf{H}^D}(s) -\eta_u \vec{\mathbf{v}}_k^F(s, \bfS^{\mathbf{H}^D}(s), \overline{\bfS}^{d,\mathbf{B}^D}(s))^\top \Big(\frac{1}{D} \sum_{d=1}^D\bfH^d_{\wedge k-1}(s) H_k^d\Big),\\
    \check{P}_k^{\bfH^D, \overline{\bfB}^D}(s) := S_k^{\mathbf{H}^D}(s)  -\eta_u \vec{\mathbf{v}}_k^F(s, \bfS^{\mathbf{H}^D}(s), \overline{\bfS}^{d,\mathbf{B}^D}(s))^\top\E\big[ \bfH^1_{\wedge k-1}(s) H_k^{1}(s)\big],
\end{cases}
\end{equation}
and introduce a pivot to obtain
\begin{multline*}
    \purple{(i)} =\eta_u\bigg\|\E\Big[\rho'\big(\hat{P}_k^{\bfH^D, \overline{\bfB}^D}(s) \big) \bnabla_{\bfz^b}\vec{\mathbf{v}}_k^F(s, \bfS^{\mathbf{H}^D}(s), \overline{\bfS}^{d,\mathbf{B}^D}(s))\frac{1}{D}\sum_{d=1}^D \bfH_{\wedge k-1}^d(s) H_k^{d}(s)\big| \Finit\Big]
    \\ -\E\Big[\rho'\big(\check{P}_k^{\bfH^D, \overline{\bfB}^D}(s) \big) \bnabla_{\bfz^b}\vec{\mathbf{v}}_k^F(s, \bfS^{\mathbf{H}^D}(s), \overline{\bfS}^{d,\mathbf{B}^D}(s))\E\big[ H_k^{1}(s)\bfH_{\wedge k-1}^1(s)\big] \big|\Finit\Big]\Big\|_{2,k}\\
    \leq\eta_u\Big\|\E\Big[\Big(\rho'\big(\hat{P}_k^{\bfH^D, \overline{\bfB}^D}(s) \big) - \rho'\big(\check{P}_k^{\bfH^D, \overline{\bfB}^D}(s) \big)\Big) \bnabla_{\bfz^b}\vec{\mathbf{v}}_k^F(s, \bfS^{\mathbf{H}^D}(s), \overline{\bfS}^{d,\mathbf{B}^D}(s))\frac{1}{D}\sum_{d=1}^D \bfH_{\wedge k-1}^d(s) H_k^{d}(s)\big| \Finit\Big]\Big\|_{2,k} \\ +\eta_u\Big\|\E\Big[\rho'\big(\check{P}_k^{\bfH^D, \overline{\bfB}^D}(s) \big) \bnabla_{\bfz^b}\vec{\mathbf{v}}_k^F(s, \bfS^{\mathbf{H}^D}(s), \overline{\bfS}^{d,\mathbf{B}^D}(s))\Big(\frac{1}{D}\sum_{d=1}^D \bfH_{\wedge k-1}^d(s) H_k^{d}(s) - \E\big[ H_k^{1}(s)\bfH_{\wedge k-1}^1(s)\big]\Big)\big|\Finit\Big]\Big\|_{2,k}.
\end{multline*}
We then use that $\rho'$ is bounded and Lipschitz, we apply the triangle and C-S inequalities, as well as the $\Finit$-measurability of $\bfH$ and Lemma~\ref{lem:explicit_regskel} to obtain
\begin{align}\label{eq:proof_cvg_TCL_cavity_bound_llnTerm}
    &\purple{(i)} \leq 
    C_\rho\eta_u \Big\| \frac{1}{D}\sum_{d=1}^D \bfH_{\wedge k-1}^d(s) H_k^{(D)}(s) - \E\big[ H_k^{1}(s)\bfH_{\wedge k-1}^1(s)\big]\Big\|_{2,k} \Big\|\frac{1}{D}\sum_{d=1}^D \bfH_{\wedge k-1}^d(s) H_k^{d}(s)\Big\|_{2,k}\notag\\
    &\qquad\qquad\qquad\qquad\cdot\E\big[\big|\vec{\mathbf{v}}_k^F(s, \bfS^{\mathbf{H}^D}(s), \overline{\bfS}^{d,\mathbf{B}^D}(s))\big|_{2,k}\big\| \bnabla_{\bfz^b}\vec{\mathbf{v}}_k^F(s, \bfS^{\mathbf{H}^D}(s), \overline{\bfS}^{d,\mathbf{B}^D}(s))\big\|_{2, k\times k}|\Finit\big]\notag\\
    & +C_\rho\eta_u \E\big[ \big\|\bnabla_{\bfz^b}\vec{\mathbf{v}}_k^F(s, \bfS^{\mathbf{H}^D}(s), \overline{\bfS}^{d,\mathbf{B}^D}(s))\big\|_{2,k\times k} | \Finit \big]\Big\|\frac{1}{D}\sum_{d=1}^D \bfH_{\wedge k-1}^d(s) H_k^{(D)}(s) - \E\big[ H_k^{1}(s)\bfH_{\wedge k-1}^1(s)\big]\Big\|_{2,k}\notag\\
    &\leq \cstW_k \llnW_k(s) \E^{1/2}\big[e^{\cstHhat_kM_{S, \overline{S}}}|\Finit\big]\frac{1}{D}\sum_{d'=1}^D\max_{\substack{j\in[0:k]\\s\in[0,1]}}|H_j^{d'}(s)|^2+\cstW_k \llnW_k(s) \E^{1/2}\big[e^{\cstHhat_kM_{S, \overline{S}}}|\Finit\big]\notag\\
    &\leq \cstW_k\llnW_k(s),
\end{align}
where we absorbed terms into $\llnW_k(s)$ and $\cstW_k$ using Lemma~\ref{lem:MF_skeleton_bounded_ndbar}. Joining all contributions yields
\begin{equation*}
    \big\|\hat{\H}^{(D)}_k ( s, \bfH^{(D)}  ) - \hat{\overline{\H}}^{(D)}_k (s)\big\|_{\ndbar}\leq \cstW_k(\purple{(i)} + \purple{(ii)})+ \frac{\cstWh_k}{\sqrt{D}}\leq \cstWh_k\big(\frac{1}{\sqrt{D}} + \llnW_k(s)\big),
\end{equation*}
which is what we wanted to prove.
\end{proof}
    
    \begin{proof}[Proof of Lemma~\ref{lem:proof_cvg_prop_caos_B}]
        The argument is analogous to that of Lemma~\ref{lem:proof_cvg_prop_caos_H}. We decompose
        \begin{align*}
            \|&\B^{(D)}(s, \bfb^{(D)}) - \B^{(D)}(s, \bfB^{(D)})\|_{\ndbar} \leq\\
            & +\Big\|\E\Big[\rho'(P_k^{\mathbf{H}^D, \mathbf{B}^D})Q_k^{\mathbf{H}^D, \mathbf{B}^D} \big( f_k^{(D)}(s, \bfS^{\mathbf{h}^D}(s), \bfS^{\mathbf{b}^D}(s)) - F_k^{(D)}(s, \bfS^{\mathbf{h}^D}(s), \bfS^{\mathbf{b}^D}(s))\big)|\Finit\Big] \Big\|_{\ndbar}\\
            & +\Big\|\E\Big[\rho'(P_k^{\mathbf{H}^D, \mathbf{B}^D})Q_k^{\mathbf{H}^D, \mathbf{B}^D} \big( F_k^{(D)}(s, \bfS^{\mathbf{h}^D}(s), \bfS^{\mathbf{b}^D}(s)) - F_k^{(D)}(s, \bfS^{\mathbf{H}^D}(s), \bfS^{\mathbf{B}^D}(s))\big)|\Finit\Big] \Big\|_{\ndbar}\\
            & +\Big\|\E\Big[\big(\rho'(P_k^{\mathbf{h}^D, \mathbf{b}^D})Q_k^{\mathbf{h}^D, \mathbf{b}^D} - \rho'(P_k^{\mathbf{H}^D, \mathbf{B}^D})Q_k^{\mathbf{H}^D, \mathbf{B}^D}\big)\sqrt{D}U^{(D)}|\Finit\Big] \Big\|_{\ndbar}\\
            & +\Big\|\E\Big[\big(\rho'(P_k^{\mathbf{h}^D, \mathbf{b}^D})Q_k^{\mathbf{h}^D, \mathbf{b}^D} - \rho'(P_k^{\mathbf{H}^D, \mathbf{B}^D})Q_k^{\mathbf{H}^D, \mathbf{B}^D}\big)f_k^{(D)}(s, \bfS^{\mathbf{h}^D}(s), \bfS^{\mathbf{b}^D}(s))|\Finit\Big] \Big\|_{\ndbar}\\
            &=: \|\blue{(i)}\|_{\ndbar} + \|\blue{(ii)}\|_{\ndbar} + \|\blue{(iii)}\|_{\ndbar} + \|\blue{(iv)}\|_{\ndbar},
        \end{align*}
        where we have posed \(P_k^{\mathbf{h}^D, \mathbf{b}^D}(s)\) and \(P_k^{\mathbf{H}^D, \mathbf{B}^D}(s)\) as in \eqref{eq:proof_cvg_prop_caos_definition_P}, and
    \[\begin{cases}
     Q_k^{\mathbf{h}^D, \mathbf{b}^D}(s) := \bfS^{\mathbf{b}^D}_k(s) + \big\langle g^{(D)}_k (s, \bfS^{\mathbf{h}^D}(s), \bfS^{\mathbf{b}^D}(s)), b_k^{(D)}(s)\big\rangle_{\ndbar}, \\  
     Q_k^{\mathbf{H}^D, \mathbf{B}^D}(s) := \bfS^{\mathbf{B}^D}_k(s) + \big\langle G^{(D)}_k (s, \bfS^{\mathbf{H}^D}(s), \bfS^{\mathbf{B}^D}(s)), B_k^{(D)}(s)\big\rangle_{\ndbar}.
    \end{cases}\]
        Each term is bounded exactly as in the proof of Lemma~\ref{lem:proof_cvg_prop_caos_H}. The only small difference comes in \blue{(iii)} and \blue{(iv)}, where we have
        \begin{multline*}
            \Big|\rho'(P_k^{\mathbf{h}^D, \mathbf{b}^D})Q_k^{\mathbf{h}^D, \mathbf{b}^D} - \rho'(P_k^{\mathbf{H}^D, \mathbf{B}^D})Q_k^{\mathbf{H}^D, \mathbf{B}^D}\Big|\\
            \leq\big|\rho'(P_k^{\mathbf{h}^D, \mathbf{b}^D}) - \rho'(P_k^{\mathbf{H}^D, \mathbf{B}^D})\big|\big| Q_k^{\mathbf{h}^D, \mathbf{b}^D}\big| + \big|\rho'(P_k^{\mathbf{H}^D, \mathbf{B}^D})\big| \big|Q_k^{\mathbf{h}^D, \mathbf{b}^D} - Q_k^{\mathbf{H}^D, \mathbf{B}^D}\big|\\
            \leq C_\rho \big|P_k^{\mathbf{h}^D, \mathbf{b}^D}-P_k^{\mathbf{H}^D, \mathbf{B}^D}\big|\big| Q_k^{\mathbf{h}^D, \mathbf{b}^D}\big| + C_\rho \big|Q_k^{\mathbf{h}^D, \mathbf{b}^D} - Q_k^{\mathbf{H}^D, \mathbf{B}^D}\big|,
        \end{multline*}
        using that $\rho'$ is bounded and Lipchitz of $\rho'$. 
        Since the factors in the first term can be correlated, we use Lemma~\ref{lem:fine_decorrelation_lemma} to bound \blue{(iii)}, instead of Lemma~\ref{lem:decorrelation_chizat}.
        This yields
        \begin{multline*}
            \|\blue{(iii)}\|_{\ndbar} \leq \cstW_k\Big(\E^{1/2}\big[\big|P_k^{\mathbf{h}^D, \mathbf{b}^D}-P_k^{\mathbf{H}^D, \mathbf{B}^D}\big|^2 | \Finit \big]\E^{1/4}\big[\big|Q_k^{\mathbf{h}^D, \mathbf{b}^D}\big|^4 |\Finit \big]  \\+ \E^{1/2}\big[\big|Q_k^{\mathbf{h}^D, \mathbf{b}^D}-Q_k^{\mathbf{H}^D, \mathbf{B}^D}\big|^2 |\Finit \big]  \Big)\\
            \leq \cstW_k\Big( \underbrace{\E^{1/2}\big[\big|P_k^{\mathbf{h}^D, \mathbf{b}^D}-P_k^{\mathbf{H}^D, \mathbf{B}^D}\big|^2 | \Finit \big]}_{\blue{(*)_1}}+ \underbrace{\E^{1/2}\big[\big|Q_k^{\mathbf{h}^D, \mathbf{b}^D}-Q_k^{\mathbf{H}^D, \mathbf{B}^D}\big|^2 |\Finit \big]}_{\blue{(*)_2}}  \Big),
        \end{multline*}
        where we absorbed \(\E^{1/4}\big[\big|Q_k^{\mathbf{h}^D, \mathbf{b}^D}\big|^4 |\Finit \big]\) into $\cstW_k$ by the same argument as in \eqref{eq:proof_cvg_prop_caos_bound_P}. On the other hand, we bound \(\blue{(*)_1}\) as in \eqref{eq:proof_cvg_prop_caos_lip_P} and similarly for \(\blue{(*)_2}\). This finally yields
        \begin{multline*}
    \|\blue{(iii)}\|_{\ndbar}\leq \cstW_k \bigg(\big\|h_k^{(D)}(s) - H_k^{(D)}(s)\big\|_{\ndbar} + \big\|b_k^{(D)}(s) - B_k^{(D)}(s)\big\|_{\ndbar}
    \\+ \sum_{i=0}^{k-1} \|h_i^{(D)}(s) - H_i^{(D)}(s)\|_{\ndbar} + \|b_i^{(D)}(s) - B_i^{(D)}(s)\|_{\ndbar}\\
    +\E^{1/2}\big[\big\| g^{(D)}_k (s, \bfS^{\mathbf{h}^D}(s), \bfS^{\mathbf{b}^D}(s))- G^{(D)}_k (s, \bfS^{\mathbf{h}^D}(s), \bfS^{\mathbf{b}^D}(s))\big\|_{\ndbar}^2\big|\Finit  \big] \\
    +\E^{1/2}\big[\big\| f^{(D)}_k (s, \bfS^{\mathbf{h}^D}(s), \bfS^{\mathbf{b}^D}(s))- F^{(D)}_k (s, \bfS^{\mathbf{h}^D}(s), \bfS^{\mathbf{b}^D}(s))\big\|_{\ndbar}^2\big|\Finit  \big]\bigg).
\end{multline*}
The term $\|\blue{(iv)}\|_{\ndbar}$ satisfies the exact same bound, using C-S and the calculations in \eqref{eq:proof_cvg_prop_caos_bound_f}.
For \blue{(i)}, we use C-S to bound
        \begin{multline*}
            \|\blue{(i)}\|_{\ndbar} \leq \E^{1/2}\big[\big|\rho'(P_k^{\mathbf{H}^D, \mathbf{B}^D})Q_k^{\mathbf{H}^D, \mathbf{B}^D}\big|^2 | \Finit \big]\\ \cdot \E^{1/2}\big[\big\| f^{(D)}_k (s, \bfS^{\mathbf{h}^D}(s), \bfS^{\mathbf{b}^D}(s))- F^{(D)}_k (s, \bfS^{\mathbf{h}^D}(s), \bfS^{\mathbf{b}^D}(s))\big\|_{\ndbar}^2\big|\Finit  \big],
        \end{multline*}
        where the first term is absorbed into \(\cstW_k\) and the second one is bounded using Proposition~\ref{prop:propcaos_Dbar_whp} at the end of the proof. 
        The term \blue{(ii)} is handled exactly as in \eqref{eq:proof_cvg_prop_caos_bound_ii}, 
        using that $F_k^d$ is Lipschitz to obtain
        \[\|\blue{(ii)}\|_{\ndbar}\leq \cstW_k\sum_{i=0}^{k-1} \|h_i^{(D)}(s) - H_i^{(D)}(s) \|_{\ndbar} + \|b_i^{(D)}(s) - B_i^{(D)}(s) \|_{\ndbar}.\]
        Summing up all contributions, we obtain
        \begin{multline*}
        \|\B^{(D)}(s, \bfb^{(D)}) - \B^{(D)}(s, \bfB^{(D)})\|_{\ndbar} \leq \cstW_k \bigg(\big\|h_k^{(D)}(s) - H_k^{(D)}(s)\big\|_{\ndbar} + \big\|b_k^{(D)}(s) - B_k^{(D)}(s)\big\|_{\ndbar}
    \\+ \sum_{i=0}^{k-1} \|h_i^{(D)}(s) - H_i^{(D)}(s)\|_{\ndbar} + \|b_i^{(D)}(s) - B_i^{(D)}(s)\|_{\ndbar}\\
    +\E^{1/2}\big[\big\| g^{(D)}_k (s, \bfS^{\mathbf{h}^D}(s), \bfS^{\mathbf{b}^D}(s))- G^{(D)}_k (s, \bfS^{\mathbf{h}^D}(s), \bfS^{\mathbf{b}^D}(s))\big\|_{\ndbar}^2\big|\Finit  \big] \\
    +\E^{1/2}\big[\big\| f^{(D)}_k (s, \bfS^{\mathbf{h}^D}(s), \bfS^{\mathbf{b}^D}(s))- F^{(D)}_k (s, \bfS^{\mathbf{h}^D}(s), \bfS^{\mathbf{b}^D}(s))\big\|_{\ndbar}^2\big|\Finit  \big]\bigg).
\end{multline*}
We now use Proposition~\ref{prop:propcaos_Dbar_whp} as in \eqref{eq:proof_cvg_prop_caos_bound_prefinal} to get
\begin{multline*}
    \|\B^{(D)}(s, \bfb^{(D)}) - \B^{(D)}(s, \bfB^{(D)})\|_{\ndbar} \leq \cstW_k \bigg(\big\|h_k^{(D)}(s) - H_k^{(D)}(s)\big\|_{\ndbar} + \big\|b_k^{(D)}(s) - B_k^{(D)}(s)\big\|_{\ndbar}
    \\+ \sum_{i=0}^{k-1} \|h_i^{(D)}(s) - H_i^{(D)}(s)\|_{\ndbar} + \|b_i^{(D)}(s) - B_i^{(D)}(s)\|_{\ndbar} + \sum_{j=0}^{k-1}\llnW_j(s)\bigg).
\end{multline*}
Finally, using our inductive hypothesis \eqref{eq:proof_cvg_inductive_hypothesis} we obtain the desired result
\begin{multline*}
    \|\B^{(D)}(s, \bfb^{(D)}) - \B^{(D)}(s, \bfB^{(D)})\|_{\ndbar} \\\leq \cstW_k \Big(\big\|h_k^{(D)}(s) - H_k^{(D)}(s)\big\|_{\ndbar} + \big\|b_k^{(D)}(s) - B_k^{(D)}(s)\big\|_{\ndbar}
    + \frac{1}{\sqrt{D}} + \sum_{j=0}^{k-1} \int_0^1\llnW_j(u)\d u + \sum_{j=0}^{k-1}\llnW_j(s)\Big).
\end{multline*}
    \end{proof}

    \begin{proof}[Proof of Lemma~\ref{lem:proof_cvg_TCL_B}]
        As in the proof of Lemma~\ref{lem:proof_cvg_TCL_H}, we have that
        \begin{multline*}
            \check{\B}^{(D)}_k ( s, \bfB^{(D)} ) - \check{\overline{\B}}^{(D)}_k (s) =\E[\rho'(P_k^{\mathbf{H}^D, \mathbf{B}^D}(s)) Q_k^{\mathbf{H}^D, \mathbf{B}^D}(s) F_k^{(D)}(s, \bfS^{\mathbf{H}^D}(s), \bfS^{\mathbf{B}^D}(s))|\Finit] \\- \E[\rho'(P_k^{\mathbf{Z}^H, \mathbf{Z}^B}(s)) Q_k^{\mathbf{Z}^H, \mathbf{Z}^B}(s) F_k^{(D)}(s, \mathbf{Z}^{H}(s), \mathbf{Z}^{B}(s))|\Finit],
        \end{multline*}
        where
        \[\begin{cases}
            P_k^{\mathbf{H}^D, \mathbf{B}^D}(s) = \bfS^{\mathbf{H}^D}_k(s) + \big\langle F^{(D)}_k (s, \bfS^{\mathbf{H}^D}(s), \bfS^{\mathbf{B}^D}(s)), H_k^{(D)}(s)\big\rangle_{\ndbar},\\
            P_k^{\mathbf{Z}^H, \mathbf{Z}^B}(s) = \mathbf{Z}^{H}_k(s) + \E\big[ F^1_k (s, \mathbf{Z}^{H}(s), \mathbf{Z}^{B}(s))H_k^{1}(s)\big \vert \mathbf{Z}^{H}, \mathbf{Z}^{B}\big],
        \end{cases}\]
        and $Q_k^{\mathbf{H}^D, \mathbf{B}^D}(s)$, $Q_k^{\mathbf{Z}^H, \mathbf{Z}^B}(s)$ are defined analogously. As in Lemma~\ref{lem:proof_cvg_TCL_H}, we define a $\Finit$-measurable function $\tilde{\varphi}^{(D)}:\R^{k+1}\times\R^{k+1}\to \R^D$, given by
    \[\tilde{\varphi}^{(D)}(\bfz^h, \bfz^b) = \rho'\big(z_k^h + \E[H_k^1(s)F_k^1(s, \bfz^h, \bfz^b)]\big)\, \big(z_k^b + \E[B_k^1(s)G_k^1(s, \bfz^h, \bfz^b)]\big)\,F_k^{(D)}(s, \bfz^h, \bfz^b),\]
    which allows us to write
\begin{multline*}
            \big\|\check{\B}^{(D)}_k ( s, \bfB^{(D)} ) - \check{\overline{\B}}^{(D)}_k (s)\big\|_{\ndbar} \\\leq\bigg\|\E[\rho'(P_k^{\mathbf{H}^D, \mathbf{B}^D}(s)) Q_k^{\mathbf{H}^D, \mathbf{B}^D}(s) F_k^{(D)}(s, \bfS^{\mathbf{H}^D}(s), \bfS^{\mathbf{B}^D}(s))-\tilde{\varphi}^{(D)}(\bfS^{\mathbf{H}^D}, \bfS^{\mathbf{B}^D})|\Finit] \bigg\|_{\ndbar} \\+ \bigg\|\E[\tilde{\varphi}^{(D)}(\bfS^{\mathbf{H}^D}, \bfS^{\mathbf{B}^D})|\Finit] -\E[\tilde{\varphi}^{(D)}(\mathbf{Z}^{H}, \mathbf{Z}^{B})|\Finit]\bigg\|_{\ndbar}\\
            =: \|\blue{(\mathfrak{i})}\|_{\ndbar} + \|\blue{(\mathfrak{ii})}\|_{\ndbar},
        \end{multline*}
        where we identified 
        $\E[\tilde{\varphi}^{(D)}(\mathbf{Z}^{H}, \mathbf{Z}^{B})|\Finit]$ and introduced $\E[\tilde{\varphi}^{(D)}(\bfS^{\mathbf{H}^D}, \bfS^{\mathbf{B}^D})|\Finit]$ as a pivot.

    The term $\|\blue{(\mathfrak{ii})}\|_{\ndbar}$ can be handled using the quantitative CLT from Corollary~\ref{cor:quantitative_noniid_clt_conditional}, exactly as in Lemma~\ref{lem:proof_cvg_TCL_H}. Indeed, from Lemma~\ref{lem:RegSkel}, for $k \in \{ 1,2,3 \}$ and $d\in[1:D]$,
    \[\|D^k\tilde{\varphi}^{d}(\bfz^h, \bfz^b)\|_{op} \leq \cstW_k\max_{t\in[0,1]}(|H_k^{d}(t)| + |B_k^{d}(t)|)\exp(\cstHhat_k\|(\bfz^h, \bfz^b)\|_{2(k+1),2}),\] where $\cstHhat_k>0$ and $\cstW_k$ only depends on $\|W_{\mathrm{in}}^{(D)}\|_{\ndbar}$, $\|W_{\mathrm{out}}^{(D)}\|_{\ndbar}$, as before. As in \eqref{eq:proof_cvg_TCL_bound_ii}, we obtain
    \[\|\blue{(\mathfrak{ii})}\|_{\ndbar}
        \leq \cstWh_k\Big(\frac{1}{\sqrt{D}} + \llnW_k(s)\Big).\]

    To handle $\|\blue{(\mathfrak{i})}\|_{\ndbar}$ we note that
    \begin{multline*}
        \|\blue{(\mathfrak{i})}\|_{\ndbar} \leq \E\big[\big|\rho'(P_k^{\mathbf{H}^D, \mathbf{B}^D}(s)) Q_k^{\mathbf{H}^D, \mathbf{B}^D}(s) -\rho'(\tilde{P}_k^{\mathbf{H}^D, \mathbf{B}^D}(s)) \tilde{Q}_k^{\mathbf{H}^D, \mathbf{B}^D}(s) \big| \big\|F_k^{(D)}(s, \bfS^{\mathbf{H}^D}(s), \bfS^{\mathbf{B}^D}(s)) \big\|_{\ndbar}\big|\Finit\big]\\
        \leq \E\big[\big|\rho'(P_k^{\mathbf{H}^D, \mathbf{B}^D}(s))\big| \big| Q_k^{\mathbf{H}^D, \mathbf{B}^D}(s) - \tilde{Q}_k^{\mathbf{H}^D, \mathbf{B}^D}(s) \big| \big\|F_k^{(D)}(s, \bfS^{\mathbf{H}^D}(s), \bfS^{\mathbf{B}^D}(s)) \big\|_{\ndbar}\big|\Finit\big]\\
        + \E\big[\big|\rho'(P_k^{\mathbf{H}^D, \mathbf{B}^D}(s))  -\rho'(\tilde{P}_k^{\mathbf{H}^D, \mathbf{B}^D}(s))\big| \big|\tilde{Q}_k^{\mathbf{H}^D, \mathbf{B}^D}(s) \big| \big\|F_k^{(D)}(s, \bfS^{\mathbf{H}^D}(s), \bfS^{\mathbf{B}^D}(s)) \big\|_{\ndbar}\big|\Finit\big],
    \end{multline*}
    where we have set, using the skeleton vectors from Definition~\ref{def:linear_skeleton_meanfield} as in \eqref{eq:proof_cvg_TCL_cavity_def_terms_with_v},
    \[\begin{cases}
            \tilde{P}_k^{\mathbf{H}^D, \mathbf{B}^D}(s) = \bfS^{\mathbf{H}^D}_k(s) -\eta_u \vec{\mathbf{v}}_k^F(s, \bfS^{\mathbf{H}^D}(s), \bfS^{\bfB^D}(s))^\top\E\big[H_k^{1}(s) \bfH^1_{\wedge k-1}(s)\big],\\ 
            \tilde{Q}_k^{\mathbf{H}^D, \mathbf{B}^D}(s) = \bfS^{\mathbf{B}^D}_k(s) -\eta_v \vec{\mathbf{v}}_k^G(s, \bfS^{\mathbf{H}^D}(s), \bfS^{\bfB^D}(s))^\top\E\big[B_k^{1}(s) \bfB^1_{\wedge k-1}(s)\big].
        \end{cases}\]
        We also write \(P_k^{\mathbf{H}^D, \mathbf{B}^D}(s)\) and \(Q_k^{\mathbf{H}^D, \mathbf{B}^D}(s)\) in terms of $\vec{\mathbf{v}}_k^F$ and $\vec{\mathbf{v}}_k^G$, respectively, using that \(\big\langle F_k^{(D)}(s, \bfz^h, \bfz^b), H_k^{(D)}(s)\big\rangle_{\ndbar} =  -\eta_u \vec{\mathbf{v}}_k^F(s, \bfz^h, \bfz^b)^\top \Big(\frac{1}{D} \sum_{d=1}^D\bfH^d_{\wedge k-1}(s) H_k^d\Big)\) and similarly for the term \(\big\langle G_k^{(D)}(s, \bfz^h, \bfz^b), B_k^{(D)}(s)\big\rangle_{\ndbar}\).
    Using that $\rho'$ is bounded and Lipschitz, as well as C-S and the $\Finit$-measurability of $(\bfH, \bfB)$, we are reduced to
    \begin{multline*}
        \|\blue{(\mathfrak{i})}\|_{\ndbar}
        \leq C_\rho \E^{1/2}\big[\big| Q_k^{\mathbf{H}^D, \mathbf{B}^D}(s) - \tilde{Q}_k^{\mathbf{H}^D, \mathbf{B}^D}(s) \big|^2 \big|\Finit\big]\E^{1/2}\big[ \big\|F_k^{(D)}(s, \bfS^{\mathbf{H}^D}(s), \bfS^{\mathbf{B}^D}(s)) \big\|_{\ndbar}^2\big|\Finit\big]\\
        + C_\rho \E^{1/2}\big[\big|P_k^{\mathbf{H}^D, \mathbf{B}^D}(s)  -\tilde{P}_k^{\mathbf{H}^D, \mathbf{B}^D}(s)\big|^2 \big|\Finit\big]\E^{1/4}\big[ \big|\tilde{Q}_k^{\mathbf{H}^D, \mathbf{B}^D}(s) \big|^4 \big|\Finit\big]\\\cdot\E^{1/4}\big[ \big\|F_k^{(D)}(s, \bfS^{\mathbf{H}^D}(s), \bfS^{\mathbf{B}^D}(s)) \big\|_{\ndbar}^4\big|\Finit\big]\\
        \leq \cstW_k \E^{1/2}\big[\|\vec{\mathbf{v}}_k^G(s, \bfS^{\mathbf{H}^D}(s), \bfS^{\bfB^D}(s))\|_{2,k}^2\big|\Finit\big]\Big\|\frac{1}{D} \sum_{d=1}^D\bfB^d_{\wedge k-1}(s) B_k^d - \E\big[B_k^{1}(s) \bfB^1_{\wedge k-1}(s)\big]\Big\|_{2,k}\\
        + \cstW_k \E^{1/2}\big[\|\vec{\mathbf{v}}_k^F(s, \bfS^{\mathbf{H}^D}(s), \bfS^{\bfB^D}(s))\|_{2,k}^2\big|\Finit\big]\Big\|\frac{1}{D} \sum_{d=1}^D\bfH^d_{\wedge k-1}(s) H_k^d - \E\big[H_k^{1}(s) \bfH^1_{\wedge k-1}(s)\big]\Big\|_{2,k}\\
        \leq \cstW_k \llnW_k(s),
    \end{multline*}
    where we have identified the subexponential LLN concentration term \(\llnW_k(s)\), 
    and we have absorbed terms into $\cstW_k$ using Lemma~\ref{lem:MF_skeleton_bounded_ndbar} and the same arguments as in \eqref{eq:proof_cvg_prop_caos_bound_P} for bounding \(\E^{1/4}\big[ \big|\tilde{Q}_k^{\mathbf{H}^D, \mathbf{B}^D}(s) \big|^4 \big|\Finit\big]\). The terms involving the skeleton vectors $\vec{\mathbf{v}}_k^F$ and $\vec{\mathbf{v}}_k^G$ are bounded using Lemmas~\ref{lem:explicit_regskel} and \ref{lem:exp_moment_magic} as in \eqref{eq:proof_cvg_prop_caos_bound_P}.

    Gathering all contributions yields
    \begin{equation*}
            \big\|\check{\B}^{(D)}_k ( s, \bfB^{(D)} ) - \check{\overline{\B}}^{(D)}_k (s)\big\|_{\ndbar} \\\leq \cstWh_k \Big(\frac{1}{\sqrt{D}}+\llnW_k(s)\Big), 
        \end{equation*}
        which is precisely what we wanted to obtain.   
        \end{proof}
    
\begin{proof}[Proof of Lemma~\ref{lem:proof_cvg_cavity_TCL_B}]
    We follow the proof of Lemma~\ref{lem:proof_cvg_cavity_TCL_H}. For $d\in[1:D]$ and $s\in[0,1]$,
    \begin{multline*}
        |\hat{\B}^d_k ( s, \bfB^{(D)} ) - \hat{\overline{\B}}^d_k (s)|\\
        = \Big|\E[\rho'(P_k^{\mathbf{H}^D, \mathbf{B}^D}(s))Q_k^{\mathbf{H}^D, \mathbf{B}^D}(s)\sqrt{D}U^d|\Finit]\qquad\qquad\qquad\qquad\qquad\qquad\qquad\qquad\qquad\qquad \\
        - \varu^2\E\big[ \rho'(P_k^{\mathbf{Z}^H, \mathbf{Z}^B}(s)) \E[B_k^1(s) \bnabla_{\bfz^{h}}G_k^1(s, \mathbf{Z}^H(s), \mathbf{Z}^B(s))|\mathbf{Z}^H, \mathbf{Z}^B]\big]\cdot \mathbf{H}_{\wedge k-1}^d(s)\\
         - \varu^2\E\big[ \rho''(P_k^{\mathbf{Z}^H, \mathbf{Z}^B}(s)) Q_k^{\mathbf{Z}^H, \mathbf{Z}^B}(s)\E[H_k^1(s) \bnabla_{\bfz^{h}}F_k^1(s, \mathbf{Z}^H(s), \mathbf{Z}^B(s))|\mathbf{Z}^H, \mathbf{Z}^B]\big]\cdot \mathbf{H}_{\wedge k-1}^d(s)\\
        - \varu^2\E\big[ \rho''(P_k^{\mathbf{Z}^H, \mathbf{Z}^B}(s)) Q_k^{\mathbf{Z}^H, \mathbf{Z}^B}(s)]H_k^d(s)\Big|,
    \end{multline*}
    where we have set $P_k^{\mathbf{H}^D, \mathbf{B}^D}(s)$, $P_k^{\mathbf{Z}^H, \mathbf{Z}^B}(s)$ as in \eqref{eq:proof_cvg_TCL_cavity_definition_P} and $Q_k^{\mathbf{H}^D, \mathbf{B}^D}(s)$, $Q_k^{\mathbf{Z}^H, \mathbf{Z}^B}(s)$
    defined analogously.
    To apply the cavity method, we write the Taylor expansion for the $\Finit$-measurable function $\tilde{\phi}:\R^{k+1}\times\R^{k+1} \to \R$ given by
    \[\tilde{\phi}(\bfz^h, \bfz^b) = \rho'(z_k^h + \langle F^{(D)}_k(s, \bfz^h, \bfz^b), H_k^{(D)}(s)\rangle_{\ndbar})(z_k^b + \langle G^{(D)}_k(s, \bfz^h, \bfz^b), B_k^{(D)}(s)\rangle_{\ndbar}),\]
    which is twice continuously differentiable. Hence, for $\overline{\bfz}^h \in\R^{k+1}$,
    \[\tilde{\phi}(\bfz^h, \bfz^b) = \tilde{\phi}(\overline{\bfz}^h, \bfz^b) + \bnabla_{\bfz^h}\tilde{\phi}(\overline{\bfz}^h, \bfz^b)\cdot(\bfz^h-\overline{\bfz}^h) + \tilde{r}(\bfz^h, \overline{\bfz}^h, \bfz^b)\]
    where \(\tilde{r}(\bfz^h, \overline{\bfz}^h, \bfz^b) = \int_0^1(1-t) \bnabla_{\bfz^h}^2\tilde{\phi}(t\bfz^h + (1-t)\overline{\bfz}^h, \bfz^b) [\bfz^h -\overline{\bfz}^h]^{\otimes2}\d t\) is the integral remainder. 
    We explicitly compute the gradient and hessian of $\tilde{\phi}$:
    \begin{align*}
        \bnabla_{\bfz^h}\tilde{\phi}(\bfz^h, \bfz^b) = \rho''(\tilde{p}_k(s, \bfz^h, \bfz^b))\tilde{q}_k(s, \bfz^h, \bfz^b)\nabla_{\bfz^h} \tilde{p}_k(s, \bfz^h, \bfz^b) + \rho'(\tilde{p}_k(s, \bfz^h, \bfz^b))\nabla_{\bfz^h} \tilde{q}_k(s, \bfz^h, \bfz^b), 
    \end{align*}
    and, underlying evaluation at $(s, \bfz^h, \bfz^b)$, 
    \begin{align*}
    \bnabla_{\bfz^h}^2\tilde{\phi}(\bfz^h, \bfz^b) &=
        \rho'''(\tilde{p}_k)\, \tilde{q}_k\,
\nabla_{\bfz^h} \tilde{p}_k (\nabla_{\bfz^h} \tilde{p}_k)^{\top}
+
\rho''(\tilde{p}_k)\Big(
\nabla_{\bfz^h} \tilde{p}_k (\nabla_{\bfz^h} \tilde{q}_k)^{\top}
+
\nabla_{\bfz^h} \tilde{q}_k (\nabla_{\bfz^h} \tilde{p}_k)^{\top}
\Big)
\\
&+
\rho''(\tilde{p}_k)\, \tilde{q}_k\, \nabla_{\bfz^h}^2 \tilde{p}_k
+
\rho'(\tilde{p}_k)\, \nabla_{\bfz^h}^2 \tilde{q}_k 
    \end{align*}
    where we used the shorthand $\tilde{p}_k(s, \bfz^h, \bfz^b) = z_k^h + \langle F^{(D)}_k(s, \bfz^h, \bfz^b), H_k^{(D)}(s)\rangle_{\ndbar}$ and $\tilde{q}_k(s, \bfz^h, \bfz^b) = z_k^b + \langle G^{(D)}_k(s, \bfz^h, \bfz^b), B_k^{(D)}(s)\rangle_{\ndbar}$, as well as their gradients \[\footnotesize\nabla_{\bfz^h} \tilde{p}_k(s, \bfz^h, \bfz^b) =\begin{pmatrix}
            \frac{1}{D} \bnabla_{\bfz^h}F^{(D)}_k(s, \bfz^h, \bfz^b) H_k^{(D)}(s)\\
            1
        \end{pmatrix} \text{, } \nabla_{\bfz^h} \tilde{q}_k(s, \bfz^h, \bfz^b) =\begin{pmatrix}
            \frac{1}{D} \bnabla_{\bfz^h}G^{(D)}_k(s, \bfz^h, \bfz^b) B_k^{(D)}(s)\\
            0
        \end{pmatrix},\]
    and their Hessians
    \begin{align*}
        \scriptsize\nabla^2_{\bfz^h} \tilde{p}_k(s, \bfz^h, \bfz^b) =\begin{pmatrix}
            \frac{1}{D} \bnabla^2_{\bfz^h}F^{(D)}_k(s, \bfz^h, \bfz^b) H_k^{(D)}(s) & 0\\
            0 &0
        \end{pmatrix} \text{, } \nabla_{\bfz^h}^2 \tilde{q}_k(s, \bfz^h, \bfz^b) =\begin{pmatrix}
            \frac{1}{D} \bnabla^2_{\bfz^h}G^{(D)}_k(s, \bfz^h, \bfz^b) B_k^{(D)}(s) & 0\\
            0 & 0
        \end{pmatrix}.
    \end{align*}
    Since \[\E\big[\rho'(P_k^{\mathbf{H}^D, \mathbf{B}^D}(s))Q_k^{\mathbf{H}^D, \mathbf{B}^D}(s)\sqrt{D}U^d|\Finit\big] = \E\big[\tilde{\phi}(\bfS^{\mathbf{H}^D}, \bfS^{\mathbf{B}^D})\sqrt{D}U^d|\Finit\big],\] 
    we can perform a Taylor expansion around \(\overline{\bfS}^{d,\mathbf{H}^D} = \bfS^{\mathbf{H}^D}-\frac{1}{\sqrt{D}}U^d\mathbf{H}_{\wedge k}^d(s) \in \R^{k+1}\), including $\bfH$ up to time $k$. This yields
    \begin{align*}
        \E[\tilde{\phi}(\bfS^{\mathbf{H}^D}, \bfS^{\mathbf{B}^D})\sqrt{D}U^d|\Finit] &= \E[\tilde{\phi}(\overline{\bfS}^{d,\mathbf{H}^D},\bfS^{\mathbf{B}^D})\sqrt{D}U^d|\Finit]\\&\qquad+\E[\bnabla_{\bfz^h}\tilde{\phi}(\overline{\bfS}^{d,\mathbf{H}^D},\bfS^{\mathbf{B}^D})\cdot\big(\frac{1}{\sqrt{D}}U^d\mathbf{H}_{\wedge k}^d(s)\big)\sqrt{D}U^d|\Finit]\\&\qquad + \E[\tilde{r}(\bfS^{\mathbf{H}^D}, \overline{\bfS}^{d,\mathbf{H}^D}, \bfS^{\mathbf{B}^D})\sqrt{D}U^d|\Finit]\\
        &=: \purple{(\mathfrak{i})_d} + \purple{(\mathfrak{ii})_d} + \purple{(\mathfrak{iii})_d}.
    \end{align*}
    Since $U^d$ is independent from $(\overline{\bfS}^{d,\mathbf{H}^D}, \bfS^{\mathbf{B}^D})$, even conditionally on $\Finit$, we have \[\purple{(\mathfrak{i})_d}=0,\;\;\text{ and }\;\;\purple{(\mathfrak{ii})_d}=\varu^2\E[\bnabla_{\bfz^h}\tilde{\phi}(\overline{\bfS}^{d,\mathbf{H}^D}, \bfS^{\mathbf{B}^D})|\Finit]\cdot \mathbf{H}_{\wedge k}^d(s).\] 
    The remainder term $\purple{(\mathfrak{iii})_d}$ is bounded exactly as in \eqref{eq:proof_cvg_TCL_cavity_iii}, yielding 
    \[|\purple{(\mathfrak{iii})_d}|\leq \frac{\cstW_k}{\sqrt{D}}\big\|\bfH_{\wedge k-1}^d(s)\big\|_{2,k}^2.\]
    Gathering terms, we have
    \begin{multline*}
        |\hat{\B}^d_k ( s ) - \hat{\overline{\B}}^d_k (s)|= \big|\purple{(\mathfrak{iii})_d}\big| +\big|\purple{(\mathfrak{ii})_d}  \\- \varu^2\E\big[ \rho'(P_k^{\mathbf{Z}^H, \mathbf{Z}^B}(s)) \E[B_k^1(s) \bnabla_{\bfz^{h}}G_k^1(s, \mathbf{Z}^H(s), \mathbf{Z}^B(s))|\mathbf{Z}^H, \mathbf{Z}^B]\big]\cdot \mathbf{H}_{\wedge k-1}^d(s)\\
         - \varu^2\E\big[ \rho''(P_k^{\mathbf{Z}^H, \mathbf{Z}^B}(s)) Q_k^{\mathbf{Z}^H, \mathbf{Z}^B}(s)\E[H_k^1(s) \bnabla_{\bfz^{h}}F_k^1(s, \mathbf{Z}^H(s), \mathbf{Z}^B(s))|\mathbf{Z}^H, \mathbf{Z}^B]\big]\cdot \mathbf{H}_{\wedge k-1}^d(s) \\- \varu^2\E\big[ \rho''(P_k^{\mathbf{Z}^H, \mathbf{Z}^B}(s)) Q_k^{\mathbf{Z}^H, \mathbf{Z}^B}(s)]H_k^d(s)
        \big|,
    \end{multline*}
    and so, by C-S,
    \begin{multline*}
        |\hat{\B}^d_k ( s ) - \hat{\overline{\B}}^d_k (s)|\leq \varu^2\big\|\E[\bnabla_{\bfz^h}\tilde{\phi}(\overline{\bfS}^{d,\mathbf{H}^D}, \bfS^{\mathbf{B}^D})|\Finit] - \E[\tilde{\psi}(\mathbf{Z}^H(s), \mathbf{Z}^B(s))]\big\|_{2,k+1}\|\bfH_{\wedge k}^d(s)\|_{2,k+1} + |\purple{(\mathfrak{iii})_d}|,
    \end{multline*}
    where, following the proof of Lemma~\ref{lem:proof_cvg_cavity_TCL_H}, we introduce the function $\tilde{\psi}:\R^{k+1}\times\R^{k+1}\to\R^{k+1}$, given by
    \begin{multline*}
    \tilde{\psi}(\bfz^h, \bfz^b) =\rho''(z_k^h + \E[F_k^1(s, \bfz^h, \bfz^b)H_k^1(s)])(z_k^b + \E[G_k^1(s, \bfz^h, \bfz^b)B_k^1(s)])\begin{pmatrix}
    \E[H_k^1(s)\bnabla_{\bfz^h}F_k^1(s, \bfz^h, \bfz^b)]\\
            1
        \end{pmatrix}\\
        +\rho'(z_k^h + \E[F_k^1(s, \bfz^h, \bfz^b)H_k^1(s)])\begin{pmatrix}
    \E[B_k^1(s)\bnabla_{\bfz^h}G_k^1(s, \bfz^h, \bfz^b)]\\
            0
        \end{pmatrix}.
    \end{multline*}
    Taking the RMS norm, as in \eqref{eq:proof_cvg_TCL_cavity_bound_RMS}, we obtain
    \begin{multline*}
        \Big\|\hat{\B}^{(D)}_k ( s, \bfB^{(D)}  ) - \hat{\overline{\B}}^{(D)}_k (s)\Big\|_{\ndbar}
        \leq \cstW_k\Big\|\E[\bnabla_{\bfz^h}\tilde{\phi}(\overline{\bfS}^{d,\mathbf{H}^D}(s), \bfS^{\mathbf{B}^D}(s))|\Finit] - \E[\tilde{\psi}(\mathbf{Z}^H(s), \mathbf{Z}^B(s))]\Big\|_{2,k+1}\\+ \frac{\cstWh_k}{\sqrt{D}},
    \end{multline*}
    with $\cstWh_k$ a constant depending further on the normalized 4-norm of $W_{\mathrm{in}}^{(D)}$ and $W_{\mathrm{out}}^{(D)}$.
    By further introducing the pivot term \(\E[\tilde{\psi}(\overline{\bfS}^{d,\mathbf{H}^D}(s),\bfS^{\mathbf{B}^D}(s))\big|\Finit]\), we are reduced to studying:
    \begin{itemize}
        \item A term \(\| \E[\tilde{\psi}(\overline{\bfS}^{d,\mathbf{H}^D}(s),\bfS^{\mathbf{B}^D}(s))\big|\Finit] - \E[\tilde{\psi}(\mathbf{Z}^H(s), \mathbf{Z}^B(s))]\|_{k+1,2}\) that is
        handled using the quantitative CLT result from Corollary~\ref{cor:quantitative_noniid_clt_conditional}. Indeed, \(\tilde{\psi}\) has $\cstHhat_k$-exponentially growing derivatives of orders up to 3, and so, for a constant $\cstWh_k>0$ depending only on $\cstHhat_k, k, \vpu, \vpv$ $\|W_{\mathrm{in}}^{(D)}\|_{\ndbar}$, $\|W_{\mathrm{out}}^{(D)}\|_{\ndbar}$, \(\|W_{\mathrm{in}}^{(D)}\|_{\overline{3}, D}\) and \(\|W_{\mathrm{out}}^{(D)}\|_{\overline{3}, D}\), we have
\[\| \E[\tilde{\psi}(\overline{\bfS}^{d,\mathbf{H}^D}(s),\bfS^{\mathbf{B}^D}(s))\big|\Finit] - \E[\tilde{\psi}(\mathbf{Z}^H(s), \mathbf{Z}^B(s))]\|_{k+1,2} \leq \cstWh_k\Big(\frac{1}{\sqrt{D}} + \llnW_k(s)\Big).\]
        \item A term \(\big\| \E\big[\bnabla_{\bfz^h}\tilde{\phi}(\overline{\bfS}^{d,\mathbf{H}^D}, \bfS^{\mathbf{B}^D}) -\tilde{\psi}(\overline{\bfS}^{d,\mathbf{H}^D}(s),\bfS^{\mathbf{B}^D}(s))\big|\Finit\big]\big\|_{k+1,2}\) that can be rearranged, introducing pivots as in \eqref{eq:proof_cvg_TCL_cavity_def_terms_with_v}, involving the skeleton vectors $\vec{\mathbf{v}}_k^F$ and $\vec{\mathbf{v}}_k^G$ from Definition~\ref{def:linear_skeleton_meanfield}, to yield a subexponential LLN concentration term as in \eqref{eq:proof_cvg_TCL_cavity_bound_llnTerm}, so that
        \[\big\| \E\big[\bnabla_{\bfz^h}\tilde{\phi}(\overline{\bfS}^{d,\mathbf{H}^D}, \bfS^{\mathbf{B}^D}) -\tilde{\psi}(\overline{\bfS}^{d,\mathbf{H}^D}(s),\bfS^{\mathbf{B}^D}(s))\big|\Finit\big]\big\|_{k+1,2}
            \leq\cstW_k \llnW_k(s),\]
        where, as in \eqref{eq:proof_cvg_TCL_cavity_bound_llnTerm}, we absorb terms into $\cstW_k$ using Lemmas~\ref{lem:explicit_regskel} and \ref{lem:MF_skeleton_bounded_ndbar}. 
    \end{itemize}
    Gathering all terms, we finally obtain
    \begin{equation*}
        \big\|\hat{\B}^{(D)}_k ( s, \bfB^{(D)}  ) - \hat{\overline{\B}}^{(D)}_k (s)\big\|_{\ndbar}
        \leq 
        \cstWh_k \Big( \llnW_k(s) +\frac{1}{\sqrt{D}}\Big),
    \end{equation*}
    which concludes.
    \end{proof}

\appendix

\section{Quantitative CLT}\label{sec:quantitative_clt}

In this section we prove the following quantitative CLT result, using standard arguments.
First recall the following elementary proposition involving subgaussian random vectors. Recall the variance-proxy seminorm for a random vector $X$ in $\R^d$:
\[\|X\|_{vp} := \inf\left\{s>0: \forall u\in \R^d,\;\E\left[e^{u^{\top(X-\E[X])}}\right]\leq e^{\frac{s^2\|u\|^2}{2}}\right\},\]
and recall that we say that $X$ is a $\sigma^2$-subgaussian random vector if $\|X\|_{vp}\leq \sigma <\infty$. 
The following properties are standard and straightforward, see e.g. \cite{vershynin2018highdimensionalproba}.

\begin{proposition}\label{prop:Gaussian_is_subgaussian}
    If $Z \sim \N(m, \Sigma)$ is a Gaussian random vector, it is also $\|\Sigma\|_{op}$-subgaussian.
\end{proposition}

\begin{proposition}\label{prop:sum_independent_subgaussians}
If $(X_i)_{i\geq 1}$ are independent, $\sigma_i^2$-subgaussian random vectors in $\R^d$, then:
\[S := \sum_{i=1}^n X_i \;\;\text{ is a }\sum_{i=1}^n \sigma^2_i \text{-subgaussian random vector.}\]
\end{proposition}

\begin{proposition}\label{prop:exp_moment_bound_subgaussian}
    Let $X$ be a $\sigma^2$-subgaussian random vector in $\R^d$. There exists an absolute constant $c >0$ such that,
    \[ \forall t \in \R, \qquad \E[\exp(t\|X\|_2)] \leq c(1\vee\sigma t) 4^d e^{2\sigma^2t^2},\]
    where $\|\cdot\|_2$ denotes the Euclidean norm in $\R^d$.
\end{proposition}

\begin{proposition}\label{prop:moment_bound_subgaussian}
    Let $X$ be a $\sigma^2$-subgaussian random vector in $\R^d$. For any $p\geq 1$, there exists $C_p >0$ that only depends on $p$ such that
    \[\E^{1/p}[\|X\|^p_2] \leq C_p\sqrt{d}\sigma,\]
    where $\|\cdot\|_2$ denotes the Euclidean norm in $\R^d$
\end{proposition}

We now state and prove the main result of this section. 

\begin{theorem}\label{thm:quantitative_noniid_clt}
  Let $(Y_i)_{i\geq 1}$ be an i.i.d.~sequence of centred, $\sigma^2$-subgaussian real-valued random variables, and $(v_i)_{i\geq 1} \subseteq \R^d$ be a fixed sequence of deterministic vectors in $\R^d$, from which we define the 
  random vectors $X_i := Y_iv_i$, $i\geq 1$. 
  Let $f\in C^3(\R^d,\R)$ be such that there exists $\alpha >0$ satisfying 
  \[ \forall k \in \{1,2,3\}, \qquad M_k :=\sup_{x\in\R^d} \frac{\| D^kf(x)\|_{op}}{\exp(\alpha\|x\|)} <\infty. \]
  Consider a Gaussian random variable $Z \sim \N(0, \Sigma)$,
  independent from $(X_i)_{i\geq 1}$, with fixed covariance matrix $\Sigma$. Finally, let $n\geq 1$, and consider the partial sums \(S_n = \frac{1}{\sqrt{n}}\sum_{i=1}^n X_i\), which have covariance matrix $\Sigma_n^S := \mathrm{Cov} (S_n)$. Also, define the deterministic constant
  \[B_n := \Big(\frac{1}{n} \sum_{i=1}^n\|v_i\|^2\Big)^{1/2} \vee \Big(\frac{1}{n} \sum_{i=1}^n\|v_i\|^3\Big)^{1/3}.\]
  
  There exists a constant $C_n>0$ that only depends on $\alpha, d, \sigma, B_n$ and $\|\Sigma\|_{op}$ such that
  \[|\E[f(S_n)] - \E[f(Z)]| \leq C_n\left[ \frac{M_3}{\sqrt{n}} + M_2\| \Sigma_n^S - \Sigma\|_{op}\right]. \]
  Note that $C_n$ only depends on $n$ through $B_n$.\footnote{In particular, if $\sup_{n\geq 1} B_n \leq B$ for some fixed $B>0$, then $C>0$ can be taken independent of $n$.}
\end{theorem}

\begin{proof}
Fix $n\geq 1$. We note that both \(\E[f(S_n)]\) and \(\E[f(Z)]\) are well defined, since $f$ has at most exponential growth. Indeed using the condition on $D^{1}f$, \(|f(x)| \leq \sup_{t\in[0,1]}\exp(\alpha t \|x\|) \|x\| \leq \exp(c\alpha\|x\|)\) for some $c>0$.

The variables $X_i$ are centred with covariance $\Sigma_i^X = \sigma^2v_iv_i^\top$. 
Since \(\|\Sigma_i^X\|_{op} = \sigma^2 \|v_i\|^2\) and the $Y_i$ are $\sigma^2$-subgaussian, this implies that $X_i$ is $\|\Sigma_i^X\|_{op}$-subgaussian for each $i \geq 1$,  
i.e. $\sigma^2\|v_i\|^2$-subgaussian.
Similarly, 
\[\Sigma_n^S = \E[S_nS_n^\top] = \frac{1}{n}\sum_{i,j=1}^n\E[X_iX_j^\top]= \frac{1}{n}\sum_{i=1}^n\Sigma_i^X,\]
where we used that the $(X_i)_{i \geq 1}$ are centred and independent. In particular, 
\begin{equation}\label{eq:quant_clt_uniform_cov_bound}
\|\Sigma_n^S\|_{op} \leq \frac{1}{n} \sum_{i=1}^n\|\Sigma_i^X\|_{op} = \sigma^2 \frac{1}{n} \sum_{i=1}^n\|v_i\|^2 \leq \sigma^2B_n^2.
\end{equation} From Proposition~\ref{prop:sum_independent_subgaussians}, we also get that $S_n$ is $\frac{1}{n}\sum_{i=1}^n\|\Sigma_i^X\|_{op}$-subgaussian, and thus $\sigma^2B^2_n$-subgaussian.

Our proof relies on Lindeberg's swapping trick. For this purpose, we consider a Gaussian random variable $Z_n \sim \N(0,\Sigma^S_n)$ and split the error into a ``CLT term'' and a ``covariance mismatch term'':
\[
\left|\E [f(S_n)] - \E [f(Z)]\right| \leq 
\underbrace{|\E [f(S_n)] - \E[ f(Z_n)]|}_{\text{CLT term}} + 
\underbrace{|\E[f(Z_n)] - \E [f(Z)]|}_{\text{covariance mismatch}},
\]
which we now proceed to separately bound.

\subsection*{Part A. Lindeberg replacement bound for $|\E[f(S_n)] - \E [f(Z_n)]|$.}
\paragraph{Step A1. Telescoping with Gaussian surrogates.}

We consider a family $(W_i)_{i\geq1}$ of {independent} Gaussian vectors that are {independent} of $(X_i)_{i\geq 1}$ and such that
\(
W_i \sim N(0, \Sigma_i^X)
\), $i\geq 1$.
We define the Gaussianized sum
\(
\tilde{S}_n \coloneqq \frac{1}{\sqrt{n}} \sum_{i=1}^n W_i,
\)
which is 
a centred Gaussian vector with covariance
\[\Cov(\tilde{S}_n) = \frac{1}{n} \sum_{i=1}^n \Cov(W_i) = \frac{1}{n} \sum_{i=1}^n \Sigma_i^X = \Sigma_n^S.
\]
In particular, $\tilde{S}_n \stackrel{d}{=} Z_n$ so that it suffices to bound
\(|\E [f(S_n)] - \E [f(\tilde{S}_n)]|\).
For $i \in [1:n]$, we define the mixed sum
\[
P_i \coloneqq \frac{1}{\sqrt{n}} \left( \sum_{j=1}^{i} W_j + \sum_{j=i+1}^{n} X_j \right),
\]
so that $P_0 = S_n$ and $P_n = \tilde{S}_n$. We have
\(
\E[f(S_n)] - \E[f(\widetilde{S}_n)] = \sum_{i=1}^n \E\left[ f(P_{i-1}) - f(P_i) \right]
\).
Fixing $i \in [1,n]$, we define the ``common part''
\[
U_i \coloneqq \frac{1}{\sqrt{n}} \left( \sum_{j=1}^{i-1} W_j + \sum_{j=i+1}^{n} X_j \right),
\]
and the ``increment parts''
\(\Delta_i \coloneqq \frac{1}{\sqrt{n}} X_i\), \(\Gamma_i \coloneqq \frac{1}{\sqrt{n}} W_i\), so that \(
P_{i-1} = U_i + \Delta_i\) and \(P_i = U_i + \Gamma_i\).
Importantly, $U_i$ is independent of both $X_i$ and $W_i$, hence independent of the centred variables $\Delta_i$ and $\Gamma_i$. 
Therefore,
\[
\E\left[ f(P_{i-1}) - f(P_i) \right] = \E\left[ f(U_i + \Delta_i) - f(U_i + \Gamma_i) \right],
\] 
and we are reduced to computing Taylor approximations for our function $f$.

\paragraph{Step A2. Second-order Taylor expansion with integral remainder.}

For $u \in \R^d$, Taylor's expansions with integral remainder yields
\[
f(u+x) = f(u) + \langle \nabla f(u), x \rangle + \frac{1}{2} D^2 f(u)[x,x] + R_u(x),
\]
with
\(R_u(x) := \frac{1}{6} \int_0^1 (1-t)^2 \, D^3 f(u+tx)[x,x,x] \, \d t\).
For $(u,x) = (U_i,\Delta_i)$ and $(u,x) = (U_i,\Gamma_i)$, this gives
\begin{multline*}
f(U_i + \Delta_i) - f(U_i + \Gamma_i) = 
\langle \nabla f(U_i), \Delta_i - \Gamma_i \rangle + \frac{1}{2} \left( D^2 f(U_i)[\Delta_i,\Delta_i] - D^2 f(U_i)[\Gamma_i,\Gamma_i] \right) \\
+ R_{U_i}(\Delta_i) - R_{U_i}(\Gamma_i).
\end{multline*}
Since $U_i$ is independent of the centred variables $\Delta_i, \Gamma_i$, we have
\begin{multline*}
  \E[\langle \nabla f(U_i), \Delta_i - \Gamma_i \rangle] = \langle \E[\nabla f(U_i)], \E[\Delta_i - \Gamma_i]\rangle =0, \\  
  \E[D^2 f(U_i)[\Delta_i,\Delta_i]] = \sum_{k,j=1}^d\E[ (\Delta_i)_k(\Delta_i)_j\big(D^2f(U_i)\big)_{ij}] = \sum_{k,j=1}^d\E[ (\Delta_i)_k(\Delta_i)_j]\E[\big(D^2f(U_i)\big)_{ij}]. \phantom{a.}
  \end{multline*}
Since
\[
\E[\Delta_i \Delta_i^\top] = \frac{1}{n} \E[X_iX_i^\top] = \frac{1}{n}\Sigma_i^X = \frac{1}{n}\E[W_iW_i^\top] = \E[\Gamma_i \Gamma_i^\top],
\]
this implies \(\E[D^2 f(U_i)[\Delta_i,\Delta_i]] = \E[D^2 f(U_i)[\Gamma_i,\Gamma_i]]\), so that the quadratic terms vanish and we get
\[
\E\left[ f(U_i + \Delta_i) - f(U_i + \Gamma_i) \right] = \E\left[ R_{U_i}(\Delta_i) - R_{U_i}(\Gamma_i) \right],
\]
and then
\[
\big|\E\left[ f(P_{i-1}) - f(P_i) \right]\big|\leq \E[\big|R_{U_i}(\Delta_i)\big|] + \E[\big|R_{U_i}(\Gamma_i)\big|].
\]

\paragraph{Step A3. Bounding the remainder using exponential growth of $D^3 f$.}

For any $u,x \in \R^d$, the definition of the operator norm yields
\[
|D^3 f(u+tx)[x,x,x]| \leq \|D^3 f(u+tx)\|_{op} \|x\|^3,
\]
so that
\[
|R_u(x)| \leq \frac{1}{6} \int_0^1 (1-t)^2 \, \|D^3 f(u+tx)\|_{op} \, \|x\|^3 \, \d t.
\]
Since $\|D^3 f(y)\|_{op} \leq M_3 \exp(\alpha \|y\|)$, we get
\[
|R_u(x)| \leq \frac{M_3}{6} \|x\|^3 \int_0^1 (1-t)^2 \exp\left(\alpha\|u+tx\|\right) \, \d t.
\]
Using $\|u+tx\| \leq \|u\| + \|x\|$ 
and $\int_0^1 (1-t)^2 dt = 1/3$, we obtain
\[
|R_u(x)| \leq \frac{M_3}{18} \|x\|^3 \exp(\alpha\|u\|)\exp(\alpha\|x\|).
\]
Applying this this with $(u,x) = (U_i, \Delta_i)$ yields
\[
\E[|R_{U_i}(\Delta_i)|] \leq \frac{M_3}{18} \, \E\left[ \|\Delta_i\|^3 \exp(\alpha\|U_i\|)\exp(\alpha\|\Delta_i\|) \right],
\]
and similarly for  $(u,x) = (U_i, \Gamma_i)$. 

Let now $\xi$ be a generic $\sigma_\xi^2$-subgaussian random vector, independent from $U_i$ and such that $\frac{\sigma_\xi}{\sqrt{n}} \leq \sigma B_n$. We have
\begin{align*}
    \E\left[\left|R_{U_i}\left(\frac{\xi}{\sqrt{n}}\right)\right|\right] &\leq \frac{M_3}{18} \, \E\left[ \left\|\frac{\xi}{\sqrt{n}}\right\|^3 \exp(\alpha\|U_i\|)\exp\Big(\alpha \Big\|\frac{\xi}{\sqrt{n}}\Big\|\Big) \right]\\
    &\leq \frac{M_3}{18} \frac{1}{n^{3/2}}\, \E\left[ \left\|\xi\right\|^3 \exp(\alpha\|U_i\|)\exp\Big(\frac{\alpha}{\sqrt{n}}\|\xi\|\Big) \right]\\
    &\leq \frac{M_3}{18} \frac{1}{n^{3/2}}\, \E\big[\exp(\alpha\|U_i\|)\big]\E\Big[ \left\|\xi\right\|^3 \exp\Big(\frac{\alpha}{\sqrt{n}}\|\xi\|\Big) \Big]\\
    &\leq \frac{M_3}{18} \frac{1}{n^{3/2}}\, \E\big[\exp(\alpha\|U_i\|)\big]\E^{1/2}\big[ \left\|\xi\right\|^6\big] \E^{1/2}\Big[\exp\Big(\frac{2\alpha}{\sqrt{n}}\|\xi\|\Big) \Big]\\
    &\leq \frac{M_3}{18} \frac{1}{n^{3/2}}\, \E\big[\exp(\alpha\|U_i\|)\big]\Big( C^6_6 d^{3}\sigma_{\xi}^6\Big)^{1/2} \bigg(c\Big(1\vee \sigma_{\xi}\frac{2\alpha}{\sqrt{n}}\Big)4^{d} \exp\Big(2\big(\sigma_{\xi}\frac{2\alpha}{\sqrt{n}}\big)^2\Big) \bigg)^{1/2} \\
    &\leq C_{\alpha, d} \frac{M_3}{n^{3/2}}\, \sigma_{\xi}^3\Big(1\vee \frac{\sigma_{\xi}}{\sqrt{n}}\Big)^{1/2} \exp\Big((2\alpha)^2\big(\frac{\sigma_{\xi}}{\sqrt{n}}\big)^2\Big)\E\big[\exp(\alpha\|U_i\|)\big]\\
    &\leq C_{\alpha, d, \sigma, B_n} \frac{M_3}{n^{3/2}}\, \sigma_{\xi}^3 \E\big[\exp(\alpha\|U_i\|)\big],
\end{align*}
where the third inequality uses the independence of $\xi$ from $U_i$, the fourth one uses C-S, and the fifth one combines Propositions~\ref{prop:exp_moment_bound_subgaussian}-\ref{prop:moment_bound_subgaussian}. 
We also absorbed the resulting constants into $C_n:=C_{\alpha, d, \sigma, B_n}>0$, which only depends on $\alpha, d, \sigma, B_n$, and may change from line to line in the following. We note that the only dependency on $n$ is through $B_n$.

We recall that the $X_i$ are $(\sigma \|v_i\|)^2$-subgaussian. 
By Proposition~\ref{prop:Gaussian_is_subgaussian} and the definition of $\|\Sigma_i^X\|_{op}$, the $W_i\sim\N(0, \Sigma_i^X)$ are also $(\sigma \|v_i\|)^2$-subgaussian. As a consequence, the quantity \(U_i =\frac{1}{\sqrt{n}} \left( \sum_{j=1}^{i-1} W_j + \sum_{j=i+1}^{n} X_j \right)\) is a sum of {independent} subgaussian random vectors and is thus subgaussian  by Proposition~\ref{prop:sum_independent_subgaussians}, with variance proxy: 
\[\sum_{\substack{j=1\\j\neq i}}^{n} \Big(\frac{\sigma \|v_j\|}{\sqrt{n}}\Big)^2 = \sigma^2\frac{1}{n}\sum_{\substack{j=1\\j\neq i}}^{n} \|v_j\|^2 \leq (\sigma B_n)^2.\] 
In particular, we can use Proposition~\ref{prop:exp_moment_bound_subgaussian} to obtain
\[\E[\exp(\alpha \|U_i\|)]\leq c(1\vee \alpha\sigma B_n)4^d e^{2\sigma^2B_n^2\alpha^2} \leq C_n. \] Joining everything, \begin{equation}\label{eq:quant_clt_remainder_bound}
    \E\left[\left|R_{U_i}\left(\frac{\xi}{\sqrt{n}}\right)\right|\right] \leq C_n \frac{M_3}{n^{3/2}}\sigma_{\xi}^3.
\end{equation}
We apply this principle to $\xi = X_i$ and $\xi = W_i$, which are both $(\sigma \|v_i\|)^2$-subgaussian. From the fact that $\frac{1}{n}\sum_{i=1}^n \|v_i\|^2 \leq B_n^2$, we further get that $\frac{\sigma \|v_i\|}{\sqrt{n}} \leq \sigma B_n$, so that we can use \eqref{eq:quant_clt_remainder_bound} as stated. We then obtain
\[\E\left[\left|R_{U_i}\left(\Delta_i\right)\right|\right] \leq C_n M_3 \frac{\|v_i\|^3}{n^{3/2}}, \;\;\E\left[\left|R_{U_i}\left(\Gamma_i\right)\right|\right] \leq C_n M_3 \frac{\|v_i\|^3}{n^{3/2}}, \]
where we further absorbed terms into the constant $C_n$.
Combining these bounds, we obtain that
\[
|\E\left[ f(P_{i-1}) - f(P_i) \right]| \leq \E[\big|R_{U_i}(\Delta_i)\big|] + \E[\big|R_{U_i}(\Gamma_i)\big|] \leq C_n M_3\frac{\|v_i\|^3}{n^{3/2}}.
\]

Summing over $i \in [1:n]$ yields
\begin{align*}
    \left|\E f(S_n) - \E f(\tilde{S}_n)\right| &= \left|\sum_{i=1}^n\E[f(P_{i-1}) - f(P_i)]\right|\leq \sum_{i=1}^n\left|\E[f(P_{i-1}) - f(P_i)]\right|\\
    &\leq \sum_{i=1}^n C_n M_3\frac{\|v_i\|^3}{n^{3/2}} = \frac{C_n M_3}{\sqrt{n}} \Big( \frac{1}{n}\sum_{i=1}^n \|v_i\|^3\Big)\\
    &\leq \frac{C_n M_3}{\sqrt{n}} B_n^3 =: \frac{C_n M_3}{\sqrt{n}},
\end{align*}
where we used again the definition of $B_n$, and absorbed terms into $C_n$.
All in all, since $\tilde{S}_n \stackrel{d}{=} Z_n$, we have proved that \(\left|\E f(S_n) - \E f(Z_n)\right|\leq \frac{C_n M_3}{\sqrt{n}}\), completing \textbf{Part A}.

\subsection*{Part B. Covariance mismatch bound for $|\E f(Z_n) - \E f(Z)|$.}
\paragraph{Step B1. A coupling that interpolates covariances linearly.}

Let $Z_0 \sim N(0,\Sigma)$ and $Z_1 \sim N(0,\Sigma^S_n)$ be independent. For $t \in [0,1]$, we define the centred Gaussian vector
\(G_t := \sqrt{1-t} \, Z_0 + \sqrt{t} \, Z_1\), whose covariance is 
\(\Sigma_t := (1-t)\Sigma + t \Sigma^S_n\).
In particular,
\(G_0 = Z_0 \stackrel{d}{=} Z\) and \(G_1 = Z_1 \stackrel{d}{=} Z_n\). We introduce
\[
\psi :
\begin{cases}
[0,1] &\rightarrow \R, \\
t &\mapsto \E[f(G_t)],
\end{cases}
\]
so that \(\E[f(Z_n)] - \E[f(Z)] = \psi(1) - \psi(0)\). 
To estimate this difference, we compute \(\frac{d}{dt} G_t = -\frac{1}{2\sqrt{1-t}} Z_0 + \frac{1}{2\sqrt{t}} Z_1\), so that$\frac{d}{dt}f(G_t) = \nabla f(G_t) \cdot \frac{d}{dt}G_t$. 
From the exponential growth of $\nabla f$, we get 
\[\left|\frac{\d}{\d t}f(G_t)\right|\leq \|\nabla f(G_t)\|\left\|\frac{\d}{\d t} G_t\right\| \leq M_1\exp(\alpha \|G_t\|) \bigg[ \frac{1}{2\sqrt{1-t}}\|Z_0\| + \frac{1}{2\sqrt{t}}\|Z_1\| \bigg]. \]
For $t\in[0,1]$, we notice that $\exp(\alpha\|G_t\|) \leq \exp(\alpha\|Z_0\|)\exp(\alpha\|Z_1\|)$. 
Since $Z_0$ and $Z_1$ are Gaussian, they have bounded exponential moments (see Proposition~\ref{prop:exp_moment_bound_subgaussian}), so that
\[ \forall t\in (0,1), \qquad \left|\frac{d}{dt}f(G_t)\right| \leq Q_{m, Z_0, Z_1} \bigg[ \frac{1}{\sqrt{t}} + \frac{1}{\sqrt{1-t}} \bigg], \] 
for an integrable random variable $Q_{m, Z_0, Z_1}$. We can thus use the Fundamental Theorem of Calculus (FTC) to write
\[
\E[f(Z_n)] - \E[f(Z)] = \psi(1) - \psi(0) = \int_0^1 \psi'(t) \, dt,
\]
where 
\[
\psi'(t) = \E\left[ \nabla f(G_t) \cdot \frac{d}{dt}G_t \right] = 
-\frac{1}{2\sqrt{1-t}} \E\left[ \nabla f(G_t) \cdot Z_0 \right] + \frac{1}{2\sqrt{t}} \E\left[ \nabla f(G_t) \cdot Z_1 \right].
\]
We now compute each expectation using Stein's Lemma.

\paragraph{Step B2. Stein's Lemma for $\E [ \nabla f(G_t) \cdot Z_0 ]$ and $\E [ \nabla f(G_t) \cdot Z_1 ]$.} 

We recall the standard Gaussian integration by parts formula.

\begin{lemma}[Stein's Lemma]
Let $W \sim N(0,\Sigma^W)$ be a centered Gaussian in $\R^d$, and let $\Phi \colon \R^d \to \R^d$ be $C^1$ with suitable integrability . Then
\[
\E [ W \cdot \Phi(W)  ] = \E [\mathrm{Tr} (\Sigma^WD\Phi(W)  ) ],
\]
where $D\Phi$ is the Jacobian matrix.
\end{lemma}

We now apply this lemma to $\E\left[ \nabla f(G_t) \cdot Z_0 \right]$ and $\E\left[ \nabla f(G_t) \cdot Z_1 \right]$. 
For the first term, we consider the $\C^1$ function
\[
\Phi_{Z_1}^{(0)}:z_0\in\R^d \mapsto \nabla f\left( \sqrt{1-t} \, z_0 + \sqrt{t} \, Z_1 \right) \in \R^d,
\]
whose Jacobian is
\(D\Phi_{Z_1}^{(0)}(z_0) = \sqrt{1-t} \, D^2 f\left( \sqrt{1-t} \, z_0 + \sqrt{t} \, Z_1 \right)\). 
By definition, \(f(G_t) =\Phi_{Z_1}^{(0)}(Z_0) \).
Since $Z_0$ and $Z_1$ are independent, we can apply Stein's Lemma to $W = Z_0 \sim N(0,\Sigma)$ and $\Phi = \Phi_{Z_1}^{(0)}$, conditionally on $Z_1$, to obtain
\[\E [ Z_0 \cdot \nabla f(G_t) \mid Z_1 ]=\E [ Z_0 \cdot \Phi_{Z_1}^{(0)}(Z_0) \mid Z_1 ] = \E [ \mathrm{Tr} ( \Sigma \, D\Phi_{Z_1}^{(0)}(Z_0) ) \mid Z_1 ].
\]
Substituting in the Jacobian, we get \(D\Phi_{Z_1}^{(0)}(Z_0) = \sqrt{1-t} \, D^2 f\left( G_t\right)\), so that taking expectations over $Z_1$ yields
\[
\E[Z_0 \cdot \nabla f(G_t)] = \E[\E [ Z_0 \cdot \nabla f(G_t) \mid Z_1 ]] = \sqrt{1-t} \, \E [ \mathrm{Tr} ( \Sigma \, D^2 f(G_t) ) ].
\]
For the term $\E\left[ \nabla f(G_t) \cdot Z_1 \right]$, we proceed completely analogously, by conditioning on $Z_0$ and defining
\(\Phi_{Z_0}^{(1)}(z_1) \coloneqq \nabla f\left( \sqrt{1-t} \, Z_0 + \sqrt{t} \, z_1 \right)\), so that \(D\Phi_{Z_0}^{(1)}(z_1) = \sqrt{t} \, D^2 f ( \sqrt{1-t} \, Z_0 + \sqrt{t} \, z_1 ) \) and Stein's Lemma on $Z_1 \sim\N(0, \Sigma_n^S)$ conditionally on $Z_1$ yields
\[\E [ Z_1 \cdot \nabla f(G_t) ] = \sqrt{t} \, \E [ \mathrm{Tr}( \Sigma_n^S \, D^2 f(G_t) ) ]. \]

\paragraph{Step B3. Integrating the error.}

Plugging back our estimates, in the previous expression of $\psi'$,
\begin{align*}
\forall t \in (0,1), \quad    \psi'(t) &= -\frac{1}{2\sqrt{1-t}} \big[ \sqrt{1-t} \, \E[ \mathrm{Tr}( \Sigma \, D^2 f(G_t) ) ] \big]
+ \frac{1}{2\sqrt{t}} \big[ \sqrt{t} \, \E [ \mathrm{Tr} ( \Sigma^S_n \, D^2 f(G_t) ) ] \big] \\
&=\frac{1}{2} \, \E [ \mathrm{Tr} ( (\Sigma^S_n - \Sigma) \, D^2 f(G_t) ) ].
\end{align*}
Thus, integrating it from 0 to 1, we get, from our FTC expression:
\[
\E[f(Z_n)] - \E[f(Z)] = \psi(1) - \psi(0) = \int_0^1 \psi'(t) \, dt = \frac{1}{2} \int_0^1 \E\left[ \mathrm{Tr}\left( (\Sigma^S_n - \Sigma) \, D^2 f(G_t) \right) \right] \d t.
\]

\paragraph{Step B4. Bounding the mismatch term.}

We now take absolute values in our last expression, using that $|\mathrm{Tr}(AB)| \leq d\|A\|_{op} \|B\|_{op}$ for $A, B\in \R^{d\times d}$, so that
\[
|\E f(Z_n) - \E f(Z)| \leq \frac{d}{2} \int_0^1 \|\Sigma_n^S - \Sigma\|_{op} \, \E\left[ \|D^2 f(G_t)\|_{op} \right] \d t.
\]
Recalling our exponential growth condition for the derivatives of $f$, we get:
\[
\E[\|D^2 f(G_t)\|] \leq M_2  \E[\exp(\alpha \|G_t\|)] .
\]
We recall that $G_t \sim N(0,\Sigma_t)$ with $\Sigma_t = (1-t)\Sigma + t\Sigma^S_n$. By \eqref{eq:quant_clt_uniform_cov_bound}, we further have \(\|\Sigma_n^S\|_{op} \leq \sigma^2 B_n^2\). Thus, \[\|\Sigma_t\|_{op} \leq (1-t)\|\Sigma\|_{op} + t\|\Sigma_n^S\|_{op} \leq \|\Sigma\|_{op}\vee(\sigma B_n)^2. \] By Proposition~\ref{prop:Gaussian_is_subgaussian}, $G_t$ is thus a $\|\Sigma\|_{op}\vee(\sigma B_n)^2$-subgaussian random vector. In particular, Proposition~\ref{prop:exp_moment_bound_subgaussian} yields
\[\sup_{t \in [0,1]} \E[\exp(\alpha\|G_t\|)] \leq c(1\vee \sqrt{\|\Sigma\|_{op}}\vee\sigma B_n \alpha)4^d e^{2(\|\Sigma\|_{op}\vee\sigma^2B_n^2)\alpha^2} \leq C_n, \]
for a constant $C_n>0$ that only depends on $\alpha, d, \sigma, B_n$ and $\|\Sigma\|_{op}$.
As a consequence,
\[
|\E[f(Z_n)] - \E[f(Z)]| \leq \frac{d}{2} \|\Sigma^S_n - \Sigma\|_{op} \int_0^1 M_2 \sup_{r \in [0,1]} \E[\exp(\alpha\|G_r\|)] \, dt \leq C_n M_2 \|\Sigma_n^S - \Sigma\|_{op},
\]
completing \textbf{Part B}.

\subsection*{Part C. Conclusion.}
Putting Parts \textbf{A} and \textbf{B} together yields
\begin{align*}
|\E[f(S_n)] - \E[f(Z)]| &\leq |\E[f(S_n)] - \E[f(Z_n)]| + |\E[f(Z_n)] - \E[f(Z)]| \\
&\leq C_n\bigg[ \frac{M_3}{\sqrt{n}} + M_2 \, \|\Sigma_n - \Sigma\|_{op} \bigg],
\end{align*}
which is the desired bound.
\end{proof}

\begin{corollary}\label{cor:quantitative_noniid_clt_vector}
In the setting of Theorem~\ref{thm:quantitative_noniid_clt}, let $f\in \C^3 (\R^{d_1}, \R^{d_2})$ be such that there exists $\alpha >0$ satisfying  
\[ \forall k \in \{1,2,3\}, \qquad M_k :=\sup_{x\in\R^d} \frac{\| D^kf(x)\|_{op}}{\exp(\alpha\|x\|)} <\infty. \]
  Similarly, we denote by $M^j_k$ the exponential growth constant of the $j$-th component $f_j$ of $f$, and we set $\vec{M}_k := (M_k^j)_{j=1}^{d_2}$. Then, there exists $C_n>0$ that only depends on $\alpha, d_1, \sigma$, $B_n$, $\|\Sigma\|_{op}$ such that
\begin{align*}
    \|\E[f(S_n)] - \E[f(Z)]\|_{2,d_2} &\leq C_n\bigg[\frac{\|\vec{M}_3\|_{2,d_2}}{\sqrt{n}} + \|\vec{M}_2\|_{2,d_2}\| \Sigma_n^S - \Sigma\|_{op}\bigg]
\end{align*}
Again, $C_n$ depends on $n\geq 1$ only through $B_n$.
\end{corollary}
\begin{proof}
Recalling that $f = (f_j)_{j=1}^{d_2}$, we have 
    \[
\max_j \|D^k f_j(x)\|_{op} \;\leq\; \|D^k f(x)\|_{op} \;\leq\; \Bigl(\sum_j \|D^k f_j(x)\|_{op}^2\Bigr)^{1/2} 
\]
In particular,
$\max_j M_k^j \leq M_k$. 
Theorem~\ref{thm:quantitative_noniid_clt} then provides $C_n>0$ as desired such that 
\[ \forall j \in [1:d_2], \qquad
    |\E[f_j(S_n)] - \E[f_j(Z)]| \leq C_n\bigg[\frac{M_3^j}{\sqrt{n}} + M_2^j\| \Sigma_n^S - \Sigma\|_{op}\bigg].
\]
Taking the $\|\cdot\|_{2,d_2}$ norm, we obtain
\begin{align*}
    \|\E[f(S_n)] - \E[f(Z)]\|_{2,d_2} &\leq \bigg\| C_n\bigg[ \frac{\vec{M}_3}{\sqrt{n}} + \vec{M_2}\| \Sigma_n^S - \Sigma\|_{op}\bigg]\bigg\|_{2,d_2}\\
    &\leq C_n\bigg[\frac{\|\vec{M}_3\|_{2,d_2}}{\sqrt{n}} + \|\vec{M}_2\|_{2,d_2}\| \Sigma_n^S - \Sigma\|_{op}\bigg]
\end{align*}
which is what we wanted to prove.
\end{proof}

We finally state a version of this result that directly matches our setting.

\begin{corollary}\label{cor:quantitative_noniid_clt_conditional}
  Let $(Y_i)_{i\geq 1}$ be an i.i.d.~sequence of centred, $\sigma^2$-subgaussian real random variables. Let $(V_i)_{i\geq 1}$ be an i.i.d.~sequence of random vectors in $L^2$, independent from $(Y_i)_{i\geq 1}$, and let $\F_V := \sigma(V_i, i\geq 1)$.
  We define the random vectors $X_i := Y_iV_i$, $i\geq 1$, as well as the partial sums \(S_n = \frac{1}{\sqrt{n}}\sum_{i=1}^n X_i\), for $n\geq 1$ with (random) conditional covariance matrix $\Sigma_n^S := \E[S_n S_n^\top|\F_V]$. For $\Sigma := \sigma^2\E[V_1V_1^T]$, we further consider a Gaussian random variable $Z \sim \N(0, \Sigma)$ independent from $(Y_i)_{i\geq 1}$ and $\F_V$.
  Let $F\in \C^3(\R^{d_1}, \R^{d_2})$ be a $\F_V$-measurable (random) function with $M_k^j <\infty$ a.s. for $k=1,2,3$ and $j\in[1:d_2]$ as in the setting of Corollary~\ref{cor:quantitative_noniid_clt_vector}.
  
Then, for $n\geq 1$ there exists $C^{\overline{V}}_n>0$ that only depends on $\alpha, d_1, \sigma$, $\frac{1}{n} \sum_{i=1}^n \|V_i\|^2$, $\frac{1}{n} \sum_{i=1}^n \|V_i\|^3$, $\E[\|V_1\|^2]$, such that
\[\|\E[F(S_n)|\F_V] - \E[F(Z)|\F_V]\|_{2,d_2} \leq C^{\overline{V}}_n \bigg[  \frac{\|\vec{M}_3\|_{2,d_2}}{\sqrt{n}} + \|\vec{M}_2\|_{2,d_2}\| \Sigma_n^S - \Sigma\|_{op} \bigg] \;\;a.s.\]

In particular, this also holds when replacing \(\|\cdot\|_{2, d_2}\) by the RMS norm \(\|\cdot\|_{\overline{2}, d_2}:=\frac{1}{\sqrt{d_2}}\|\cdot\|_{2, d_2}\). Also, if $F$ is deterministic, we can replace $\E[F(Z)|\F_V] = \E[F(Z)]$ by independence.
\end{corollary}

\begin{proof}
    Conditionally on $\F_V$, the sequence $(X_i)_{i\geq 1}$ satisfies the conditions of Theorem~\ref{thm:quantitative_noniid_clt}. Similarly, $\|\Sigma\|_{op} \leq \sigma^2\E[\|V_1\|^2] $ as required in Theorem~\ref{thm:quantitative_noniid_clt}. Finally, conditionally on $\F_V$, $F\in\C(\R^{d_1}, \R^{d_2})$ satisfies the exponential growth conditions from Corollary~\ref{cor:quantitative_noniid_clt_vector}, allowing us to conclude that, conditionally on $\F_V$,
    \begin{align*}
        \|\E[F(S_n)|\F_V] - \E[F(Z)|\F_V]\|_{2,d_2} &\leq C_n^{\overline{V}}\bigg[ \frac{\|\vec{M}_3\|_{2,d_2}}{\sqrt{n}} + \|\vec{M}_2\|_{2,d_2}\| \Sigma_n^S - \Sigma\|_{op}\bigg]
    \end{align*}
    which gives the desired result.
\end{proof}

\section{Extension to trainable embeddings}\label{sec:extensions_of_results}
Throughout this work, we have taken the embedding matrices \(\Win,\Wout\) as fixed during training. Despite this, and as mentioned in Remarks~\ref{rem:training_embedding_matrices_finiteD} and \ref{rem:training_embedding_matrices_meanfield}, our results readily generalize to the setting of trainable embeddings.

Namely, the well-posedness from Theorem~\ref{thm:well_posedness_main} follows almost immediately without changing the proof, since the fixed-point argument for the limit ODE does not explicitly involve the updates of $\Win, \Wout$.

For the quantitative convergence in Theorem~\ref{thm:quantitative_convergence_main}, further details need to be considered.
Let us denote by $\Big((\hat{W}_{\mathrm{in}}^{(D)})_k\Big)_{k\geq 1}$ the iterates of the embedding matrices in the finite-$D$ system \eqref{eq:original_finite_D_system}; while we let $\Big((W_{\mathrm{in}}^{(D)})_k\Big)_{k\geq 1}$ be those of the limit one \eqref{eq:LimH}-\eqref{eq:LimB}.
The coupling of both systems occurs at initialization; namely, at iteration $k=0$, both are set as $(\hat{W}_{\mathrm{in}}^{(D)})_0 = (W_{\mathrm{in}}^{(D)})_0$.

The finite-dimensional update rule from Remark~\ref{rem:training_embedding_matrices_finiteD} reads, at each iteration $k\in[0:K-1]$ and for each row $d\in[1:D]$,
\[\begin{cases}
    (\hat{W}_{\mathrm{in}}^d)_{k+1} = (\hat{W}_{\mathrm{in}}^d)_{k} - \eta_{\mathrm{in}}b_k^d(0)x,\\
    (\hat{W}_{\mathrm{out}}^d)_{k+1} = (\hat{W}_{\mathrm{out}}^d)_{k} - \eta_{\mathrm{out}}h_k^d(1)\nabla\ell(\hat{y}_k^D).
\end{cases}\]
While the updates for each i.i.d.~copy of the limit system read, for $d\in[1:D]$,
\[\begin{cases}
    (W_{\mathrm{in}}^d)_{k+1} = (W_{\mathrm{in}}^d)_{k} - \eta_{\mathrm{in}}B_k^d(0)x,\\
    (W_{\mathrm{out}}^d)_{k+1} = (W_{\mathrm{out}}^d)_{k} - \eta_{\mathrm{out}}H_k^d(1)\nabla\ell(y_k).
\end{cases}\]

The adaptation of the proof of Theorem~\ref{thm:quantitative_convergence_main} then follows from the following facts.
\begin{enumerate}[label=(\roman*),ref=(\roman*)]
\item The following a priori estimates hold:
\[\|(\hat{W}_{\mathrm{in}}^{(D)})_{k+1} \|_{\ndbar} \leq \|(\hat{W}_{\mathrm{in}}^{(D)})_{k}\|_{\ndbar} + \eta_{\mathrm{in}} \|x\|_{2, \dimin}\|b_k^{(D)}(0) \|_{\ndbar},\]
\[\|(\hat{W}_{\mathrm{out}}^{(D)})_{k+1} \|_{\ndbar} \leq \|(\hat{W}_{\mathrm{out}}^{(D)})_{k}\|_{\ndbar} + \eta_{\mathrm{out}} \|\nabla\ell(\hat{y}_k^D)\|_{2, \dimin}\|h_k^{(D)}(1) \|_{\ndbar},\]
which allow for Lemma~\ref{lem:embeddingcontrol} to still hold despite the updates on \(\Win, \Wout\). The analogous estimates for \(\|(W_{\mathrm{in}}^{(D)})_{k+1} \|_{\ndbar}\) and \(\|(W_{\mathrm{our}}^{(D)})_{k+1} \|_{\ndbar}\) further allow for getting the result of Lemma~\ref{lem:MF_skeleton_bounded_ndbar}. Namely, we have for every $k\in[0:K-1]$,
\[\|(\hat{W}_{\mathrm{in}}^{(D)})_{k} \|_{\ndbar}\vee\|(W_{\mathrm{in}}^{(D)})_{k} \|_{\ndbar}\vee\|(\hat{W}_{\mathrm{out}}^{(D)})_{k} \|_{\ndbar}\vee\|(W_{\mathrm{out}}^{(D)})_{k} \|_{\ndbar} \leq \cstW_k \;\; \text{a.s.},\]
where \(\cstW_k\) depends solely on the RMS norm of the initialization of the embedding matrices.
\item The following propagation of chaos estimates hold:
\[\|(\hat{W}_{\mathrm{in}}^{(D)})_{k+1} - W_{\mathrm{in}}^{(D)})_{k+1}\|_{\ndbar} \leq \|(\hat{W}_{\mathrm{in}}^{(D)})_{k} - (W_{\mathrm{in}}^{(D)})_{k}\|_{\ndbar} + \eta_{\mathrm{in}} \|x\|_{2, \dimin}\|b_k^{(D)}(0) - B_k^{(D)}(0)\|_{\ndbar},\]
and 
\begin{multline*}
    \|(\hat{W}_{\mathrm{out}}^{(D)})_{k+1} - W_{\mathrm{out}}^{(D)})_{k+1}\|_{\ndbar} \leq \|(\hat{W}_{\mathrm{out}}^{(D)})_{k} - (W_{\mathrm{out}}^{(D)})_{k}\|_{\ndbar} + \eta_{\mathrm{out}} \|h_k^{(D)}(1)\|_{\ndbar}\|\nabla\ell(\hat{y}_k^D)-\nabla\ell(y_k)\|_{2, \dimout} \\+\eta_{\mathrm{out}} \|\nabla\ell(y_k)\|_{2, \dimout}\|h_k^{(D)}(1) - H_k^{(D)}(1)\|_{\ndbar}.
\end{multline*}
Namely, in our induction argument from Section~\ref{ssec:proof_thm_recurrence_relation} all of the terms on the RHS are controlled with high probability by $O(D^{-1/2})$. More explicitly, if we suppose
\[\|(\hat{W}_{\mathrm{in}}^{(D)})_{k} - W_{\mathrm{in}}^{(D)})_{k}\|_{\ndbar}\vee\|(\hat{W}_{\mathrm{out}}^{(D)})_{k} - W_{\mathrm{out}}^{(D)})_{k}\|_{\ndbar}\leq \cstWh_k\Big( \frac{1}{\sqrt{D}} + \sum_{j=0}^{k-1}\int_0^1\llnW_j(u)\d u\Big)\]
we get by induction, using the bounds from Section~\ref{ssec:proof_thm_recurrence_relation},
\[\|(\hat{W}_{\mathrm{in}}^{(D)})_{k+1} - W_{\mathrm{in}}^{(D)})_{k+1}\|_{\ndbar}\vee\|(\hat{W}_{\mathrm{out}}^{(D)})_{k+1} - W_{\mathrm{out}}^{(D)})_{k+1}\|_{\ndbar} \leq \cstWh_k\Big( \frac{1}{\sqrt{D}} + \sum_{j=0}^k\int_0^1\llnW_j(u)\d u\Big),\]
which allows us to adapt the proof of Theorem~\ref{thm:whp_recurrence relation} and to conclude that Theorem~\ref{thm:quantitative_convergence_main} still holds.
\end{enumerate}
For conciseness, the details of these estimates are left to the reader.

\section{The linear case: explicit system}\label{sec:linear_case_explicit_system}

We here provide some additional details regarding the limit system in the linear case.

\paragraph{Explicit structure in the linear case.} Even in the nonlinear case, the gradient matrices \(\bnabla_{\bfz^{\mathfrak{j}}} \vec{\mathbf{v}}_{k}^F\) and \(\bnabla_{\bfz^{\mathfrak{j}}} \vec{\mathbf{v}}_{k}^G\) are upper triangular. In the linear case, we further get an explicit description for them.
Let $k\in[1,K-1]$ and define, for \(i\in[0:k-1], d \in[0: k-i-1]\), $m \in[0:d]$: \(\Phi^{\pm}(0,i,0) := 1\) and for $d\geq 1$,
    \[\Phi^{\pm}(m,i,d) (s)= \sum_{\vec{u}\in \Upsilon_m(i, i+d)} (-a)^m\prod_{\ell=0}^{m-1} \tilde\Gamma_{u_\ell, u_{\ell+1}}^{(-1)^{\ell}.(\pm1)}(s),\]
    where for $n,\bar{n}\in\mathbb{N}$, $n<\bar{n}$, \(\Upsilon_m(n, \bar{n}) := \{(u_\ell)_{\ell=0}^m \,:\, u_0 =n, \, u_m = \bar{n},\, u_\ell < u_{\ell+1}\,\forall \ell\}\), and we identify $\tilde\Gamma_{i,j}^{(+1)}:=\eta_u\Gamma_{i,j}^H$ and $\tilde\Gamma_{i,j}^{(-1)} = \eta_v\Gamma_{i,j}^B$. 
    Then, for \(i\in[0:k-1]\) and \( d \in[0, k-i-1]\),
    \begin{align*}
    (\bnabla_{\bfz^{\mathfrak{j}}} \vec{\mathbf{v}}_{k}^F(s))_{i,(i+d)} &= a\Big(\sum_{m=0}^d \mathbbm{1}_{\mathfrak{j}=h,\,m \text{ odd}} \Phi^{-}(m,i,d)(s) +\mathbbm{1}_{\mathfrak{j}=b,\,m \text{ even}}\Phi^{+}(m,i,d)(s)\Big)\\
    (\bnabla_{\bfz^{\mathfrak{j}}} \vec{\mathbf{v}}_{k}^G(s))_{i,(i+d)} &= a\Big(\sum_{m=0}^d \mathbbm{1}_{\mathfrak{j}=b,\,m \text{ odd}} \Phi^{+}(m,i,d)(s) +\mathbbm{1}_{\mathfrak{j}=h,\,m \text{ even}}\Phi^{-}(m,i,d)(s)\Big)
    \end{align*}
This can be proved by induction, critically using the disjoint union
    \(\Upsilon_{m+1}(n, \bar{n}) = \bigsqcup_{d=m}^{\bar{n}-n-1} (\Upsilon_m(n,n+d), \bar{n})\), which turns a matrix multiplication into the $m+1$-th iterate in the construction.

This explicit expression suggests that, for small learning rates $\eta_u, \eta_v$, the involved matrices are expected to be close to diagonal, potentially allowing to simplify computations through suitable approximations of \eqref{eq:linear_system_explicit}. Further analysing this idea could be potentially interesting for future work, e.g. to obtain faster algorithms for solving the limit system.

\paragraph{Explicit covariance structure.} 
The experiments presented in Section~\ref{ssec:numerical_experiments} consider $W_{\mathrm{in}}, W_{\mathrm{out}}$ to be centered, Gaussian and --for simplicity-- independent and with i.i.d.~coordinates. In particular, \((\bfH, \bfB)\) is a centered Gaussian process, for which we only need to understand its covariance structure. We explicitly compute it as follows.

Consider the following quantities:
\[\begin{cases}
  \boldsymbol{\Lambda}_{\wedge k-1, \wedge k-1}^{H, B}(s) = \E[\bfH_{\wedge k-1}(s)\bfB_{\wedge k-1}(s)^\top],\\
  \boldsymbol{\Lambda}_{k,\wedge k-1}^{H, B}(s) = \E[H_k(s) \bfB_{\wedge k-1}(s)] = (\boldsymbol{\Lambda}^{H, B}_{k,j}(s))_{j=0}^{k-1},\\
  \boldsymbol{\Lambda}_{\wedge k-1,k}^{H, B}(s) = \E[B_k(s) \bfH_{\wedge k-1}(s)] = (\boldsymbol{\Lambda}^{H, B}_{j,k}(s))_{j=0}^{k-1}.
\end{cases}\]
Similarly, we write
\[\begin{cases}
  \boldsymbol{\xi}^{H, W_{\mathrm{in}}}_{\wedge k-1}(s) = (\xi^{H, W_{\mathrm{in}}}_j)_{j=0}^{k-1} =\E[\Win \bfH_{\wedge k-1}(s)^\top] \in\R^{\dimin\times k},\\
  \boldsymbol{\xi}^{H, \Wout}_{\wedge k-1}(s)= (\xi^{H, W_{\mathrm{out}}}_j)_{j=0}^{k-1} =\E[\Wout \bfH_{\wedge k-1}(s)^\top]\in\R^{\dimout\times k}.
\end{cases}\]
and analogously for \(\boldsymbol{\xi}^{B, W_{\mathrm{in}}}_{\wedge k-1}(s)\) and \(\boldsymbol{\xi}^{B, \Wout}_{\wedge k-1}(s)\). 

 At iteration $k=0$, we set 
 \[\begin{cases}
     \Gamma_{0,0}^H = \varin^2\|x\|^2,\\
     \Gamma_{0,0}^B = \varout^2\|\nabla\ell(0)\|^2,\\
     \boldsymbol{\Lambda}_{0,0}^{H,B} = 0,
 \end{cases} \quad \begin{cases}
     (\xi_0^{H,\Win}, \xi_0^{H,\Wout}) = (\varin^2x,0),\\
     (\xi_0^{B,\Win}, \xi_0^{B,\Wout}) = (0, \varout^2\nabla\ell(0)).
 \end{cases}\] Then, at every stage $k\in[1:K-1]$, supposing \(\boldsymbol{\Gamma}_{\wedge k-1}^H\), \(\boldsymbol{\Gamma}_{\wedge k-1}^B\), \(\boldsymbol{\Lambda}_{\wedge k-1, \wedge k-1}^{H,B}\), \(\boldsymbol{\xi}_{\wedge k-1}^{H, W_{\mathrm{in}}}\), \(\boldsymbol{\xi}_{\wedge k-1}^{H, W_{\mathrm{out}}}\), \(\boldsymbol{\xi}_{\wedge k-1}^{B, W_{\mathrm{in}}}\), \(\boldsymbol{\xi}_{\wedge k-1}^{B, W_{\mathrm{out}}}\) have been constructed, we solve, in order, the following set of first-order homogeneous linear systems of ODEs with variable coefficients
 \begin{enumerate}[label=(\roman*),ref=(\roman*)]
        \item 
    \(\begin{pmatrix}
            \vec{\boldsymbol{\Gamma}}_k^H(0)\\\boldsymbol{\Lambda}_{k,\wedge k-1}^{H, B}(0)
        \end{pmatrix} = \begin{pmatrix}
            \sigma_{W_{\mathrm{in}}}^2\|x\|^2.\vec{1}_k\\\boldsymbol{\xi}^{B, W_{\mathrm{in}}}_{\wedge k-1}(0)^\top x
        \end{pmatrix};
        \;\;\partial_s\begin{pmatrix}
            \vec{\boldsymbol{\Gamma}}_k^H(s)\\\boldsymbol{\Lambda}_{k,\wedge k-1}^{H, B}(s)
        \end{pmatrix} = \mathbf{C}_k^H(s).\begin{pmatrix}
            \vec{\boldsymbol{\Gamma}}_k^H(s)\\\boldsymbol{\Lambda}_{k,\wedge k-1}^{H, B}(s)
        \end{pmatrix}\)
        \item \(\begin{pmatrix}
            \xi_k^{H, \Win}\\
            \xi_k^{H, \Wout}
        \end{pmatrix}(0) = \begin{pmatrix}
            \vpin^2x\\ \vec{0}
        \end{pmatrix};\;\; \partial_s\begin{pmatrix}
            \xi_k^{H, \Win}\\
            \xi_k^{H, \Wout}
        \end{pmatrix}(s) = \begin{pmatrix}
            \boldsymbol{\xi}_{\wedge k-1}^{B,\Win}(s)\\
            \boldsymbol{\xi}_{\wedge k-1}^{B,\Wout}(s)
        \end{pmatrix}\mathbf{M}_k^{\mathbf{H}, \mathbf{B}}(s)^\top \vec{\boldsymbol{\Gamma}}_k^H(s)\)
        \item Compute $y_k = \xi^{H, W_{\mathrm{out}}}_k(1)$.
        \item 
        \(\begin{pmatrix}
            \vec{\boldsymbol{\Gamma}}_k^B(1)\\\boldsymbol{\Lambda}_{\wedge k-1,k}^{H, B}(1)
        \end{pmatrix} = \begin{pmatrix} \sigma_{W_{\mathrm{out}}}^2.\begin{pmatrix}
            \nabla\ell(y_k)^\top \nabla\ell(y_0)\\\vdots\\
            \nabla\ell(y_k)^\top \nabla\ell(y_{k-1})
        \end{pmatrix}
            \\\boldsymbol{\xi}^{H, W_{\mathrm{out}}}_{\wedge k-1}(1)^\top \nabla \ell(y_k)
        \end{pmatrix};
        \;\; \partial_s\begin{pmatrix}
            \vec{\boldsymbol{\Gamma}}_k^B(s)\\\boldsymbol{\Lambda}_{\wedge k-1,k}^{H, B}(s)
        \end{pmatrix} = \mathbf{C}_k^B(s).\begin{pmatrix}
            \vec{\boldsymbol{\Gamma}}_k^B(s)\\\boldsymbol{\Lambda}_{\wedge k-1,k}^{H, B}(s)
        \end{pmatrix}\)

    \item
    \(\begin{pmatrix}
            \xi_k^{B, \Win}\\
            \xi_k^{B, \Wout}
        \end{pmatrix}(1) = \begin{pmatrix}
            \vec{0}\\\sigma_{\mathrm{out}}^2\nabla\ell(y_k)
        \end{pmatrix};\;\; \partial_s\begin{pmatrix}
            \xi_k^{B, \Win}\\
            \xi_k^{B, \Wout}
        \end{pmatrix}(s) = -\begin{pmatrix}
            \boldsymbol{\xi}_{\wedge k-1}^{B,\Win}(s)\\
            \boldsymbol{\xi}_{\wedge k-1}^{B,\Wout}(s)
        \end{pmatrix}\mathbf{M}_k^{\mathbf{H}, \mathbf{B}}(s) \vec{\boldsymbol{\Gamma}}_k^B(s)\)
        \item \(\boldsymbol{\Lambda}_{ k,k}^{H, B}(s) \equiv x^\top \xi_k^{B, W_{\mathrm{in}}}(0) = \nabla\ell(y_k)^\top\xi_k^{H, W_{\mathrm{out}}}(1)\), and
    \begin{align*}
        \Gamma_{k,k}^H(0) = \sigma_{W_{\mathrm{in}}}^2\|x\|^2; & \qquad \partial_s\Gamma_{k,k}^H(s) = 2 (\vec{\boldsymbol{\Gamma}}_k^H(s))^T\mathbf{M}_k^{\mathbf{H}, \mathbf{B}}\boldsymbol{\Lambda}_{k,\wedge k}^{H, B}(s)\\
        \Gamma_{k,k}^B(0) = \sigma_{W_{\mathrm{out}}}^2\|\nabla\ell(y_k)\|^2; & \qquad \partial_s\Gamma_{k,k}^B(s) = -2 (\vec{\boldsymbol{\Gamma}}_k^B(s))^T(\mathbf{M}_k^{\mathbf{H}, \mathbf{B}})^\top\boldsymbol{\Lambda}_{\wedge k,k}^{H, B}(s)
    \end{align*}
    \end{enumerate}
    where we set, for \(i\in[0:k-1]\), \((\mathbf{W}_k^H(s))_{i,\cdot} = ((\vec{\boldsymbol{\Gamma}}_i^H(s))^{\top}\mathbf{M}^{\mathbf{H},\mathbf{B}}_i(s) | \mathbf{0}_{1 \times k-i})\) and similarly \((\mathbf{W}_k^B(s))_{i,\cdot} = ((\vec{\boldsymbol{\Gamma}}_i^B(s))^{\top}\tilde{\mathbf{M}}^{\mathbf{H},\mathbf{B}}_i(s) | \mathbf{0}_{1 \times k-i})\); and define by blocks,
    \[\mathbf{C}_k^H(s) = \begin{pmatrix}
    \boldsymbol{\Lambda}_{\wedge k-1,\wedge k-1}^{H, B}(s)(\mathbf{M}^{\mathbf{H},\mathbf{B}}_k(s))^{\top} & \mathbf{W}_k^H(s)\\
    \boldsymbol{\Gamma}_{\wedge k-1}^B(s)(\mathbf{M}^{\mathbf{H},\mathbf{B}}_k(s))^{\top} + \mathbf{W}_k^B(s) & \mathbf{0}_{k\times k}
\end{pmatrix}\]
\[\mathbf{C}_k^B(s) = \begin{pmatrix}
    -(\boldsymbol{\Lambda}_{\wedge k-1,\wedge k-1}^{H, B}(s))^\top\mathbf{M}^{\mathbf{H},\mathbf{B}}_k(s) & \mathbf{W}_k^B(s)\\
    -\boldsymbol{\Gamma}_{\wedge k-1}^H(s)\mathbf{M}^{\mathbf{H},\mathbf{B}}_k(s) + \mathbf{W}_k^H(s) & \mathbf{0}_{k\times k}
\end{pmatrix}\]
With this, we are able to add the $k+1$-th row/column to all relevant matrices, and be ready for step $k+1$.
    The above procedure can be implemented \textit{exactly}, modulo the choice of algorithm for numerically solving each linear ODE (e.g. 4th-order Runge-Kutta method). For $N=1$ and small $K$, solving this system is much more efficient than training a very large neural network (a few seconds vs. more than 20 minutes). However, as deduced from the path-dependent construction, the compute and memory required for determining $\Gamma_{K,K}^H$, $\Gamma_{K,K}^B$ scales with $\Omega((KN)^3)$, so that only modest examples allow for calculating it.
    
\printbibliography

\end{document}